%% file: icml2021_decentralized_noniid.tex
\documentclass{article}

\usepackage[accepted]{configuration/icml2021}
\definecolor{darkblue}{rgb}{0.0, 0.0, 0.55}
\usepackage[hyperindex,breaklinks,colorlinks,citecolor=darkblue]{hyperref} % allow link break (for arxiv) 

\usepackage[utf8]{inputenc}
\usepackage[T1]{fontenc}
\usepackage{hyperref}
\usepackage{booktabs,caption}
\usepackage{amsfonts}       %
\usepackage{nicefrac}       %
\usepackage{microtype}      %
\usepackage{xcolor}
\usepackage[flushleft]{threeparttable}

\usepackage{url}
\usepackage{float}
\usepackage{subfigure}
\usepackage{graphicx}
\usepackage{multirow}
\usepackage{xspace}
\usepackage{natbib}
\usepackage{enumitem}
\usepackage[font=small]{caption}
\usepackage[toc, page, header]{appendix}
\input{configuration/epfml_macros}

\input{configuration/myalgo}

\newcommand{\iid}{i.i.d.\ }
 
\newcommand{\qg}{quasi-global momentum\xspace} 
\newcommand{\algoptsgdm}{QG-DSGDm\xspace} 
\newcommand{\algoptsgdmn}{QG-DSGDm-N\xspace} 
\newcommand{\algopadam}{QG-DAdam\xspace} 

\newcommand{\salgoptsgdm}{QG-SGDm\xspace} 
\newcommand{\salgoptsgdmn}{QG-SGDm-N\xspace}

\usepackage[textwidth=5cm]{todonotes}

\icmltitlerunning{Quasi-Global Momentum}

\begin{document}

\twocolumn[
	\icmltitle{Quasi-Global Momentum: \\ Accelerating Decentralized Deep Learning on Heterogeneous Data}

	\begin{icmlauthorlist}
		\icmlauthor{Tao Lin}{to}
		\icmlauthor{Sai Praneeth Karimireddy}{to}
		\icmlauthor{Sebastian U. Stich}{to}
		\icmlauthor{Martin Jaggi}{to}
	\end{icmlauthorlist}

	\icmlaffiliation{to}{EPFL, Lausanne, Switzerland}

	\icmlcorrespondingauthor{Tao Lin}{tao.lin@epfl.ch}

	\icmlkeywords{Machine Learning, ICML}

	\vskip 0.3in
]

\printAffiliationsAndNotice{} %

\input{main.tex}

% \clearpage
\begin{small}
	\bibliography{icml2021}
	\bibliographystyle{configuration/icml2021}
\end{small}

%\end{document}

\clearpage
\appendix
\input{appendix.tex}

\end{document}

%% file: configuration/epfml_macros.tex
\usepackage{amsmath,amsthm,amssymb}
\newtheorem{theorem}{Theorem}[section]
\newtheorem{lemma}[theorem]{Lemma}

\newtheorem{remark}[theorem]{Remark}
\newtheorem{assumption}{Assumption}

\providecommand{\lin}[1]{\ensuremath{\left\langle #1 \right\rangle}}
\providecommand{\abs}[1]{\left\lvert#1\right\rvert}
\providecommand{\norm}[1]{\left\lVert#1\right\rVert}

\providecommand{\R}{\mathbb{R}} %

\providecommand{\E}{{\mathbb E}}
\providecommand{\E}[1]{{\mathbb E}\left.#1\right. }        %
\providecommand{\Eb}[1]{{\mathbb E}\left[#1\right] }       %
\providecommand{\EEb}[2]{{\mathbb E}_{#1}\left[#2\right] } %

\providecommand{\0}{\mathbf{0}}
\providecommand{\1}{\mathbf{1}}
\renewcommand{\aa}{\mathbf{a}}
\providecommand{\bb}{\mathbf{b}}

\providecommand{\dd}{\mathbf{d}}
\providecommand{\ee}{\mathbf{e}}

\let\ggg\gg
\renewcommand{\gg}{\mathbf{g}}

\providecommand{\mm}{\mathbf{m}}

\providecommand{\pp}{\mathbf{p}}
\providecommand{\qq}{\mathbf{q}}

\renewcommand{\ss}{\mathbf{s}}

\providecommand{\vv}{\mathbf{v}}

\providecommand{\xx}{\mathbf{x}}
\providecommand{\yy}{\mathbf{y}}
\providecommand{\zz}{\mathbf{z}}

\providecommand{\mG}{\mathbf{G}}

\providecommand{\mI}{\mathbf{I}}

\providecommand{\mM}{\mathbf{M}}

\providecommand{\mW}{\mathbf{W}}
\providecommand{\mX}{\mathbf{X}}

\providecommand{\mZ}{\mathbf{Z}}

\providecommand{\cD}{\mathcal{D}}

\providecommand{\cN}{\mathcal{N}}
\providecommand{\cO}{\mathcal{O}}

\providecommand{\cS}{\mathcal{S}}
\providecommand{\cT}{\mathcal{T}}
\providecommand{\cU}{\mathcal{U}}
\providecommand{\cV}{\mathcal{V}}

\newenvironment{talign}
{\align}
{\endalign}
\newenvironment{talign*}
{\csname align*\endcsname}
{\endalign}

%% file: configuration/myalgo.tex
\usepackage{algorithm}
\usepackage[noend]{algpseudocode}

\newcommand{\mycaptionof}[2]{\captionof{#1}{#2}}

\errorcontextlines\maxdimen

\makeatletter
\newcommand*{\algrule}[1][\algorithmicindent]{\makebox[#1][l]{\hspace*{.5em}\thealgruleextra\vrule height \thealgruleheight depth \thealgruledepth}}%
\newcommand*{\thealgruleextra}{}
\newcommand*{\thealgruleheight}{.75\baselineskip}
\newcommand*{\thealgruledepth}{.25\baselineskip}

\newcount\ALG@printindent@tempcnta
\def\ALG@printindent{%
	\ifnum \theALG@nested>0%
	\ifx\ALG@text\ALG@x@notext%
	\else
		\unskip
		\addvspace{-1pt}%
		\ALG@printindent@tempcnta=1
		\loop
		\algrule[\csname ALG@ind@\the\ALG@printindent@tempcnta\endcsname]%
		\advance \ALG@printindent@tempcnta 1
		\ifnum \ALG@printindent@tempcnta<\numexpr\theALG@nested+1\relax%
			\repeat
		\fi
	\fi
}%
\usepackage{etoolbox}
\patchcmd{\ALG@doentity}{\noindent\hskip\ALG@tlm}{\ALG@printindent}{}{\errmessage{failed to patch}}
\makeatother

\newbox\statebox
\newcommand{\myState}[1]{%
	\setbox\statebox=\vbox{#1}%
	\edef\thealgruleheight{\dimexpr \the\ht\statebox+1pt\relax}%
	\edef\thealgruledepth{\dimexpr \the\dp\statebox+1pt\relax}%
	\ifdim\thealgruleheight<.75\baselineskip
		\def\thealgruleheight{\dimexpr .75\baselineskip+1pt\relax}%
	\fi
	\ifdim\thealgruledepth<.25\baselineskip
		\def\thealgruledepth{\dimexpr .25\baselineskip+1pt\relax}%
	\fi
	\State #1%
	\def\thealgruleheight{\dimexpr .75\baselineskip+1pt\relax}%
	\def\thealgruledepth{\dimexpr .25\baselineskip+1pt\relax}%
}

%% file: main.tex
\begin{abstract}
	Decentralized training of deep learning models
	is a key element for enabling data privacy and on-device learning over networks.
	In realistic learning scenarios,
	the presence of heterogeneity across different clients' local datasets
	poses an optimization challenge and may severely deteriorate the generalization performance.\\
	In this paper, we investigate and identify the limitation of several decentralized optimization algorithms
	for different degrees of data heterogeneity.
	We propose a novel momentum-based method
	to mitigate this decentralized training difficulty.
	We show in extensive empirical experiments
	on various CV/NLP datasets (CIFAR-10, ImageNet, and AG News)
	and several network topologies (Ring and Social Network) that
	our method is much more robust to the heterogeneity of clients' data than other existing methods,
	by a significant improvement in test performance ($1\% \!-\! 20\%$).
	Our code is publicly available.\footnote{Code: \scalebox{0.95}{\mbox{\url{github.com/epfml/quasi-global-momentum}}}}
\end{abstract}

\section{Introduction}
\begin{figure*}[!t]
	\centering
	\vspace{-1em}
	\subfigure[CIFAR-10, $n\!=\!16$, $\alpha=10$.]{
		\includegraphics[width=.315\textwidth,]{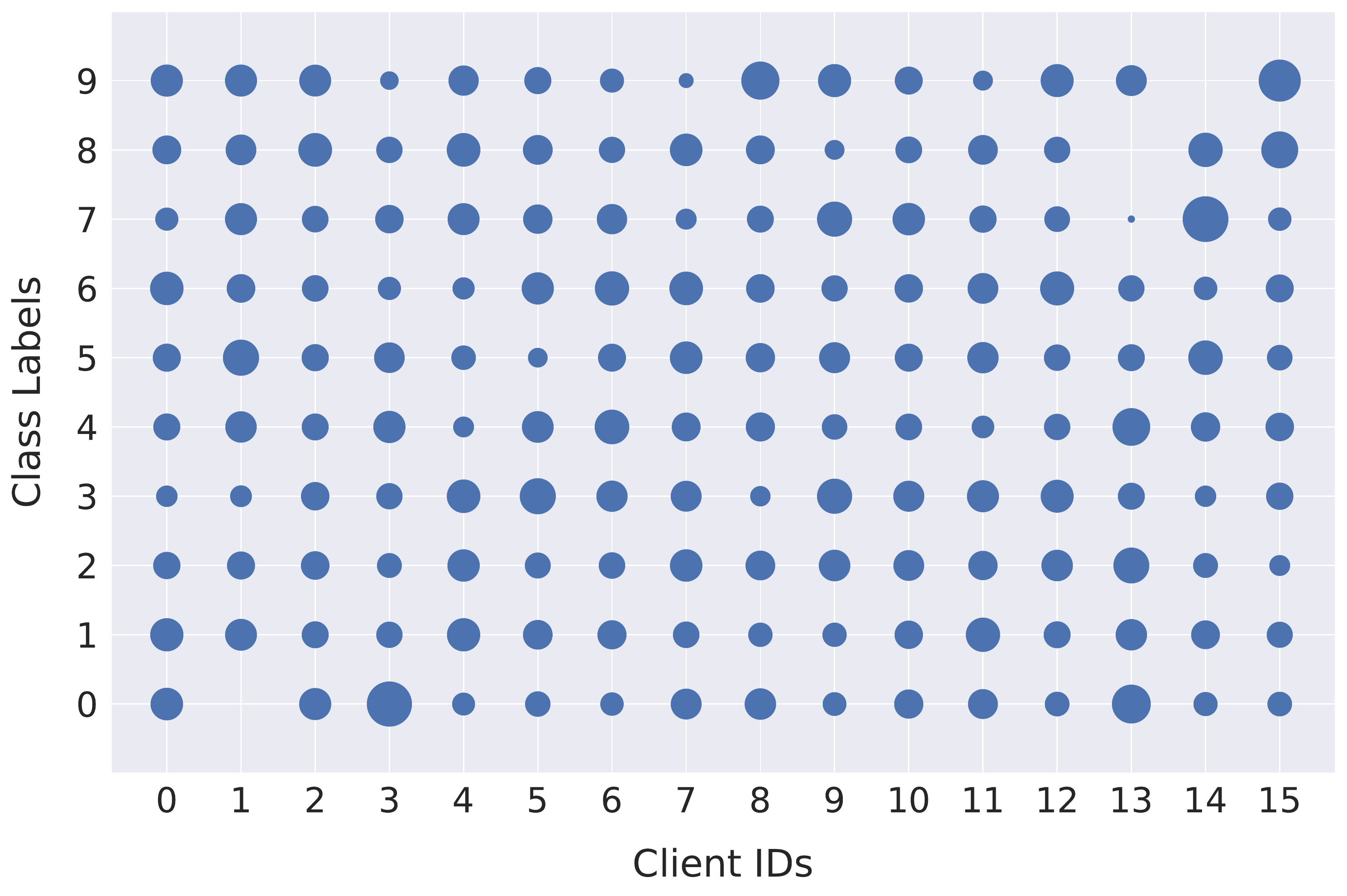}
		\label{fig:cifar10_non_iid_dirichlet_n16}
	}
	\subfigure[CIFAR-10, $n\!=\!16$, $\alpha=1$.]{
		\includegraphics[width=.315\textwidth,]{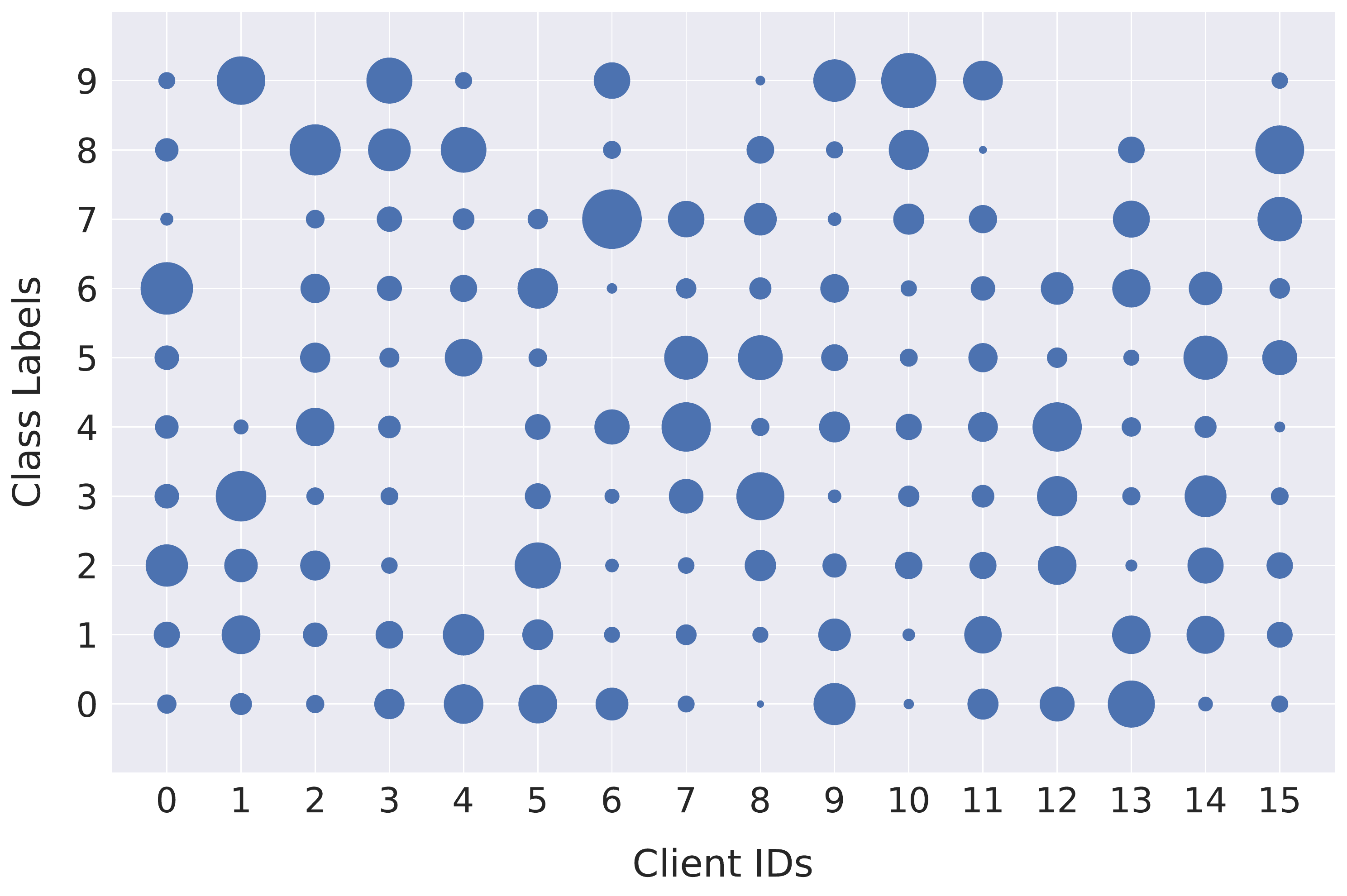}
		\label{fig:cifar10_non_iid_dirichlet_n16}
	}
	\subfigure[CIFAR-10, $n\!=\!16$, $\alpha=0.1$.]{
		\includegraphics[width=.315\textwidth,]{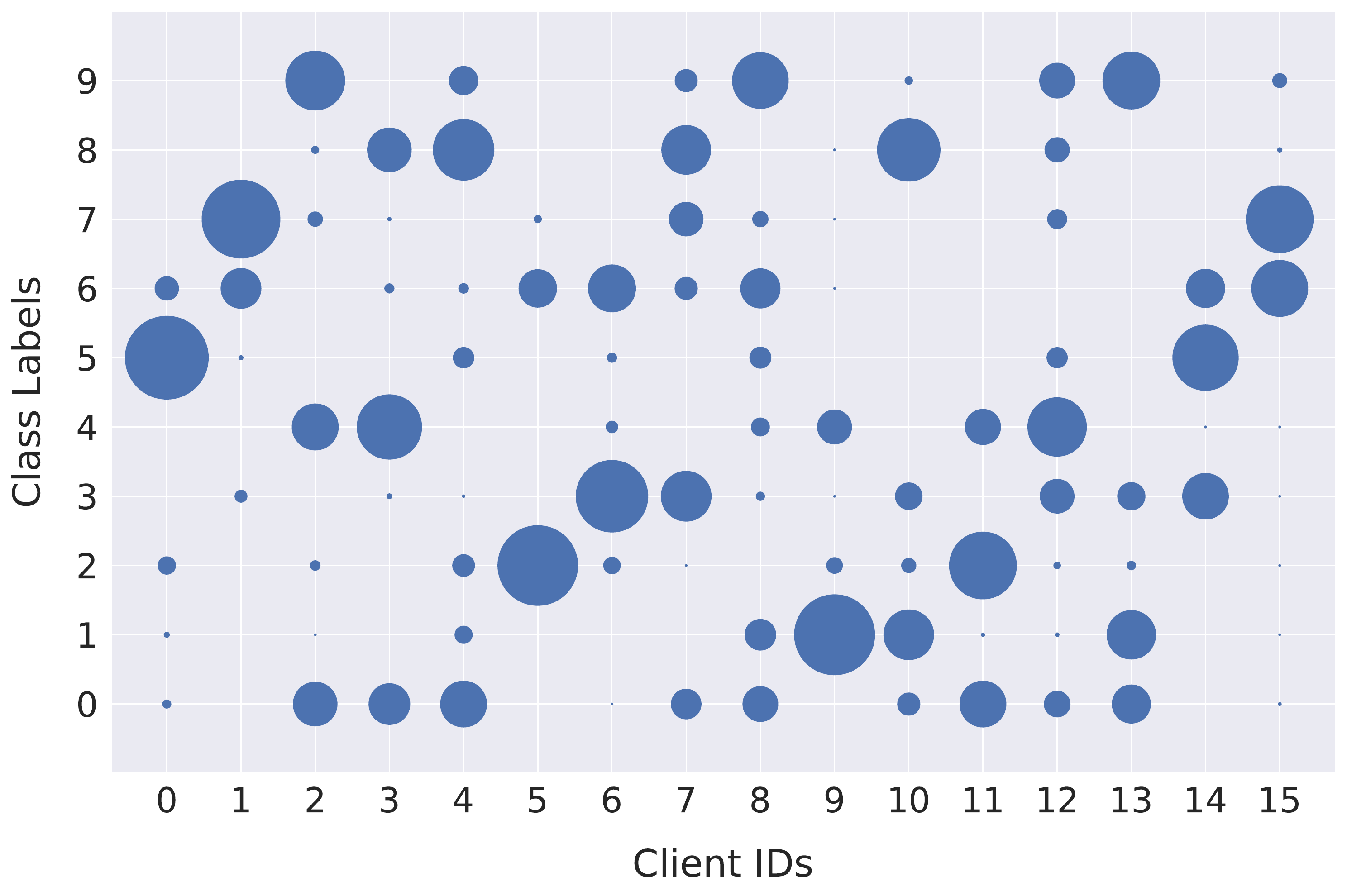}
		\label{fig:cifar10_non_iid_dirichlet_n16}
	}
	\vspace{-0.5em}
	\resizebox{.9\textwidth}{!}{%
		\begin{tabular}{ccccccc}
			\toprule
			\multirow{2}{*}{Methods} & \multicolumn{3}{c}{Ring ($n\!=\!16$)} & \multicolumn{3}{c}{Social Network ($n\!=\!32$)} \\ \cmidrule(lr){2-4} \cmidrule(lr){5-7}
			                    & $\alpha=10$               & $\alpha=1$                & $\alpha=0.1$              & $\alpha=10$               & $\alpha=1$                & $\alpha=0.1$              \\ \midrule
			DSGD                & $89.90 \pm 0.26$          & $88.88 \pm 0.26$          & $74.55 \pm 2.07$          & $89.95 \pm 0.23$          & $88.41 \pm 0.27$          & $77.56 \pm 1.65$          \\
			DSGDm-N             & $91.47 \pm 0.23$          & $89.98 \pm 0.10$          & $77.48 \pm 2.67$          & $91.17 \pm 0.11$          & $89.96 \pm 0.35$          & $80.59 \pm 2.32$          \\
			\algoptsgdmn (ours) & $\textbf{91.90} \pm 0.17$ & $\textbf{91.28} \pm 0.38$ & $\textbf{82.20} \pm 1.27$ & $\textbf{91.51} \pm 0.02$ & $\textbf{91.00} \pm 0.24$ & $\textbf{85.19} \pm 0.98$ \\
			\bottomrule
		\end{tabular}%
	}
	\caption{\small
		\textbf{Illustrating the challenge of heterogeneous data in decentralized deep learning},
		for training ResNet-EvoNorm-20 on CIFAR-10.
		The Dirichlet distribution $\alpha$ values control different non-\iid degrees~\citep{yurochkin2019bayesian,hsu2019measuring,he2020fedml};
		the smaller $\alpha$ is, the more likely the clients hold examples from only one class.
		\textbf{The inline figures illustrate the \# of samples per class allocated to each client (indicated by dot sizes)},
		for the case of $\alpha \!=\! 10, 1, 0.1$.
		The test top-1 accuracy results in the table are averaged over three random seeds,
		with learning rate tuning for each setting.
		The performance upper bound
		(i.e.\ centralized training without local data re-shuffling)
		for $n\!=\!16$ and $32$ nodes are $92.95 \pm 0.13$ and $92.88 \pm 0.07$ respectively.
		Following prior work, the evaluated DSGD methods maintain local momentum buffer (without synchronization) for each worker;
		other experimental setup refers to Section~\ref{sec:exp_setup}. \looseness=-1
	}
	\label{fig:resnet20_cifar10_motivation_example}
	\vspace{-1.em}
\end{figure*}

Decentralized machine learning methods---%
allowing %
communications in a peer-to-peer fashion on an underlying communication network topology (without a central coordinator)---%
have emerged as an important paradigm
in large-scale machine learning~\citep{lian2017can,%
	lian2018asynchronous,koloskova2019choco,koloskova2020unified}.
Decentralized Stochastic Gradient Descent (DSGD) methods offer
(1) scalability to large datasets and systems in large data-centers~\citep{lian2017can,assran2019stochastic,koloskova2020decentralized},
as well as (2) privacy-preserving learning for the emerging EdgeAI applications~\citep{kairouz2019advances,koloskova2020decentralized},
where the training data remains distributed over a large number of clients
(e.g.\ mobile phones, sensors, or hospitals)
and is kept locally (never transmitted during training).

A key challenge---in particular in the second scenario---is the large heterogeneity (non-i.i.d.-ness)
in the data present on the different clients~\citep{zhao2018federated,kairouz2019advances,hsieh2020non}.
Heterogeneous data (e.g.\ as illustrated in Figure~\ref{fig:resnet20_cifar10_motivation_example})
causes very diverse optimization objectives on each client,
which results in slow and unstable global convergence, as well as poor generalization performance
(shown in the inline table of Figure~\ref{fig:resnet20_cifar10_motivation_example}).
Addressing these optimization difficulties %
is essential to realize reliable decentralized deep learning applications. %
Although such challenges have been theoretically pointed out in~\citep{shi2015extra,%
	lee2015distributed,tang2018d,koloskova2020unified},
the empirical performance of different DSGD methods remains
poorly understood.
To the best of our knowledge,
there currently exists no efficient, effective, and robust optimization algorithm yet for decentralized deep learning on heterogeneous data.\looseness=-1

In the meantime, SGD with momentum acceleration (SGDm)
remains the current workhorse for the state-of-the-art (SOTA)
centralized deep learning training~\citep{he2016deep,goyal2017accurate,he2019bag}.
For decentralized deep learning,
the currently used training recipes (i.e.\ DSGDm)
maintain a local momentum buffer on each worker~\citep{assran2019stochastic,koloskova2020decentralized,nadiradze2020swarmsgd,singh2020squarm,kong2021consensus}
while only communicating the model parameters to the neighbors.
However, these attempts in prior work mainly consider homogeneous decentralized data---%
and there is no evidence that local momentum enhances generalization performance of decentralized deep learning on heterogeneous data.

As our first contribution, we investigate how DSGD and DSGDm are impacted by the degree of data heterogeneity  and the choice of the network topology.
We find that heterogeneous data hinders the local momentum acceleration in DSGDm.
We further show that using a high-quality shared momentum buffer (e.g.\ synchronizing the momentum buffer globally) improves the optimization and generalization performance of DSGDm.
However, such a global communication significantly increases
the communication cost and violates the decentralized learning setup.

We instead propose Quasi-Global (QG) momentum,
a simple, yet effective, method that
mitigates the difficulties for decentralized learning on heterogeneous data.
Our approach is based on locally approximating the global optimization direction without introducing extra communication overhead.
We demonstrate in extensive empirical results that
QG momentum
can stabilize the optimization trajectory, and that
it can accelerate decentralized learning achieving much better generalization performance under high data heterogeneity than previous methods.
\begin{itemize}[nosep,leftmargin=12pt]
	\item We systematically examine the behavior of decentralized optimization algorithms
	      on standard deep learning benchmarks for various degrees of data heterogeneity.
	\item We propose a novel momentum-based decentralized optimization method---%
	      \algoptsgdm and \algoptsgdmn---to stabilize the local optimization.
	      We validate the effectiveness of our method on a spectrum of non-\iid degrees and network topologies---%
	      it is much more robust to the data heterogeneity than all other existing methods.
	\item We rigorously prove the convergence of our scheme.
	\item We additionally investigate different normalization methods
	      alternative to Batch Normalization
	      (BN)~\citep{ioffe2015batch} in CNNs,
	      %SEB: alternative formatting:
	      %\citep[BN,][]{ioffe2015batch} in CNNs,
	      due to its particular vulnerable to non-\iid local data and the caused severe quality loss.
\end{itemize}

\section{Related Work}
\textbf{Decentralized Deep Learning.}
The study of decentralized optimization algorithms dates back to~\citet{tsitsiklis1984problems},
relating to use gossip algorithms~\citep{kempe2003gossip,xiao2004fast,boyd2006randomized}
to compute aggregates (find consensus) among clients.
In the context of machine learning/deep learning,
combining SGD with gossip averaging~\citep{lian2017can,%
	lian2018asynchronous,assran2019stochastic,koloskova2020unified}
has gained a lot of attention recently
for the benefits of \emph{computational scalability}, \emph{communication efficiency}, \emph{data locality},
as well as the favorable leading term in the convergence rate
$\smash{\cO\big(\frac{1}{n\varepsilon^2}\big)}$~\citep{lian2017can,%
	scaman2017optimal,scaman2018optimal,%
	tang2018d,koloskova2019choco,koloskova2020decentralized,koloskova2020unified}
which is the same as in centralized mini-batch SGD~\citep{dekel2012optimal}.
A weak version of decentralized learning
also covers the recent emerging federated learning (FL) setting~\citep{konevcny2016federated,%
	mcmahan2017communication,kairouz2019advances,karimireddy2019scaffold,lin2020ensemble}
by using (centralized) star-shaped network topology and local updates.
Note that specializing our results to the FL setting is beyond the scope of our work. %
It is also non-trivial to adapt certain very recent techniques developed in FL for heterogeneous data~\citep{karimireddy2019scaffold,karimireddy2020mime,lin2020ensemble,wang2020tackling,das2020faster,haddadpour2021federated} to the gossip-based decentralized deep learning.\looseness=-1

A line of recent works on decentralized stochastic optimization, like D$^2$/Exact-diffusion~\cite{tang2018d,yuan2020influence,yuan2021removing}, and gradient tracking~\cite{pu2020distributed,pan2020d,lu2019gnsd}, proposes different techniques to theoretically eliminate the influence of data heterogeneity between nodes.
However, it remains unclear if these theoretically sound methods still endow with superior convergence and generalization properties in deep learning.

Other works focus on improving communication efficiency,
from the aspect of communication compression~\citep{tang2018communication,%
	koloskova2019choco,koloskova2020decentralized,lu2020moniqua,taheri2020quantized,%
	singh2020squarm,vogels2020powergossip,taheri2020quantized,nadiradze2020swarmsgd},
less frequent communication through multiple local updates~\citep{hendrikx2019accelerated,%
	koloskova2020unified,nadiradze2020swarmsgd},
or better communication topology design~\citep{nedic2018network,assran2019stochastic,%
	wang2019matcha,wang2020exploring,neglia2020decentralized,%
	nadiradze2020swarmsgd,kong2021consensus}.

\textbf{Mini-batch SGD with Momentum Acceleration.}
Momentum is a critical component for training the SOTA deep neural networks~\citep{sutskever2013importance,lucas2018aggregated}.
Despite various empirical successes, the current theoretical understanding
of momentum-based SGD methods remains limited~\citep{bottou2016optimization}.
A line of work on the serial (centralized) setting has aimed to develop a convergence analysis for different momentum methods as a special case~\citep{yan2018unified,gitman2019understanding}.
However, SGD is known to be optimal in the worst case for stochastic non-convex optimization~\citep{arjevani2019lower}.
\\
In distributed deep learning,
most prior works focus on homogeneous data (especially for numerical evaluations)
and incorporate momentum with a locally maintained buffer
(which has no synchronization)~\citep{lian2017can,assran2019stochastic,%
	lin2020dont,lin2020extrapolation,koloskova2020decentralized,singh2020squarm}.
\citet{yu2019linear} propose synchronizing the local momentum buffer periodically for better performance
at the cost of doubling the communication.
SlowMo~\citep{wang2020slowmo} instead proposes to
periodically perform a slow momentum update on the globally synchronized model parameters
(with additional All-Reduce communication cost),
for centralized or decentralized methods.
Parallel work~\citep{balu2020decentralized}
introduces DMSGD for decentralized learning\footnote{
	We detail the DMSGD algorithm and clarify the difference
	in Appendix~\ref{appendix:connection_between_our_and_DMSGD};
	we empirically compare with DMSGD in Table~\ref{tab:ablation_study_compare_with_other_momentum_SGDs}.
}---%
it constructs the acceleration momentum
from the mixture of the local momentum and consensus momentum.
Our proposed method has no extra communication overhead
and significantly outperforms all existing methods (in Section~\ref{sec:main_results}).

\begin{figure*}[!t]
	\centering
	\subfigure[\small
		w/o local momentum.
	]{
		\includegraphics[width=.275\textwidth,]{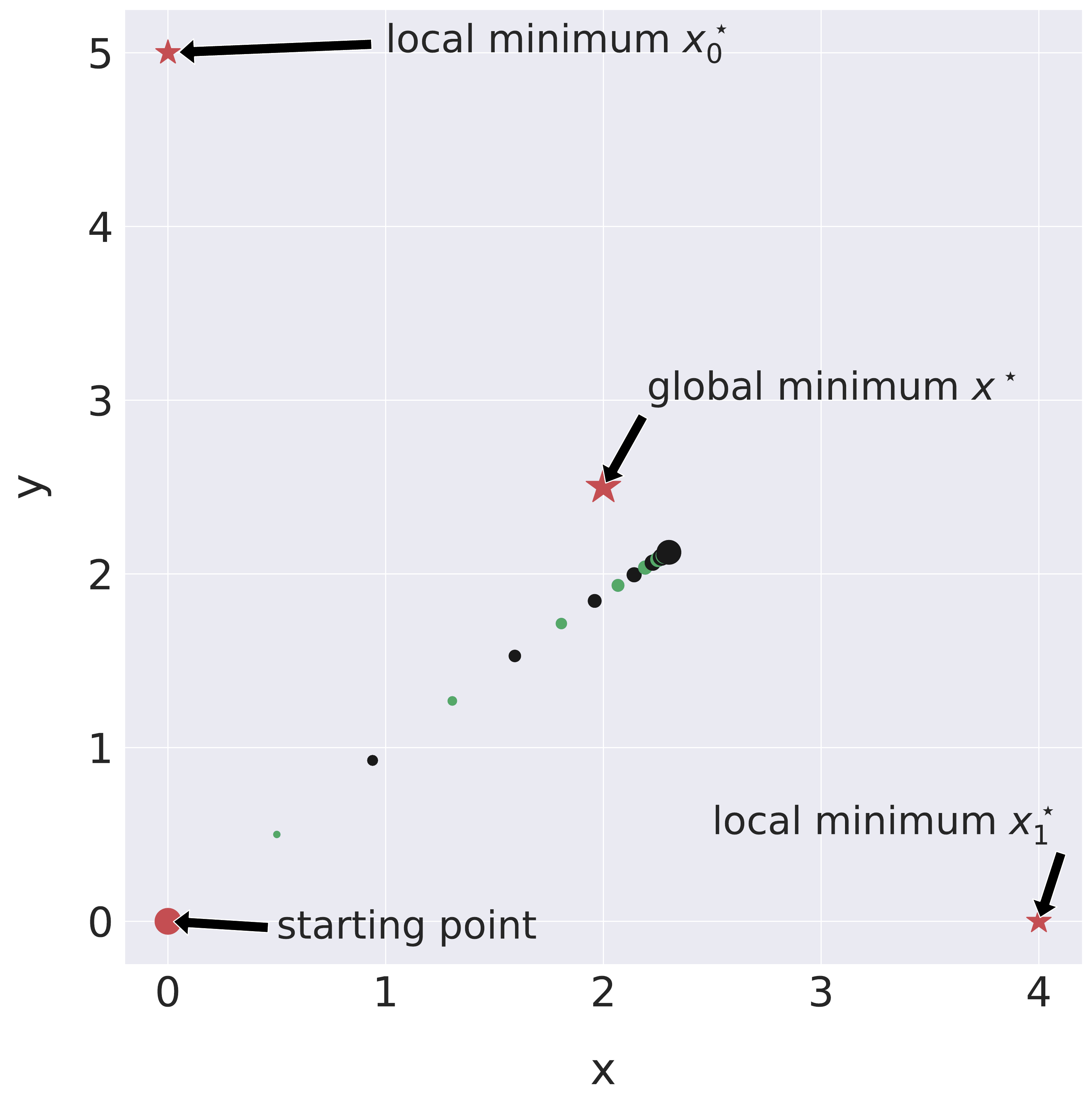}
		\label{fig:2d_illustration_procedure_wo_hb_momentum}
	}
	\subfigure[\small
		w/ local HeavyBall momentum.
	]{
		\includegraphics[width=.275\textwidth,]{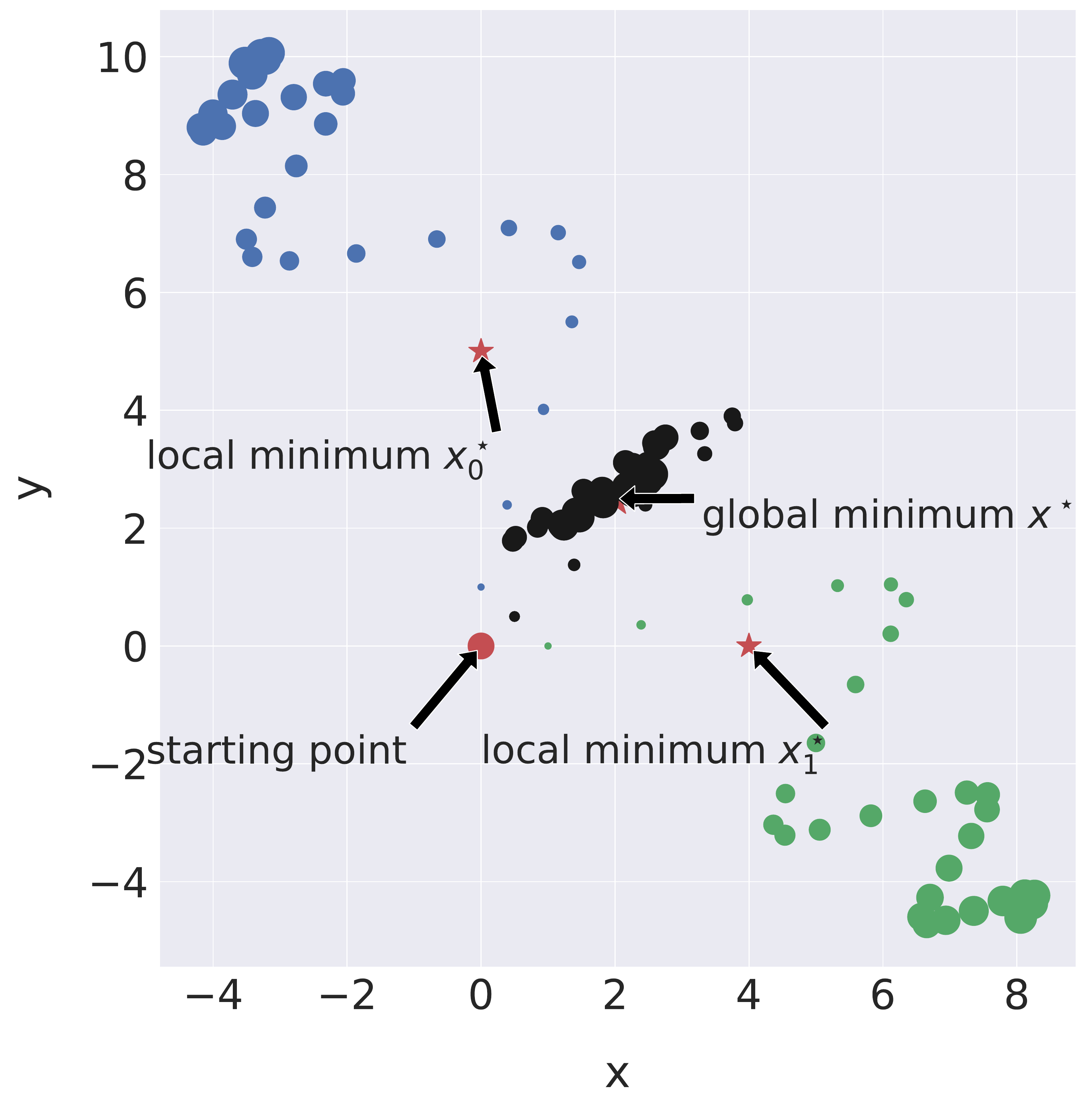}
		\label{fig:2d_illustration_procedure_w_hb_momentum}
	}
	\subfigure[\small
		Our scheme.
	]{
		\includegraphics[width=.275\textwidth,]{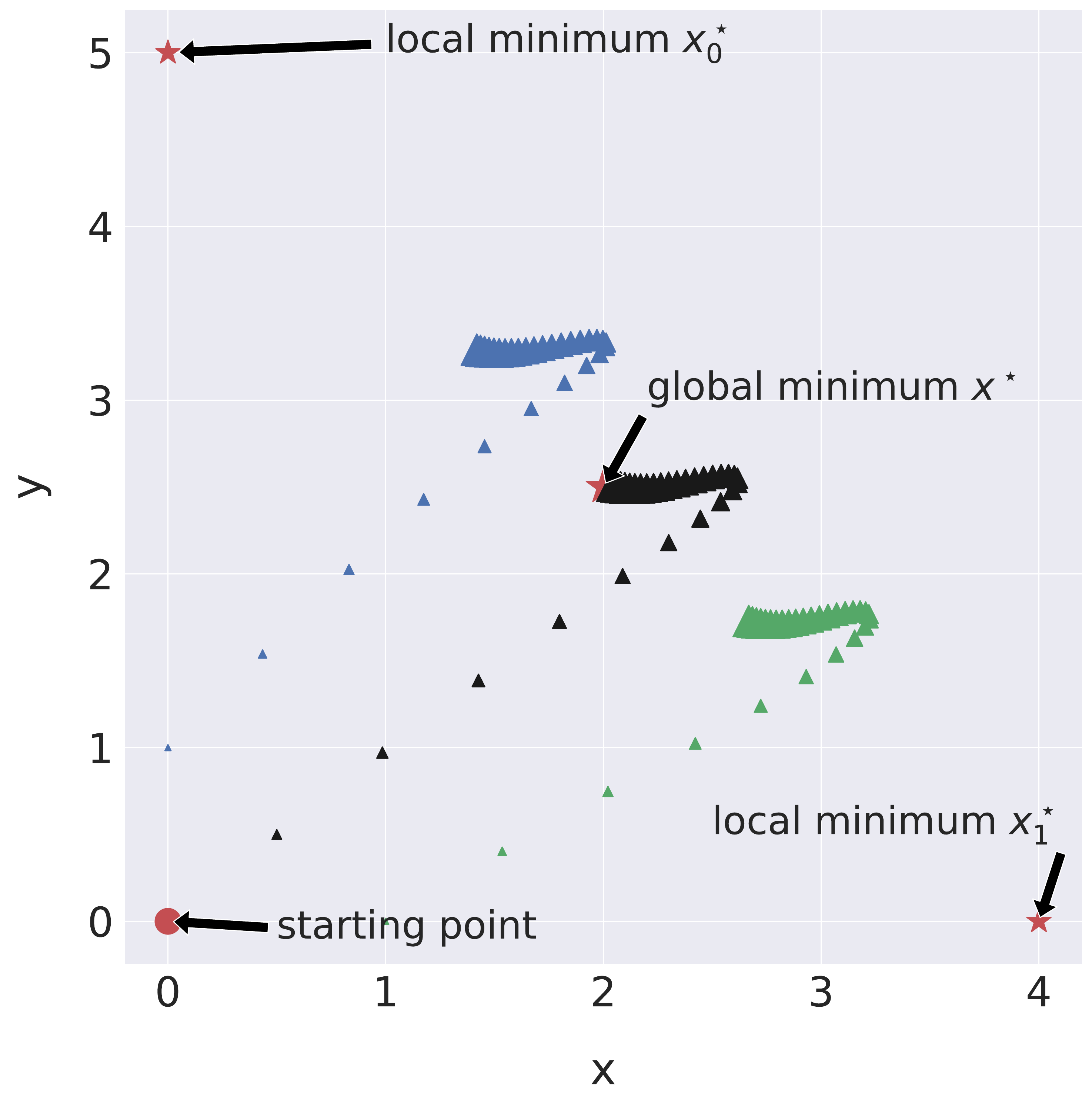}
		\label{fig:2d_illustration_procedure_our_scheme}
	}
	\vspace{-1em}
	\caption{\small
		\textbf{The ineffectiveness of local momentum acceleration under heterogeneous data setup}:\
		the local momentum buffer accumulates ``biased'' gradients,
		causing unstable and oscillation behaviors.
		The size of marker increase proportional to the number of update steps;
		colors blue and green indicate the local models of two workers (after performing local update),
		while black color indicates the synchronized global model.
		Uniform weight averaging is performed after each update step,
		and the new gradients are computed on the averaged model.
		We use the common $\beta \!=\! 0.9$ in this illustration.
		For additional results with different $\beta$ values refer to Appendix~\ref{appendix:more_on_2D_illustration}.
	}
	\vspace{-1em}
	\label{fig:2d_illustration_procedures}
\end{figure*}

\textbf{Batch Normalization in Distributed Learning.}
Batch Normalization (BN)~\citep{ioffe2015batch}
is an indispensable component in deep learning~\citep{santurkar2018does,luo2018towards}
and has been employed by default in most SOTA CNNs~\citep{he2016deep,huang2016densely,tan2019efficientnet}.
However, it often fails on distributed deep learning
with heterogeneous local data due to the discrepancies between local activation statistics~\citep[see %
	recent empirical examination for federated learning in][]{hsieh2020non,andreux2020siloed,li2021fedbn,diao2021heterofl}.
As a remedy, \citet{hsieh2020non} propose to replace BN with Group Normalization (GN)~\citep{wu2018group}
to address the issue of local BN statistics,
while~\citet{andreux2020siloed,li2021fedbn,diao2021heterofl}
modify the way of synchronizing the local BN weight/statistics for better generalization performance.
In the scope of decentralized learning, the effect of batch normalization has not been investigated yet.

\section{Method} \label{sec:method}
\subsection{Notation and Setting}
We consider sum-structured distributed optimization problems $f \colon \R^d \to \R$ of the form
\begin{align}
	f^\star := \textstyle \min_{\xx \in \R^d} \left[ f(\xx) := \frac{1}{n} \sum_{i=1}^n f_i(\xx) \right] \,,
\end{align}
where the components $f_i: \R^d \rightarrow \R$ are distributed among the $n$ nodes
and are given in stochastic form:
$
	f_i(\xx) := \EEb{\xi \sim \cD_i}{F_i (\xx, \xi)}
$,
where $\cD_i$ denotes the local data distribution on node $i \in [n]$.
In D(ecentralized)SGD,
each node $i$ maintains local parameters $\smash{\xx_i^{(t)}} \in \R^d$,
and updates them as:
\begin{small}%
	\begin{talign} \label{eq:d-sgd}
		\textstyle
		\xx_{i}^{(t+1)} = \sum_{j=1}^n w_{ij} \left(\xx_j^{(t)} - \eta \nabla F_j (\xx_j^{(t)},\xi_j^{(t)})\right) \,,
		\tag{DSGD}
	\end{talign}
\end{small}%
that is, by a stochastic gradient step based on a sample $\xi_i^{(i)}\sim \cD_i$,
followed by gossip averaging with neighboring nodes in the network topology
encoded by the mixing weights $w_{ij}$.

In this paper,
we denote DSGD with local HeavyBall momentum by DSGDm,
and DSGD with local Nesterov momentum by DSGDm-N;
the naming rule also applies to our method.
For the sake of simplicity,
we use HeavyBall momentum variants in Section~\ref{sec:method} and~\ref{sec:understanding} for analysis purposes. \looseness=-1

\subsection{\algoptsgdm Algorithm}
To motivate the algorithm design,
we first illustrate the impact of using different momentum buffers (local vs.\ global) on distributed training on heterogeneous data.

\paragraph{Heterogeneous data hinders local momentum acceleration---an example 2D optimization illustration.}
Figure~\ref{fig:2d_illustration_procedures} shows a toy 2D optimization example that simulates the biased local gradients caused by heterogeneous data. It depicts the optimization trajectories of
two agents ($n=2$) that start the optimization from the position $(0, 0)$
and receive
in every iteration a
gradient that points to the local minimum $(0, 5)$ and $(4, 0)$ respectively.
The gradient is given by the direction from the current model (position) to the local minimum, and scaled to a constant update magnitude.
Model synchronization (i.e.\ uniform averaging) is performed for every local model update step. \looseness=-1

Heterogeneous data strongly influences the effectiveness of the local momentum acceleration.
Though local momentum in Figure~\ref{fig:2d_illustration_procedure_w_hb_momentum}
assists the models to converge to the neighborhood of the global minimum
(better convergence than when excluding local momentum
in Figure~\ref{fig:2d_illustration_procedure_wo_hb_momentum}),
it also causes an unstable and oscillation optimization trajectory.
The problem gets even worse in decentralized deep learning,
where the learning relies on stochastic gradients from non-convex function
and only has limited communication.

\paragraph{Synchronizing the local momentum buffers boosts decentralized learning.} %
We here consider a hypothetical method,
which synchronizes the local momentum buffer as in~\citep{yu2019linear},
to use the global momentum buffer locally
(avoid using ill-conditioned local momentum buffer caused by heterogeneous data,
as shown by the poor performance in Figure~\ref{fig:resnet20_cifar10_motivation_example}).
We can witness from Table~\ref{tab:ablation_study_compare_with_other_momentum_SGDs} that
synchronizing the buffer per update step by global averaging
to some extent mitigates the issue caused by heterogeneity ($1\% \!-\! 5\%$ improvement comparing row 3 with row 7 in Table~\ref{tab:ablation_study_compare_with_other_momentum_SGDs}).
Despite its effectiveness,
the global synchronization fundamentally violates the realistic decentralized learning setup
and introduces extra communication overhead.\looseness=-1

\begin{algorithm}[!h]
	\begin{algorithmic}[1]
		\Procedure{worker-$i$}{}
		\For{$t \in \{ 1, \ldots, T \}$}
		\myState{sample $\xi_{i}^{(t)}$ and compute $\gg_{i}^{(t)} = \nabla F_i(\xx_{i}^{(t)}, \xi_{i}^{(t)})$}
		\myState{ \colorbox{green!30}{ $ \mm_i^{(t)} = \beta \mm_{i}^{(t-1)} + \gg_{i}^{(t)} $ } }
		\myState{ \colorbox{blue!30}{ $ \mm_i^{(t)} = \beta \hat \mm_{i}^{(t-1)} + \gg_{i}^{(t)} $ } }
		\myState{$ \xx_{i}^{(t+\frac{1}{2})} = \xx_{i}^{(t)} - \eta \mm_i^{(t)} $}
		\myState{$\xx_{i}^{(t+1)} = \sum_{j \in \cN_{i}^{(t)}} w_{ij} \xx_{j}^{(t + \frac{1}{2})}$}
		\myState{ \colorbox{blue!30}{ $ \dd_i^{(t)} = \frac{ \xx_{i}^{(t)} - \xx_{i}^{(t + 1)}  }{\eta} $ } }
		\myState{ \colorbox{blue!30}{ $ \hat \mm_{i}^{(t)} = \mu \hat \mm_{i}^{(t - 1)} + (1 - \mu) \dd_i^{(t)} $ } }
		\EndFor
		\myState{\Return $\xx_i^{(T)}$}
		\EndProcedure
	\end{algorithmic}

	\mycaptionof{algorithm}{\small
		Decentralized learning algorithms:
		\colorbox{blue!30}{\algoptsgdm} v.s.\
		\colorbox{green!30}{DSGDm};
		Colors indicate the two alternative algorithm variants.
		At initialization $\mm_{i}^{(0)} = \hat \mm_{i}^{(0)} := \0$.
	}
	\label{alg:hbsgdm_vs_algoptsgdm}
\end{algorithm}

\paragraph{Our proposal---\algoptsgdm.}
Motivated by the performance gain brought by employing a global momentum buffer,
we propose a \textbf{Q}uasi-\textbf{G}lobal (QG) momentum buffer---a communication-free approach to mimic the global optimization direction---to mitigate the difficulties for decentralized learning on heterogeneous data.
Integrating quasi-global momentum with local stochastic gradients
alleviates the drift in the local optimization direction,
and thus results in a stabilized training and high robustness to heterogeneous data.\looseness=-1

Algorithm~\ref{alg:hbsgdm_vs_algoptsgdm} highlights
the difference between DSGDm and \algoptsgdm.
Instead of using local gradients from heterogeneous data
to form the local momentum (line~4 for DSGDm),
which may significantly deflect from the global optimization direction,
for \algoptsgdm, we use the difference of two consecutive synchronized models (line~8)
\begin{align} \label{eq:movement_directions}
	\dd_i^{(t)} = \frac{1}{{\eta}} \left( \xx_{i}^{(t)} - \xx_{i}^{(t + 1)} \right) \,,
\end{align}
to update the momentum buffer (line 9) by
$\hat \mm_{i}^{(t)} = \mu \hat \mm_{i}^{(t - 1)} + (1 - \mu) \dd_i^{(t)}$.
We set $\mu = \beta$ for all our numerical experiments, without needing hyper-parameter tuning.\looseness=-1

The update scheme of~\algoptsgdm can be re-formulated in matrix form ($\mX=[\xx_1,\dots,\xx_n] \in \R^{d \times n}$, etc.) as follows
\begin{small}
	\begin{talign} \label{eq:our_scheme_matrix_form}
		\begin{split}
			\mX^{(t+1)} &= \mW \left( \mX^{(t)} - \eta \left( \beta \mM^{(t-1)} + \mG^{(t)} \right) \right) \\
			\mM^{(t)}   &= \mu \mM^{(t-1)} + (1 - \mu) \frac{ \mX^{(t)} - \mX^{(t + 1)} }{\eta} \,.
		\end{split}
	\end{talign}\vspace{-2mm}
\end{small}

\begin{figure*}[!t]
	\vspace{-1em}
	\centering
	\subfigure[\small
		$n\!=\!16$.
	]{
		\includegraphics[width=.31\textwidth,]{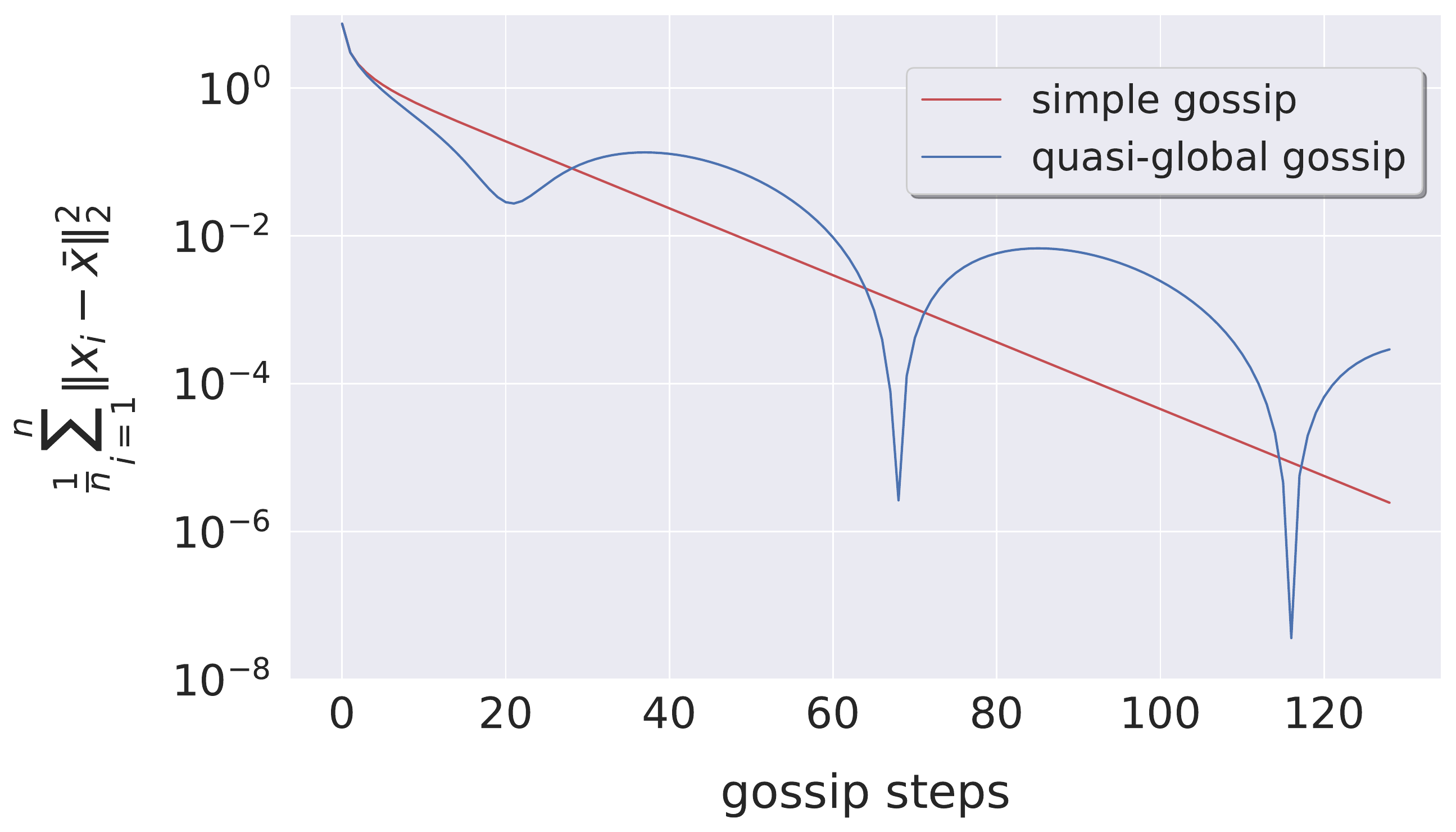}
		\label{fig:understanding_via_consensus_averaging_n16_128}
	}
	\hfill
	\subfigure[\small
		$n\!=\!32$.
	]{
		\includegraphics[width=.31\textwidth,]{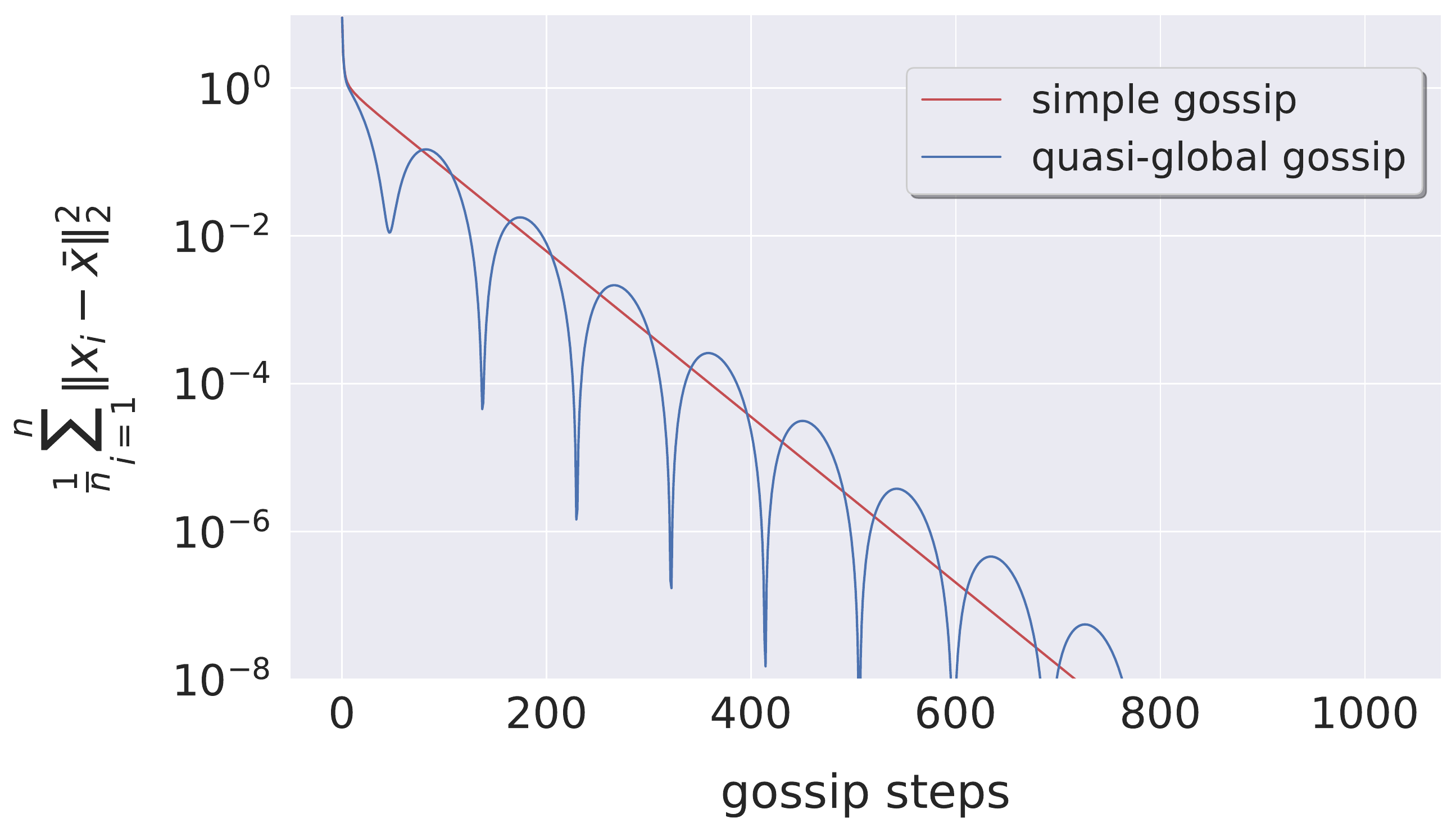}
		\label{fig:understanding_via_consensus_averaging_n32_1024}
	}
	\hfill
	\subfigure[\small
		$n\!=\!64$.
	]{
		\includegraphics[width=.31\textwidth,]{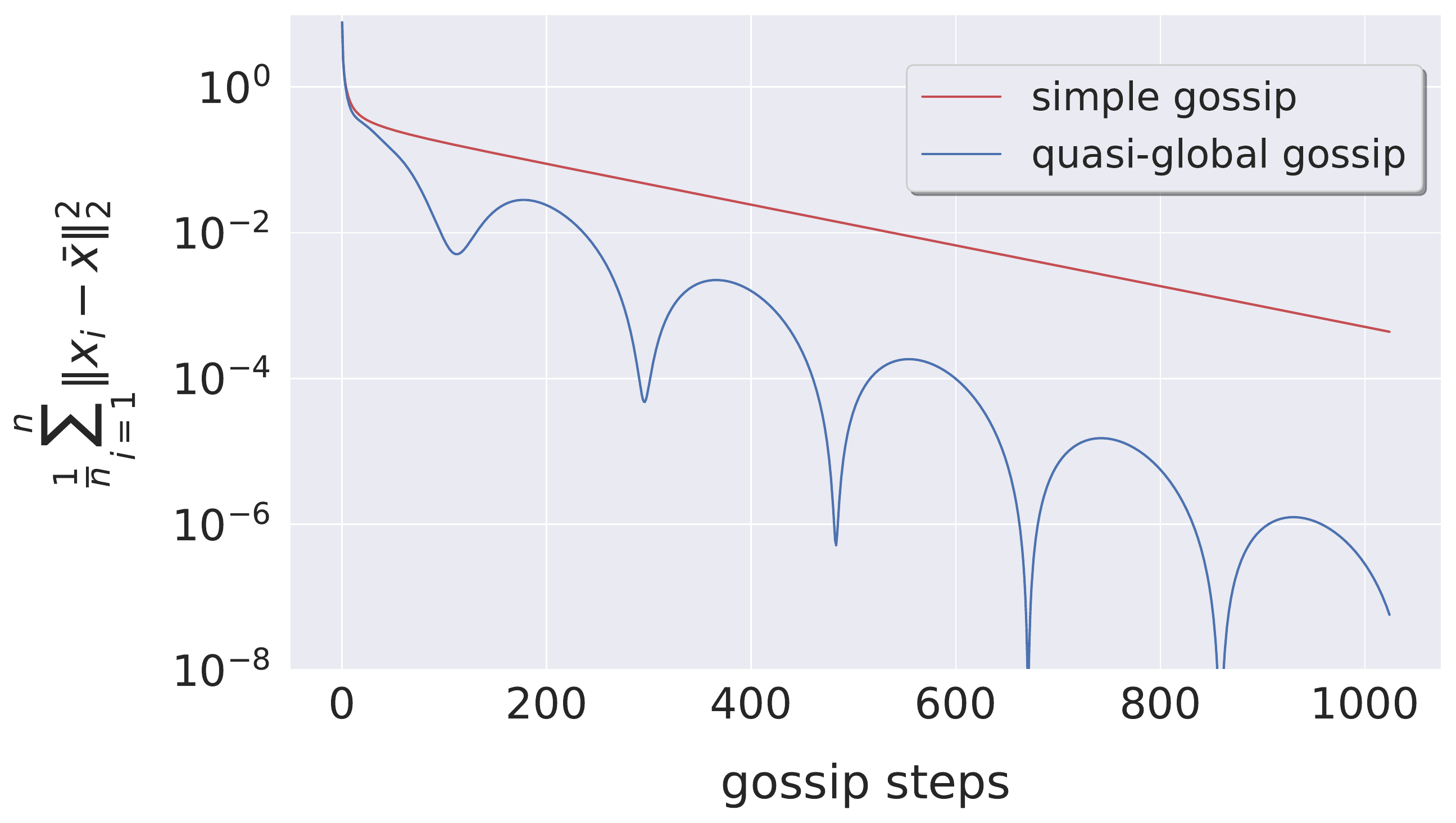}
		\label{fig:understanding_via_consensus_averaging_n64_1024}
	}
	\vspace{-1em}
	\caption{\small
		\textbf{Understanding \algoptsgdm through the distributed average consensus problem
			on a fixed ring topology}.
		\algoptsgdm without gradient update step \eqref{eq:update_without_stochastic_gradients}
		still presents faster convergence (to a relative high precision) than the standard gossip averaging.
		Appendix~\ref{appendix:more_on_consensus_averaging_problem}
		illustrates the results on other communication topologies and topology scales.
	}
	\vspace{-1em}
	\label{fig:understanding_via_consensus_averaging}
\end{figure*}

\subsection{Convergence Analysis} \label{sec:convergence_analysis}
We provide a convergence analysis for our novel \algoptsgdm for non-convex functions.
The proof details can be found in Appendix~\ref{appendix:convergence_rate_proofs}.

\begin{assumption} \label{asm:main_assumptions_for_convergence}
	We assume that the following hold:
	\begin{itemize}[nosep,leftmargin=12pt,itemsep=1pt]
		\item The function $f(\xx)$ we are minimizing is lower bounded from below by $f^\star$, and each node's loss $f_i$ is smooth satisfying	$\norm{\nabla f_i(\yy) \!-\! \nabla f_i(\xx)} \leq L\norm{\yy \!-\! \xx}$.
		\item The stochastic gradients within each node satisfy $\Eb{ g_i(\xx) } \!=\! \nabla f_i(\xx)$
		      and $\E\norm{g_i(\xx) \!-\! \nabla f_i(\xx)}^2 \leq \sigma^2$.
		      The variance across the workers is also bounded as $\frac{1}{n}\sum_{i=1}^n \norm{\nabla f_i(\xx) \!-\! \nabla f(\xx)}^2 \leq \zeta^2$.
		\item The mixing matrix is doubly stochastic: for the all ones vector $\1$,
		      we have $\mW \1 \!=\! \1$ and $\mW^\top\1 \!=\! \1$.
		      %SEB: and symmetic?
		\item Define $\bar \mZ \!=\! \mZ \frac{1}{n}\1 \1^\top$ for any matrix $\mZ \in  \R^{d \times n}$, then the mixing matrix satisfies $\E_\mW \norm{ \mZ \mW - \bar{\mZ} }_F^2 \leq (1 - \rho) \norm{\mZ - \bar{\mZ}}_F^2$.
	\end{itemize}
\end{assumption}

\begin{theorem}[Convergence of \algoptsgdm for non-convex functions]
	\label{thm:1}
	Given Assumption~\ref{asm:main_assumptions_for_convergence},
	the sequence of iterates generated by \eqref{eq:our_scheme_matrix_form}
	for step size $\eta \!=\!\smash{\cO\left(\sqrt{\frac{n}{\sigma^2 T}}\right)}$
	and momentum parameter $\frac{\beta}{1 - \beta} \leq \frac{\rho}{21}$ satisfies $\frac{1}{T}\sum_{t=0}^{T-1}\E\norm{\nabla f(\bar\xx^{t})}^2 \leq \epsilon$ in iterations
	\begin{talign*}
		T  = \cO\Big( \frac{L\sigma^2}{n \epsilon^2} +  \frac{L \tilde\zeta }{\rho \epsilon^{3/2}} + \frac{L}{\epsilon}(\frac{1}{\rho} + \frac{1}{(1-\mu)(1-\beta)^2})\Big)
		\,,
	\end{talign*}
	where $\tilde\zeta^2 := \zeta^2 + (1 + \frac{1-\beta}{1-\mu}) \sigma^2$.
\end{theorem}

\begin{remark}
	The asymptotic number of iterations required,  $\cO\bigl(\frac{\sigma^2}{n \epsilon^2}\bigr)$ shows perfect linear speedup in the number of workers $n$, independent of the communication topology. This upper bound matches the convergence bounds of DSGD~\cite{lian2017can} and centralized mini-batch SGD~\cite{dekel2012optimal}, and is optimal \citep{arjevani2019lower}. This significantly improves over previous analyses of distributed momentum methods which need $\frac{L\sigma^2}{n (1-\beta)\epsilon^2}$ iterations, slowing down for larger values of $\beta$ \citep{yu2019linear,balu2020decentralized}. The second \emph{drift} term $\frac{1}{\rho \epsilon^{3/2}}$ arises due to the non-iid data distribution, and matches the tightest analysis of DSGD without momentum~\citep{koloskova2020unified}. Finally, our theorem imposes some constraint on the momentum parameter $\beta$ (but not on $\mu$). In practice however, \algoptsgdm performs well even when this constraint is violated.
\end{remark}

\subsection{Connection with Other Methods}
We bridge \qg with two recent works below.
The corresponding algorithm details are included
in Appendix~\ref{appendix:detailed_algorithm} for clarity.\looseness=-1
\vspace{-2mm}

\paragraph{Connection with MimeLite.}
MimeLite~\citep{karimireddy2020mime} was recently introduced %in a preprint 
for FL on heterogeneous data.
It shares a similar ingredient as ours:
a ``global'' movement direction $\dd$ is used locally
to alleviate the issue caused by heterogeneity.
The difference falls into the way of forming $\dd$
(c.f.\ line 8 in Algorithm~\ref{alg:hbsgdm_vs_algoptsgdm}):
in MimeLite, $\dd$ is the full batch gradients
computed on the previously synchronized model,
while the $\dd$ in our \algoptsgdm is the difference on two consecutive synchronized models.\looseness=-1
\\
MimeLite only addresses the FL setting,
which results in a computation and communication overhead (to form $\dd$),
and is non-trivial to extend to decentralized learning.

\vspace{-2mm}
\paragraph{Connection with SlowMo.}
SlowMo and its ``noaverage'' variant~\citep{wang2020slowmo} aim to improve generalization performance
in the homogeneous data-center training scenario,
while \algoptsgdm is targeting learning with data heterogeneity.
In terms of update scheme,
SlowMo variants update the slow momentum buffer through the model difference $\dd$
of $\tau \!\ggg\! 1$ local update (and synchronization) steps,
while \algoptsgdm only considers consecutive models (analogously $\tau \!=\! 1$)\footnote{
	We also study the variant of \algoptsgdm with $\tau \!>\! 1$
	in Appendix~\ref{appendix:multiple_step_algoptsgdmn}---we stick to $\tau \!=\! 1$ in the main paper
	for its superior performance and hyper-parameter ($\tau$) tuning free.
}.
Besides, in contrast to \algoptsgdm,
the slow momentum buffer in SlowMo will never interact with the local update---%
setting $\tau$ to 1 in SlowMo variants cannot recover \algoptsgdm.\looseness=-1
\\
SlowMo variants are orthogonal to \algoptsgdm;
combining these two algorithms may lead to a better generalization performance,
and we leave it for future work.

\vspace{1em}
\section{Understanding \algoptsgdm} \label{sec:understanding}
\subsection{Faster Convergence in Average Consensus}
We now consider the simpler averaging consensus problem (isolated from the learning part of \algoptsgdm):
we simplify~\eqref{eq:our_scheme_matrix_form} by removing gradients and step-size:\vspace{-1mm}
\begin{small}
	\begin{talign} \label{eq:update_without_stochastic_gradients}
		\begin{split}
			\mX^{(t+1)} &= \mW \left( \mX^{(t)} - \beta \mM^{(t-1)} \right) \\
			\mM^{(t)}   &= \mu \mM^{(t-1)} + (1 - \mu) \left( \mX^{(t)} - \mX^{(t + 1)} \right) \,,
		\end{split}
	\end{talign}
\end{small}%
and compare it with gossip averaging
$\mX^{(t+1)} = \mW \mX^{(t)}$.

Figure~\ref{fig:understanding_via_consensus_averaging}
depicts the advantages of~\eqref{eq:update_without_stochastic_gradients} over standard gossip averaging,
where \algoptsgdm can quickly converge to a critical consensus distance (e.g.\ $10^{-2}$).
It partially explains the performance gain of \algoptsgdm
from the aspect of improved decentralized communication (which leads to better optimization)---%
decentralized training can converge as fast as its centralized counterpart
once the consensus distance is lower than the critical one, as stated in~\citep{kong2021consensus}.\looseness=-1

\vspace{-1mm}
\subsection{\algoptsgdm (Single Worker Case) Recovers QHM}
Considering the single worker case,
\algoptsgdm can be further simplified to
(derivations in Appendix~\ref{appendix:hd_sgd_with_single_worker_sgdm}):
\begin{small}
	\begin{talign*}
		\begin{split}
			\hat \mm^{(t)} &= \hat \beta \hat \mm^{(t-1)} + \gg^{(t)} \\
			\xx^{(t+1)}
			&
			= \xx^{(t)} - \eta \left( ( 1 - \frac{\mu}{\hat \beta} ) \hat \mm^{(t)} + \frac{\mu}{\hat \beta} \gg^{(t)} \right)
			\,,
		\end{split}
	\end{talign*}
\end{small}%
where $\hat \beta := \mu + (1 - \mu) \beta$.
Thus, the single worker case of \algoptsgdm (i.e.\ \salgoptsgdm)
recovers Quasi-Hyperbolic Momentum
(QHM)~\citep{ma2018quasihyperbolic,gitman2019understanding}.
We illustrate its acceleration benefits as well as the performance gain
in Figure~\ref{fig:understanding_optsgd_on_n1_complete} and Figure~\ref{fig:understanding_optsgd_on_n1_and_wo_weight_decay} of Appendix~\ref{appendix:understanding_single_worker}.
We elaborate in Appendix~\ref{appendix:connection_for_single_worker_case}
that SGDm is only a special case of \salgoptsgdm/QHM (by setting $\mu \!=\! 0$).
Besides, it is non-trivial to adapt (centralized) QHM to (decentralized) \algoptsgdm
due to discrepant motivation.
\looseness=-1

\paragraph{Stabilized optimization trajectory.}
We study the optimization trajectory of Rosenbrock function~\cite{rosenbrock1960automatic} $f(x, y) = (y - x^2)^2 + 100 (x - 1)^2$ as in~\citep{lucas2018aggregated}
to better understand the performance gain of \salgoptsgdm (with zero stochastic noise).\footnote{
	We further study the optimization trajectory for more complicated non-convex function
	in Appendix~\ref{appendix:toy_problem_trajectory}.
}
Figure~\ref{fig:understanding_on_new_2d_toy_example_complete}
illustrates the effects of stabilization in \salgoptsgdm (much less oscillation than SGDm).\looseness=-1

\begin{figure}[!h]
	\centering
	\vspace{-0.5em}
	\includegraphics[width=.4\textwidth,]{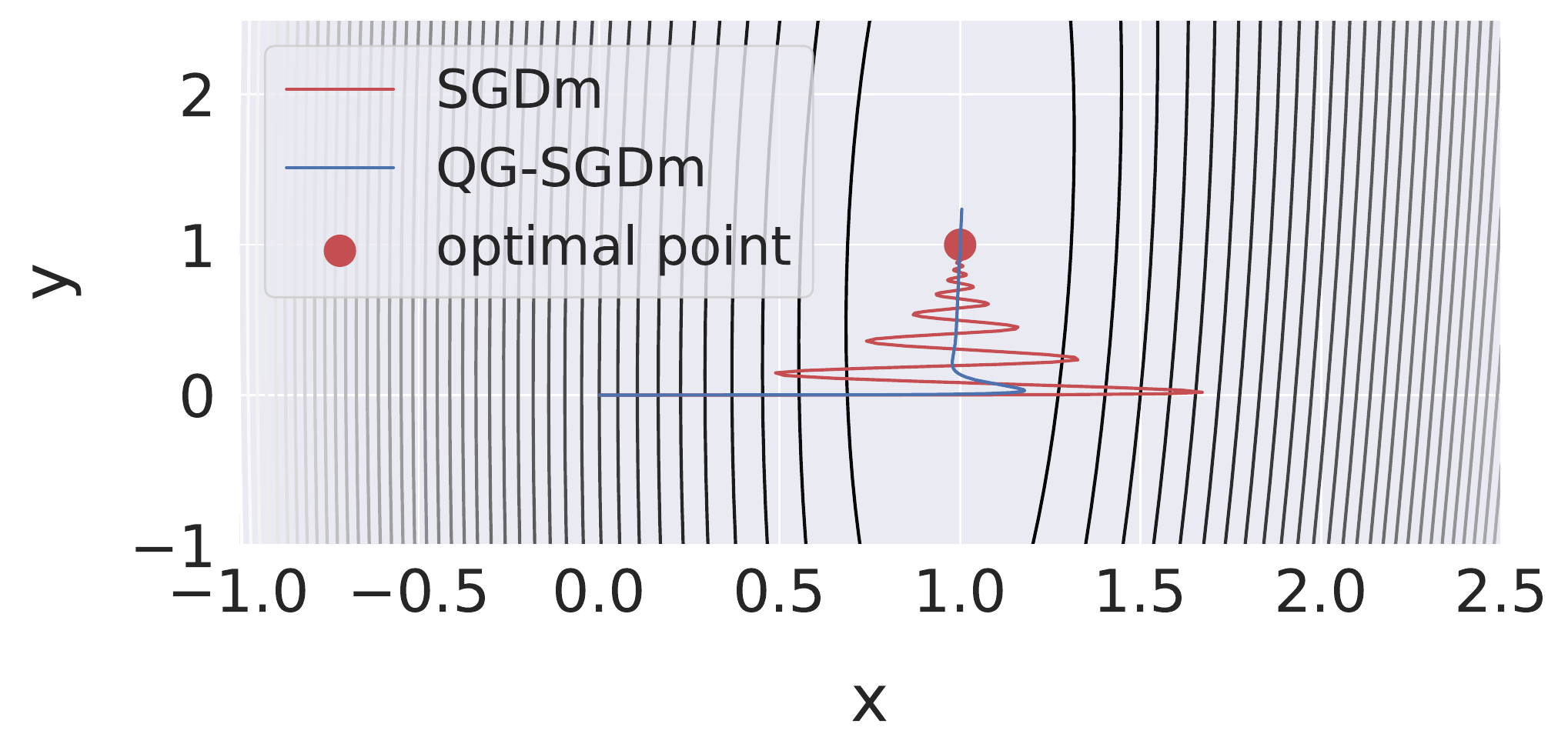}
	\vspace{-1em}
	\caption{\small
		\textbf{Understanding the optimization trajectory of \salgoptsgdm and SGDm}
		(i.e.\ single worker case)
		via a 2D toy function $f(x, y) = (y - x^2)^2 + 100 (x - 1)^2$.
		This function has a global minimum at $(x, y) = (1, 1)$.
		SGDm and \salgoptsgdm use $\beta = 0.9, \eta = 0.001$, with initial point $(0, 0)$.
		For additional trajectories with different initial points and/or $\beta$ values
		refer to Appendix~\ref{appendix:toy_problem_trajectory}.
	}
	\vspace{-1em}
	\label{fig:understanding_on_new_2d_toy_example}
\end{figure}

\begin{figure}[!h]
	\centering
	\subfigure[\small ResNet-BN-20.]{
		\includegraphics[width=0.225\textwidth,]{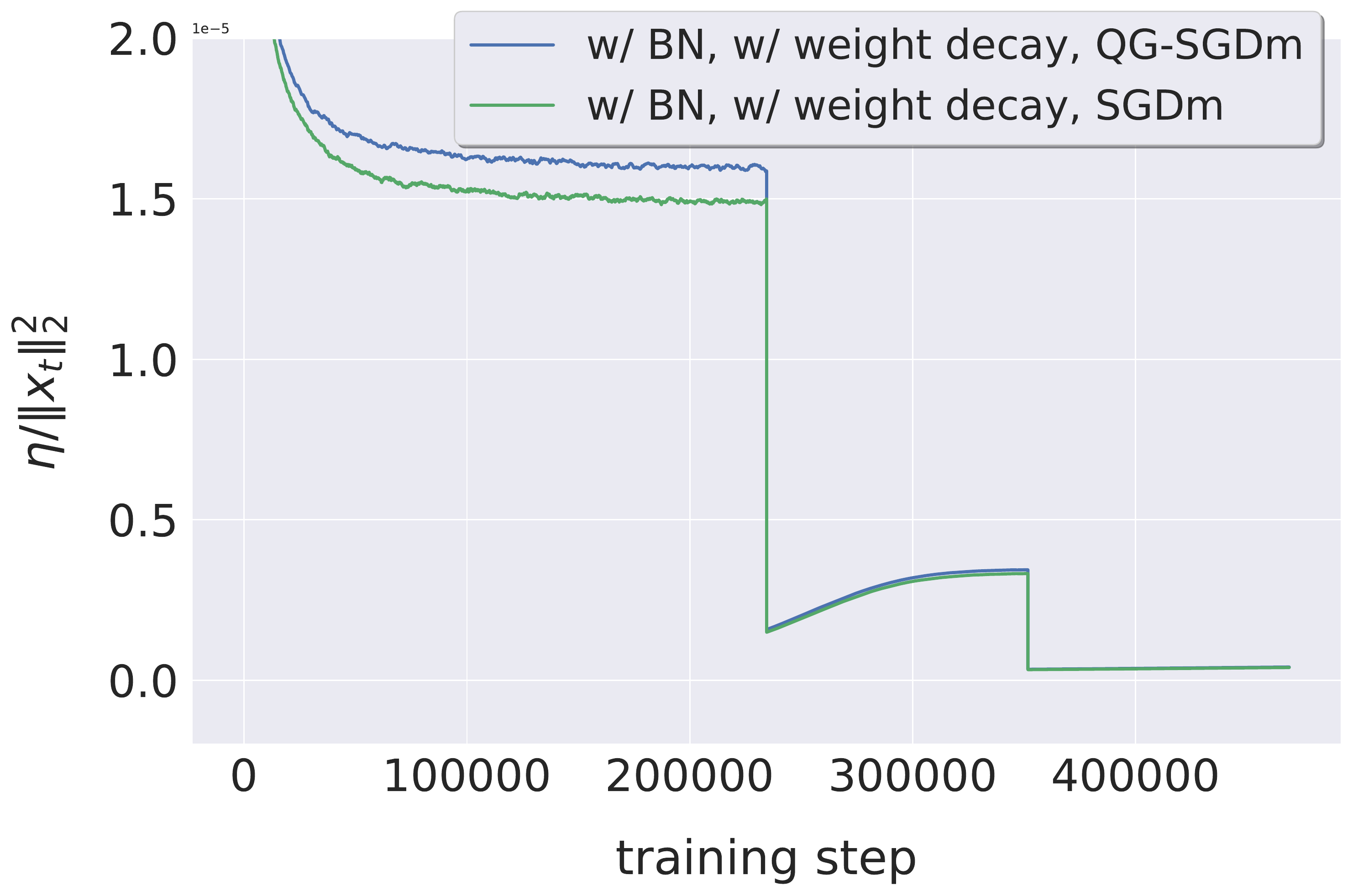}
		\label{fig:n1_resnet20_cifar10_understanding_effective_stepsize_heavyball_momentum_w_bn_w_wd}
	}
	\subfigure[\small ResNet-GN-20.]{
		\includegraphics[width=0.225\textwidth,]{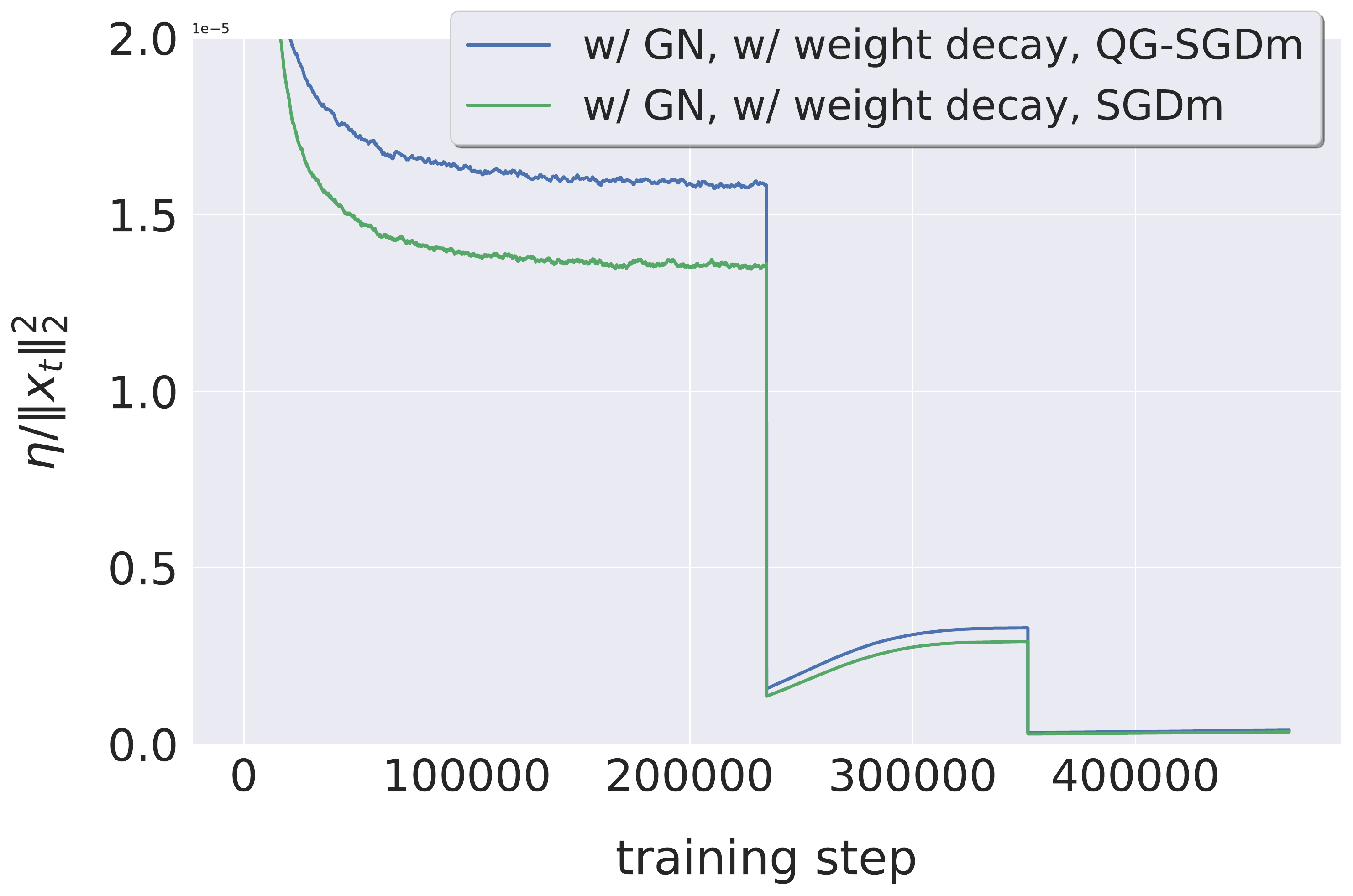}
		\label{fig:n1_resnet20_cifar10_understanding_effective_stepsize_heavyball_momentum_w_gn_w_wd}
	}
	\vspace{-1em}
	\caption{\small
		\textbf{The effective step-size $ \eta / \norm{ \xx_t }_2^2 $ of \salgoptsgdm and SGDm}
		(single worker case) on CIFAR-10.
		For the weight norm curves refer to
		Figure~\ref{fig:n1_resnet20_cifar10_understanding_from_deep_learning_aspect_complete}
		in Appendix~\ref{appendix:understanding_single_worker}.
	}
	\vspace{-1em}
	\label{fig:n1_resnet20_cifar10_understanding_from_deep_learning_aspect}
\end{figure}

\paragraph{Larger effective step-size.}
Recent works~\citep{hoffer2018norm,zhang2018three}
point out the larger effective step-size
(i.e.\ $ \eta / \norm{\xx_t}_2^2 $) brought by weight decay
provides the primary regularization effect for deep learning training.
Figure~\ref{fig:n1_resnet20_cifar10_understanding_from_deep_learning_aspect}
examines the effective step-size during the optimization procedure:
\salgoptsgdm illustrates a larger effective step-size than SGDm,
explaining the performance gain e.g.\ in
Figure~\ref{fig:understanding_optsgd_on_n1_complete}
and Figure~\ref{fig:understanding_optsgd_on_n1_and_wo_weight_decay}
of Appendix~\ref{appendix:understanding_single_worker}.

\section{Experiments}
\subsection{Setup} \label{sec:exp_setup}
\paragraph{Datasets and models.}
We empirically study the decentralized training behavior on both CV and NLP benchmarks,
on the architecture of ResNet~\citep{he2016deep}, VGG~\citep{simonyan2014very}
and DistilBERT~\citep{sanh2019distilbert}.
$\bullet$ Image classification (CV) benchmark:
we consider training CIFAR-10~\citep{krizhevsky2009learning},
ImageNet-32 (i.e.\ image resolution of $32$)~\citep{chrabaszcz2017downsampled},
and ImageNet~\citep{imagenet_cvpr09} from scratch,
with standard data augmentation and preprocessing scheme~\citep{he2016deep}.
We use VGG-11 (with width factor $1/2$ and without BN) and ResNet-20 for CIFAR-10,
ResNet-20 with width factor $2$ (noted as ResNet-20-x2) for ImageNet-32,
and ResNet-18 for ImageNet.
The width factor indicates the proportional scaling of the network width corresponding to the original neural network.
Weight initialization schemes follow~\citep{he2015delving,goyal2017accurate}.
$\bullet$ Text classification (NLP) benchmark:
we perform fine-tuning on a $4$-class classification dataset (AG News~\citep{zhang2015character}).
% and a $2$-class sentiment classification dataset
% (Stanford Sentiment Treebank, SST2~\citep{socher2013recursive}).
Unless mentioned otherwise, all our experiments are repeated over three random seeds.
We report the averaged performance of local models on the full test dataset.
\looseness=-1

\begin{table*}[!t]
	\centering
	\caption{\small
		\textbf{
			The test top-1 accuracy of different decentralized optimization algorithms
			evaluated on different degrees of non-\iid local CIFAR-10 data,
			for various neural architectures and network topologies.
		}
		The results are averaged over three random seeds,
		with learning rate tuning for each setting.
		We also include the results of centralized baseline for reference purposes,
		following the decentralized experiment configuration,
		except that the centralized baseline uses randomly partitioned local training data (i.e.\ independent of $\alpha$).\looseness=-1
	}
	\vspace{-1em}
	\label{tab:cv_main_results}
	\resizebox{1.\textwidth}{!}{%
		\begin{tabular}{ccccccccc}
			\toprule
			\multirow{2}{*}{Datasets} & \multirow{2}{*}{Neural Architectures} &   \multirow{2}{*}{Methods}                 & \multicolumn{3}{c}{Ring ($n\!=\!16$)}  & \multicolumn{3}{c}{Social Network ($n\!=\!32$)} \\ \cmidrule(lr){4-6} \cmidrule(lr){7-9}
			  &   &              & $\alpha=10$               & $\alpha=1$                & $\alpha=0.1$              & $\alpha=10$               & $\alpha=1$                & $\alpha=0.1$              \\ \midrule
			\multirow{16}{*}{CIFAR-10}   & \multirow{4}{*}{ResNet-BN-20} 			& SGDm-N (centralized) & \multicolumn{3}{c}{$92.95 \pm 0.13$}             								               &             \multicolumn{3}{c}{$92.88 \pm 0.07$}                          \\
			  &   & DSGD         & $90.94 \pm 0.15$          & $88.95 \pm 0.59$          & $54.66 \pm 3.58$          & $90.52 \pm 0.24$          & $89.22 \pm 0.35$          & $58.32 \pm 3.27$          \\
			  &   & DSGDm-N      & $92.53 \pm 0.27$          & $89.13 \pm 0.81$          & $57.19 \pm 2.65$          & $92.20 \pm 0.24$          & $90.19 \pm 0.54$          & $63.00 \pm 2.50$          \\
			  &   & \algoptsgdmn & $\textbf{92.65} \pm 0.17$ & $\textbf{91.21} \pm 0.28$ & $\textbf{58.16} \pm 3.32$ & $\textbf{92.52} \pm 0.09$ & $\textbf{91.20} \pm 0.16$ & $\textbf{64.32} \pm 2.43$ \\ \cmidrule(lr){2-9}
			& \multirow{4}{*}{ResNet-GN-20}            & SGDm-N (centralized) & \multicolumn{3}{c}{$88.06 \pm 1.12$}												               &            \multicolumn{3}{c}{$86.19 \pm 1.08$}              \\
			  &   & DSGD         & $86.86 \pm 0.37$          & $85.93 \pm 0.14$          & $73.14 \pm 3.92$          & $84.00 \pm 0.67$          & $82.98 \pm 0.47$          & $67.84 \pm 3.94$          \\
			  &   & DSGDm-N      & $89.86 \pm 0.15$          & $88.30 \pm 0.49$          & $71.86 \pm 2.22$          & $88.54 \pm 0.22$          & $86.36 \pm 0.55$          & $72.02 \pm 1.79$          \\
			  &   & \algoptsgdmn & $\textbf{90.18} \pm 0.44$ & $\textbf{89.68} \pm 0.41$ & $\textbf{82.78} \pm 2.05$ & $\textbf{88.58} \pm 0.09$ & $\textbf{88.19} \pm 0.30$ & $\textbf{83.60} \pm 1.83$ \\ \cmidrule(lr){2-9}
			& \multirow{4}{*}{ResNet-EvoNorm-20}       & SGDm-N (centralized) & \multicolumn{3}{c}{$92.18 \pm 0.19$}                                         &             \multicolumn{3}{c}{$91.92 \pm 0.33$}              \\
			  &   & DSGD         & $89.90 \pm 0.26$          & $88.88 \pm 0.26$          & $74.55 \pm 2.07$          & $89.95 \pm 0.23$          & $88.41 \pm 0.27$          & $77.56 \pm 1.65$          \\
			  &   & DSGDm-N      & $91.47 \pm 0.23$          & $89.98 \pm 0.10$          & $77.48 \pm 2.67$          & $91.17 \pm 0.11$          & $89.96 \pm 0.35$          & $80.59 \pm 2.32$          \\
			  &   & \algoptsgdmn & $\textbf{91.90} \pm 0.17$ & $\textbf{91.28} \pm 0.38$ & $\textbf{82.20} \pm 1.27$ & $\textbf{91.51} \pm 0.02$ & $\textbf{91.00} \pm 0.24$ & $\textbf{85.19} \pm 0.98$ \\ \cmidrule(lr){2-9}
			& \multirow{3}{*}{ \parbox{3.5cm}{ \centering VGG-11 \\ (w/o normalization layer) } }                  & SGDm-N (centralized) & \multicolumn{3}{c}{$88.87 \pm 0.29$}                                         &             \multicolumn{3}{c}{$87.38 \pm 0.39$}              \\
			  &   & DSGDm-N      & $88.68 \pm 0.30$          & $88.52 \pm 0.24$          & $77.45 \pm 3.15$          & $86.39 \pm 0.06$          & $85.85 \pm 0.22$          & $77.02 \pm 2.66$          \\
			  &   & \algoptsgdmn & $\textbf{89.01} \pm 0.04$ & $\textbf{89.00} \pm 0.22$ & $\textbf{83.41} \pm 2.20$ & $\textbf{86.87} \pm 0.60$ & $\textbf{86.09} \pm 0.30$ & $\textbf{84.86} \pm 0.58$ \\
			\bottomrule
		\end{tabular}%
	}
	\vspace{-.5em}
\end{table*}

\vspace{-1mm}
\paragraph{Heterogeneous distribution of client data.}
We use the Dirichlet distribution to create disjoint non-\iid
client training data~\citep{yurochkin2019bayesian,hsu2019measuring,he2020fedml}---the created client data
is fixed and never shuffled across clients during the training.
The degree of non-i.i.d.-ness is controlled by the value of $\alpha$;
the smaller $\alpha$ is, the more likely the clients hold examples from only one class.
An illustration regarding how samples are distributed among $16$ clients
on CIFAR-10 can be found in Figure~\ref{fig:resnet20_cifar10_motivation_example};
more visualizations on other datasets/scales
are shown in Appendix~\ref{appendix:non_iid_complete_visualization}.
Besides, Figure~\ref{fig:nx_davis_southern_women_graph}
in Appendix~\ref{appendix:visualize_communication_topologies}
visualizes the Social Network topology.
\looseness=-1

\vspace{-1mm}
\paragraph{Training schemes.}
Following the SOTA deep learning training scheme,
we use mini-batch SGD as the base optimizer for CV benchmark~\citep{he2016deep,goyal2017accurate},
and similarly, Adam for NLP benchmark~\citep{zhang2019adam,mosbach2021on}.
In Section~\ref{sec:main_results},
we adapt these base optimizers to different distributed variants\footnote{
	We by default use local momentum variants without buffer synchronization.
	We consider DSGDm-N as our primary competitor for CNNs,
	as Nesterov momentum is the SOTA training scheme.
	We also investigate the performance of DSGDm
	in Table~\ref{tab:ablation_study_compare_with_other_momentum_SGDs}.
}.

For the CV benchmark,
the models are trained for $300$ and $90$
epochs for CIFAR-10 and ImageNet(-32) respectively;
the local mini-batch size are set to $32$ and $64$.
All experiments use the SOTA learning rate scheme
in distributed deep learning training~\citep{goyal2017accurate,he2019bag}
with learning rate scaling and warm-up.
The learning rate is always gradually warmed up
from a relatively small value (i.e.\ $0.1$) for the first $5$ epochs.
Besides, the learning rate will be divided by $10$
when the model has accessed specified fractions of
the total number of training samples---%
$\{ \frac{1}{2}, \frac{3}{4} \}$ for CIFAR and $\{ \frac{1}{3}, \frac{2}{3}, \frac{8}{9} \}$ for ImageNet.
\\
For the NLP benchmark, we fine-tune the \textit{distilbert-base-uncased}
from HuggingFace~\citep{wolf2019huggingface}
with constant learning rate and mini-batch of size $32$ for $10$ epochs.

We fine-tune the learning rate for both CV\footnote{
	We tune the initial learning rate and warm it up from $0.1$
	(if the tuned one is above $0.1$).
} and NLP tasks; we use constant weight decay (1e{-}4).
The tuning procedure ensures that the best hyper-parameter
lies in the middle of our search grids; otherwise we extend our search grid.\\
Regarding momentum related hyper-parameters,
we follow the common practice in the community
($\beta \!=\! 0.9$ and without dampening for Nesterov/HeavyBall momentum variants,
and $\beta_1 \!=\! 0.9, \beta_2 \!=\! 0.99$ for Adam variants).

% \vspace{-1mm}
\paragraph{BN and its alternatives for distributed deep learning.}
The existence of BN layer is challenging for the SOTA distributed training,
especially for heterogeneous data setting.
To better understand the impact of different normalization schemes in distributed deep learning,
we investigate:
\begin{itemize}[nosep,leftmargin=12pt]
	\item Distributed BN implementation.
	      Our default implementation\footnote{
		      We also try the BN variant~\citep{li2021fedbn} proposed for FL,
		      but we exclude it in our comparison due to its poor performance.
	      } follows~\citep{goyal2017accurate,andreux2020siloed}
	      that computes the BN statistics independently for each client
	      while only synchronizing the BN weights.

	\item Using other normalization layers:
	      for instance on ResNet with BN layers (denoted by ResNet-BN-20),
	      we can instead use ResNet-GN
	      by replacing all BN with GN with group number of $2$,
	      as suggested in~\citep{hsieh2020non}.
	      %Seb: \citet if we are referring to the person,
	      % suggested by Hsieh et al. (2020).
	      % suggested in [the paper] (Hsieh et al. 2020).
	      We also examine the recently proposed S0 variant of EvoNorm~\citep{liu2020evolving} (which does not use runtime mini-batches statistics), noted as ResNet-EvoNorm.
\end{itemize}

\begin{table*}[!t]
	\centering
	% \vspace{-1em}
	\caption{\small
		\textbf{Comparison with Gradient Tracking (GT) methods} for training CIFAR-10.
		D$^2$ and D$^2_+$ do not include the momentum acceleration.
		% The top table evaluates ResNet-EvoNorm-20, while the bottom table examines the VGG-11 (w/o normalization layer).
		We carefully tune the learning rate for each case, and results are averaged over three seeds where std is indicated.
	}
	\label{tab:comparison_with_gt_methods}
	\vspace{-1.em}
	\resizebox{1.\textwidth}{!}{%
		\begin{tabular}{cccccccccc}
			\toprule
			& \multicolumn{6}{c}{ResNet-EvoNorm-20 on Ring ($n\!=\!16$)} & \multicolumn{3}{c}{ResNet-EvoNorm-20 on Ring ($n\!=\!32$)} \\ \cmidrule(lr){2-7} \cmidrule(lr){8-10}
			                   & DSGD (w/ GT)     & DSGDm-N          & DSGDm-N (w/ GT)  & D$^2$   & D$^2_+$          & QG-DSGDm-N                & DSGDm-N          & DSGDm-N (w/ GT)  & QG-DSGDm-N                \\
			\midrule
			$\alpha \!=\! 1$   & $87.36 \pm 0.40$ & $89.98 \pm 0.10$ & $90.38 \pm 0.41$ & $74.89$ & $85.70 \pm 0.29$ & $\textbf{91.28} \pm 0.38$ & $88.46 \pm 0.29$ & $89.44 \pm 0.60$ & $\textbf{90.27} \pm 0.27$ \\
			$\alpha \!=\! 0.1$ & $66.16 \pm 1.05$ & $77.48 \pm 2.67$ & $78.64 \pm 1.84$ & $49.80$ & $69.18 \pm 3.30$ & $\textbf{82.20} \pm 1.27$ & $78.17 \pm 1.63$ & $79.25 \pm 2.17$ & $\textbf{83.18} \pm 1.11$ \\
			\bottomrule
		\end{tabular}%
	}
	\vspace{-0.5em}
\end{table*}

\subsection{Results} \label{sec:main_results}
\paragraph{Comments on BN and its alternatives.}
Table~\ref{tab:cv_main_results} and Table~\ref{tab:cv_imagenet_main_results}
examine the effects of BN and its alternatives
on the training quality of decentralized deep learning on CIFAR-10 and ImageNet dataset.
\emph{ResNet with EvoNorm replacement outperforms its GN counterpart}
on a spectrum of optimization algorithms, non-\iid degrees, and network topologies,
\emph{illustrating its efficacy to be a new alternative to BN in CNNs for distributed learning on heterogeneous data.} \looseness=-1

\paragraph{Superior performance of \qg.}
We evaluate \algoptsgdmn and compare it with several DSGD variants
in Table~\ref{tab:cv_main_results},
for training different neural networks on CIFAR-10 in terms of different non-\iid degrees
on Ring ($n\!=\!16$) and Social Network ($n \!=\! 32$).
\emph{\algoptsgdmn accelerates the training
	by stabilizing the oscillating optimization trajectory caused by heterogeneity
	and leads to a significant performance gain over all other strong competitors on all levels of data heterogeneity.
	The benefits of our method are further pronounced when considering a higher degree of non-i.i.d.-ness.}
These observations are consistent with the results on the challenging ImageNet(-32) dataset
in Table~\ref{tab:cv_imagenet_main_results}
(and the learning curves in Figure~\ref{fig:learning_curves_cv_tasks} in Appendix~\ref{appendix:learning_curves_cv_task}).

\begin{table}[!h]
	\centering
	\caption{\small
		\textbf{Test top-1 accuracy of different decentralized optimization algorithms
			evaluated on different degrees of non-\iid local ImageNet data}.
		The results are over three random seeds.
		We perform sufficient learning rate tuning on ImageNet-32 for each setup
		while we use the same one for ImageNet due to the computational feasibility.
		``$\star$'' indicates non-convergence.
	}
	\vspace{-1em}
	\label{tab:cv_imagenet_main_results}
	\resizebox{.475\textwidth}{!}{%
		\begin{tabular}{ccccc}
			\toprule
			\multirow{2}{*}{Datasets} & \multirow{2}{*}{ \parbox{2.cm}{ \centering Neural \\ Architectures } } &   \multirow{2}{*}{Methods}                 & \multicolumn{2}{c}{Ring ($n\!=\!16$)}  \\ \cmidrule(lr){4-5}
			  &   &              & $\alpha=1$                & $\alpha=0.1$              \\ \midrule
			\multirow{6}{*}{ \parbox{2.cm}{ \centering ImageNet-32 \\ (resolution 32) } } & \multirow{3}{*}{ \parbox{2.cm}{ \centering ResNet-20-x2 \\ (EvoNorm) } } & SGDm-N (centralized) & \multicolumn{2}{c}{$44.43 \pm 0.20$}                                                    \\
			  &   & DSGDm-N      & $30.35 \pm 0.05$          & $16.71 \pm 0.17$          \\
			  &   & \algoptsgdmn & $\textbf{31.24} \pm 0.27$ & $\textbf{19.53} \pm 0.91$ \\ \cmidrule(lr){2-5}
			& \multirow{3}{*}{ \parbox{2.cm}{ \centering ResNet-20-x2 \\ (GN) }  }      & SGDm-N (centralized) & \multicolumn{2}{c}{$37.89 \pm 0.67$}                                                      \\
			  &   & DSGDm-N      & $34.16 \pm 1.37$          & $\star$                   \\
			  &   & \algoptsgdmn & $\textbf{38.57} \pm 0.45$ & $\textbf{21.42} \pm 0.81$ \\ \midrule
			\multirow{6}{*}{ImageNet} & \multirow{3}{*}{ \parbox{2.cm}{ \centering ResNet-18 \\ (EvoNorm) } } & SGDm-N (centralized) & \multicolumn{2}{c}{$69.55 \pm 0.25$}                                                    \\
			  &   & DSGDm-N      & $68.77 \pm 0.05$          & $53.15 \pm 0.14$          \\
			  &   & \algoptsgdmn & $\textbf{69.20} \pm 0.08$ & $\textbf{56.50} \pm 0.01$ \\ \cmidrule(lr){2-5}
			& \multirow{3}{*}{ \parbox{2.cm}{ \centering ResNet-18 \\ (GN) } }      & SGDm-N (centralized) & \multicolumn{2}{c}{$62.59 \pm 0.01$}                                                      \\
			  &   & DSGDm-N      & $60.76 \pm 0.48$          & $39.57 \pm 1.22$          \\
			  &   & \algoptsgdmn & $\textbf{64.92} \pm 0.27$ & $\textbf{47.86} \pm 1.05$ \\
			\bottomrule
		\end{tabular}%
	}
	\vspace{-0.1em}
\end{table}

\begin{table}[!h]
	\centering
	\caption{\small
		\textbf{Test top-1 accuracy of different decentralized SGD algorithms evaluated on different degrees of non-i.i.d.-ness and communication topologies}, for training ResNet-EvoNorm-18 on ImageNet.
		The results are over three random seeds.
		We use the same learning rate for different experiments due to the computational feasibility.
		Centralized SGDm-N reaches $69.55 \pm 0.25$.
	}
	\vspace{-1em}
	\label{tab:cv_imagenet_main_results_other_topologies}
	\resizebox{.475\textwidth}{!}{%
		\begin{tabular}{cccc}
			\toprule
			\multirow{2}{*}{Communication Topology}  &   \multirow{2}{*}{Methods}                 & \multicolumn{2}{c}{Test Top-1 Accuracy}  \\ \cmidrule(lr){3-4}
			                                   &              & $\alpha=1$                & $\alpha=0.1$              \\ \midrule
			\multirow{2}{*}{Ring ($n\!=\!16$)} & DSGDm-N      & $68.77 \pm 0.05$          & $53.15 \pm 0.14$          \\
			                                   & \algoptsgdmn & $\textbf{69.20} \pm 0.08$ & $\textbf{56.50} \pm 0.01$ \\ \midrule
			\multirow{2}{*}{ \parbox{5.cm}{ 1-peer directed exponential graph \\ ($n\!=\!16$)~\cite{assran2019stochastic} } } & DSGDm-N      & $69.00 \pm 0.11$                     & $58.52 \pm 0.27$                     \\
			                                   & \algoptsgdmn & $\textbf{69.34} \pm 0.17$ & $\textbf{61.44} \pm 0.20$ \\
			\bottomrule
		\end{tabular}%
	}
	\vspace{-0.5em}
\end{table}

\vspace{-1mm}
\paragraph{Decentralized Adam.}
We further extend the idea of \qg to the Adam optimizer for decentralized learning, noted as \algopadam (the algorithm details are deferred to Algorithm~\ref{alg:algopadam} in Appendix~\ref{appendix:detailed_algorithm}).
We validate the effectiveness of \algopadam over D(decentralized)Adam in Table~\ref{tab:nlp_main_results},
on fine-tuning DistilBERT on AG News and training ResNet-EvoNorm-20 on CIFAR-10 from scratch: \emph{\algopadam is still preferable over DAdam}.
We leave a better adaptation and theoretical proof for future work.
% for its high robustness to the heterogeneous data.}

\vspace{-1mm}
\paragraph{Generalizing \qg to time-varying topologies.}
The benefits of \qg are not limited to the fixed and undirected communication topologies, e.g.\ Ring and Social network in Table~\ref{tab:cv_main_results}---it also generalizes to other topologies, like the time-varying directed topology~\citep{assran2019stochastic}, as shown in Table~\ref{tab:cv_imagenet_main_results_other_topologies} for training ResNet-EvoNorm-18 on ImageNet.
These results are aligned with the findings on the influence of the consensus distance on the generalization performance of decentralized deep learning~\citep{kong2021consensus}, supporting the fact that \emph{\qg can be served as a simple plugin to further improve the performance of decentralized deep learning.}

\begin{table*}[!t]
	\centering
	\caption{\small
		\textbf{An extensive investigation for a wide spectrum of DSGD variants},
		for training ResNet-EvoNorm-20 on CIFAR-10.
		The results are averaged over three seeds, each with learning rate tuning.
		We use ``communication topology'' to synchronize the model parameters,
		while some methods involve ``extra communication'',
		with specified objective to be communicated on the given network topology.\looseness=-1
	}
	\vspace{-1em}
	\label{tab:ablation_study_compare_with_other_momentum_SGDs}
	\resizebox{.85\textwidth}{!}{%
		\begin{tabular}{cccccc}
			\toprule
			\multirow{2}{*}{Methods}          &      \multirow{2}{*}{ \parbox{2.cm}{ \centering Communication \\ Topology }}  &      \multirow{2}{*}{ \parbox{2.cm}{ \centering Extra \\ Communication} } & \multirow{2}{*}{ \parbox{2.cm}{ \centering Momentum \\ Type} } & \multicolumn{2}{c}{Test Top-1 Accuracy ($n\!=\!16$)} \\ \cmidrule(lr){5-6}
			             &          &                             &                 & $\alpha=1$                & $\alpha=0.1$              \\ \midrule
			SGDm-N & complete & - & global &             \multicolumn{2}{c}{$92.18 \pm 0.19$}                                        \\
			\midrule
			DSGDm-N      & complete & -                           & local           & $91.47 \pm 0.10$          & $71.24 \pm 3.08$          \\
			DSGDm-N      & ring     & momentum buffer (complete)  & local           & $90.96 \pm 0.33$          & $81.22 \pm 1.78$          \\
			SlowMo       & ring     & model parameters (complete) & local \& global & $91.06 \pm 0.26$          & $79.20 \pm 1.16$          \\
			\midrule
			DSGD         & ring     & -                           & -               & $88.88 \pm 0.26$          & $74.55 \pm 2.07$          \\
			DSGDm        & ring     & -                           & local           & $89.67 \pm 0.33$          & $77.66 \pm 0.95$          \\
			DSGDm-N      & ring     & -                           & local           & $89.98 \pm 0.10$          & $77.48 \pm 2.67$          \\
			DSGDm        & ring     & momentum buffer (ring)      & local           & $90.42 \pm 0.32$          & $78.69 \pm 2.39$          \\
			DSGDm-N      & ring     & momentum buffer (ring)      & local           & $90.48 \pm 0.67$          & $79.83 \pm 2.29$          \\
			DSGDm-N      & ring     & local gradients (ring)      & local           & $90.10 \pm 0.61$          & $78.58 \pm 4.12$          \\
			DMSGD        & ring     & -                           & local           & $90.06 \pm 0.04$          & $79.89 \pm 0.97$          \\
			\midrule
			\algoptsgdm  & ring     & -                           & local           & $91.22 \pm 0.41$          & $\textbf{82.24} \pm 1.05$ \\
			\algoptsgdmn & ring     & -                           & local           & $\textbf{91.28} \pm 0.38$ & $82.20 \pm 1.27$          \\
			\bottomrule
		\end{tabular}%
	}
	\vspace{-0.5em}
\end{table*}

\begin{table}[!h]
	\centering
	\caption{\small
		\textbf{Test accuracy of different decentralized optimization algorithms (with Adam), evaluated on different degrees of non-\iid local data}.
		The results are over three random seeds, with tuned learning rate.
		% We include the centralized results for reference purposes (following the same configuration as the decentralized one), where the All-Reduced gradients enable the same moment buffers.
	}
	\vspace{-1em}
	\label{tab:nlp_main_results}
	\resizebox{.475\textwidth}{!}{%
		% \begin{tabular}{cccc}
		% 	\toprule
		% 	\multirow{2}{*}{Models \& Datasets}  			&   \multirow{2}{*}{Methods}                 			& \multicolumn{2}{c}{Ring ($n\!=\!16$)} \\ \cmidrule(lr){3-4}
		% 	  &            & $\alpha=1$       & $\alpha=0.1$              \\ \midrule
		% 	\multirow{3}{*}{ \parbox{5.cm}{ \centering Fine-tuning DistilBERT-base \\ (AG News) } }  	& Adam (centralized)    &        \multicolumn{2}{c}{$ \pm $}                                        \\
		% 	  & DAdam      & $ \pm $          & $ 87.29 \pm 0.60 $        \\
		% 	  & \algopadam & $\textbf{} \pm $ & $\textbf{88.33} \pm 0.67$ \\ \midrule
		% 	\multirow{3}{*}{ \parbox{5.cm}{ \centering Training ResNet-EvoNorm-20 \\ from scratch (CIFAR-10) } } & Adam (centralized) &                   \multicolumn{2}{c}{$ \pm $}                                         \\
		% 	  & DAdam      & $\pm $           & $65.52 \pm 3.32$          \\
		% 	  & \algopadam & $\textbf{} \pm $ & $\textbf{66.86} \pm 2.81$ \\
		% 	\bottomrule
		% \end{tabular}%
		\begin{tabular}{ccc}
			\toprule
			Models \& Datasets & Methods    & $\alpha=0.1$              \\ \midrule
			\multirow{2}{*}{ \parbox{5.cm}{ \centering Fine-tuning DistilBERT-base \\ (AG News) } }  & DAdam                & $ 87.29 \pm 0.60 $        \\
			                   & \algopadam & $\textbf{88.33} \pm 0.67$ \\ \midrule
			\multirow{2}{*}{ \parbox{5.cm}{ \centering Training ResNet-EvoNorm-20 \\ from scratch (CIFAR-10) } } & DAdam                 & $65.52 \pm 3.32$          \\
			                   & \algopadam & $\textbf{66.86} \pm 2.81$ \\
			\bottomrule
		\end{tabular}%
	}
	\vspace{-0.5em}
\end{table}

\vspace{-2mm}
\paragraph{Comparison with D$^2$ and Gradient Tracking (GT).}
As shown in Table~\ref{tab:comparison_with_gt_methods}, D$^2$~\citep{tang2018d} and GT methods~\citep{pu2020distributed,pan2020d,lu2019gnsd} cannot achieve comparable test performance on the standard deep learning benchmark, while \algoptsgdmn outperforms them significantly. Additional detailed comparisons are deferred to Appendix~\ref{appendix:comparison_with_gradient_tracking}.\\
It is non-trivial to integrate D$^2$ with momentum.
Besides, D$^2$ requires constant learning rate, which does not fit the SOTA learning rate schedules (e.g.\ stage-wise) in deep learning\footnote{D$^2$ can be rewritten as $\scriptstyle \mW ( \mX^{(t)} - \eta ( (\mX^{(t-1)} - \mX^{(t)} ) / \eta + \nabla f(\mX^{(t)}) - \nabla f(\mX^{(t-1)}) ) )$, and the update would break if the magnitude of $\scriptstyle \mX^{(t-1)} - \mX^{(t)}$ is a factor of $10 \eta$ (i.e.\ performing learning rate decay at step $t$).}.
We include an improved D$^2$ variant\footnote{The update scheme of D$^2_+$ follows $\scriptstyle \mW ( \mX^{(t)} - \eta^{(t)} ( (\mX^{(t-1)} - \mX^{(t)} ) / \eta^{(t-1)} + \nabla f(\mX^{(t)}) - \nabla f(\mX^{(t-1)}) ) )$.} (denoted as D$^2_+$) to address this learning rate decay issue in D$^2$, but the performance of D$^2_+$ still remains far behind our scheme.

\vspace{-1mm}
\paragraph{Ablation study.}
Table~\ref{tab:ablation_study_compare_with_other_momentum_SGDs}
empirically investigates a wide range of different DSGD variants,
in terms of the generalization performance on different degrees of data heterogeneity.
We can witness that
(1) DSGD variants with \qg always significantly surpass all other methods
(excluding the centralized upper bound), without introducing extra communication cost;
(2) local momentum accelerates the decentralized optimization
(c.f.\ the results of DSGD v.s.\ DSGDm and DSGDm-N),
while our \qg further improves the performance gain;
(3) synchronizing local momentum buffer or local gradients
only marginally improves the generalization performance, but the gains fall behind our \qg (as we accelerate the consensus and stabilize trajectories, as illustrated in Section~\ref{sec:understanding});
(4) the parallel work DMSGD\footnote{
	We tune both learning rate $\eta$ and weighting factor $\mu$
	(using the grid suggested in~\citet{balu2020decentralized}) for DMSGD (option I).
}~\citep{balu2020decentralized} does show some improvements,
but its performance gain is much less significant than ours.
Table~\ref{fig:tuning_momentum_factors} in the Appendix~\ref{appendix:tuning_lr_and_momentum_for_sgdmN}
further shows that tuning momentum factor for DSGDm-N cannot alleviate the training difficulty caused by data heterogeneity.

Besides, Figure~\ref{fig:cv_different_ring_scale_results} showcases the generality of \qg for achieving remarkable performance gain on different topology scales and non-i.i.d.\ degrees.

\begin{figure}[!h]
	\centering
	\vspace{-1em}
	\subfigure[$\alpha\!=\!1$.]{
		\includegraphics[width=.225\textwidth,]{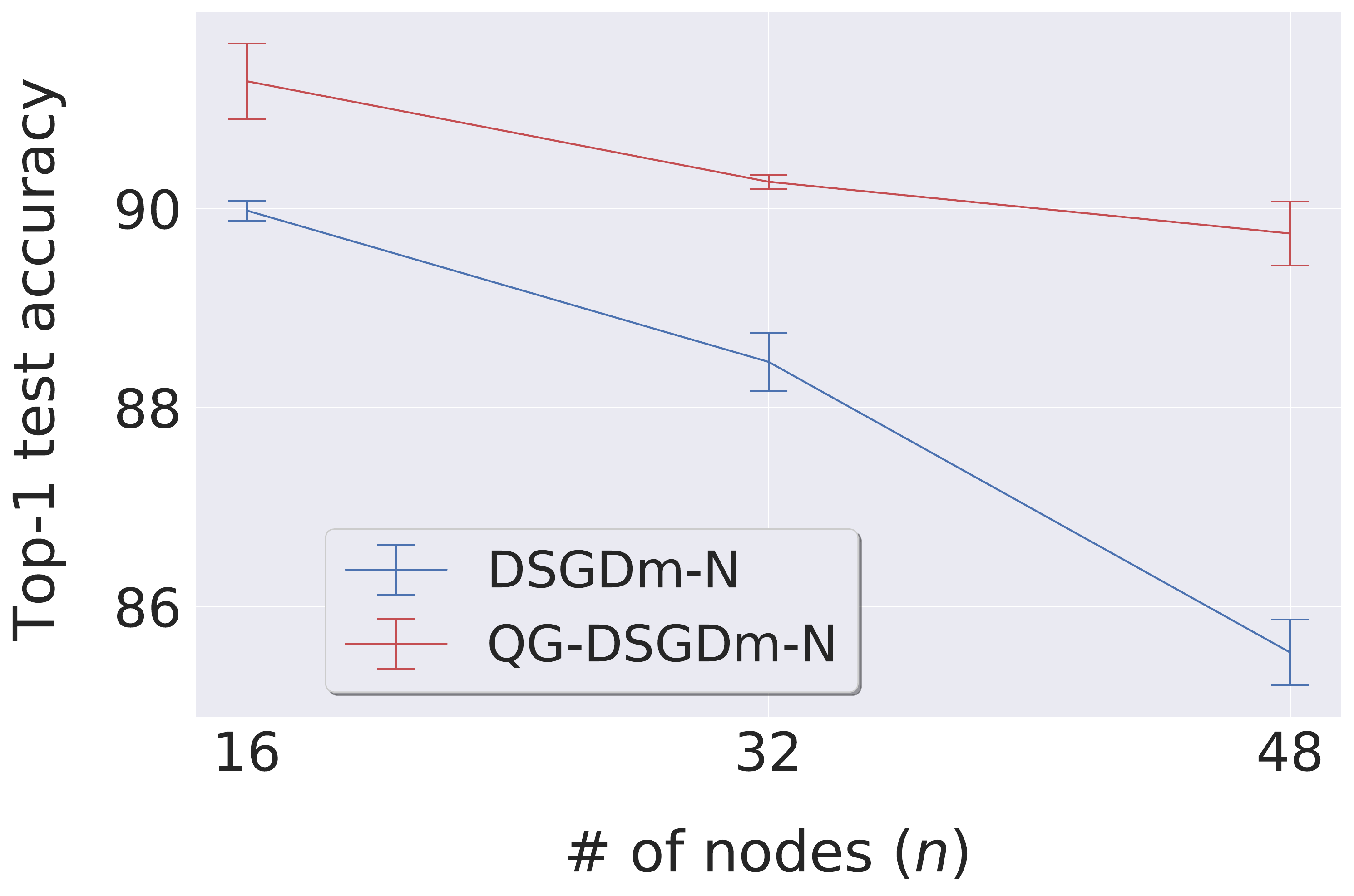}
		\label{fig:perf_vs_scales_resnet_evonorm_20_cifar10_alpha1}
	}
	\subfigure[$\alpha\!=\!0.1$.]{
		\includegraphics[width=.225\textwidth,]{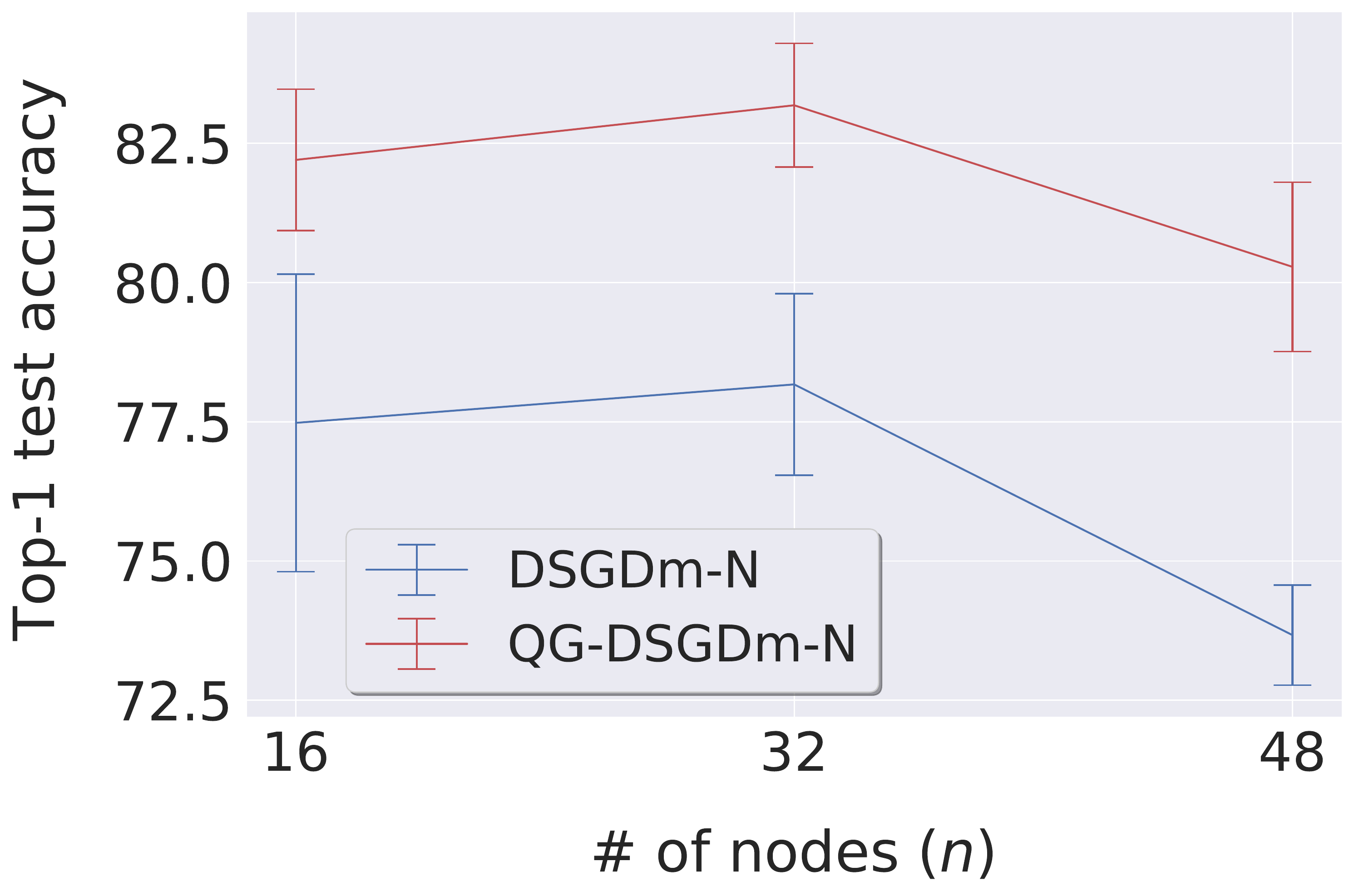}
		\label{fig:perf_vs_scales_resnet_evonorm_20_cifar10_alpha01}
	}
	\vspace{-1em}
	\caption{\small
		\textbf{Test top-1 accuracy of different decentralized algorithms
			evaluated on different topology scales and non-i.i.d.\ degrees},
		for training ResNet-EvoNorm-20 on CIFAR-10.
		The results are over three random seeds,
		each with sufficient learning rate tuning.
		Colors blue and red indicate DSGDm-N and \algoptsgdmn respectively.
		Numerical results refer to Table~\ref{tab:cv_different_ring_scale_results}
		in Appendix~\ref{appendix:cv_different_ring_scale_results}.
	}
	\label{fig:cv_different_ring_scale_results}
\end{figure}

\vspace{-0.5em}
\section*{Conclusion}
We demonstrated that heterogeneity has an out sized impact on the performance of deep learning models, leading to unstable convergence and poor performance.
We proposed a novel momentum-based algorithm to stabilize the training and established its efficacy through thorough empirical evaluations.
Our method, especially for mildly heterogeneous settings, leads to a 10--20\% increase in accuracy.
However, a gap still remains between the centralized training. Closing this gap, we believe, is critical for wider adoption of decentralized learning.

\section*{Acknowledgements}
We acknowledge funding from a Google Focused Research Award, Facebook, and European Horizon 2020 FET Proactive Project DIGIPREDICT.

%% file: appendix.tex
\onecolumn
{
	\hypersetup{linkcolor=black}
	\parskip=0em
	\renewcommand{\contentsname}{Contents of Appendix}
	\tableofcontents
	\addtocontents{toc}{\protect\setcounter{tocdepth}{3}}
}

\clearpage %(just because otherwise the next page is half empty)

\section{Detailed Experimental Setup} \label{appendix:detailed_exp_setup}

\subsection{Visualization for Communication Topologies} \label{appendix:visualize_communication_topologies}
Figure~\ref{fig:nx_davis_southern_women_graph} visualizes the Social Network topology
we evaluated in the main paper.

\begin{figure*}[!h]
	\centering
	\includegraphics[width=.7\textwidth,]{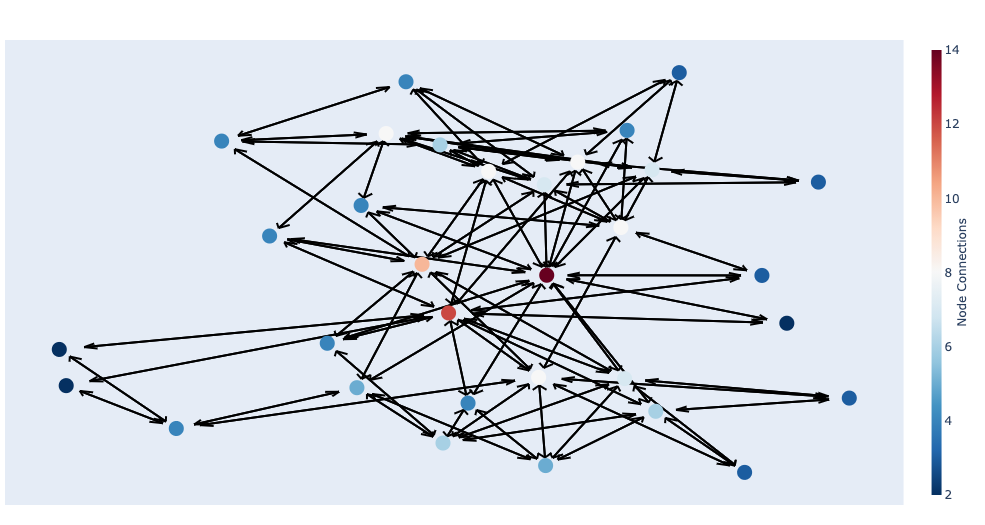}
	\caption{\small
		The visualization of the examined social topology
		(generated from ``networkx.generators.social.davis\_southern\_women\_graph'').
	}
	\label{fig:nx_davis_southern_women_graph}
\end{figure*}

\subsection{Visualization of Non-IID Local Data} \label{appendix:non_iid_complete_visualization}
\paragraph{The synthetic formulation of non-\iid client data.}
We re-iterate the partition scheme introduced and stated
in~\citep{yurochkin2019bayesian,hsu2019measuring} for completeness reasons.

Assume every client training example is drawn
independently with class labels following a categorical distribution
over $M$ classes parameterized by a vector $\qq$ ($q_i \geq 0, i \in [1, M]$ and $\norm{\qq}_1 = 1$).
To synthesize client non-\iid local data distributions,
we draw $\alpha \sim \text{Dir} (\alpha \pp)$ from a Dirichlet distribution,
where $\pp$ characterizes a prior class distribution over $M$ classes,
and $\alpha > 0$ is a concentration parameter controlling the identicalness among clients.
With $\alpha \rightarrow \infty$, all clients have identical distributions to the prior;
with $\alpha \rightarrow 0$, each client holds examples from only one random class.

To better understand the local data distribution for the datasets we considered in the experiments,
in Figure~\ref{fig:illustration_noniid_for_cv_datasets}
we visualize the partition results of CIFAR-10 and ImageNet(-32)
for various degrees of non-i.i.d.-ness and network scales;
in Figure~\ref{fig:illustration_noniid_for_nlp_datasets},
we visualize the partitioned local data on $16$ clients with $\alpha \!=\! \{ 10, 1, 0.1 \}$
for AG News and SST-2.

\begin{figure*}[!h]
	\centering
	\subfigure[\small
		CIFAR-10, $n\!=\!16$, $\alpha=10$.
	]{
		\includegraphics[width=.315\textwidth,]{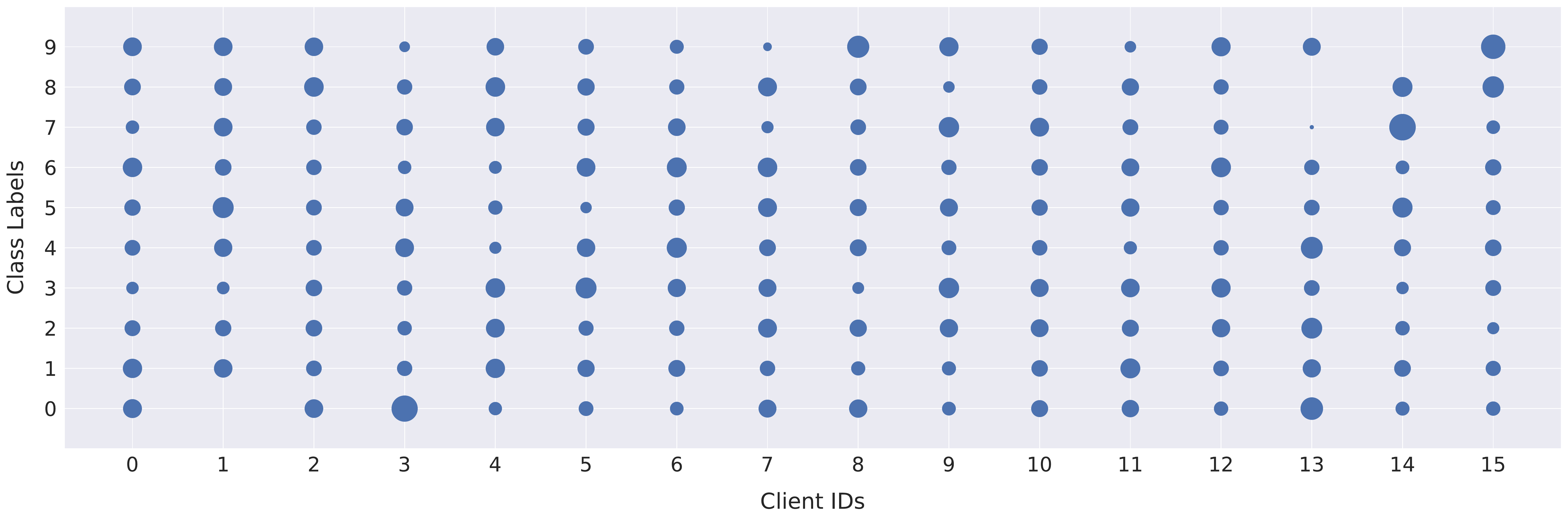}
		\label{fig:cifar10_non_iid_dirichlet_10_n16}
	}
	\hfill
	\subfigure[\small
		CIFAR-10, $n\!=\!16$, $\alpha=1$.
	]{
		\includegraphics[width=.315\textwidth,]{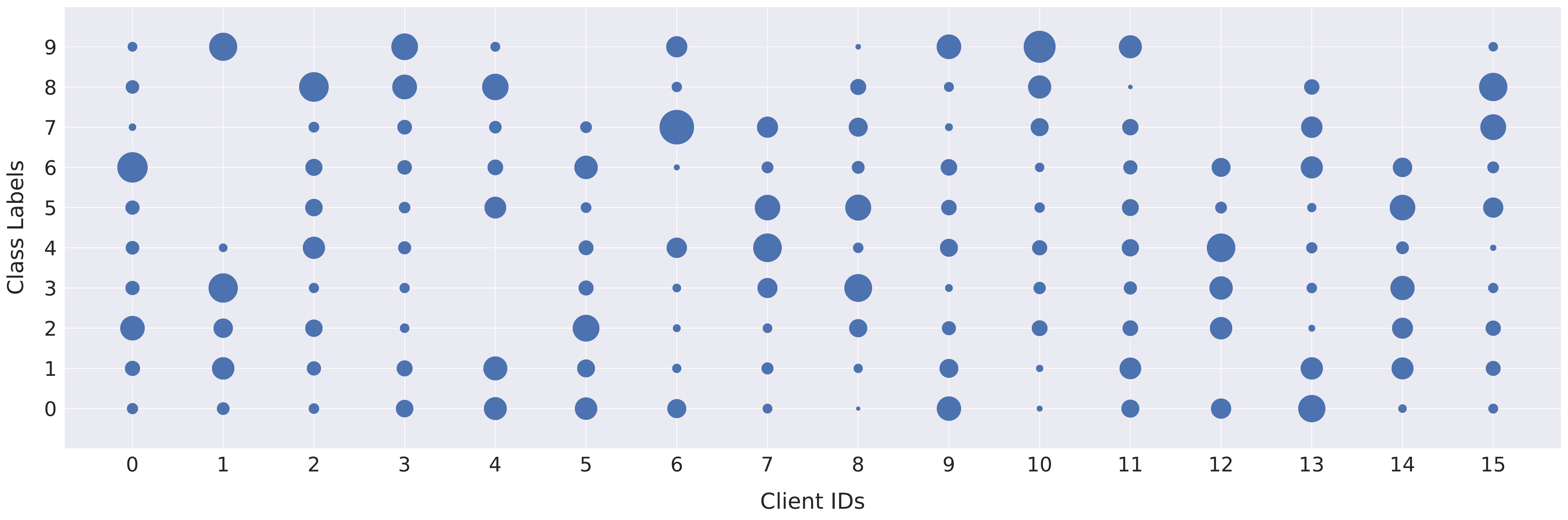}
		\label{fig:cifar10_non_iid_dirichlet_1_n16}
	}
	\hfill
	\subfigure[\small
		CIFAR-10, $n\!=\!16$, $\alpha=0.1$.
	]{
		\includegraphics[width=.315\textwidth,]{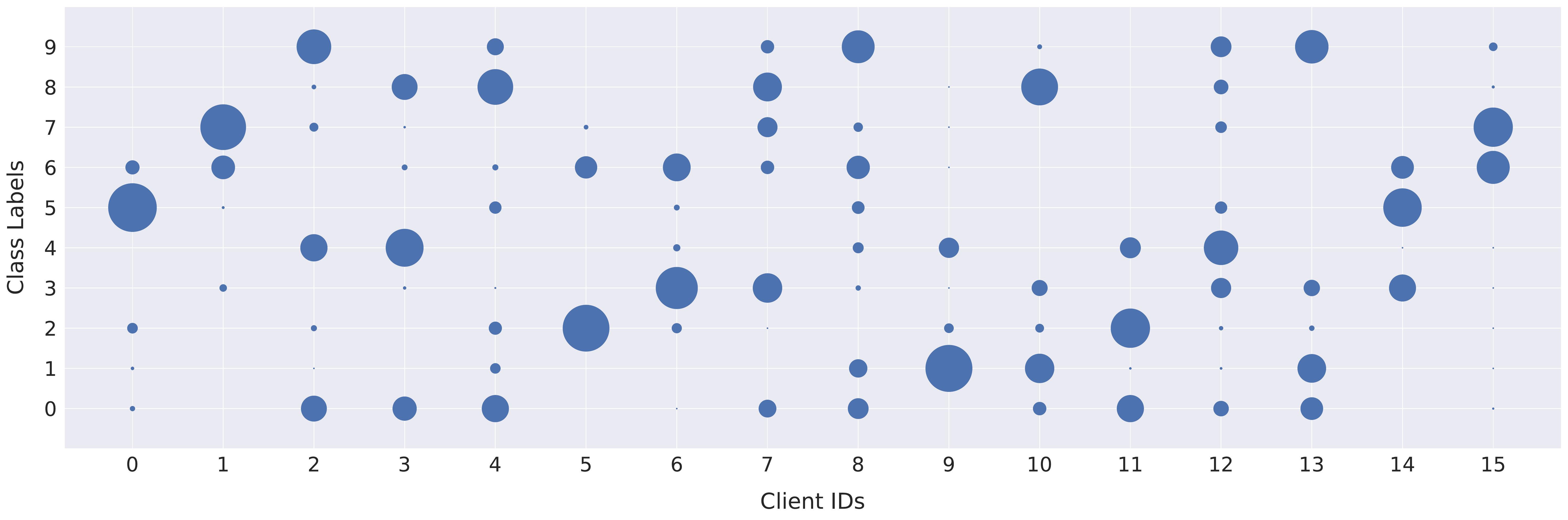}
		\label{fig:cifar10_non_iid_dirichlet_01_n16}
	}
	\hfill
	\subfigure[\small
		CIFAR-10, $n\!=\!32$, $\alpha=1$.
	]{
		\includegraphics[width=.475\textwidth,]{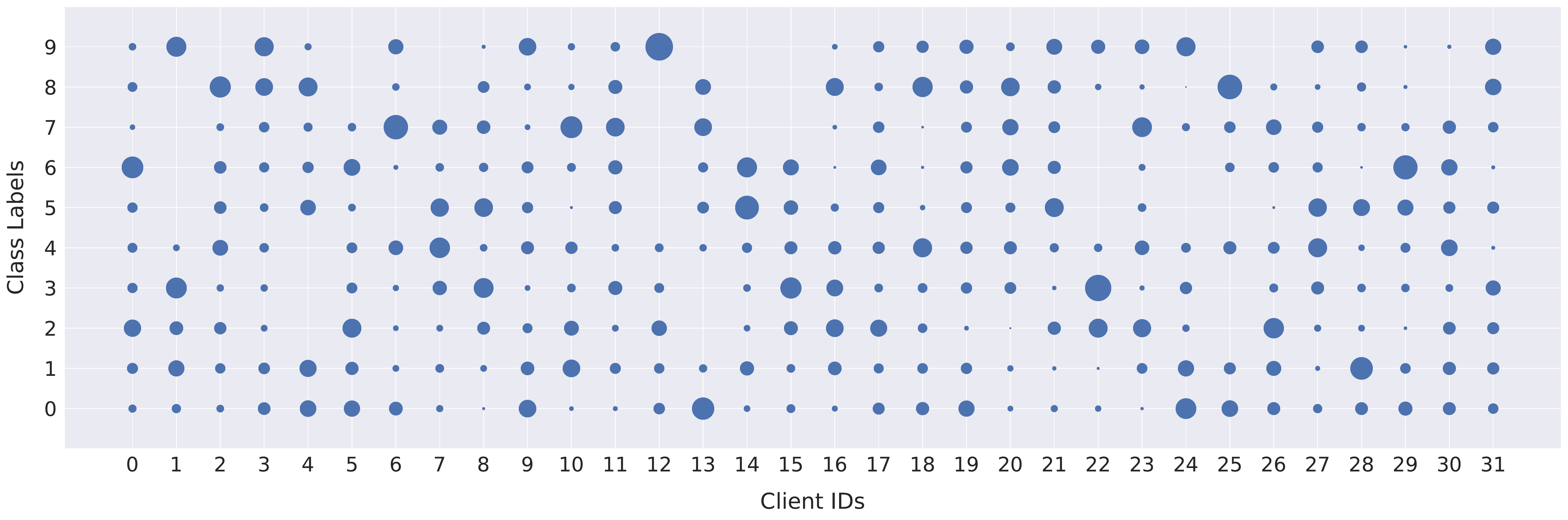}
		\label{fig:cifar10_non_iid_dirichlet_1_n32}
	}
	\hfill
	\subfigure[\small
		CIFAR-10, $n\!=\!32$, $\alpha=0.1$.
	]{
		\includegraphics[width=.475\textwidth,]{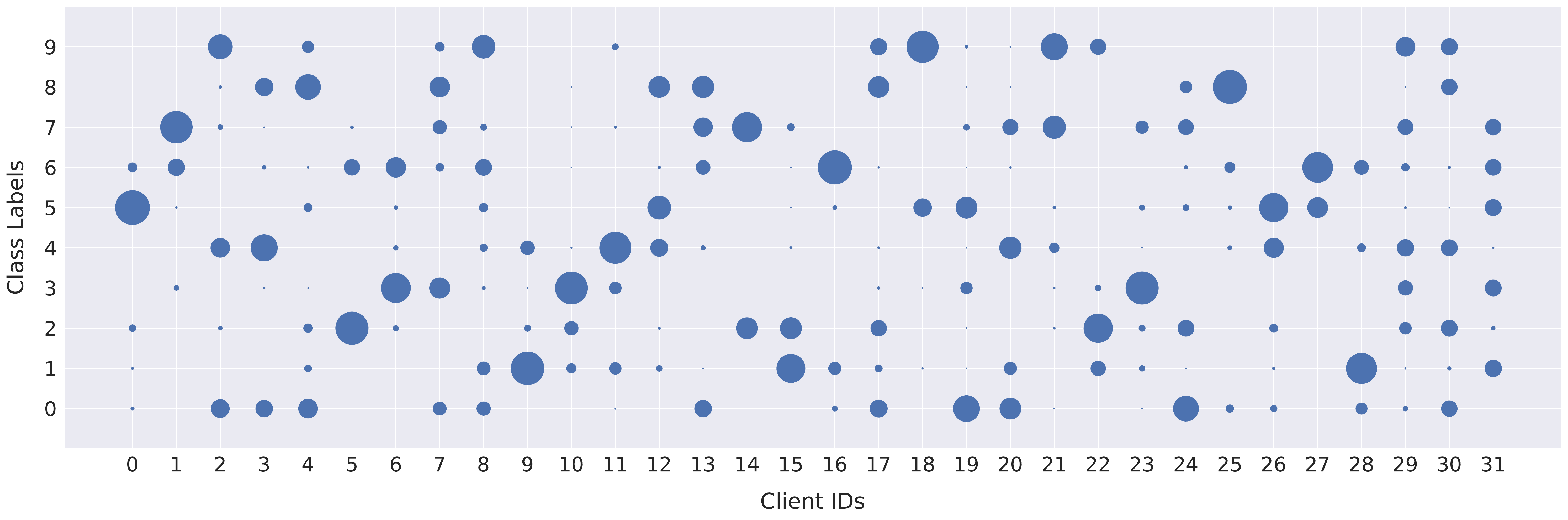}
		\label{fig:cifar10_non_iid_dirichlet_01_n32}
	}
	\hfill
	\subfigure[\small
		CIFAR-10, $n\!=\!48$, $\alpha=1$.
	]{
		\includegraphics[width=.475\textwidth,]{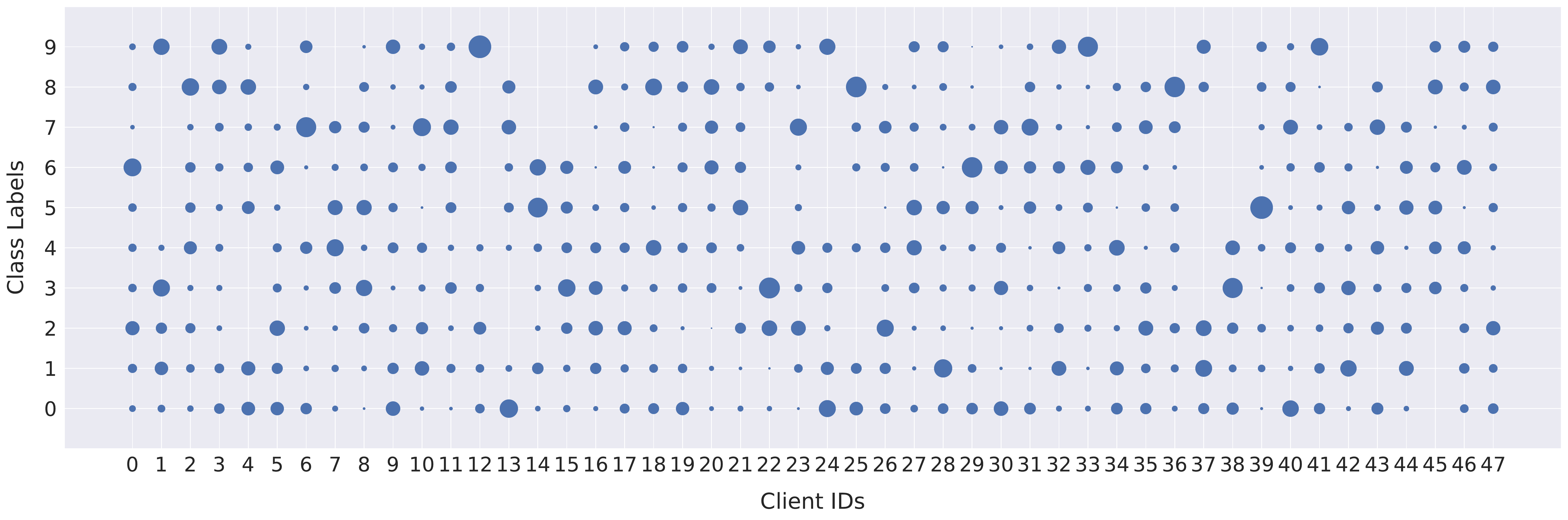}
		\label{fig:cifar10_non_iid_dirichlet_1_n48}
	}
	\hfill
	\subfigure[\small
		CIFAR-10, $n\!=\!48$, $\alpha=0.1$.
	]{
		\includegraphics[width=.475\textwidth,]{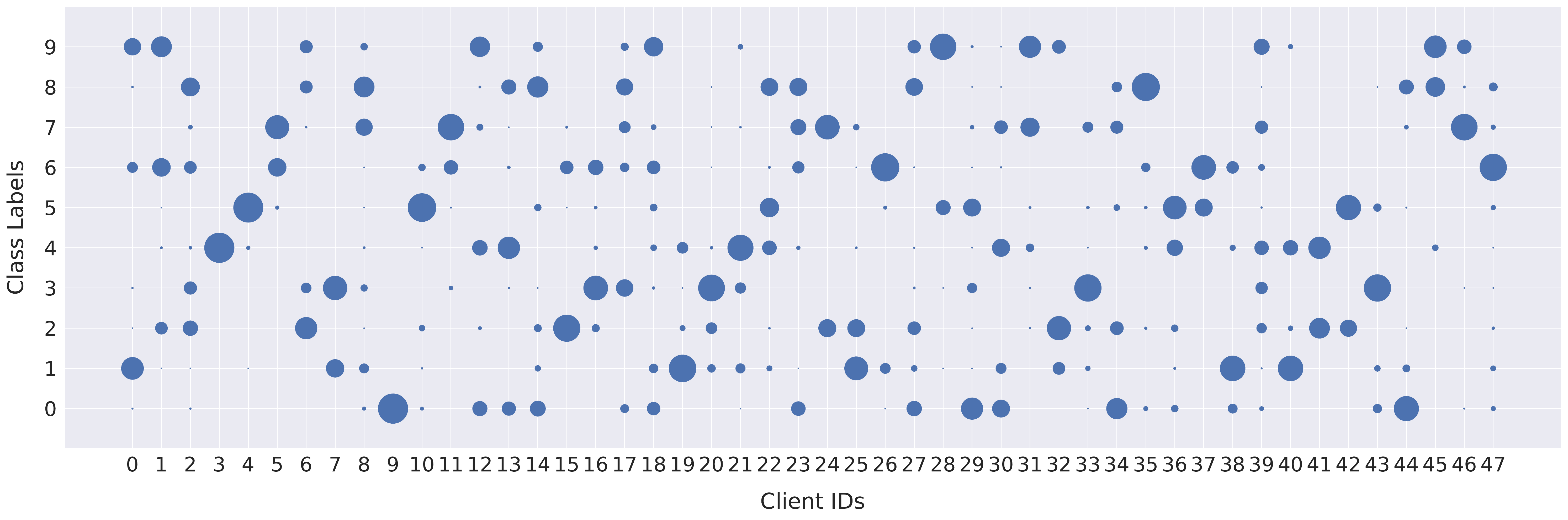}
		\label{fig:cifar10_non_iid_dirichlet_01_n48}
	}
	\hfill
	\subfigure[\small
		ImageNet, $n\!=\!16$, $\alpha=1$.
	]{
		\includegraphics[width=.475\textwidth,]{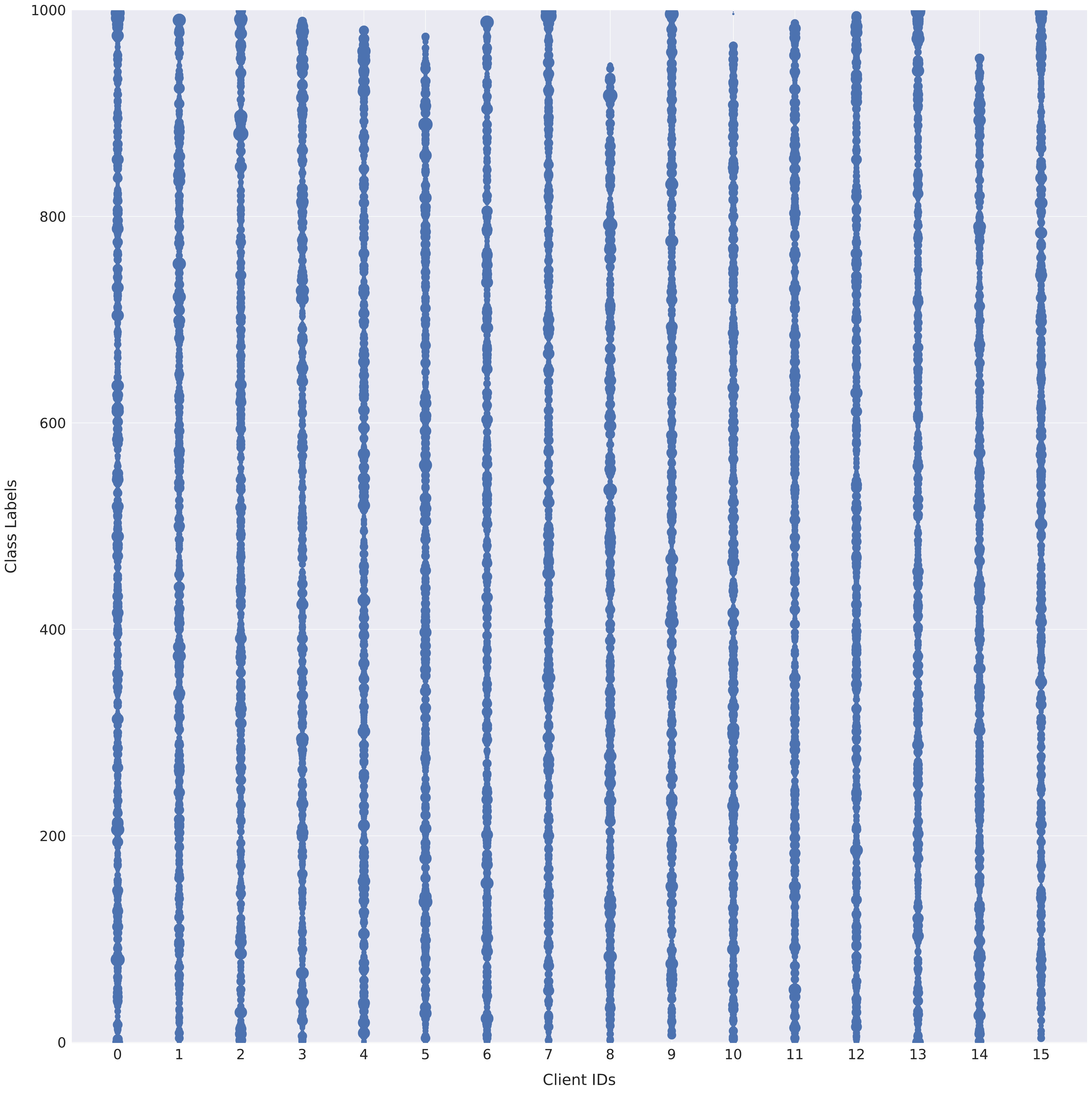}
		\label{fig:imagenet_non_iid_dirichlet_1_n16}
	}
	\hfill
	\subfigure[\small
		ImageNet, $n\!=\!16$, $\alpha=0.1$.
	]{
		\includegraphics[width=.475\textwidth,]{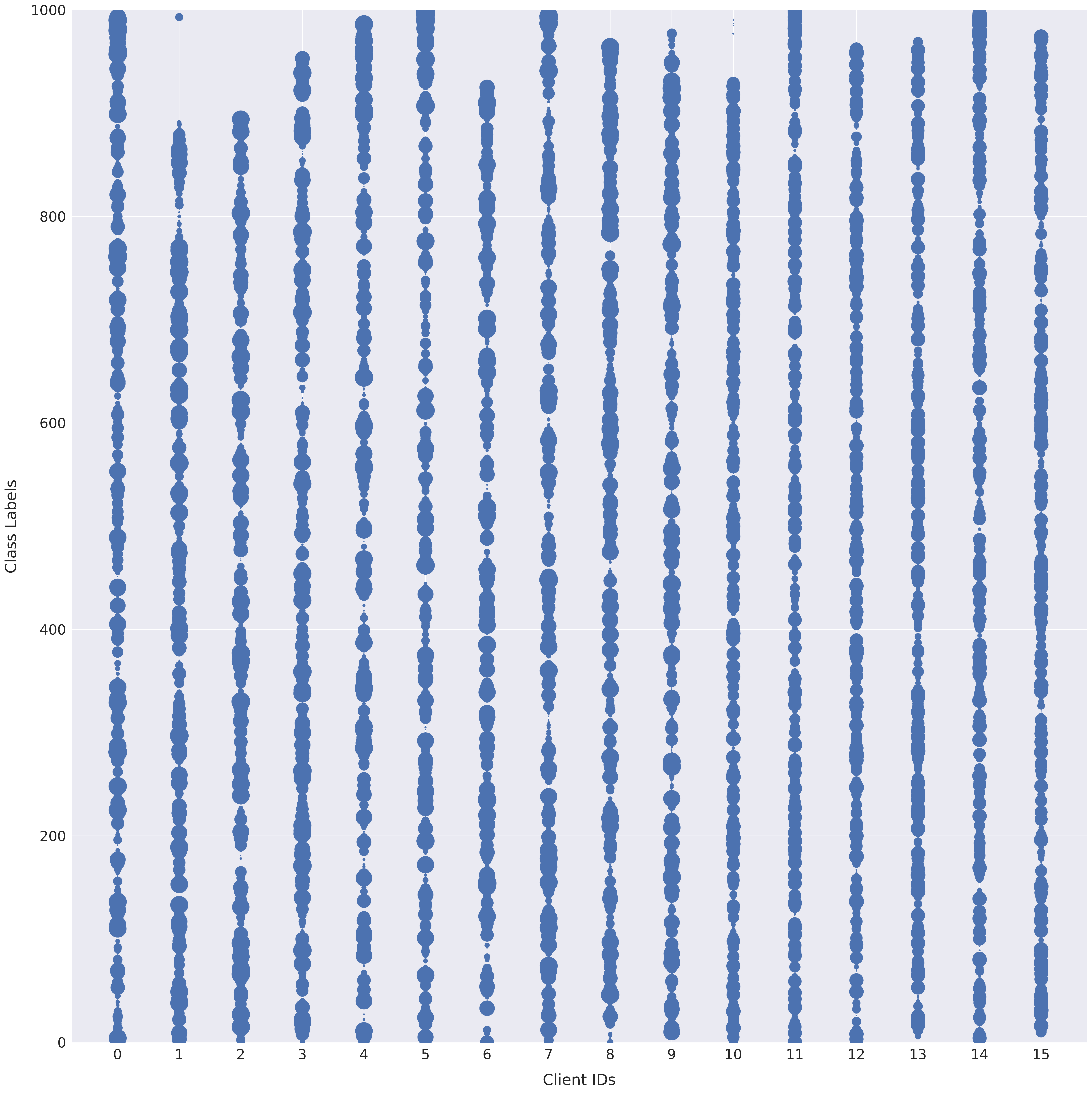}
		\label{fig:imagenet_non_iid_dirichlet_01_n16}
	}
	\vspace{-1em}
	\caption{\small
		Illustration of \# of samples per class allocated to each client (indicated by dot sizes),
		for different Dirichlet distribution $\alpha$ values on CV datasets.
	}
	\label{fig:illustration_noniid_for_cv_datasets}
\end{figure*}

\begin{figure*}[!h]
	\vspace{-1em}
	\centering
	\subfigure[\small
		AG News, $n\!=\!16$, $\alpha=10$.
	]{
		\includegraphics[width=.315\textwidth,]{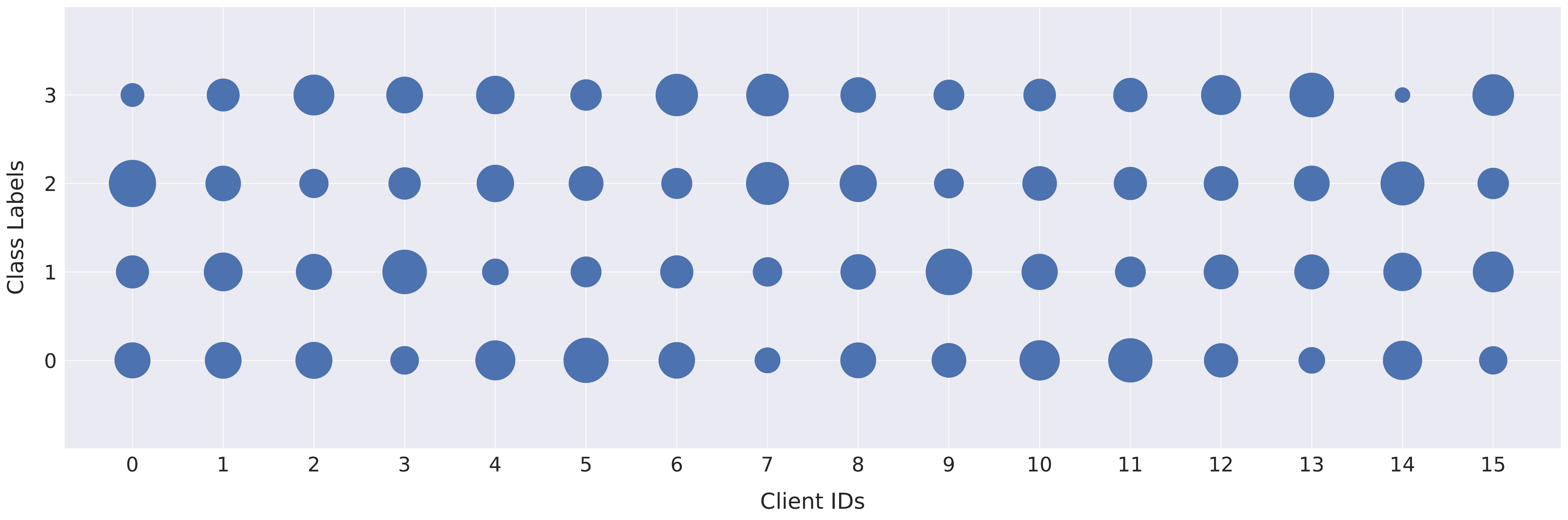}
		\label{fig:agnews_non_iid_dirichlet_10_n16}
	}
	\hfill
	\subfigure[\small
		AG News, $n\!=\!16$, $\alpha=1$.
	]{
		\includegraphics[width=.315\textwidth,]{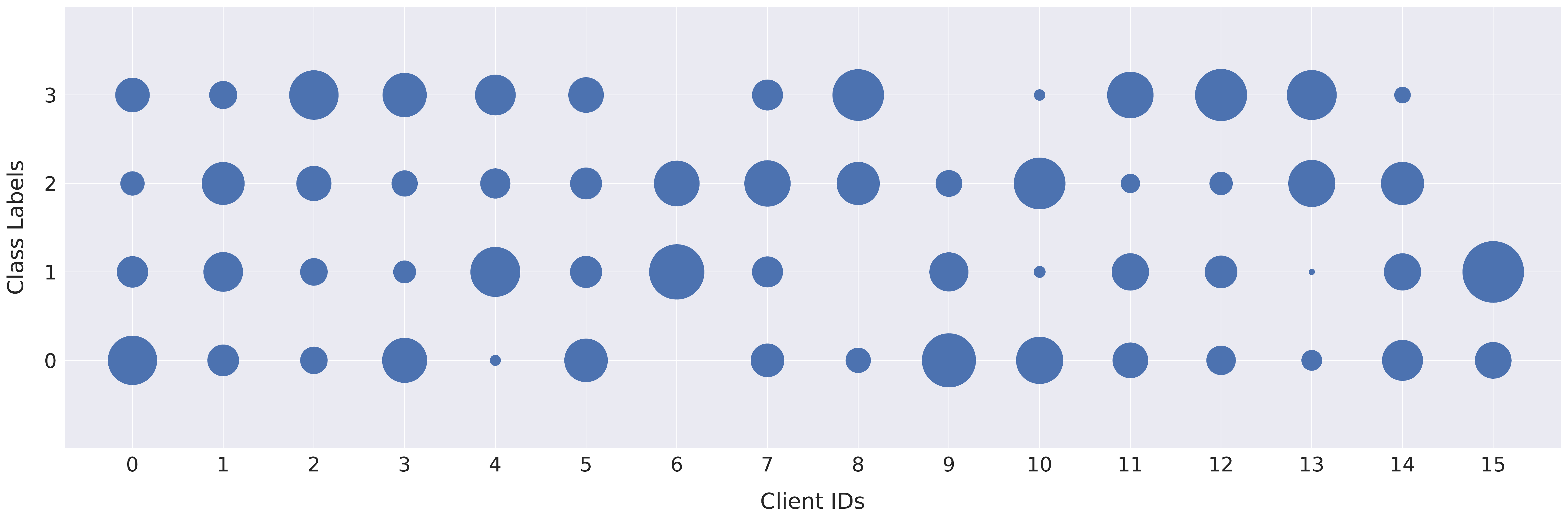}
		\label{fig:agnews_non_iid_dirichlet_1_n16}
	}
	\hfill
	\subfigure[\small
		AG News, $n\!=\!16$, $\alpha=0.1$.
	]{
		\includegraphics[width=.315\textwidth,]{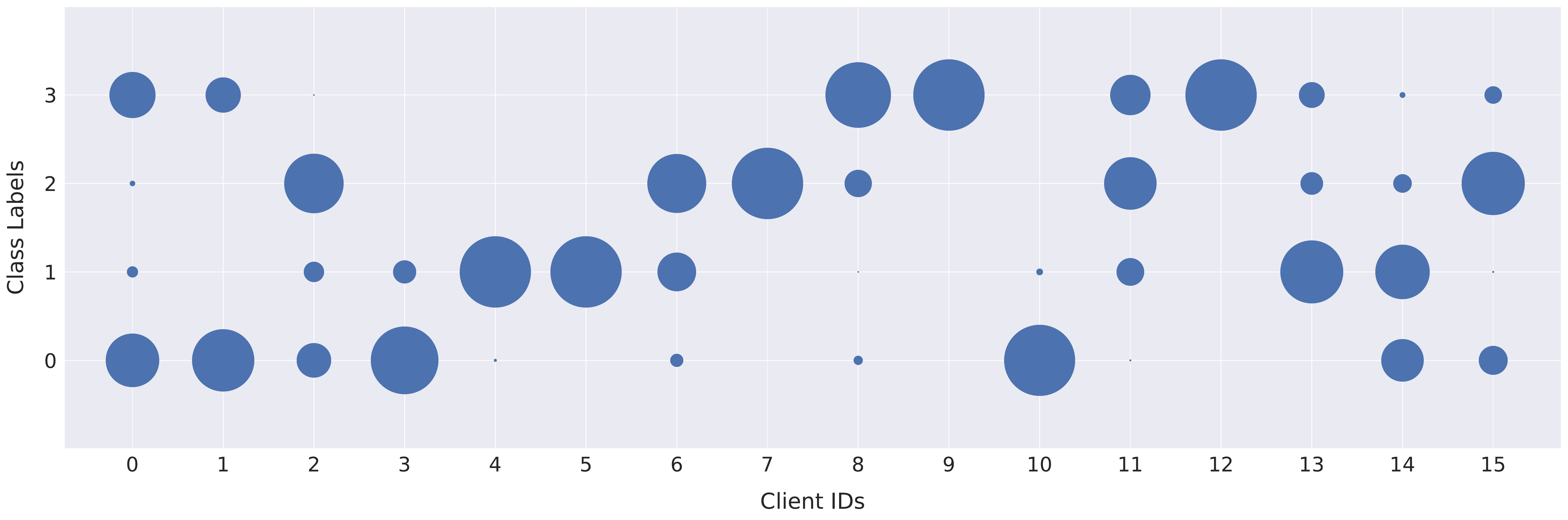}
		\label{fig:agnews_non_iid_dirichlet_01_n16}
	}
	% \subfigure[\small
	% 	SST2, $n\!=\!16$, $\alpha=10$.
	% ]{
	% 	\includegraphics[width=.315\textwidth,]{figures/noniid_data/sst2_non_iid_dirichlet_10_n16.pdf}
	% 	\label{fig:sst2_non_iid_dirichlet_10_n16}
	% }
	% \hfill
	% \subfigure[\small
	% 	SST2, $n\!=\!16$, $\alpha=1$.
	% ]{
	% 	\includegraphics[width=.315\textwidth,]{figures/noniid_data/sst2_non_iid_dirichlet_1_n16.pdf}
	% 	\label{fig:sst2_non_iid_dirichlet_1_n16}
	% }
	% \hfill
	% \subfigure[\small
	% 	SST2, $n\!=\!16$, $\alpha=0.1$.
	% ]{
	% 	\includegraphics[width=.315\textwidth,]{figures/noniid_data/sst2_non_iid_dirichlet_01_n16.pdf}
	% 	\label{fig:sst2_non_iid_dirichlet_01_n16}
	% }
	\vspace{-1em}
	\caption{\small
		Illustration of \# of samples per class allocated to each client (indicated by dot sizes),
		for different Dirichlet distribution $\alpha$ values on NLP datasets.
	}
	\label{fig:illustration_noniid_for_nlp_datasets}
\end{figure*}

\clearpage
\section{Detailed Algorithm Description and Connections}
\subsection{Detailed Algorithm Description} \label{appendix:detailed_algorithm}
The variant of Adam with the idea of \qg is detailed in Algorithm~\ref{alg:algopadam}.

\begin{algorithm}[!h]
	\begin{algorithmic}[1]
		\Procedure{}{}
		\For{$t \in \{ 1, \ldots, T \}$}
		\myState{sample $\xi_{i}^{(t)}$ and compute $\gg_{i}^{(t)} = \nabla F_i(\xx_{i}^{(t)}, \xi_{i}^{(t)})$}
		\myState{$\mm_{i}^{(t)} = \beta_1 \hat \mm_{i}^{(t-1)} + (1 - \beta_1) \hat \gg_{i}^{(t)}$}
		\myState{$\vv_{i}^{(t)} = \beta_2 \hat \vv_{i}^{(t-1)} + (1 - \beta_2) \hat \gg_{i}^{(t)} \odot \hat \gg_{i}^{(t)}$}
		\myState{$\xx_{i}^{(t+\frac{1}{2})} = \xx_{i}^{(t)} - \eta \frac{ \mm_{i}^{(t)} }{ \sqrt{ \vv_{i}^{(t)} } + \epsilon }$}
		\myState{$\xx_{i}^{(t+1)} = \sum_{j \in \cN_{i}^{(t)}} w_{ij} \xx_{j}^{(t + \frac{1}{2})}$}
		\myState{$ \dd_i^{(t)} = \xx_{i}^{(t)} - \xx_{i}^{(t + 1)} $}
		\myState{$ \hat \dd_i^{(t)} = \frac{ \dd_i^{(t)} }{ \norm{ \dd_i^{(t)} }_2 } $}
		\myState{$ \hat \mm_{i}^{(t)} = \beta_1 \hat \mm_{i}^{(t - 1)} + (1 - \beta_1) \hat \dd_i^{(t)} $}
		\myState{$ \hat \vv_{i}^{(t)} = \beta_2 \hat \vv_{i}^{(t - 1)} + (1 - \beta_2) \hat \dd_i^{(t)} \odot \hat \dd_i^{(t)} $}
		\EndFor
		\myState{\Return $\xx_i^{(T)}$}
		\EndProcedure
	\end{algorithmic}

	\mycaptionof{algorithm}{\small
		\algopadam.
		$\hat \mm_{i}^{(0)}, \hat \vv_{i}^{(0)}$ are initialized as $\0$ for all workers.
	}
	\label{alg:algopadam}
\end{algorithm}

\begin{algorithm}[!h]
	\begin{algorithmic}[1]
		\Procedure{worker-$i$}{}
		\For{$t \in \{ 1, \ldots, T \}$}
		\myState{sample $\xi_{i}^{(t)}$ and compute $\gg_{i}^{(t)} = \nabla F_i(\xx_{i}^{(t)}, \xi_{i}^{(t)})$}
		\myState{ $ \mm_i^{(t)} = \beta \hat \mm_{i}^{(t-1)} + \gg_{i}^{(t)} $ }
		\myState{ $ \xx_{i}^{(t+\frac{1}{2})} = \xx_{i}^{(t)} - \eta \mm_i^{(t)} $}
		\myState{ $ \xx_{i}^{(t+1)} = \sum_{j \in \cN_{i}^{(t)}} w_{ij} \xx_{j}^{(t + \frac{1}{2})} $ }
		\If{$\text{mod}(t, \tau) \neq 0$}
		\myState{ $ \hat \mm_{i}^{(t)} = \hat \mm_{i}^{(t-1)} $ }
		\Else
		\myState{ $ \dd_i^{(t)} = \frac{ \xx_{i}^{(t)} - \xx_{i}^{(t + 1)}  }{\eta} $ }
		\myState{ $ \hat \mm_{i}^{(t)} = \mu \hat \mm_{i}^{(t - 1)} + (1 - \mu) \dd_i^{(t)} $ }
		\EndIf
		\EndFor
		\myState{\Return $\xx_i^{(T)}$}
		\EndProcedure
	\end{algorithmic}

	\mycaptionof{algorithm}{\small
		Multiple-step variant of \algoptsgdm.
		$\mm_{i}^{(0)} = \hat \mm_{i}^{(0)} := \0$.
		$\tau$ is the number of local steps.
	}
	\label{alg:multiple_step_algoptsgdm_varaint}
\end{algorithm}

Algorithm~\ref{alg:mimelite} depicts the general procedure of MimeLite in~\citep{karimireddy2020mime}.
For SGDm, the update step $\cU$ and the tracking step $\cV$ follow
\begin{align*}
	\begin{split}
		\cU \left( \nabla F_i(\yy_i, \xi), \ss \right)
		&:= (1 - \beta) \nabla F_i(\yy_i, \xi) + \beta \ss \\
		\cV \left( \frac{1}{\abs{\cS}} \sum_{i \in \cS} \nabla f_i (\xx), \ss \right)
		&:= (1 - \beta) \frac{1}{\abs{\cS}} \sum_{i \in \cS} \nabla f_i (\xx) + \beta \ss \,.
	\end{split}
\end{align*}

\begin{algorithm}[!h]
	\begin{algorithmic}[1]
		\Procedure{}{}
		\For{each round $t \in [T]$}
		\myState{sample subset $\cS$ of clients}
		\myState{communicate $(\xx, \ss)$ to all clients $i \in \cS$}

		\For{client $i \in \cS$ in parallel}
		\myState{initialize local model $\yy_i \leftarrow \xx$}
		\For{client $i \in \cS$ in parallel}
		\For{$k \in [\tau]$}
		\myState{sample mini-batch $\xi$ from local data}
		\myState{$\yy_i \leftarrow \yy_i - \eta \cU \left( \nabla F_i(\yy_i, \xi), \ss \right)$}
		\EndFor
		\EndFor
		\myState{compute full local-batch gradient $\nabla f_i (\xx)$}
		\myState{communicate $(\yy_i, \nabla f_i(\xx))$}
		\EndFor
		\myState{$\ss \leftarrow \cV \left( \frac{1}{\abs{\cS}} \sum_{i \in \cS} \nabla f_i (\xx), \ss \right)$} \Comment{update optimization statistics}
		\myState{$\xx \leftarrow \frac{1}{\abs{\cS}} \sum_{i \in \cS} \yy_i$} \Comment{update server parameters}
		\EndFor
		\myState{\Return $\xx_T$}
		\EndProcedure
	\end{algorithmic}

	\mycaptionof{algorithm}{\small MimeLite~\citep{karimireddy2020mime}.}
	\label{alg:mimelite}
\end{algorithm}

Algorithm~\ref{alg:slowmo} shows the pseudocode of SlowMo~\citep{wang2020slowmo}.
For our evaluation in Table~\ref{tab:ablation_study_compare_with_other_momentum_SGDs},
we follow the hyper-parameter suggestion mentioned in~\citep{wang2020slowmo}:
for CIFAR-10 dataset, we set $\alpha=1, \tau=12, \beta=0.7$.

\begin{algorithm}[!h]
	\begin{algorithmic}[1]
		\Procedure{}{}
		\For{$t \in [T]$ at worker-$i$ in parallel}
		\myState{Maintain/Average base optimizer buffers}
		\For{$k \in [\tau]$}
		\myState{Base optimizer step: $ \xx_{i, k+1}^{(t)} = \xx_{i, k}^{(t)} - \gamma^{(t)} \dd_{i, k}^{(t)} $}
		\EndFor
		\myState{Exact-Average: $\xx_{\tau}^{(t)} = \frac{1}{n} \sum_{i=1}^n \xx_{i, \tau}^{(t)} $}
		\myState{Update slow momentum: $ \mm^{(t+1)} = \beta \mm^{(t)} + \frac{1}{\gamma^{(t)}} ( \xx_{i, 0}^{(t)} - \xx_{\tau}^{(t)} ) $}
		\myState{Update outer iterates: $ \xx_{i, 0}^{(t+1)} = \xx_{i, 0}^{(t)} - \alpha \gamma^{(t)} \mm^{(t+1)}$}
		\EndFor
		\myState{\Return $\xx_T$}
		\EndProcedure
	\end{algorithmic}

	\mycaptionof{algorithm}{\small
		SlowMo~\citep{wang2020slowmo}.
		$\dd_{i, k}^{(t)}$ indicates the local update direction for communication round $t$
		at local update steps $k$.
	}
	\label{alg:slowmo}
\end{algorithm}

\clearpage
\subsection{Difference between DMSGD and \algoptsgdm} \label{appendix:connection_between_our_and_DMSGD}
We first re-iterate DMSGD~\citep{balu2020decentralized} in Algorithm~\ref{alg:original_dmsgd} with slightly adjusted notations.
\begin{algorithm}[!h]
	\begin{algorithmic}[1]
		\Procedure{worker-$i$}{}
		\For{$t \in \{ 1, \ldots, T \}$}
		\myState{ $ \vv_{i}^{(t)} = \sum_{j \in \cN_{i}^{(t)}} w_{ij} \xx_{j}^{(t)}$} \Comment{Consensus step}
		\myState{ $ \hat \mm_i^{(t)} = \mu ( \xx_i^{(t)} - \xx_i^{(t - 1)} ) + (1 - \mu) ( \vv_{i}^{(t)} - \vv_{i}^{(t - 1)} )  $ } \Comment{Momentum step}
		\myState{sample $\xi_{i}^{(t)}$ and compute $\gg_{i}^{(t)} = \nabla F_i(\xx_{i}^{(t)}, \xi_{i}^{(t)})$}
		\myState{ $ \xx_i^{(t+1)} = \vv_{i}^{(t)} - \eta \gg_{i}^{(t)} + \beta \hat \mm_i^{(t)} $ } \Comment{Option I of local gradient step}
		\myState{ $ \xx_i^{(t+1)} = \xx_{i}^{(t)} - \eta \gg_{i}^{(t)} + \beta \hat \mm_i^{(t)} $ } \Comment{Option II of local gradient step}
		\EndFor
		\EndProcedure
	\end{algorithmic}

	\mycaptionof{algorithm}{\small
		Original formulation of DMSGD.
		$\hat \mm_{i}^{(0)}$ are initialized as $\0$ for all workers.
	}
	\label{alg:original_dmsgd}
\end{algorithm}

By re-organizing, we can further simplify Algorithm~\ref{alg:original_dmsgd} to Algorithm~\ref{alg:original_dmsgd2}.
\begin{algorithm}[!h]
	\begin{algorithmic}[1]
		\Procedure{worker-$i$}{}
		\For{$t \in \{ 1, \ldots, T \}$}
		\myState{sample $\xi_{i}^{(t)}$ and compute $\gg_{i}^{(t)} = \nabla F_i(\xx_{i}^{(t-\frac{1}{2})}, \xi_{i}^{(t)})$}
		\myState{ $ \xx_i^{(t+\frac{1}{2})} = \xx_{i}^{(t)} - \eta ( \beta \hat \mm_i^{(t-1)} + \gg_{i}^{(t)} ) $ } \Comment{Option I of local gradient step}
		\myState{ $ \xx_i^{(t+\frac{1}{2})} = \xx_{i}^{(t-\frac{1}{2})} - \eta ( \beta \hat \mm_i^{(t-1)} + \gg_{i}^{(t)} ) $ } \Comment{Option II of local gradient step}
		\myState{ $ \xx_{i}^{(t+1)} = \sum_{j \in \cN_{i}^{(t)}} w_{ij} \xx_{j}^{(t+\frac{1}{2})}$} \Comment{Consensus step}
		\myState{ $ \hat \mm_i^{(t)} = \mu ( \xx_i^{(t - \frac{1}{2})} - \xx_i^{(t+\frac{1}{2})} ) + (1 - \mu) ( \xx_{i}^{(t)} - \xx_{i}^{(t+1)} ) $ } \Comment{Momentum step}
		\EndFor
		\EndProcedure
	\end{algorithmic}
	\mycaptionof{algorithm}{\small
		Re-organized formulation of DMSGD.
		$\hat \mm_{i}^{(0)}$ are initialized as $\0$ for all workers.
	}
	\label{alg:original_dmsgd2}
\end{algorithm}

For a fair comparison, we unify Algorithm~\ref{alg:original_dmsgd2} with \algoptsgdm,
as in Algorithm~\ref{alg:appendix_adapted_dmsgd}
(we slightly abuse the notations for comparison purpose).
\begin{algorithm}[!h]
	\begin{algorithmic}[1]
		\Procedure{worker-$i$}{}
		\For{$t \in \{ 1, \ldots, T \}$}
		\myState{sample $\xi_{i}^{(t)}$ and compute $\gg_{i}^{(t)} = \nabla F_i(\xx_{i}^{(t)}, \xi_{i}^{(t)})$}
		\myState{$\xx_{i}^{(t+\frac{1}{2})} = \xx_{i}^{(t)} - \eta (\beta \hat \mm_{i}^{(t-1)} + \gg_{i}^{(t)})$}
		\myState{$\xx_{i}^{(t+1)} = \sum_{j \in \cN_{i}^{(t)}} w_{ij} \xx_{j}^{(t + \frac{1}{2})}$}
		\myState{$ \hat \mm_{i}^{(t)} $ is determined by the algorithm.}
		\EndFor
		\myState{\Return $\xx_i^{(T)}$}
		\EndProcedure
	\end{algorithmic}

	\mycaptionof{algorithm}{\small
		DMSGD v.s.\ \algoptsgdm.
		$\hat \mm_{i}^{(0)}$ are initialized as $\0$ for all workers.
	}
	\label{alg:appendix_adapted_dmsgd}
\end{algorithm}

Note that in Algorithm~\ref{alg:appendix_adapted_dmsgd}
(slightly different from the $\hat \mm^{(t)}$ in Algorithm~\ref{alg:original_dmsgd2}),
$\hat \mm_i^{(t)}$ in DMSGD is defined as
\begin{align*}
	\hat \mm_i^{(t)}
	= \frac{
	\mu ( \xx_i^{(t - \frac{1}{2})} - \xx_i^{(t + \frac{1}{2})} ) + (1 - \mu) ( \xx_{i}^{(t)} - \xx_{i}^{(t + 1)} )
	}{\eta} \,,
\end{align*}
while for \algoptsgdm, we have
\begin{small}
	\begin{align*}
		\begin{split}
			\hat \mm_i^{(t)}
			= \mu \hat \mm_{i}^{(t - 1)} + (1 - \mu) \frac{ \xx_{i}^{(t)} - \xx_{i}^{(t + 1)} }{\eta} \,.
		\end{split}
	\end{align*}
\end{small}

Thus, for option I of DMSGD, we have
\begin{small}
	\begin{align*}
		\begin{split}
			&\hat \mm_{i}^{(t)}
			= \frac{ \mu ( \xx_i^{(t - \frac{1}{2})} - \xx_i^{(t + \frac{1}{2})} ) + (1 - \omega) ( \xx_{i}^{(t)} - \xx_{i}^{(t + 1)} ) }{\eta} \\
			&= \frac{\mu}{\eta} \left(
			\left( \xx_{i}^{(t-1)} - \eta (\beta \hat \mm_{i}^{(t-2)} + \gg_{i}^{(t-1)}) \right) -
			\left( \xx_{i}^{(t)} - \eta (\beta \hat \mm_{i}^{(t-1)} + \gg_{i}^{(t)}) \right)
			\right)
			+ (1 - \mu) \frac{ \xx_{i}^{(t)} - \xx_{i}^{(t + 1)} }{\eta} \\
			&= \mu \left(
			\frac{\xx_{i}^{(t-1)} - \xx_{i}^{(t)}}{\eta}
			-
			\left(
				\left( \beta \hat \mm_{i}^{(t-2)} + \gg_{i}^{(t-1)} \right)-
				\left( \beta \hat \mm_{i}^{(t-1)} + \gg_{i}^{(t)} \right)
				\right)
			\right)
			+ (1 - \mu) \frac{ \xx_{i}^{(t)} - \xx_{i}^{(t + 1)} }{\eta} \\
			&= \mu \left(  \frac{ \xx_{i}^{(t-1)} - \xx_{i}^{(t)}  }{\eta}
			- \beta ( \hat \mm_i^{(t-2)} - \hat \mm_i^{(t-1)} )
			- ( \gg_i^{(t-1)} - \gg_i^{(t)} )
			\right)
			+ (1 - \mu) \frac{ \xx_{i}^{(t)} - \xx_{i}^{(t + 1)} }{\eta} \\
			&= \mu \left(
			\beta \hat \mm_i^{(t-1)} + \gg_i^{(t)}
			+ \frac{ \xx_{i}^{(t-1)} - \xx_{i}^{(t)} }{\eta} - \beta \hat \mm_i^{(t-2)} - \gg_i^{(t-1)}
			\right)
			+ (1 - \mu) \frac{ \xx_{i}^{(t)} - \xx_{i}^{(t + 1)} }{\eta}
			\,,
		\end{split}
	\end{align*}
\end{small}
for option II of DMSGD, we have
\begin{small}
	\begin{align*}
		\begin{split}
			&\hat \mm_{i}^{(t)}
			=\frac{ \mu ( \xx_i^{(t - \frac{1}{2})} - \xx_i^{(t + \frac{1}{2})} ) + (1 - \omega) ( \xx_{i}^{(t)} - \xx_{i}^{(t + 1)} ) }{\eta} \\
			&=\mu \left( \beta \hat \mm_i^{(t-1)} + \gg_i^{(t)} \right)
			+ (1 - \mu) \frac{ \xx_{i}^{(t)} - \xx_{i}^{(t + 1)}  }{\eta} \,.
		\end{split}
	\end{align*}
\end{small}

It is obvious that the design of DMSGD is different from our \algoptsgdm:
\begin{itemize}[itemsep=0pt,leftmargin=12pt]
	\item The update scheme on the momentum buffer $\hat \mm_i$ is different, as illustrate above.
	\item DMSGD is based on the heavy-ball momentum,
	      while our scheme can generalize to heavy ball momentum SGD, Nesterov momentum SGD, and even Adam variants.
\end{itemize}

\subsection{Connections Between SGDm and \salgoptsgdm} \label{appendix:connection_for_single_worker_case}
\subsubsection{Connections Between SGDm and \algoptsgdm} \label{appendix:hd_sgd_with_single_worker_sgdm}
Note our scheme \algoptsgdm on the single worker case (i.e.\ \salgoptsgdm) has the following equation:
\begin{align*}
	\begin{split}
		\xx^{(t+1)} &= \xx^{(t)} - \eta \left( \beta \mm^{(t-1)} + \gg^{(t)} \right) \\
		\mm^t
		&= \mu \mm^{(t-1)} + (1 - \mu) \frac{ \xx^{(t)} - \xx^{(t+1)} }{\eta}
		= \left( \mu + (1 - \mu) \beta \right) \mm^{(t-1)} + (1 - \mu) \gg^{(t)} \,.
	\end{split}
\end{align*}

By letting $\hat \mm^{(t)} := \frac{\mm^{(t)}}{1 - \mu}$, we have
\begin{align*}
	\begin{split}
		\xx^{(t+1)}
		&= \xx^{(t)} - \eta \left( \beta (1 - \mu) \hat \mm^{(t-1)} + \gg^{(t)} \right) \\
		\hat \mm^{(t)}
		&= \left( \mu + (1 - \mu) \beta \right) \hat \mm^{(t-1)} + \gg^{(t)} \,.
	\end{split}
\end{align*}

We further let $\hat \beta := \mu + (1 - \mu) \beta$, then we have
\begin{align*}
	\begin{split}
		\xx^{(t+1)}
		&= \xx^{(t)} - \eta \left( \hat \beta \hat \mm^{(t-1)} + \gg^{(t)} + \left( \beta (1 - \mu) - \hat \beta \right) \hat \mm^{(t-1)} \right)
		= \xx^{(t)} - \eta \left( \hat \beta \hat \mm^{(t-1)} + \gg^{(t)} -\mu \hat \mm^{(t-1)} \right) \\
		\hat \mm^{(t)}
		&= \hat \beta \hat \mm^{(t-1)} + \gg^{(t)} \,.
	\end{split}
\end{align*}

By re-organizing, we have
\begin{align} \label{eq:appendix_qhm}
	\begin{split}
		\hat \mm^{(t)} &= \hat \beta \hat \mm^{(t-1)} + \gg^{(t)} \\
		\xx^{(t+1)}
		&= \xx^{(t)} - \eta \left( \hat \mm^{(t)} -\mu \hat \mm^{(t-1)} \right)
		= \xx^{(t)} - \eta \left( \hat \mm^{(t)} -\mu \hat \mm^{(t-1)} \right)
		= \xx^{(t)} - \eta \left( ( 1 - \frac{\mu}{\hat \beta} ) \hat \mm^{(t)} + \frac{\mu}{\hat \beta} \gg^{(t)} \right)
		\,,
	\end{split}
\end{align}
which recovers the QHM~\citep{gitman2019understanding}.

Comparing to the case of SGD with Heavy-ball Momentum (SGDm), where
\begin{align*}
	\begin{split}
		\hat \mm^{(t)} &= \beta \hat \mm^{(t-1)} + \gg^{(t)} \\
		\xx^{(t+1)}
		&= \xx^{(t)} - \eta \hat \mm^{(t)}
		\,,
	\end{split}
\end{align*}
we can witness that SGDm is only a special case of \eqref{eq:appendix_qhm} (when $\mu \!=\! 0$).

\subsubsection{Connections Between SGDm-N and \salgoptsgdm} \label{appendix:nag_with_single_worker_sgdm}
First note that our simplified version of \algoptsgdm (i.e.\ \salgoptsgdm) can recover the QHM~\citep{gitman2019understanding},
as illustrated in Appendix~\ref{appendix:hd_sgd_with_single_worker_sgdm}.
Furthermore, as pointed out in~\citep{gitman2019understanding} that,
the QHM is indeed equivalent to the original SGDm-N
with re-scaling of $\eta \rightarrow \eta / (1 - \beta)$.
Therefore, we can argue that our simplified version of \algoptsgdm (i.e.\ \salgoptsgdm or QHM)
is equivalent to the original SGDm-N with re-scaling of $\eta \rightarrow \eta / (1 - \beta)$.

For the reason of completeness, we include the derivatives below.

First of all, SGD with Nesterov Momentum (SGDm-N) can be rewritten as
\begin{align*}
	\begin{split}
		\xx^{(t + \frac{1}{2})} 	&= \xx^{(t)} + \beta \mm^{(t - 1)} \\
		\mm^{(t)} 					&= \beta \mm^{(t-1)} - \eta \nabla f(\xx^{(t + \frac{1}{2})}) \\
		\xx^{(t+1)} 				&= \xx^{(t + \frac{1}{2})} - \eta \nabla f(\xx^{(t + \frac{1}{2})}) \,,
	\end{split}
\end{align*}
and in PyTorch, we instead have
\begin{align} \label{eq:appendix_pytorch_nag}
	\begin{split}
		\xx^{(t + \frac{1}{2})} 	&= \xx^{(t)} + \beta \mm^{(t - 1)} \\
		\mm^{(t)} 					&= \beta \mm^{(t-1)} + \nabla f(\xx^{(t + \frac{1}{2})}) \\
		\xx^{(t + 1)} 				&= \xx^{(t)} - \eta \mm^{(t)} \,,
	\end{split}
\end{align}
where the above equations are equivalent for the constant $\beta$.

Then we reiterate the derivatives from~\citet{gitman2019understanding} below,
from SGDm-N to QHM (which is equivalently Equation~\eqref{eq:appendix_qhm}).
The SGDm-N in~\citet{gitman2019understanding} follows
\begin{align*}
	\begin{split}
		\mm^{(t)} &= \beta \mm^{(t-1)} - \eta \nabla f \left( \xx^{(t)} + \beta \mm^{(t-1)} \right) \\
		\xx^{(t+1)} &= \xx^{(t)} + \mm^{(t)} \,,
	\end{split}
\end{align*}
where we can move the learning rate out of the momentum into the iterates update:
\begin{align} \label{eq:appendix_qhm_nag_2}
	\begin{split}
		\mm^{(t)} 	&= \beta \mm^{(t-1)} + \nabla f \left( \xx^{(t)} - \eta \beta \mm^{(t-1)} \right) \\
		\xx^{(t+1)} &= \xx^{(t)} - \eta \mm^{(t)} \,,
	\end{split}
\end{align}
where the above two methods produce the same sequence of iterates $\xx^{(t)}$
if $\mm^{(0)}$ is initialized at $\0$.
The second equation~\eqref{eq:appendix_qhm_nag_2}
is equivalent to the Pytorch implementation in Equation~\eqref{eq:appendix_pytorch_nag}.

Let's normalize the momentum update by $1 - \beta$:
\begin{align*}
	\begin{split}
		\mm^{(t)} 	&= \beta \mm^{(t-1)} + (1 - \beta) \nabla f \left( \xx^{(t)} - \eta \beta \mm^{(t-1)} \right) \\
		\xx^{(t+1)} &= \xx^{(t)} - \eta \mm^{(t)} \,,
	\end{split}
\end{align*}
which is equivalent to the un-normalized one by re-scaling $\eta \rightarrow \frac{\eta}{1 - \beta}$
for constant parameters.
We make a change of variables $\yy^{(t)} = \xx^{(t)} - \eta \beta \mm^{(t-1)}$,
\begin{align*}
	\begin{split}
		\mm^{(t)} 	&= \beta \mm^{(t-1)} + (1 - \beta) \nabla f ( \yy^{(t)} ) \\
		\yy^{(t+1)} &= \xx^{(t+1)} - \eta \beta \mm^{(t)}
		= \xx^{(t)} - \eta \mm^{(t)} - \eta \beta \mm^{(t)} \\
		&= \yy^{(t)} + \eta \beta \mm^{(t-1)} - \eta \mm^{(t)} - \eta \beta \mm^{(t)} \\
		&= \yy^{(t)} + \eta \left( \mm^{(t)} - (1 - \beta) \nabla f ( \yy^{(t)} ) \right)
		- \eta \mm^{(t)} - \eta \beta \mm^{(t)} \\
		&= \yy^{(t)} - \eta \left(
		(1 - \beta) \nabla f ( \yy^{(t)} ) + \beta \mm^{(t)}
		\right) \,,
	\end{split}
\end{align*}
where by renaming $\yy^{(t)}$ back to $\xx^{(t)}$, we obtain the exact formula used in QHM update.

\subsubsection{The Simplification of \algoptsgdmn on Single Worker Case}
We can further simplify our scheme \algoptsgdmn on the single worker case (and obtain \salgoptsgdmn):
\begin{align*}
	\begin{split}
		\xx^{(t + \frac{1}{2})} &= \xx^{(t)} + \beta \mm^{(t-1)} 								\\
		\hat \mm^{(t)} 			&= \beta \mm^{(t-1)} + \nabla f(\xx^{(t + \frac{1}{2})}) 		\\
		\xx^{(t + 1)} 			&= \xx^{(t)} - \eta \hat \mm^{(t)} 								\\
		\mm^{(t)} 				&= \mu \mm^{(t-1)} + (1 - \mu) \frac{ \xx^{(t)} - \xx^{(t+1)} }{\eta} 	\,.
	\end{split}
\end{align*}

We rewrite $\mm^{(t)}$ as
\begin{align*}
	\mm^{(t)}
	  & = \mu \mm^{(t-1)} + (1 - \mu) \frac{ \xx^{(t)} - \xx^{(t+1)} }{\eta}
	= \mu \mm^{(t-1)} + (1 - \mu) \hat \mm^{(t)} \\
	  & = \left( \mu + \beta - \mu \beta \right) \mm^{(t-1)} + (1 - \mu) \nabla f(\xx^{(t + \frac{1}{2})}) \,.
\end{align*}

Similar to the treatment in Appendix~\ref{appendix:hd_sgd_with_single_worker_sgdm},
by letting $\hat \mm^{(t)} := \frac{ \mm^{(t)} }{1 - \mu}$, we have
\begin{align*}
	\xx^{(t + 1)}
	  & = \xx^{(t)} - \eta \hat \mm^{(t)}
	= \xx^{(t)} - \eta \left( \beta (1 - \mu) \hat \mm^{(t-1)} + \nabla f(\xx^{(t + \frac{1}{2})}) \right) \\
	\hat \mm^{(t)}
	  & = \left( \mu + \beta - \mu \beta \right) \hat \mm^{(t-1)} + \nabla f(\xx^{(t + \frac{1}{2})}) \,,
\end{align*}
and thus, in the end we have the following equations for \salgoptsgdmn
\begin{align*}
	\begin{split}
		\xx^{(t + \frac{1}{2})} 	&= \xx^{(t)} + \beta (1 - \mu) \hat \mm^{(t-1)} \\
		\hat \mm^{(t)} 				&= \hat \beta \hat \mm^{(t-1)} + \nabla f(\xx^{(t + \frac{1}{2})}) \\
		\xx^{(t+1)}
		&= \xx^{(t)} -
		\eta \left(
		( 1 - \frac{\mu}{\hat \beta} ) \hat \mm^{(t)}
		+ \frac{\mu}{\hat \beta} \nabla f(\xx^{(t + \frac{1}{2})})
		\right)
		\,,
	\end{split}
\end{align*}
where $\hat \beta := \mu + (1 - \mu) \beta$ and we can recover SGDm-N by setting $\mu = 0$.

\clearpage
\section{Global Convergence Rate Proofs}\label{appendix:convergence_rate_proofs}
We reiterate the update scheme of~\algoptsgdm in a matrix form:
\begin{align} \label{eq:appendix_our_scheme_matrix_form}
	\begin{split}
		\mX^{(t+1)} 	&= \mW \left( \mX^{(t)} - \eta \left( \beta \mM^{(t)} + \mG^{(t)} \right) \right) \\
		\mM^{(t + 1)}   &= \mu \mM^{(t)} + (1 - \mu) \frac{ \mX^{(t)} - \mX^{(t + 1)} }{\eta} \\
		&= \left( \mu + (1 - \mu) \beta \mW \right) \mM^{(t)} + (1-\mu)\mW\mG^{(t)} + \frac{1 - \mu}{\eta}(\mI-\mW)\mX^{(t)} \,,
	\end{split}
\end{align}%
where our numerical experiments by default use $\mu = \beta$.
For each matrix $\mZ$, we define an averaged vector $\bar \zz = \mZ \frac{1}{n}\1$
and matrix $\bar\mZ = \mZ \frac{1}{n}\1 \1^\top$.
Note that we use bold lower-case to indicate vectors and bold upper-case to denote matrices.

First, we state some standard definitions and regularity conditions.
\begin{assumption}\label{asm:convergence}
	We assume that the following hold:
	\begin{enumerate}
		\item The function $f(\xx)$ we are minimizing is lower bounded from below by $f^\star$,
		      and each node's loss $f_i$ is smooth satisfying
		      $\norm{\nabla f_i(\yy) - \nabla f_i(\xx)} \leq L\norm{\yy - \xx}$.

		      The $L$-smoothness implies the following quadratic upper bound on $f_i$
		      \begin{align*}
			      f_i (\yy) \leq f_i (\xx) + \langle \nabla f_i(\xx), \yy - \xx \rangle + \frac{L}{2} \norm{\yy - \xx}^2 \,.
		      \end{align*}

		      If additionally the function $ f $ is convex and $\xx^\star$ is an optimum of $f$, then it also implies
		      \begin{align*}
			      \norm{\nabla f(\xx) - \nabla f(\xx^\star)} \leq 2L \left( f(\xx) - f^\star \right) \,.
		      \end{align*}

		\item The stochastic gradients within each node satisfies $\Eb{ g_i(\xx) } = \nabla f_i(\xx)$
		      and $\E\norm{g_i(\xx) - \nabla f_i(\xx)}^2 \leq \sigma^2$.
		      The variance across the workers is also bounded as
		      $\frac{1}{n}\sum_{i=1}^n \norm{\nabla f_i(\xx) - \nabla f(\xx)}^2 \leq \zeta^2$.
		\item The mixing matrix is doubly stochastic where for the all ones vector $\1$,
		      we have $\mW \1 = \1$ and $\mW^\top\1 = \1$.
		      Further, define $\bar \mZ = \mZ \frac{1}{n}\1 \1^\top$ for any matrix $\mZ \in  \R^{d \times n}$.
		      Then, the mixing matrix satisfies
		      $\E_\mW \norm{ \mZ \mW - \bar{\mZ} }_F^2 \leq (1 - \rho) \norm{\mZ - \bar{\mZ}}_F^2$.
	\end{enumerate}
\end{assumption}

\paragraph{Average parameters.}
Let us examine the effect of updates in \eqref{eq:appendix_our_scheme_matrix_form}
on $\bar\xx^{(t)}$ which is the parameters averaged across the nodes.
Note that since $\mW$ is doubly stochastic, we can simplify the updates as follows:
\begin{align} \label{eq:appendix_our_scheme_averaged_form}
	\begin{split}
		\bar\xx^{(t+1)} 	&= \bar\xx^{(t)} - \eta \left( \beta \bar\mm^{(t)} + \bar\gg^{(t)} \right) \text{, and } \\
		\bar \mm^{(t + 1)}   &= \mu \bar \mm^{(t)} + (1 - \mu) \frac{ \bar \xx^{(t)} - \bar\xx^{(t + 1)} }{\eta} \\
		&= (1 - (1 - \mu)(1- \beta)) \bar \mm^{(t)} + (1 - \mu)\bar\gg^{(t)}\,.
	\end{split}
\end{align}%
Here, $\bar \gg^{(t)} := \frac{1}{n}\sum_{i=1}^n \nabla F_i(\xx_i^{(t)}, \xi_i^{(t)})$
is the average of the stochastic gradients across the nodes.

\paragraph{Virtual Sequence.}
Now we define a virtual sequence of parameters $\{\hat\xx^{(t)}\}$ which has a simple SGD style update,
and an error sequence which will be easy to analyze:
\begin{equation}\label{eq:appendix_virtual_sequence}
	\begin{split}
		\hat\xx^{(t+1)} &= \hat\xx^{(t)} - \frac{\eta}{1-\beta}\bar\gg^{(t)}\text{, and}\\
		\bar\ee^{(t)} &:= \hat\xx^{(t)} - \bar\xx^{(t)}\,.
	\end{split}
\end{equation}
Our strategy for the analysis will be to analyze the virtual sequence $\{\hat\xx^{(t)}\}$
and prove that the real sequence of iterates remains close i.e. that the $\ee^{(t)}$ remains small.

\paragraph{Single-step progress of virtual update.}
We will show that every step we make some progress, but have to balance three sources of error:
i) the stochastic error which depends on $\sigma^2$ due to using stochastic gradients,
ii) consensus error which depends on $\mX^{t} - \bar\mX^{t}$, and finally
iii) momentum error due to using momentum which depends on $\ee^{(t)}$.

\begin{lemma}[Non-convex one step progress]\label{lem:non-convex-progress}
	Given Assumption~\ref{asm:convergence},
	the sequence of iterates generated by~\eqref{eq:appendix_our_scheme_matrix_form}
	using $\eta\leq \frac{1 - \beta}{4L}$ satisfy
	\[
		\E f(\hat\xx^{(t+1)}) \leq
		\E f(\hat\xx^{(t)})
		- \frac{\tilde\eta}{4}\norm{\nabla f(\bar\xx^{(t)})}^2
		- \frac{\tilde\eta}{4}\E\norm{\frac{1}{n}\sum_{i=1}^n \nabla f_i(\xx_i^{(t)})}^2
		+ \frac{L \tilde\eta^2 \sigma^2}{n} + \frac{3 L^2 \tilde\eta}{2}\norm{\ee^{(t)}}^2
		+ \frac{3L^2 \tilde\eta}{n}\norm{\mX^{(t)} - \bar\mX^{(t)}}^2_F\,,
	\]
	where we define $\tilde\eta := \frac{\eta}{1-\beta}$.
\end{lemma}
\begin{proof}
	Starting from the smoothness of $f$, we have
	\begin{align*}
		\E f(\hat\xx^{(t+1)}) & \leq \E f(\hat\xx^{(t)}) + \E\lin{\nabla f(\hat\xx^{(t)}) , \hat\xx^{(t+1)} - \hat\xx^{(t)}} + \frac{L}{2}\E\norm{\hat\xx^{(t+1)} - \hat\xx^{(t)}}^2                                                                                                                    \\
		                      & = \E f(\hat\xx^{(t)}) - \frac{\eta}{1-\beta} \frac{1}{n}\sum_{i=1}^n\E\lin{\nabla f(\hat\xx^{(t)}) , \nabla f_i(\xx_i^{(t)})} + \frac{L\eta^2}{(1-\beta)^2}\E\norm{\bar \gg^{(t)}}^2                                                                                    \\
		                      & \leq \E f(\hat\xx^{(t)}) - \tilde\eta \frac{1}{n}\sum_{i=1}^n\E\lin{\nabla f(\hat\xx^{(t)}) , \nabla f_i(\xx_i^{(t)})} + L\tilde\eta^2 \E\norm{\frac{1}{n}\sum_{i=1}^n \nabla f_i(\xx_i^{(t)})}^2 + \frac{L\tilde\eta^2 \sigma^2}{n}                                    \\
		                      & = \E f(\hat\xx^{(t)})+ L\tilde\eta^2 \E\norm{\frac{1}{n}\sum_{i=1}^n \nabla f_i(\xx_i^{(t)})}^2 + \frac{L\tilde\eta^2 \sigma^2}{n}                                                                                                                                      \\
		                      & \hspace{0.5cm} - \frac{\tilde\eta}{2}\E\norm{\frac{1}{n}\sum_{i=1}^n \nabla f_i(\xx_i^{(t)})}^2 - \frac{\tilde\eta}{2}\E\norm{\nabla f(\hat\xx^{(t)})}^2 +  \frac{\tilde\eta}{2}\E\norm{\frac{1}{n}\sum_{i=1}^n \nabla f_i(\xx_i^{(t)}) - \nabla f(\hat\xx^{(t)})}^2    \\
		                      & \leq \E f(\hat\xx^{(t)})+ (L\tilde\eta^2 -\tilde\eta/2) \E\norm{\frac{1}{n}\sum_{i=1}^n \nabla f_i(\xx_i^{(t)})}^2 + \frac{L\tilde\eta^2 \sigma^2}{n}                                                                                                                   \\
		                      & \hspace{0.5cm} - \frac{\tilde\eta}{4}\E\norm{\nabla f(\bar\xx^{(t)})}^2 + \frac{\tilde\eta}{2}\E\norm{\nabla f(\hat\xx^{(t)}) - \nabla f(\bar\xx^{(t)})}^2  +  \frac{\tilde\eta}{2}\E\norm{\frac{1}{n}\sum_{i=1}^n \nabla f_i(\xx_i^{(t)}) - \nabla f(\hat\xx^{(t)})}^2 \\
		                      & \leq \E f(\hat\xx^{(t)}) - \frac{\tilde\eta}{4}\E\norm{\nabla f(\bar\xx^{(t)})}^2   + \frac{L\tilde\eta^2 \sigma^2}{n} - \tilde\eta/4 \E\norm{\frac{1}{n}\sum_{i=1}^n \nabla f_i(\xx_i^{(t)})}^2                                                                        \\
		                      & \hspace{0.5cm} + \frac{3\tilde\eta}{2}\E\norm{\nabla f(\hat\xx^{(t)}) - \nabla f(\bar\xx^{(t)})}^2  +  \frac{\tilde\eta}{n}\sum_{i=1}^n \E\norm{ \nabla f_i(\xx_i^{(t)}) - \nabla f_i(\bar\xx^{(t)})}^2
		\,.
	\end{align*}
	In the last inequality we used our bound on the step-size that $\eta\leq \frac{1 - \beta}{4L}$.
	In the rest of the inequalities, we repeatedly use the identity that $2ab = -a^2 - b^2 + (a-b)^2$.
	Finally, using the smoothness of the function $f$ and the definition $\tilde\eta := \frac{\eta}{1-\beta}$, we get
	\begin{align*}
		\E f(\hat\xx^{(t+1)}) & \leq \E f(\hat\xx^{(t)}) - \frac{\tilde\eta}{4}\norm{\nabla f(\bar\xx^{(t)})}^2  + \frac{L \tilde\eta^2 \sigma^2}{n} + \frac{3L^2 \tilde\eta}{2}\norm{\bar\xx^{(t)} - \hat\xx^{(t)}}^2 + \frac{L^2 \tilde\eta}{n} \sum_{i=1}^n\norm{\xx_i^{(t)} - \bar\xx^{(t)}}^2\,.
	\end{align*}
	Recalling the definition of $\ee^{(t)}$ and $\mX^{(t)}$ yields the lemma.
\end{proof}

\begin{lemma}[Strongly-convex one step progress]\label{lem:convex-progress}
	Suppose that the set of functions $\{f_i\}$ are $\mu$-strongly convex in addition to Assumption~\ref{asm:convergence}. Then the sequence of iterates generated by \eqref{eq:appendix_our_scheme_matrix_form} using $\eta \leq \frac{1-\beta}{4L}$ satisfy for $\tilde\eta := \frac{\eta}{1 - \beta}$,
	\[
		\E \norm{\hat\xx^{(t+1)} - \xx^\star}^2
		\leq (1-\mu\tilde\eta/2)\E \norm{\hat\xx^{(t)} - \xx^\star}^2
		+ \frac{\tilde\eta^2\sigma^2}{n} - 3\tilde\eta/4(f(\bar\xx^{(t)}) - f^\star)
		+ 8L\tilde\eta \norm{\ee^{(t)}}^2 + \frac{5L\tilde\eta}{n}\norm{\mX^{t} - \bar\mX^{t}}^2_F\,.
	\]
\end{lemma}
\begin{proof}
	Starting from the update rule for $\hat\xx^{(t+1)}$, and expanding very similar to the steps performed in the previous lemma we get,
	\begin{align*}
		\E \norm{\hat\xx^{(t+1)} - \xx^\star}^2 & = \E \norm{\hat\xx^{(t)} - \xx^\star}^2 + 2\E\lin{\hat\xx^{(t+1)} - \hat\xx^{(t)}, \hat\xx^{(t)} - \xx^\star} + \E\norm{\hat\xx^{(t+1)} - \hat\xx^{(t)}}^2                                          \\
		                                        & \leq  \E \norm{\hat\xx^{(t)} - \xx^\star}^2 - \frac{2\tilde\eta}{n}\sum_{i=1}^n\lin{\nabla f_i(\xx_i^{(t)}) , \hat\xx^{(t)} - \xx^\star}                                                            \\
		                                        & \hspace*{2cm}+ \frac{\tilde\eta^2}{n}\sum_{i=1}^n \E\norm{\nabla f_i(\xx_i^{(t)}) - \nabla f_i(\bar\xx^{(t)})}^2 + \tilde\eta^2\E\norm{\nabla f(\bar\xx^{(t)})}^2 + \frac{\tilde\eta^2 \sigma^2}{n} \\
		                                        & \leq  \E \norm{\hat\xx^{(t)} - \xx^\star}^2 -2\tilde\eta \underbrace{\tfrac{1}{n}\sum_{i=1}^n\lin{\nabla f_i(\xx_i^{(t)}) , \hat\xx^{(t)} - \xx^\star}}_{\cT_1}                                     \\
		                                        & \hspace*{2cm}+ \frac{\tilde\eta^2 L^2}{n}\sum_{i=1}^n \E\norm{\xx_i^{(t)} - \bar\xx^{(t)}}^2 + 2L\tilde\eta^2\E(f(\bar\xx^{(t)}) - f^\star) + \frac{\tilde\eta^2 \sigma^2}{n}
		\,.
	\end{align*}
	We will examine the term $\cT_1$ now. Using strong convexity and smoothness of each of the functions $\{f_i\}$, we have
	\begin{align*}
		\frac{1}{n}\sum_{i=1}^n\lin{\nabla f_i(\xx_i), \xx_i - \xx^\star} & \geq \frac{1}{n}\sum_{i=1}^n f_i(\xx_i) - f(\xx^\star) + \frac{\mu}{2}\frac{1}{n}\sum_{i=1}^n\norm{\xx_i - \xx^\star}^2                                                                 \\
		                                                                  & \geq \frac{1}{n}\sum_{i=1}^n f_i(\xx_i) - f(\xx^\star) + \frac{\mu}{4}\norm{\hat\xx - \xx^\star}^2 - \frac{\mu}{2n}\sum_{i=1}^n\norm{\xx_i - \hat\xx}^2                                 \\
		                                                                  & \geq \frac{1}{n}\sum_{i=1}^n f_i(\xx_i) - f(\xx^\star) + \frac{\mu}{4}\norm{\hat\xx - \xx^\star}^2 - \frac{\mu}{n}\sum_{i=1}^n\norm{\xx_i - \bar\xx}^2 - \mu \norm{\hat\xx - \bar\xx}^2 \\
		                                                                  & \geq \frac{1}{n}\sum_{i=1}^n f_i(\xx_i) - f(\xx^\star) + \frac{\mu}{4}\norm{\hat\xx - \xx^\star}^2 - \frac{L}{n}\sum_{i=1}^n\norm{\xx_i - \bar\xx}^2 - L \norm{\hat\xx - \bar\xx}^2 \,,
	\end{align*}
	and
	\begin{align*}
		\frac{1}{n}\sum_{i=1}^n\lin{\nabla f_i(\xx_i), \bar\xx - \xx_i}   \geq f(\bar\xx) - \frac{1}{n}\sum_{i=1}^n f_i(\xx_i) - \frac{L}{2n}\sum_{i=1}^n \norm{\xx_i - \bar\xx}^2 \,,
	\end{align*}
	and finally,
	\begin{align*}
		\frac{1}{n}\sum_{i=1}^n\lin{\nabla f_i(\xx_i), \hat\xx - \bar\xx} & = \lin{\frac{1}{n}\sum_{i=1}^n \left( \nabla f_i(\xx_i) \pm \nabla f_i(\bar\xx) \right), \hat\xx - \bar\xx}                                           \\
		                                                                  & \geq -\frac{1}{8L}\norm{\frac{1}{n}\sum_{i=1}^n (\nabla f_i(\xx_i) \pm \nabla f_i(\bar \xx))}^2 -2L\norm{\hat\xx - \bar\xx}^2                         \\
		                                                                  & \geq -\frac{1}{4Ln}\sum_{i=1}^n\norm{\nabla f_i(\xx_i) - \nabla f_i(\bar\xx)}^2 -\frac{1}{4L}\norm{\nabla f(\bar\xx)}^2 -2L\norm{\hat\xx - \bar\xx}^2 \\
		                                                                  & \geq -\frac{L}{2n}\sum_{i=1}^n\norm{\xx_i - \bar\xx}^2 - \frac{1}{2}(f(\bar\xx) - f(\xx^\star)) - 3L\norm{\hat\xx - \bar\xx}^2 \,,
	\end{align*}
	where the first inequality uses the basic inequality $\langle \aa, \bb \rangle \leq \frac{\beta}{2} \norm{\aa}^2 + \frac{1}{2 \beta} \norm{\bb}^2, \forall \beta > 0$, and the second inequality uses the relaxed triangle inequality $\norm{\aa + \bb}^2 \leq (1 + \alpha) \norm{\aa}^2 + (1 + \alpha^{-1}) \norm{\bb}^2, \forall \alpha > 0$.

	Adding up the three inequalities together yields the following expression for the term $\cT_1$
	\begin{align*}
		\frac{1}{n}\sum_{i=1}^n\lin{\nabla f_i(\xx_i^{(t)}) , \hat\xx^{(t)} - \xx^\star}
		  & \geq \frac{1}{2}(f(\bar\xx^{(t)}) - f(\xx^\star)) + \frac{\mu}{4} \norm{\hat\xx^{(t)} - \xx^\star}^2 - \frac{2L}{n}\sum_{i=1}^n \norm{ \xx_i^{(t)} - \bar\xx^{(t)} }^2 - 4L \norm{ \hat\xx^{(t)} - \bar\xx^{(t)} }^2\,.
	\end{align*}
	Plugging this back into the previous inequality and using $\tilde\eta \leq \frac{1}{4L}$ finishes the proof of the lemma.
	\begin{align*}
		\E \norm{\hat\xx^{(t+1)} - \xx^\star}^2
		  & \leq  (1- \tilde\eta\mu / 2)\E \norm{\hat\xx^{(t)} - \xx^\star}^2
		- \tilde\eta(1- 2L\tilde\eta)\E(f(\bar\xx^{(t)}) - f^\star)
		+ \frac{\tilde\eta^2 \sigma^2}{n} \\
		  & \hspace*{2cm}+ \frac{\tilde\eta^2 L^2 + 4\tilde\eta L}{n}\sum_{i=1}^n \E\norm{\xx_i^{(t)} - \bar\xx^{(t)}}^2
		+ 8L\tilde\eta\norm{ \hat\xx^{(t)} - \bar\xx^{(t)} }^2\,. \qedhere
	\end{align*}
\end{proof}

\paragraph{Bounding the consensus error.}
We will now try to bound the consensus error $(\mX^{t} - \bar\mX^{t})$ between the node's parameters and its average.
During each step, we perform a diffusion step (communication with neighbors) which brings the parameters of the nodes closer to each other.
However we also perform additional gradient/momentum steps which moves the distance away from each other.

\begin{lemma}[One step consensus change]\label{lem:consensus-change}
	Given Assumption~\ref{asm:convergence}, the sequence of iterates generated by \eqref{eq:appendix_our_scheme_matrix_form}
	using $\eta \leq \frac{\rho}{7L}$ satisfy for $\tilde\eta := \frac{\eta}{1 - \beta}$,
	\[
		\frac{1}{n}\E\norm{\mX^{t+1} - \bar\mX^{t+1}}^2_F \leq \frac{(1 - \rho/4)}{n}\E\norm{\mX^{t} - \bar\mX^{t}}^2_F
		+ \frac{12\eta^2\zeta^2}{\rho} + 4(1-\rho)\eta^2\sigma^2 + \frac{6\eta^2\beta^2}{\rho n}\E\norm{\mM^{(t)} - \bar\mM^{(t)}}^2_F \,. \]
\end{lemma}
\begin{proof}
	Starting from the update step \eqref{eq:appendix_our_scheme_matrix_form},
	\begin{align*}
		\frac{1}{n}\E\norm{\mX^{t+1} - \bar\mX^{t+1}}^2_F & = \frac{1}{n} \E\norm{\mW \left( \mX^{(t)} - \eta \left( \beta \mM^{(t)} + \mG^{(t)} \right) \right) - \left( \bar\mX^{(t)} - \eta \left( \beta \bar\mM^{(t)} + \bar\mG^{(t)} \right) \right)}^2_F                        \\
		                                                  & \leq \frac{1 - \rho}{n}\E\norm{\left( \mX^{(t)} - \eta \left( \beta \mM^{(t)} + \mG^{(t)} \right) \right) - \left( \bar\mX^{(t)} - \eta \left( \beta \bar\mM^{(t)} + \bar\mG^{(t)} \right) \right)}^2_F                   \\
		                                                  & \leq \frac{1 - \rho}{n}\E\norm{\left( \mX^{(t)} - \eta \left( \beta \mM^{(t)} + \EEb{t}{\mG^{(t)}} \right) \right) - \left( \bar\mX^{(t)} - \eta \left( \beta \bar\mM^{(t)} + \EEb{t}{\bar\mG^{(t)}} \right) \right)}^2_F \\&\hspace*{2cm}+ 4(1-\rho)\eta^2\sigma^2\\
		                                                  & \leq \frac{(1 - \rho)(1+\rho/2)}{n}\E\norm{\mX^{(t)} - \bar\mX^{(t)}}^2_F + \frac{6\eta^2\beta^2}{\rho n}\E\norm{\mM^{(t)} - \bar\mM^{(t)}}^2_F                                                                           \\&\hspace*{2cm} + \frac{6\eta^2}{\rho n}\E\norm{ \EEb{t}{ \mG^{(t)} } - \EEb{t}{ \bar\mG^{(t)} } }^2_F + 4(1-\rho)\eta^2\sigma^2\,.
	\end{align*}
	Here we used the contractivity of the mixing matrix and Young's inequality. We can proceed as
	\begin{align*}
		\frac{1}{n}\E\norm{\mX^{t+1} - \bar\mX^{t+1}}^2_F & \leq \frac{(1 - \rho/2)}{n}\E\norm{\mX^{(t)} - \bar\mX^{(t)}}^2_F + \frac{6\eta^2\beta^2}{\rho n}\E\norm{\mM^{(t)} - \bar\mM^{(t)}}^2_F  + 4(1-\rho)\eta^2\sigma^2                                                     \\
		                                                  & \hspace*{2cm} + \frac{6\eta^2}{\rho n}\E\norm{\EEb{t}{\mG^{(t)}} - \nabla f(\bar\xx^{(t)})}^2_F                                                                                                                        \\
		                                                  & =  \frac{(1 - \rho/2)}{n}\E\norm{\mX^{(t)} - \bar\mX^{(t)}}^2_F + \frac{6\eta^2\beta^2}{\rho n}\E\norm{\mM^{(t)} - \bar\mM^{(t)}}^2_F  + 4(1-\rho)\eta^2\sigma^2                                                       \\
		                                                  & \hspace*{2cm} + \frac{6\eta^2}{\rho n}\sum_{i=1}^n\E\norm{\nabla f_i(\xx_i^{(t)}) \pm \nabla f_i(\bar\xx^{(t)}) - \nabla f(\bar\xx^{(t)})}^2                                                                           \\
		                                                  & \leq  \frac{(1 - \rho/2)}{n}\E\norm{\mX^{(t)} - \bar\mX^{(t)}}^2_F + \frac{6\eta^2\beta^2}{\rho n}\E\norm{\mM^{(t)} - \bar\mM^{(t)}}^2_F  + 4(1-\rho)\eta^2\sigma^2                                                    \\
		                                                  & \hspace*{2cm} + \frac{12\eta^2}{\rho n}\sum_{i=1}^n\E\norm{\nabla f_i(\xx_i^{(t)}) - \nabla f_i(\bar\xx^{(t)})}^2 + \frac{12\eta^2}{\rho n}\sum_{i=1}^n\E\norm{\nabla f_i(\bar \xx^{(t)}) - \nabla f(\bar\xx^{(t)})}^2 \\
		                                                  & \leq  \frac{(1 - \rho/2)}{n}\E\norm{\mX^{(t)} - \bar\mX^{(t)}}^2_F + \frac{6\eta^2\beta^2}{\rho n}\E\norm{\mM^{(t)} - \bar\mM^{(t)}}^2_F  + 4(1-\rho)\eta^2\sigma^2                                                    \\
		                                                  & \hspace*{2cm} + \frac{12\eta^2 L^2}{\rho n}\sum_{i=1}^n\E\norm{\xx_i^{(t)} - \bar\xx^{(t)}}^2 + \frac{12\eta^2\zeta^2}{\rho}\,.
	\end{align*}
	Our assumption that the step size $\eta \leq \frac{\rho}{7L}$ ensures that $12\eta^2 L^2 \leq \rho^2/4$, finishing the proof.
\end{proof}

We will now try to bound the momentum error $(\mX^{t} - \bar\mX^{t})$ between the momentum on each node and its average across nodes.

\begin{lemma}[One step momentum change]\label{lem:momentum}
	Given Assumption~\ref{asm:convergence}, the sequence of iterates generated by \eqref{eq:appendix_our_scheme_matrix_form}
	using momentum satisfying $\frac{\beta}{1 - \beta} \leq \frac{\rho}{21}$,
	\begin{align*}
		\frac{6\eta^2\beta^2}{n\rho(1-\mu)(1-\beta)}\E\norm{\mM^{t+1} - \bar\mM^{t+1}}^2_F & \leq \left( \frac{6\eta^2\beta^2}{n\rho(1-\mu)(1-\beta)} - \frac{6\eta^2\beta^2}{n\rho} \right)\E\norm{(\mM^{(t)} -\bar\mM^{(t)})}^2_F \\&\hspace*{0.5cm} +  \frac{\rho}{8n}\E\norm{\mX^{(t)} - \bar\mX^{(t)}}^2_F + \frac{\eta^2\rho \zeta^2 }{8} + \frac{\eta^2\rho \sigma^2 (1-\beta)}{8(1-\mu)}\,.
	\end{align*}
\end{lemma}
\begin{proof}
	Starting from the update step \eqref{eq:appendix_our_scheme_matrix_form} and proceeding similar to the previous lemma, we have
	\begin{align*}
		\frac{1}{n}\E\norm{\mM^{(t+1)} - \bar\mM^{(t+1)}}^2_F
		  &                                                                                                                                                                                                  \\
		  & \hspace*{-3cm}= \frac{1}{n}\E\norm{\left( \mu\mI + (1 - \mu) \beta \mW \right)  (\mM^{(t)} -\bar\mM^{(t)}) + (1-\mu)\mW(\mG^{(t)} - \bar\mG^{(t)}) + \frac{1 - \mu}{\eta}(\mI-\mW)\mX^{(t)}}^2_F \\
		  & \hspace*{-3cm}= \frac{1}{n}\E\norm{
			\left( \mu\mI + (1 - \mu) \beta \mW \right)  (\mM^{(t)} -\bar\mM^{(t)})
			+ \frac{1 - \mu}{\eta}(\mI - \mW) \mX^{(t)}
			+ (1 - \mu) \mW( \Eb{\mG^{(t)} - \bar\mG^{(t)}} ) }^2_F \\
		  & \hspace*{-2cm} + \frac{1}{n} \E \norm{
			(1-\mu) \mW \left( \mG^{(t)} - \Eb{\mG^{(t)}} - ( \bar\mG^{(t)} - \Eb{\bar\mG^{(t)}} ) \right)
		}^2_F\\
		  & \hspace*{-3cm}= \frac{1}{n}\E\norm{
			\left( \mu\mI + (1 - \mu) \beta \mW \right)  (\mM^{(t)} -\bar\mM^{(t)})
			+ \frac{1 - \mu}{\eta}(\mI - \mW) \mX^{(t)}
			+ (1 - \mu) \mW( \Eb{\mG^{(t)} - \bar\mG^{(t)}} ) }^2_F + 4 \sigma^2 \\
		  & \hspace*{-3cm}\leq \frac{1}{n}
		\left( 1 + \frac{ (1-\mu)(1-\beta) }{1 - (1-\mu)(1-\beta)} \right)
		\E\norm{\left( \mu\mI + (1 - \mu) \beta \mW \right)  (\mM^{(t)} -\bar\mM^{(t)}) }_F^2
		+ 4\sigma^2 \\
		  & \hspace*{-2cm} + \frac{1}{n} \left( 1 + \frac{ 1 - (1-\mu)(1-\beta) }{ (1-\mu)(1-\beta) } \right)
		\E \norm{
			\frac{1 - \mu}{\eta}(\mI-\mW)\mX^{(t)} + (1-\mu) \mW \left( \Eb{ \mG^{(t)} - \bar\mG^{(t)}} \right)
		}^2_F \,.
	\end{align*}
	Note that since $\mW \prec \mI$, we have $\left( \mu\mI + (1 - \mu) \beta \mW \right) \prec ( \mu + (1 - \mu) \beta)\mI = (1 - (1-\beta)(1-\mu))\mI$. Further, since $-\mI \prec \mW$, we have $\mI -\mW \prec 2\mI$. With these observations, we can continue
	\begin{align*}
		\frac{1}{n}\E\norm{\mM^{(t+1)} - \bar\mM^{(t+1)}}^2_F
		  &                                                                                                   \\
		  & \hspace*{-3cm}\leq \frac{1}{n}
		\left( 1 + \frac{ (1-\mu)(1-\beta) }{1 - (1-\mu)(1-\beta)} \right)
		\E\norm{\left( 1- (1 - \mu)(1 - \beta)\right)(\mM^{(t)} -\bar\mM^{(t)}) }_F^2
		+ 4\sigma^2 \\
		  & \hspace*{-2cm} + \frac{1}{n} \left( 1 + \frac{ 1 - (1-\mu)(1-\beta) }{ (1-\mu)(1-\beta) } \right)
		\E \norm{
			\frac{1 - \mu}{\eta}(\mI-\mW)\mX^{(t)} + (1-\mu) \mW \left( \Eb{ \mG^{(t)} - \bar\mG^{(t)}} \right)
		}^2_F\\
		  & \hspace*{-3cm}\leq \frac{1}{n}
		\left(1 - (1-\mu)(1-\beta) \right)
		\E\norm{\mM^{(t)} -\bar\mM^{(t)}}_F^2
		+ 4\sigma^2 \\
		  & \hspace*{-2cm} + \frac{ 1 }{ (1-\mu)(1-\beta) n }
		\E \norm{
			\frac{1 - \mu}{\eta}(\mI-\mW)\mX^{(t)} + (1-\mu) \mW \left( \Eb{ \mG^{(t)} - \bar\mG^{(t)}} \right)
		}^2_F \\
		  & \hspace*{-3cm}\leq \frac{1}{n}
		\left(1 - (1-\mu)(1-\beta) \right)
		\E\norm{\mM^{(t)} -\bar\mM^{(t)}}_F^2
		+ 4\sigma^2  \\
		  & \hspace*{-2cm} + \frac{ 4(1-\mu)}{ (1-\beta) n \eta^2 }
		\E \norm{ \mX^{(t)} - \bar \mX^{(t)} }_F^2
		+ \frac{ 2(1-\mu) }{ (1-\beta) n } \E \norm{ \Eb{ \mG^{(t)} } - \bar\mG^{(t)} }^2_F\,.
	\end{align*}

	From the proof of the previous lemma, we can simplify the last term as
	\begin{align*}
		\frac{1}{n}\E\norm{\mM^{(t+1)} - \bar\mM^{(t+1)}}^2_F
		  & \leq \frac{1}{n} \left( 1 - (1 - \mu) (1 - \beta) \right)
		\E \norm{ \mM^{(t)} -\bar\mM^{(t)} }_F^2 + 4 \sigma^2\\
		  & \hspace*{1cm} + \frac{4(1-\mu)}{ n \eta^2 (1-\beta) } \E \norm{ \mX^{(t)} - \bar \mX^{(t)} }_F^2
		+ \frac{2(1-\mu) }{ n (1-\beta) } \E \norm{ \Eb{ \mG^{(t)} } \pm \nabla f( \bar\mX^{(t)} ) - \bar\mG^{(t)} }^2_F  \\
		  & \leq \frac{1}{n} \left( 1 - (1 - \mu) (1 - \beta) \right)
		\E \norm{ \mM^{(t)} -\bar\mM^{(t)} }_F^2 + 4 \sigma^2\\
		  & \hspace*{1cm} + \frac{4(1-\mu)}{ n \eta^2 (1-\beta) } \E \norm{ \mX^{(t)} - \bar \mX^{(t)} }_F^2
		+ \frac{8(1-\mu)\zeta^2}{(1-\beta)} + \frac{4(1-\mu) L^2 }{ n (1-\beta) } \E \norm{ \mX^{(t)} - \bar\mX^{(t)} }^2_F  \\
		  & = \frac{1}{n} \left( 1 - (1 - \mu) (1 - \beta) \right)
		\E \norm{ \mM^{(t)} -\bar\mM^{(t)} }_F^2 + 4 \sigma^2\\
		  & \hspace*{1cm} + \frac{4(1-\mu)(1 + \eta^2 L^2) }{ n \eta^2 (1-\beta) } \E \norm{ \mX^{(t)} - \bar \mX^{(t)} }_F^2
		+ \frac{8(1-\mu)\zeta^2}{(1-\beta)}  \\
	\end{align*}

	Multiplying both sides by $\frac{6\eta^2\beta^2}{\rho(1-\mu)(1-\beta)}$ yields
	\begin{align*}
		\frac{6\eta^2\beta^2}{n\rho(1-\mu)(1-\beta)}\E\norm{\mM^{t+1} - \bar\mM^{t+1}}^2_F \hspace*{-2cm}
		  &                                                                                                                                                            \\
		  & \leq \frac{6\eta^2\beta^2}{n\rho(1-\mu)(1-\beta)}\E\norm{\mM^{(t)} -\bar\mM^{(t)}}^2_F - \frac{6\eta^2\beta^2}{n\rho}\E\norm{\mM^{(t)} -\bar\mM^{(t)}}^2_F
		\\&\hspace*{1cm} +  \frac{48\beta^2 (1+L^2\eta^2)}{n\rho (1-\beta)^2}\E\norm{\mX^{(t)} - \bar\mX^{(t)}}^2_F + \frac{24\eta^2\beta^2 \sigma^2}{\rho(1-\mu)(1-\beta)} + \frac{48\eta^2 \beta^2\zeta^2}{\rho(1-\beta)^2}\\
		  & \leq \frac{6\eta^2\beta^2}{n\rho(1-\mu)(1-\beta)}\E\norm{\mM^{(t)} -\bar\mM^{(t)}}^2_F - \frac{6\eta^2\beta^2}{n\rho}\E\norm{\mM^{(t)} -\bar\mM^{(t)}}^2_F
		\\&\hspace*{1cm} +  \frac{\rho}{8n}\E\norm{\mX^{(t)} - \bar\mX^{(t)}}^2_F + \frac{\eta^2\rho \zeta^2 }{8} + \frac{\eta^2\rho \sigma^2 (1-\beta)}{8(1-\mu)}\,.
	\end{align*}
	The last step follows from our assumption that the momentum parameter that $\frac{\beta}{1 - \beta} \leq \frac{\rho}{21}$ and $\eta\leq \frac{1}{7L}$. This ensures that $\frac{48\beta^2 (1+L^2\eta^2)}{\rho (1-\beta)^2} \leq \frac{49\beta^2}{\rho (1-\beta)^2} \leq \frac{\rho}{8}$.
\end{proof}

We can now exactly describe the progress in consensus made each round.
\begin{lemma}[One step consensus improvement]\label{lem:consensus}
	Given Assumption~\ref{asm:convergence}, the sequence of iterates generated by \eqref{eq:appendix_our_scheme_matrix_form} using step-size $\eta \leq \frac{\rho}{7L}$ and momentum $\frac{\beta}{1 - \beta} \leq \frac{\rho}{21}$, satisfy
	\begin{align*}
		\frac{1}{n}\E\norm{\mX^{t+1} - \bar\mX^{t+1}}^2_F + \frac{6\eta^2\beta^2}{n\rho(1-\mu)(1-\beta)}\E\norm{\mM^{t+1} - \bar\mM^{t+1}}^2_F \hspace*{-7cm} &                                                                                                                                              \\
		                                                                                                                                                      & \leq \frac{1 - \rho/8}{n}\E\norm{\mX^{t} - \bar\mX^{t}}^2_F + \frac{6\eta^2\beta^2}{n\rho(1-\mu)(1-\beta)}\E\norm{\mM^{t} - \bar\mM^{t}}^2_F
		+ \frac{13\eta^2\zeta^2}{\rho} + \frac{13\eta^2\sigma^2(2 -\beta -\mu))}{(1-\mu)\rho}
	\end{align*}

\end{lemma}
\begin{proof}
	Simply adding the results of Lemmas~\ref{lem:consensus-change} and \ref{lem:momentum} gives the result.
\end{proof}

\paragraph{Average and virtual sequences.} We will now bound the difference between the average and the virtual sequences $\ee^{(t)}$ which recall was defined to be $\ee^{(t)} = \hat\xx^{(t)} - \bar\xx^{(t)}$. The latter runs SGD with momentum whereas the former only runs SGD. We view the momentum terms as accumulating the gradient terms, delayed over time and hence the proof views SGDm as simply SGD run with a larger step-size.

\begin{lemma}[One step error contraction]\label{lem:error-sequence}
	Given Assumption~\ref{asm:convergence}, the sequence of iterates generated by \eqref{eq:appendix_our_scheme_matrix_form} satisfy
	\begin{align*}
		\E\norm{\ee^{(t+1)}}^2 \leq (1- (1-\mu)(1-\beta))\E\norm{\ee^{(t)}}^2  + \frac{2\tilde\eta^2 \beta^2}{(1-\beta) (1-\mu)}\E\norm{\E_t[\bar\gg^t]}^2 + \tilde\eta^2\beta^2 \sigma^2 \,.
	\end{align*}
\end{lemma}
\begin{proof}
	By definition, $\hat\xx^{(0)} = \bar\xx^{(0)}$ and hence we have $\ee^{(0)}=0$. For $t \geq 0$, starting from the definition of the error term we have
	\begin{align*}
		\ee^{(t+1)} & = \hat\xx^{(t+1)} - \bar\xx^{(t+1)}                                                                                                          \\
		            & = \left( \hat\xx^{(t)} - \frac{\eta}{1 - \beta}\bar\gg^{(t)} \right) - \left(\bar\xx^{(t)} - \eta (\beta\bar\mm^{(t)} +\bar\gg^{(t)})\right) \\
		            & = \ee^{(t)} - \eta\beta(\frac{1}{1-\beta}\bar\gg^{(t)} - \bar\mm^{(t)})                                                                      \\
		            & = \sum_{k=0}^{t} -\eta\beta(\frac{1}{1-\beta}\bar\gg^{(k)} - \bar\mm^{(k)})\,.
	\end{align*}
	Using the update of the average momentum \eqref{eq:appendix_our_scheme_averaged_form}, we can write
	\begin{align*}
		\ee^{t+1} & = \sum_{k=0}^{(t)} -\eta\beta(\frac{1}{1-\beta}\bar\gg^{(k)} - \bar\mm^{(k)})                                                                             \\
		          & = \sum_{k=0}^{(t)} -\eta\beta \left( \frac{1}{1-\beta}\bar\gg^{(k)} - \left( (1- (1-\mu)(1-\beta))\bar\mm^{(k-1)} + (1-\mu)\bar\gg^{(k-1)}\right) \right) \\
		          & = (1- (1-\mu)(1-\beta)) \sum_{k=0}^t -\eta\beta \left( \frac{1}{1-\beta}\bar\gg^{(k-1)} - \bar\mm^{(k-1)} \right)
		+ \sum_{k=0}^{(t)} -\frac{\eta\beta}{1-\beta} \left(\bar\gg^{(k)} - \bar\gg^{(k-1)}\right)\\
		          & = (1- (1-\mu)(1-\beta))\ee^{(t)} - \frac{\eta\beta}{1-\beta}\bar\gg^{(t)} \,.
	\end{align*}
	By convention, we assume that vectors with negative indices are 0. Taking norms and expectations gives
	\begin{align*}
		\E\norm{\ee^{(t+1)}}^2 & = \E\norm{(1- (1-\mu)(1-\beta))\ee^{(t)} - \frac{\eta\beta}{1-\beta}\bar\gg^{(t)}}^2                                                                                                 \\
		                       & \leq \E\norm{(1- (1-\mu)(1-\beta))\ee^{(t)} - \frac{\eta\beta}{1-\beta}\E_t[\bar\gg^{(t)}]}^2 + \frac{\eta^2\beta^2 \sigma^2}{(1-\beta)^2}                                           \\
		                       & \leq (1- (1-\mu)(1-\beta))\E\norm{\ee^{(t)}}^2 + \frac{2\eta^2 \beta^2}{(1-\beta)^3 (1-\mu)}\E\norm{\E_t[\bar\gg^{(t)}]}^2 + \frac{\eta^2\beta^2 \sigma^2}{(1-\beta)^2} \,. \qedhere
	\end{align*}
\end{proof}

\paragraph{Convergence rate for non-convex case.}
\begin{theorem}
	Given Assumption~\ref{asm:convergence}, the sequence of iterates generated by \eqref{eq:appendix_our_scheme_matrix_form} for step size $\eta = \min \left(\frac{\rho}{7 L}, \frac{1-\beta}{4L} , \frac{(1-\mu)(1-\beta)^2}{4\beta L} , \sqrt{\frac{4n (f(\bar\xx^0) - f^\star)}{L \sigma^2 T}}\right)$ and momentum parameter $\frac{\beta}{1 - \beta} \leq \frac{\rho}{21}$ satisfies
	\begin{align*}
		\frac{1}{T}\sum_{t=0}^{T-1}\E\norm{\nabla f(\bar\xx^{t})}	\leq \cO\Bigg( & \sqrt{\frac{L\sigma^2 (f(\bar\xx^0) - f^\star)}{n T}} \ +                                                                        \\
		                                                                        & \sqrt[3]{L^2(f(\bar\xx^0) - f^\star)^2\frac{\tilde\zeta^2}{\rho^2 T^2}}\  +                                                      \\
		                                                                        & \left(\frac{1}{\rho} + \frac{1}{1-\beta} + \frac{\beta}{(1-\mu)(1-\beta)^2}\right)\frac{L(f(\bar\xx^0) - f^\star)}{T} \Bigg) \,.
	\end{align*}
\end{theorem}
\begin{proof}
	Define $\tilde\zeta^2 := \zeta^2 + \sigma^2\left(1+ \frac{1-\beta}{1-\mu} \right)$. Scaling Lemma~\ref{lem:consensus} by $\frac{24L^2 \tilde\eta}{\rho}$ gives
	\begin{align*}
		\frac{24L^2 \tilde\eta}{\rho n}\E\norm{\mX^{t+1} - \bar\mX^{t+1}}^2_F + \frac{144 L^2 \tilde\eta^3\beta^2 (1-\beta)}{n\rho^2 (1-\mu)}\E\norm{\mM^{t+1} - \bar\mM^{t+1}}^2_F \hspace*{-8cm} &                                                                                                                                                                          \\
		                                                                                                                                                                                           & \leq \frac{24L^2 \tilde\eta}{\rho n}\E\norm{\mX^{t} - \bar\mX^{t}}^2_F + \frac{144 L^2 \tilde\eta^3\beta^2 (1-\beta)}{n\rho^2 (1-\mu)}\E\norm{\mM^{t} - \bar\mM^{t}}^2_F \\
		                                                                                                                                                                                           & \hspace*{1cm}  - \frac{3L^2\tilde\eta}{n}\E\norm{\mX^{t} - \bar\mX^{t}}^2_F + \frac{312 L^2 \tilde\eta^3(1-\beta)^2 \tilde\zeta^2}{\rho^2} \,.
	\end{align*}
	Scaling Lemma~\ref{lem:error-sequence} by $\frac{3L^2\tilde\eta}{2(1-\mu)(1-\beta)}$ gives
	\begin{align*}
		\frac{3L^2\tilde\eta}{2(1-\mu)(1-\beta)}\E\norm{\ee^{(t+1)}}^2 &\leq \frac{3L^2\tilde\eta}{2(1-\mu)(1-\beta)}\E\norm{\ee^{(t)}}^2 - \frac{3L^2\tilde\eta}{2}\E\norm{\ee^{(t)}}^2 \\ &\hspace*{2cm}+ \frac{3L^2\tilde\eta^3 \beta^2}{(1-\beta)^2 (1-\mu)^2}\E\norm{\E_t[\bar\gg^t]}^2 + \frac{3L^2\tilde\eta^3 \beta^2 \sigma^2}{2(1-\mu)(1-\beta)} \,.
	\end{align*}
	Finally Lemma~\ref{lem:non-convex-progress} gives
	\[
		\E f(\hat\xx^{(t+1)}) \leq \E f(\hat\xx^{(t)}) - \frac{\tilde\eta}{4}\norm{\nabla f(\bar\xx^{(t)})}^2 - \frac{\tilde\eta}{4}\E\norm{\frac{1}{n}\sum_{i=1}^n \nabla f_i(\xx_i^{(t)})}^2 + \frac{L \tilde\eta^2 \sigma^2}{n} + \frac{3 L^2 \tilde\eta}{2}\norm{\ee^{(t)}}^2 + \frac{3L^2 \tilde\eta}{n}\norm{\mX^{t} - \bar\mX^{t}}^2_F\,.
	\]
	Define
	\[
		\Phi^t := \frac{24L^2 \tilde\eta}{\rho n}\E\norm{\mX^{t} - \bar\mX^{t}}^2_F + \frac{144 L^2 \tilde\eta^3\beta^2 (1-\beta)}{n\rho^2 (1-\mu)}\E\norm{\mM^{t} - \bar\mM^{t}}^2_F + \frac{3L^2\tilde\eta}{2(1-\mu)(1-\beta)}\E\norm{\ee^{(t)}}^2 + \E[f(\bar\xx^{t}) - f^\star] \,.
	\]
	Note that $\Phi^0 = \E[f(\bar\xx^0)] - f^\star$ and that $\Phi^t \geq 0$ for any $t$. Then adding the three inequalities from the lemmas as described above gives
	\begin{align*}
		\Phi^{t+1} & \leq \Phi^t - \frac{\tilde\eta}{4}\norm{\nabla f(\bar\xx^{(t)})}^2 + \left(  \frac{3L^2\tilde\eta^3 \beta^2}{(1-\beta)^2 (1-\mu)^2} - \frac{\tilde\eta}{4}\right) \E\norm{\frac{1}{n}\sum_{i=1}^n \nabla f_i(\xx_i^{(t)})}^2 \\&\hspace*{2cm}+ \frac{L \tilde\eta^2 \sigma^2}{n} + \frac{3L^2\tilde\eta^3 \beta^2 \sigma^2}{2(1-\mu)(1-\beta)}  + \frac{312 L^2 \tilde\eta^3(1-\beta)^2 \tilde\zeta^2}{\rho^2} \,.
	\end{align*}
	Since, $\eta \leq \frac{(1-\mu)(1-\beta)^2}{4\beta L}$, we have that $\frac{3L^2\tilde\eta^3 \beta^2}{(1-\beta)^2 (1-\mu)^2} \leq \frac{\tilde \eta}{4}$. Rearranging the terms and averaging over $t$, we get
	\begin{align*}
		\frac{1}{T}\sum_{t=0}^{T-1}\E\norm{\nabla f(\xx^{(t)})}^2 & \leq \frac{4}{\tilde\eta T}(\Phi^0 - \Phi^{T}) + \frac{L \tilde\eta \sigma^2}{n} + \frac{6L^2\tilde\eta^2 \beta^2 \sigma^2}{(1-\mu)} + \frac{1248 L^2 \tilde\eta^2(1-\beta)^2 \tilde\zeta^2}{\rho^2}                                  \\
		                                                          & \leq \frac{1}{\tilde\eta T}4(f(\bar\xx^0) - f^\star) + \tilde\eta\left( \frac{L\sigma^2}{n}\right) + \tilde\eta^2 \left( \frac{L^2 \rho^2 \sigma^2 (1-\beta)}{(1-\mu)} + \frac{1248 L^2 (1-\beta)^2 \tilde\zeta^2}{\rho^2} \right)\,.
		\\
		                                                          & \leq \frac{1}{\tilde\eta T}4(f(\bar\xx^0) - f^\star) + \tilde\eta\left( \frac{L\sigma^2}{n}\right) + \tilde\eta^2 \left( \frac{L^2\sigma^2(1-\beta)}{(1-\mu)} + \frac{1248 L^2 \tilde\zeta^2}{\rho^2} \right)                         \\
		                                                          & \leq \frac{1}{\tilde\eta T}4(f(\bar\xx^0) - f^\star) + \tilde\eta\left( \frac{L\sigma^2}{n}\right) + \tilde\eta^2 \left(\frac{1249 L^2 \tilde\zeta^2}{\rho^2} \right)\,.
	\end{align*}
	Choosing an appropriate steps-size $\eta$ proves the theorem.
\end{proof}

\clearpage
\section{Additional Results} \label{appendix:additional_results}
\subsection{Results on Distributed Average Consensus Problem} \label{appendix:more_on_consensus_averaging_problem}
Figure~\ref{fig:more_understanding_via_consensus_averaging}
illustrates the results for average consensus problem on other communication topologies and topology scales.

\begin{figure*}[!h]
	\vspace{-1em}
	\centering
	\subfigure[\small
		torus, $n\!=\!16$.
	]{
		\includegraphics[width=.315\textwidth,]{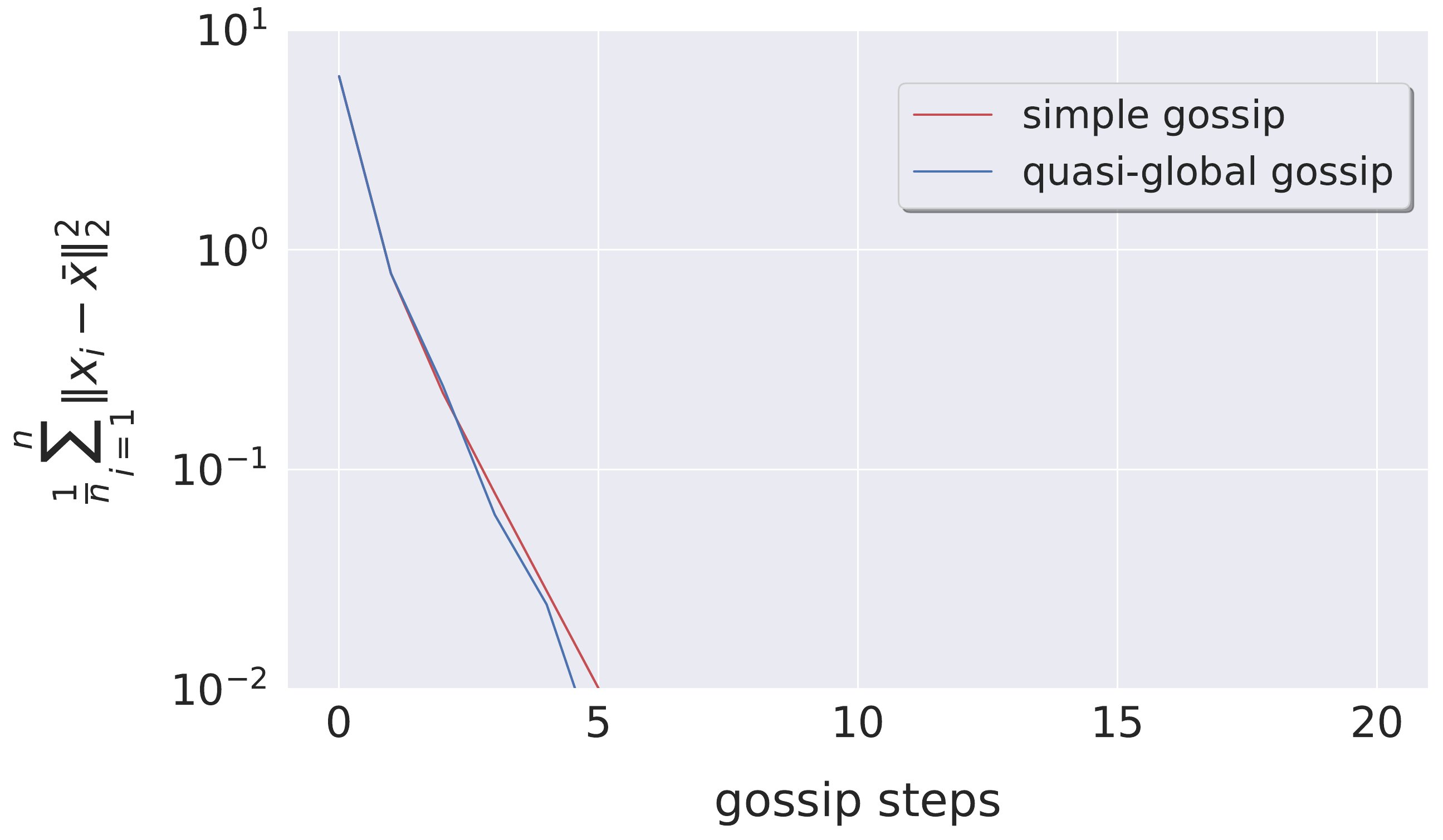}
		\label{fig:understanding_via_consensus_averaging_n16_128}
	}
	\hfill
	\subfigure[\small
		torus, $n\!=\!64$.
	]{
		\includegraphics[width=.315\textwidth,]{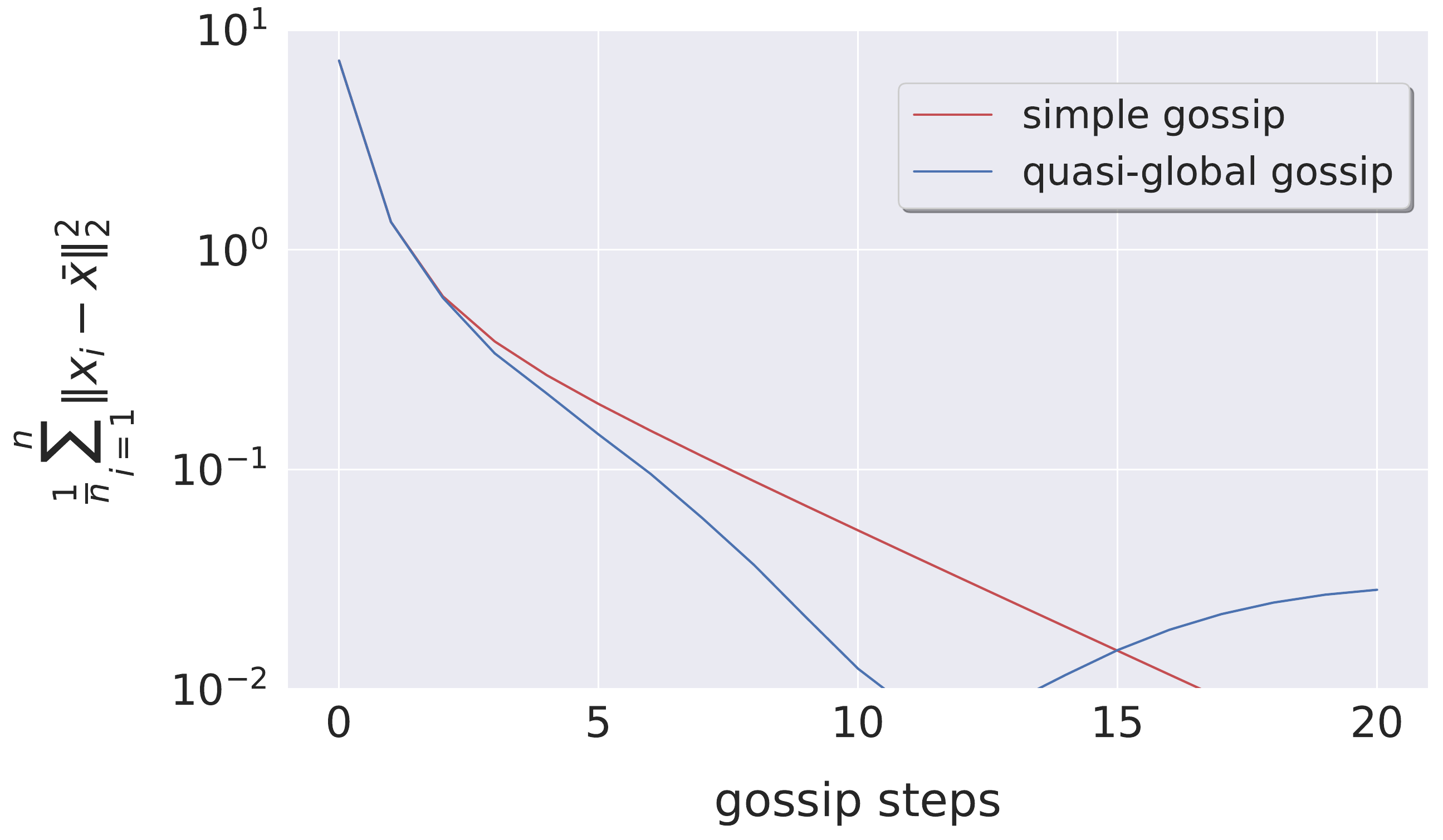}
		\label{fig:understanding_via_consensus_averaging_n16_128}
	}
	\hfill
	\subfigure[\small
		torus, $n\!=\!100$.
	]{
		\includegraphics[width=.315\textwidth,]{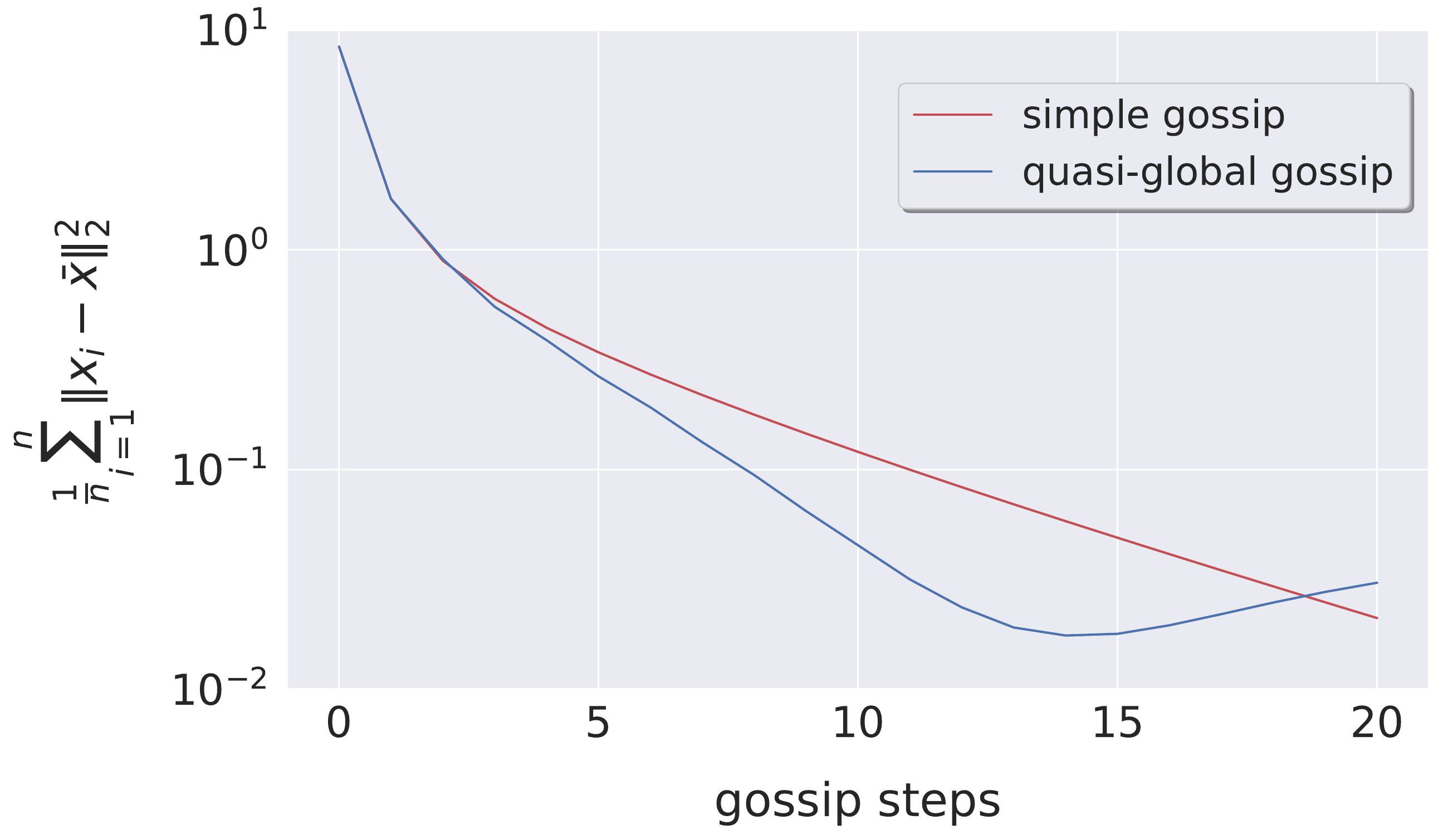}
		\label{fig:understanding_via_consensus_averaging_n16_128}
	}
	\hfill
	\subfigure[\small
		torus, $n\!=\!256$.
	]{
		\includegraphics[width=.315\textwidth,]{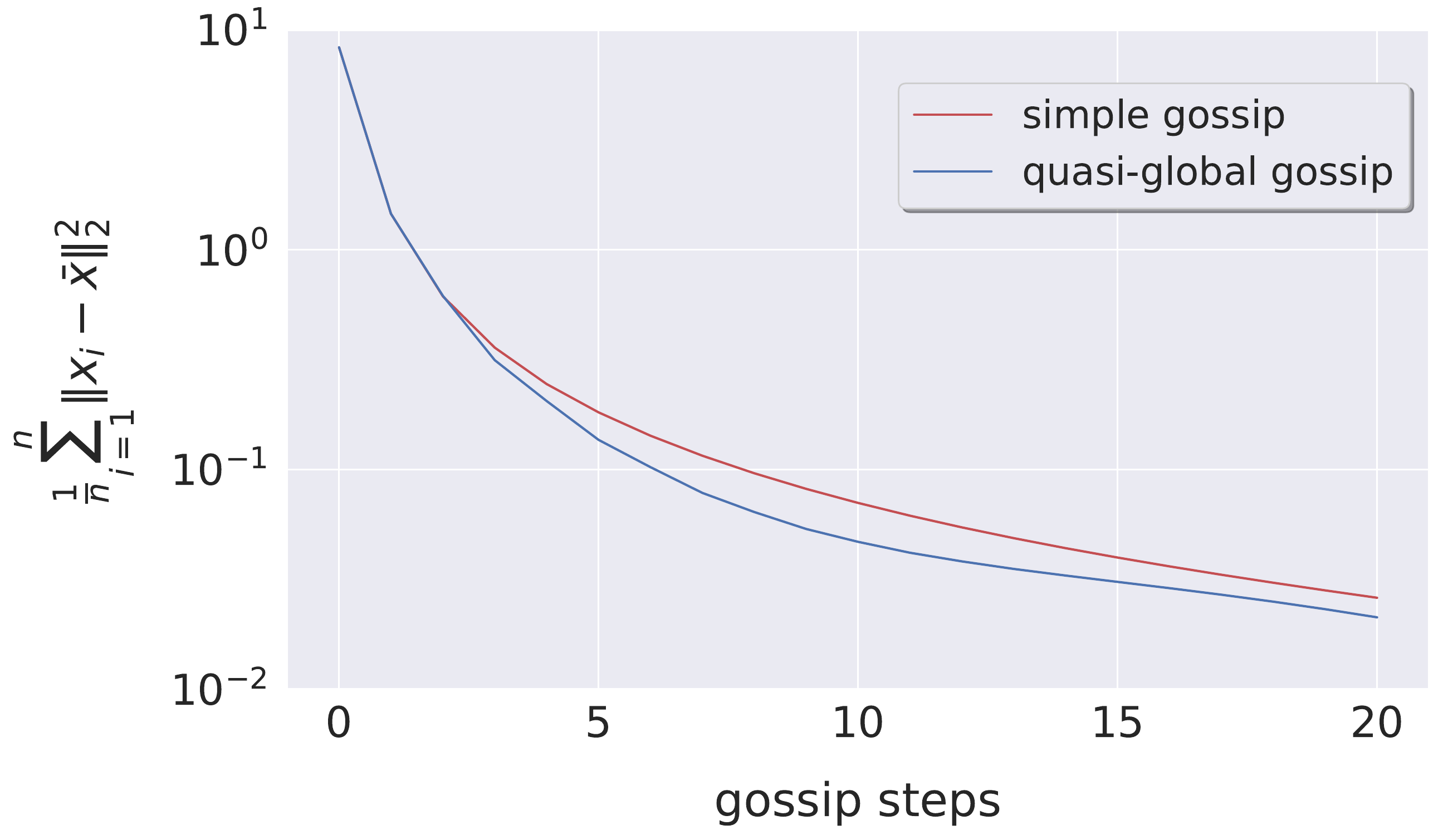}
		\label{fig:understanding_via_consensus_averaging_n16_128}
	}
	\subfigure[\small
		social network, $n\!=\!32$.
	]{
		\includegraphics[width=.315\textwidth,]{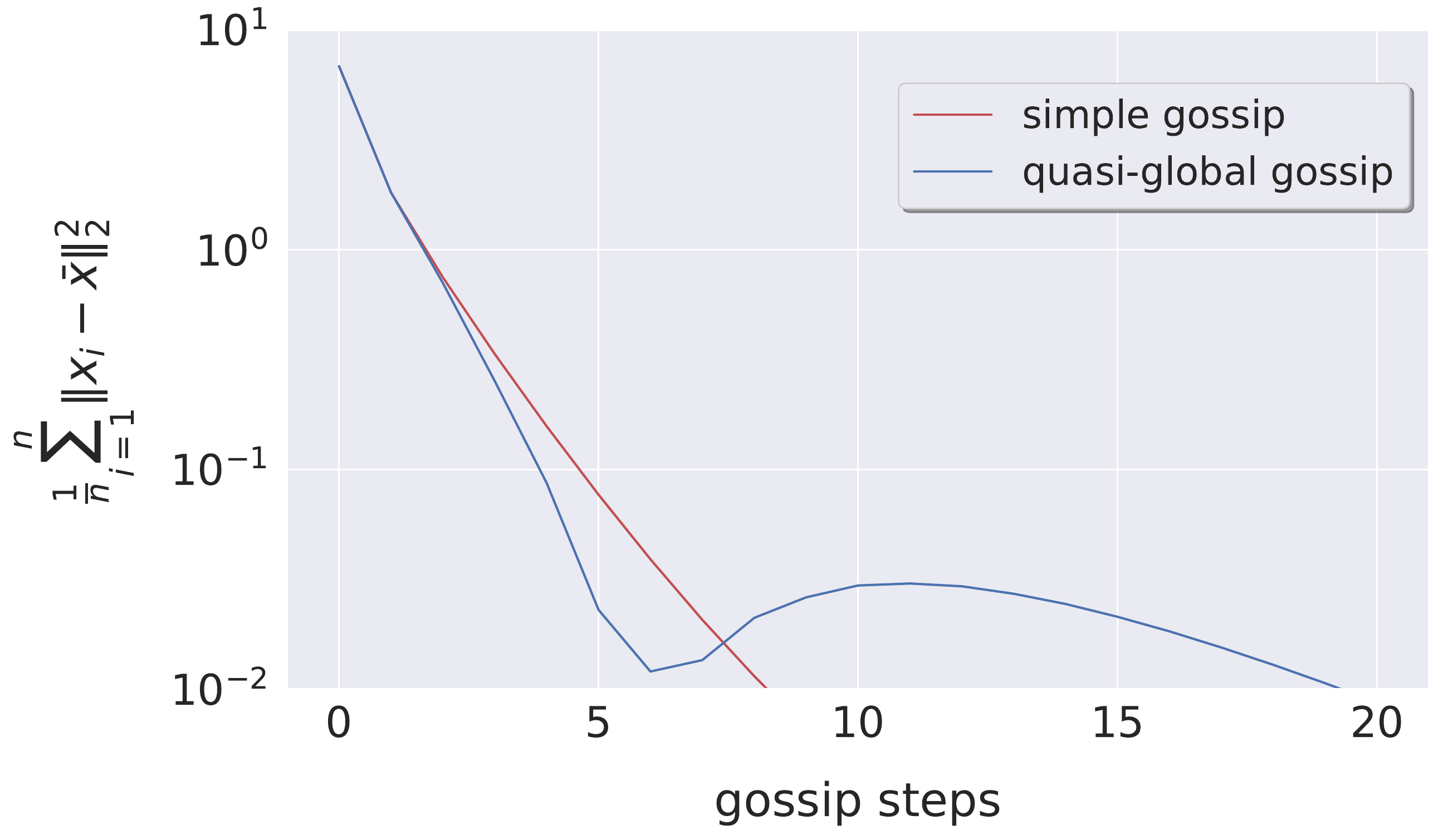}
		\label{fig:understanding_via_consensus_averaging_n16_128}
	}
	\hfill
	\vspace{-1em}
	\caption{\small
		More results on understanding \algoptsgdm through the aspect of distributed consensus averaging problem
		on different communication topologies and scales.
		\algoptsgdm without gradient update step (as in \eqref{eq:update_without_stochastic_gradients})
		still presents faster convergence (to a relative high precision) than the normal gossip algorithm.
	}
	\label{fig:more_understanding_via_consensus_averaging}
\end{figure*}

\subsection{Results on 2D Illustration} \label{appendix:more_on_2D_illustration}
Following the 2D illustration in Figure~\ref{fig:2d_illustration_procedures},
we elaborate below the different choices of momentum factor for SGDm and \salgoptsgdm
in Figure~\ref{fig:more_2d_illustration_procedures}.
We can witness that the effectiveness of local momentum (oscillation) in SGDm
is always impacted by the data heterogeneity,
no matter the choices of momentum factor.
While for \salgoptsgdm, there exists a trade-off between stabilized optimization and fast convergence,
controlled by the momentum factor $\beta$.
Note that we always set $\mu := \beta$ in \salgoptsgdm, which may result in undesirable behavior.

\begin{figure}[!h]
	\centering
	\subfigure[\small
		SGDm, $\beta=0.25$.
	]{
		\includegraphics[width=.225\textwidth,]{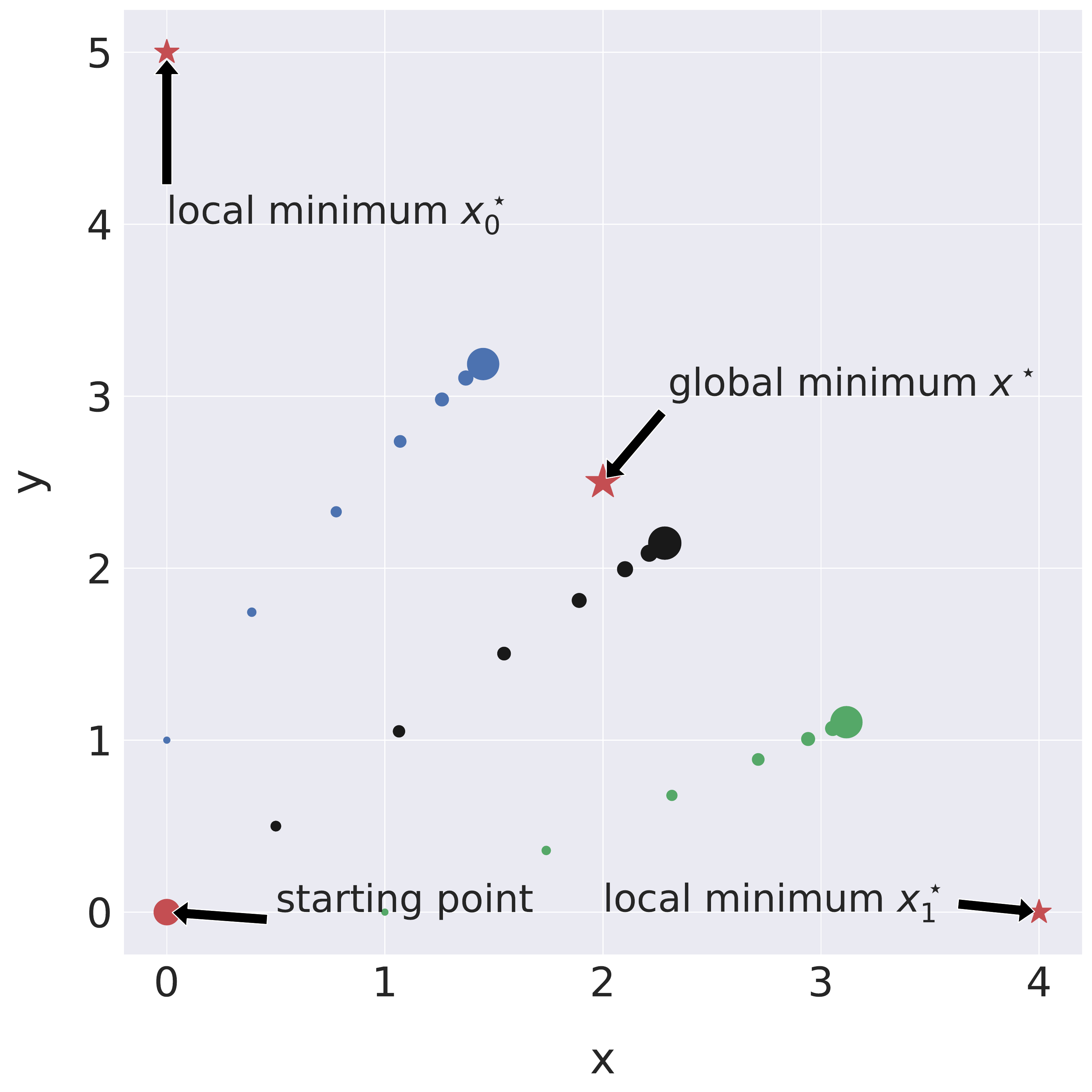}
	}
	\subfigure[\small
		\salgoptsgdm, $\beta=0.25$.
	]{
		\includegraphics[width=.225\textwidth,]{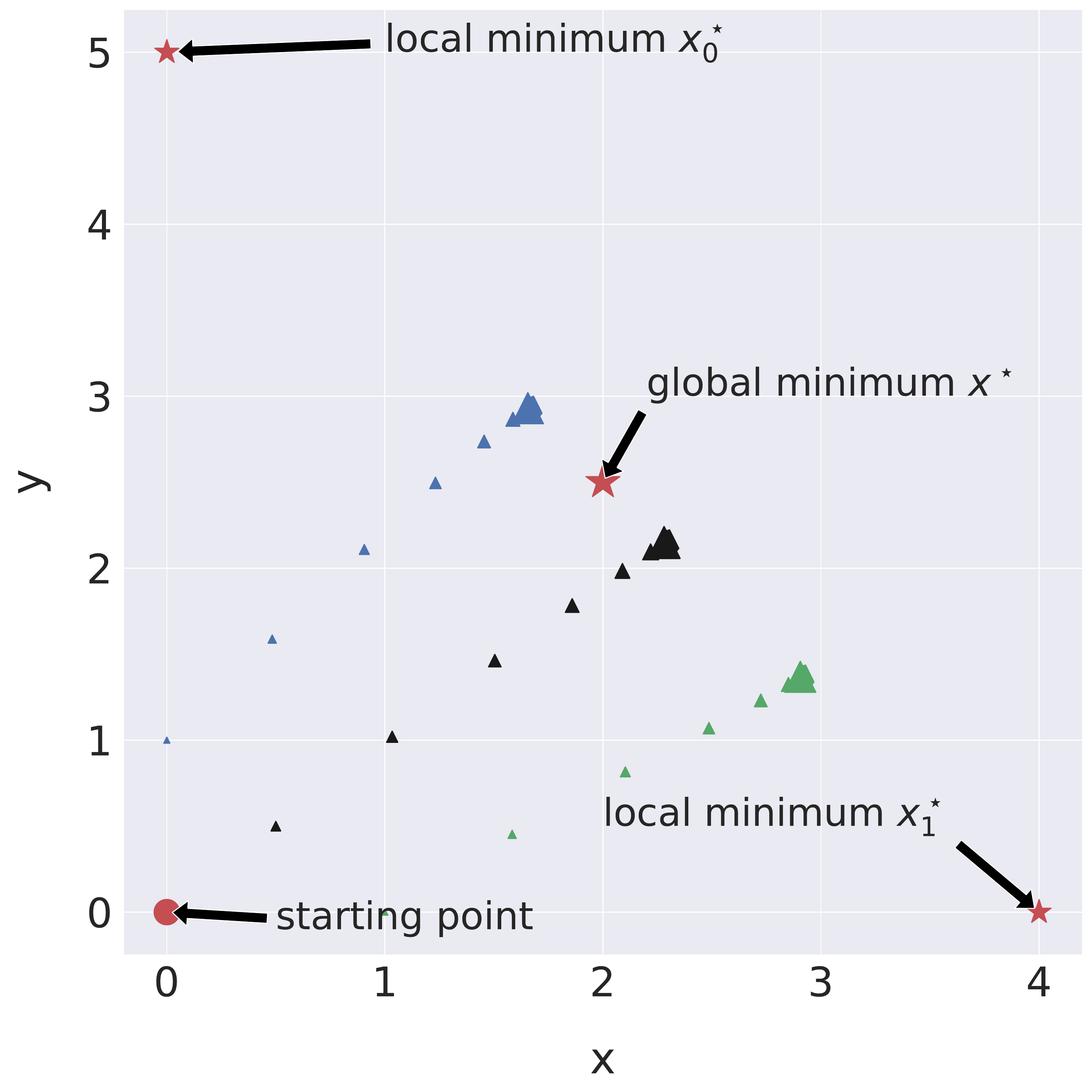}
	}
	\subfigure[\small
		SGDm, $\beta=0.5$.
	]{
		\includegraphics[width=.225\textwidth,]{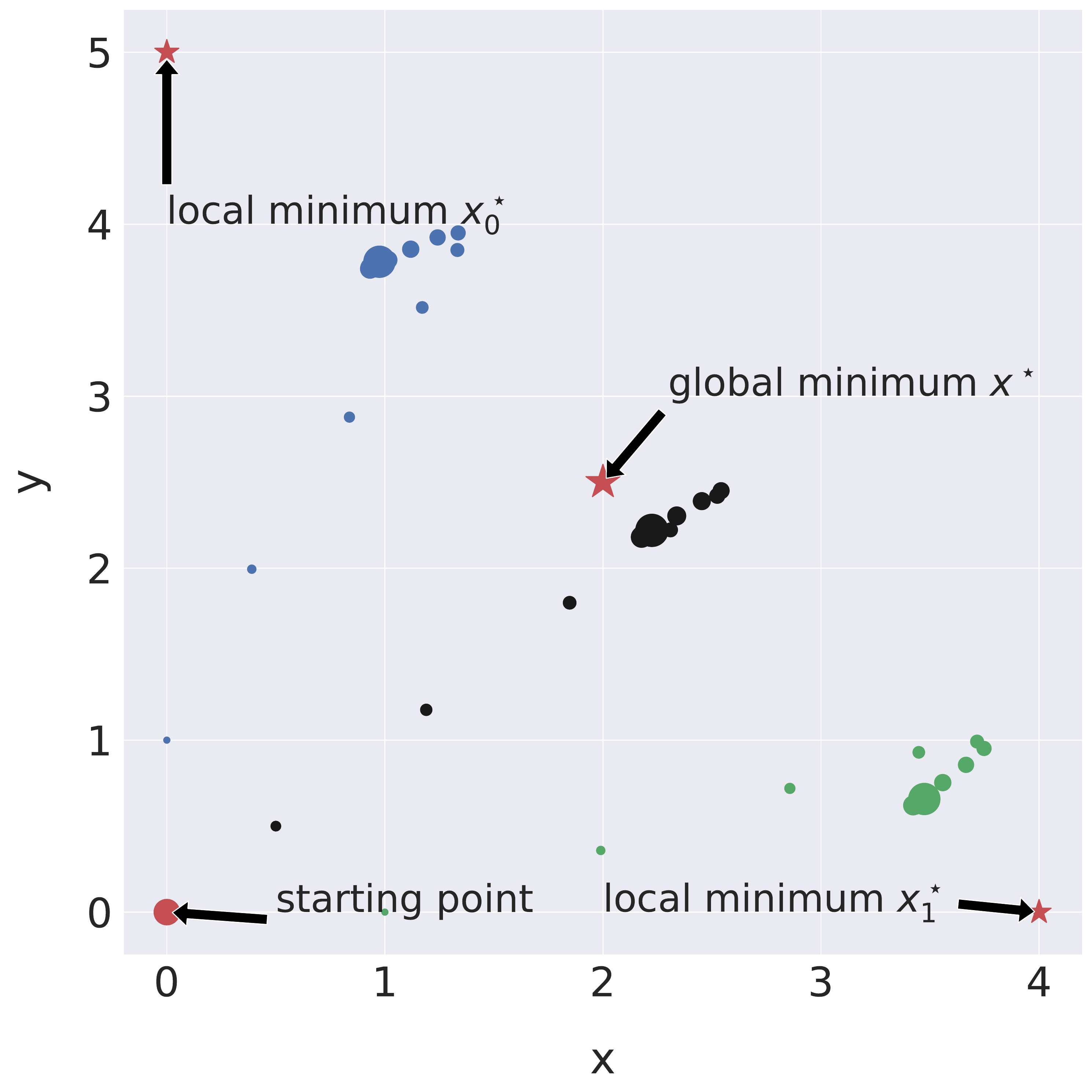}
	}
	\subfigure[\small
		\salgoptsgdm, $\beta=0.5$.
	]{
		\includegraphics[width=.225\textwidth,]{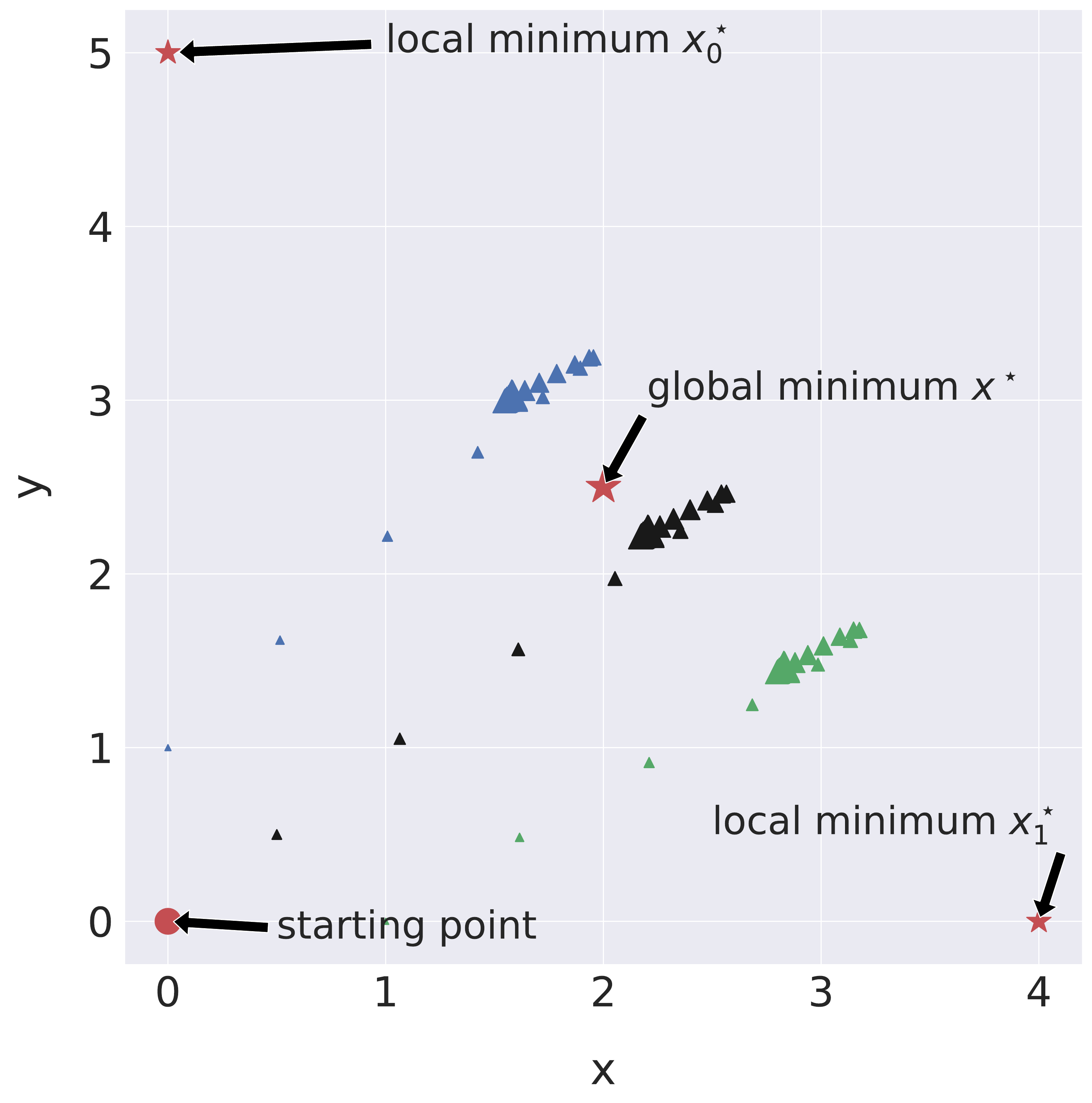}
	}
	\subfigure[\small
		SGDm, $\beta=0.75$.
	]{
		\includegraphics[width=.225\textwidth,]{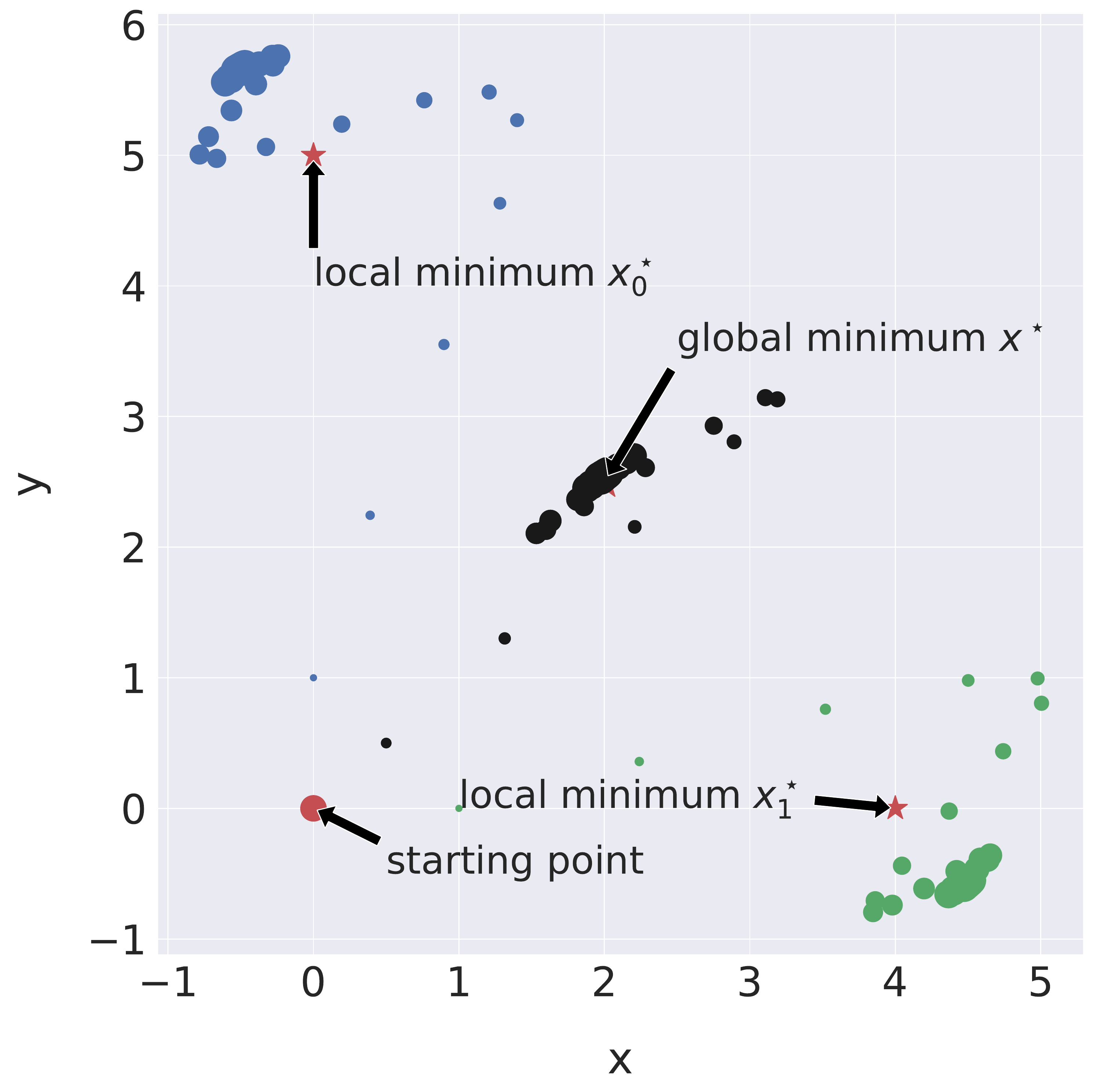}
	}
	\subfigure[\small
		\salgoptsgdm, $\beta=0.75$.
	]{
		\includegraphics[width=.225\textwidth,]{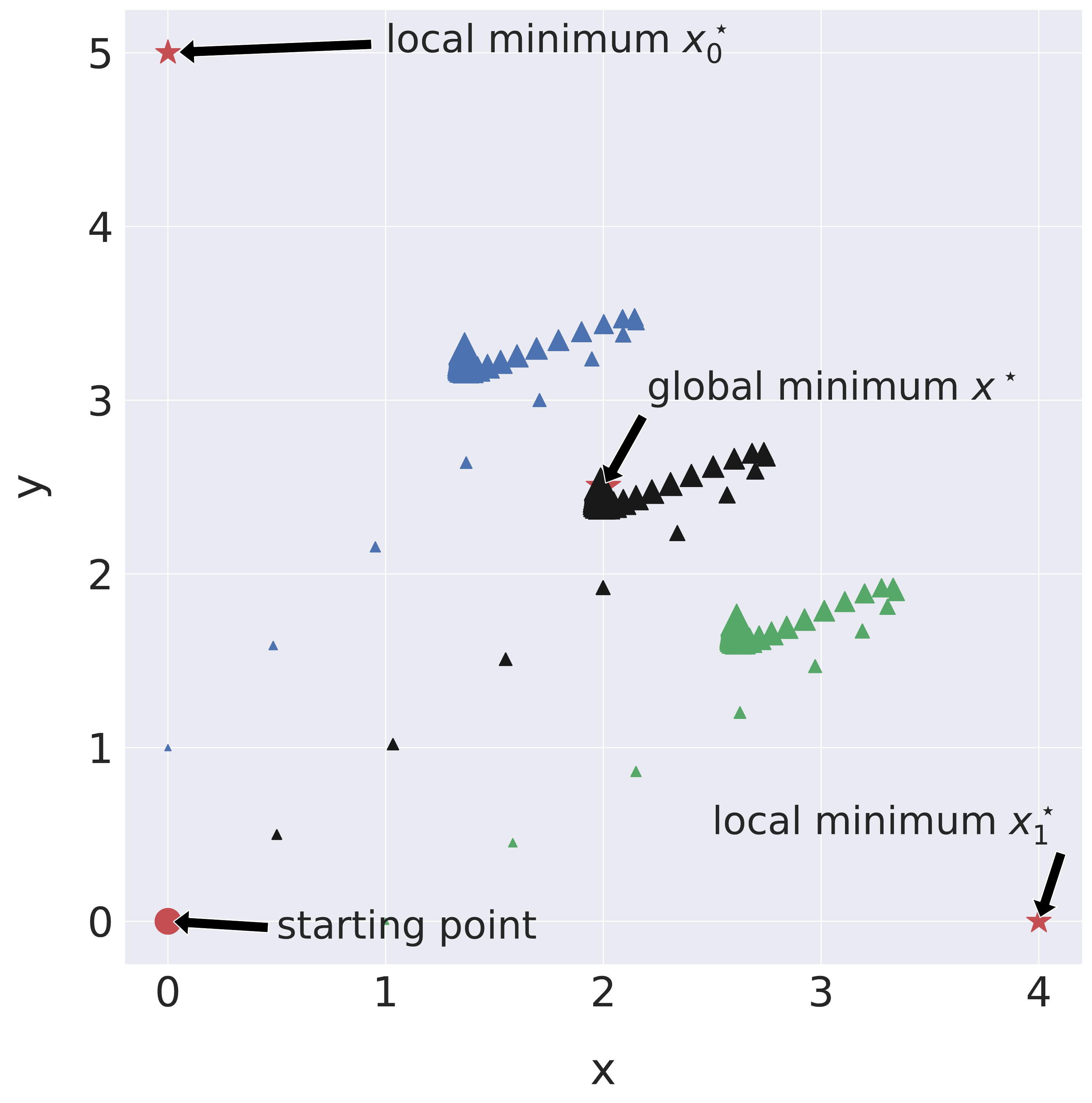}
	}
	\subfigure[\small
		SGDm, $\beta=0.90$.
	]{
		\includegraphics[width=.225\textwidth,]{figures/2d_illustration/procedure0.pdf}
	}
	\subfigure[\small
		\salgoptsgdm, $\beta=0.90$.
	]{
		\includegraphics[width=.225\textwidth,]{figures/2d_illustration/procedure1.pdf}
	}
	\subfigure[\small
		SGDm, $\beta=0.95$.
	]{
		\includegraphics[width=.225\textwidth,]{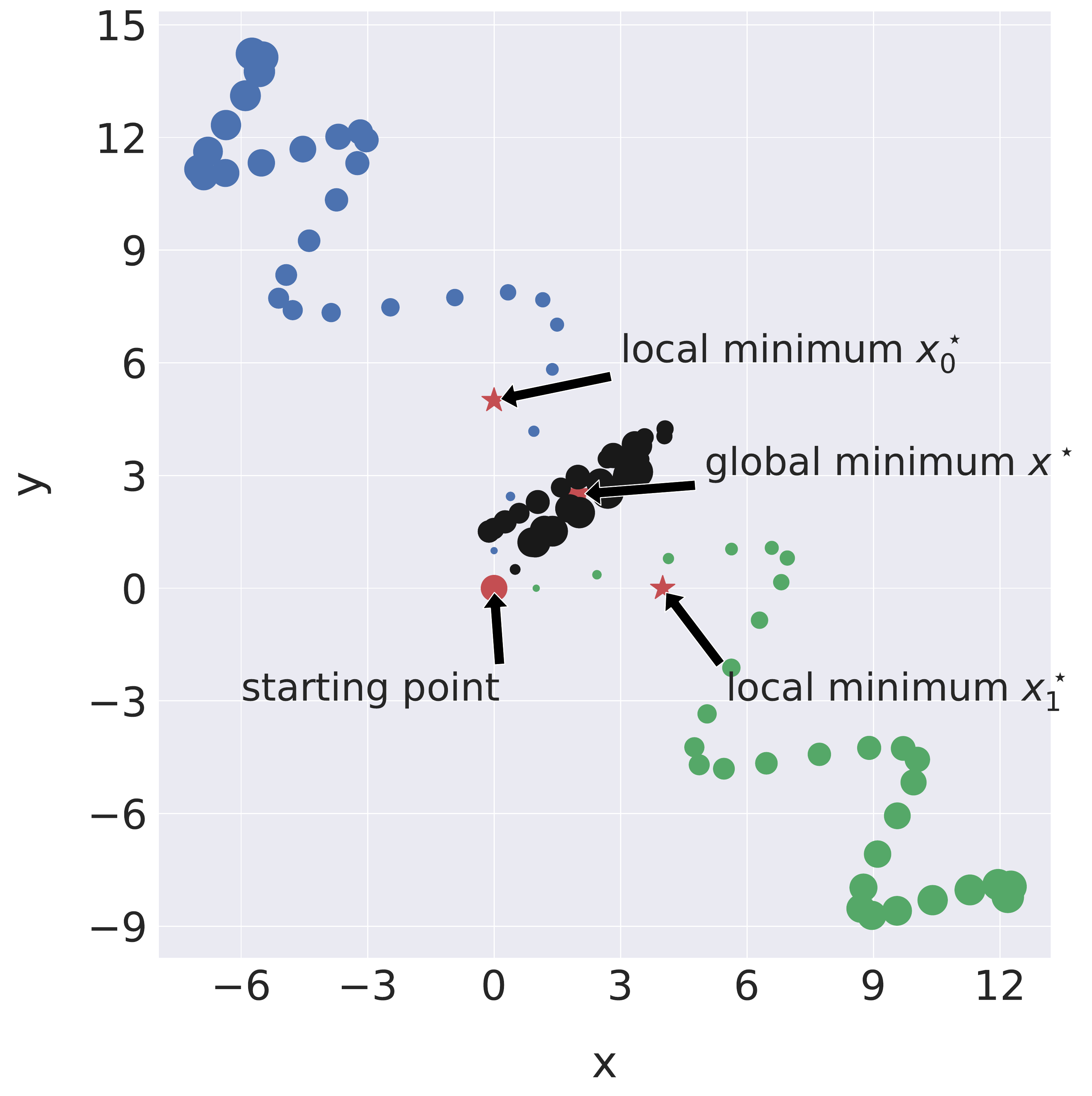}
	}
	\subfigure[\small
		\salgoptsgdm, $\beta=0.95$.
	]{
		\includegraphics[width=.225\textwidth,]{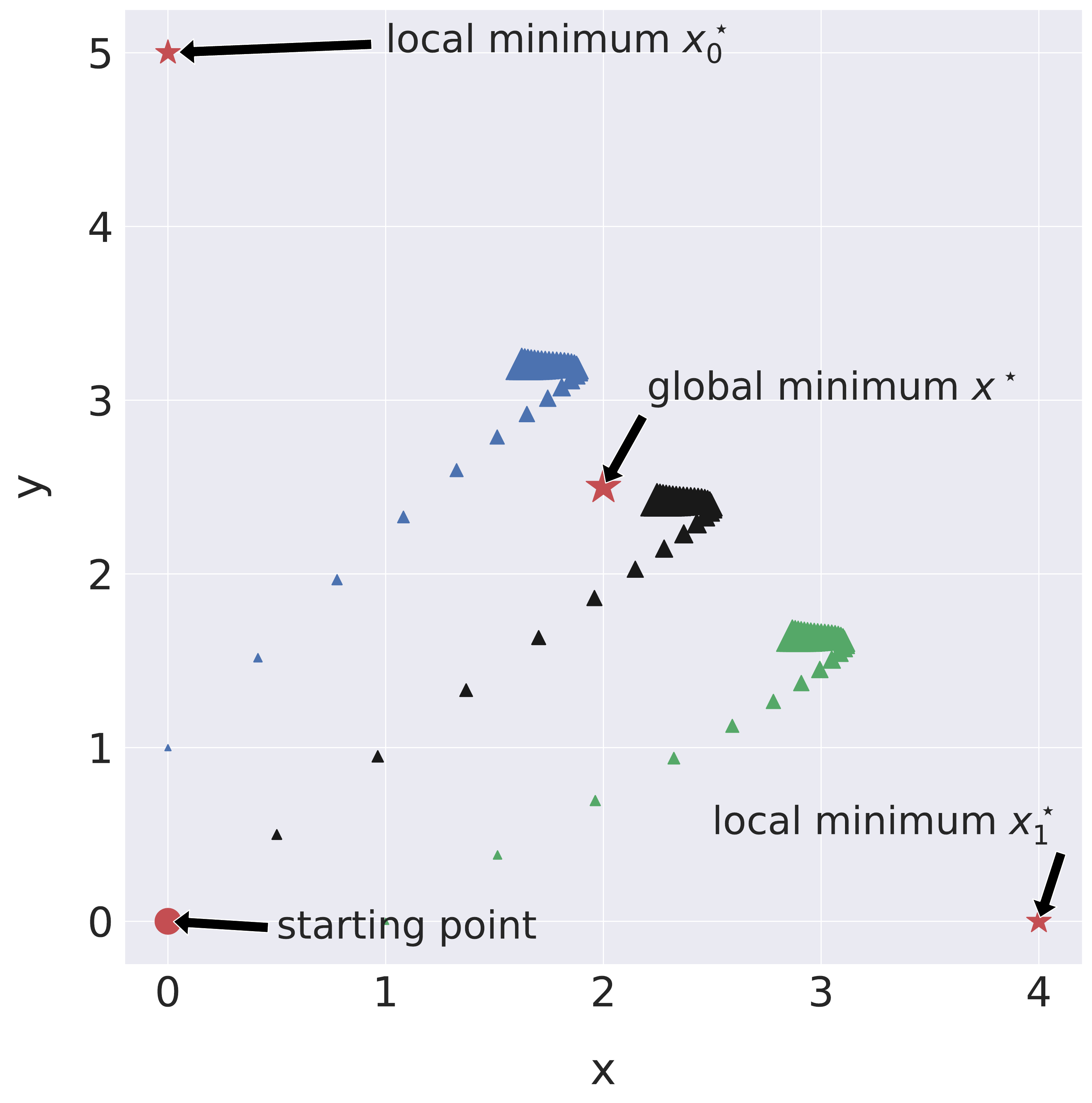}
	}
	\subfigure[\small
		SGDm, $\beta=0.99$.
	]{
		\includegraphics[width=.225\textwidth,]{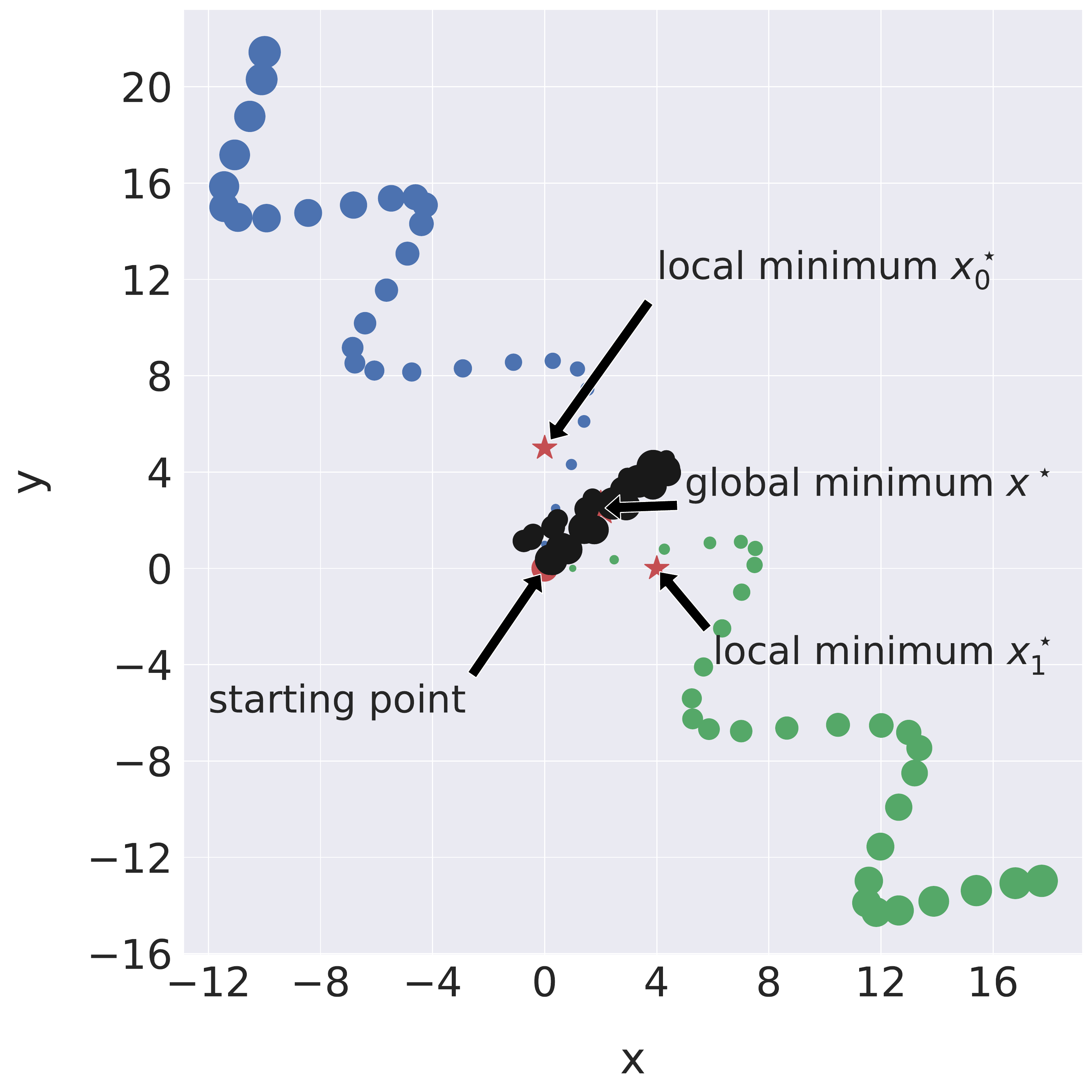}
	}
	\subfigure[\small
		\salgoptsgdm, $\beta=0.99$.
	]{
		\includegraphics[width=.225\textwidth,]{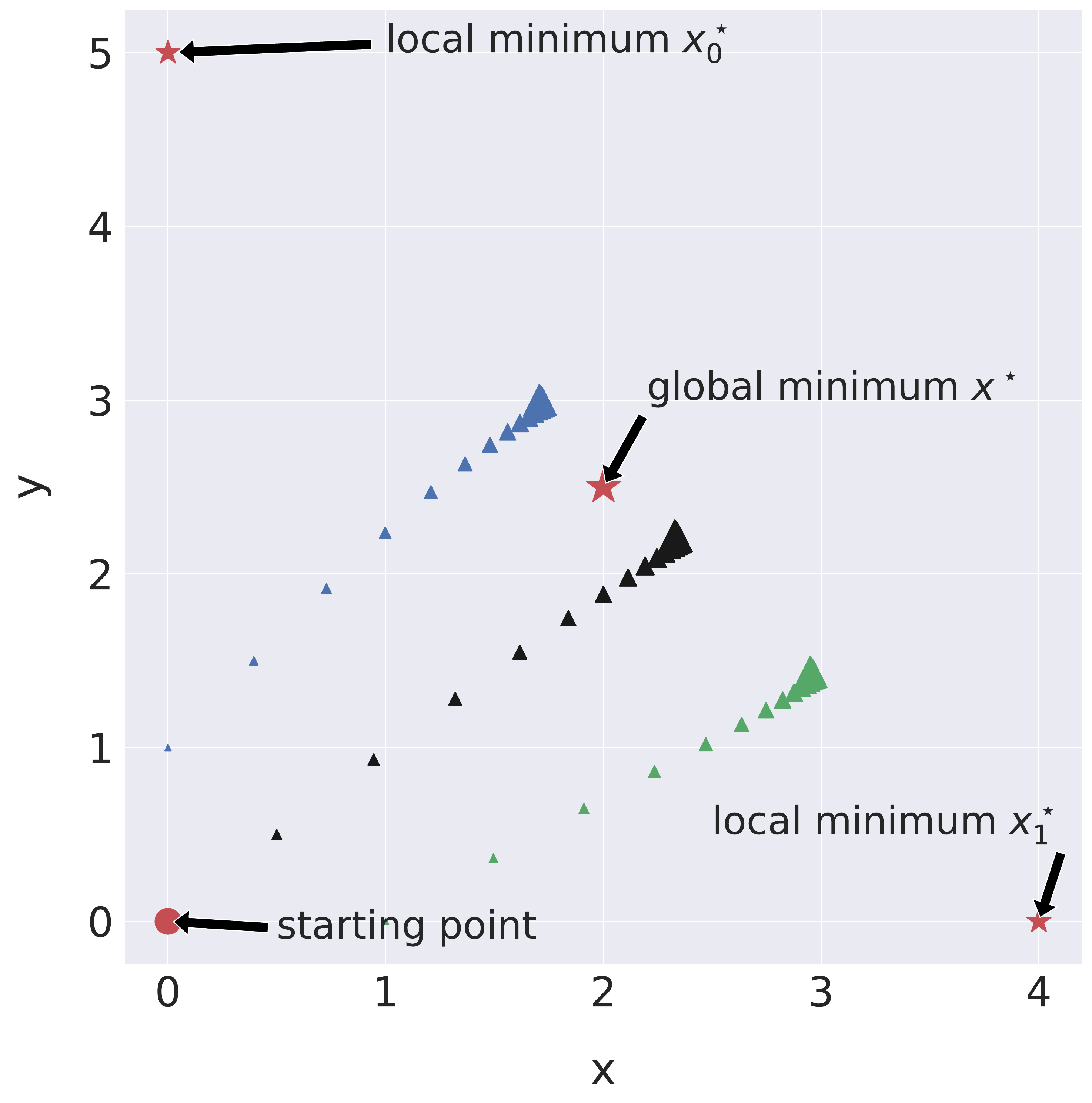}
	}
	\vspace{-1em}
	\caption{\small
		\textbf{The ineffectiveness of local momentum acceleration under heterogeneous data setup}:
		the local momentum buffer accumulates ``biased'' gradients,
		causing unstable and oscillation behaviors.
		The gradient is estimated by
		the direction from a given model to the local minimum
		with a constant update magnitude.
		The size of marker will increase by the number of update steps;
		colors blue and green indicate the local models of two workers (after performing local update),
		while color black is the synchronized global model.
		Uniform weight averaging is performed after each update step,
		and the new gradients will be computed on the averaged model.
	}
	\label{fig:more_2d_illustration_procedures}
\end{figure}

\clearpage
\subsection{Understanding \algoptsgdm and \algoptsgdmn on the Single Worker Case}
\label{appendix:understanding_single_worker}
Recall that the single worker case of \algoptsgdm and \algoptsgdmn
refers to \salgoptsgdm and \salgoptsgdmn.

Figure~\ref{fig:understanding_optsgd_on_n1_complete}
studies the learning behavior (learning curves for both training loss and top-1 test accuracy, as well as final best test accuracy)
of \salgoptsgdm and \salgoptsgdmn on two different normalization methods (BN and GN) for ResNet-20 on CIFAR-10.
In general Nesterov momentum variants outperforms that of HeavyBall momentum,
and we can witness a larger performance gain when the optimization is challenging
(e.g.\ in the of using GN replacement).

Figure~\ref{fig:understanding_optsgd_on_n1_and_wo_weight_decay}
further investigates the impact of weight decay on Nesterov momentum variants.
Excluding weight decay from the training procedure is detrimental to the final generalization performance.
We also notice larger benefits of \salgoptsgdmn when the optimization procedure is fragile/unstable.

Figure~\ref{fig:n1_resnet20_cifar10_understanding_from_deep_learning_aspect_complete}
in addition illustrates the curves of weight norm and effective step-size during the optimization procedure,
to interpret the potential causes of the performance gain.

\begin{figure*}[!h]
	\centering
	\vspace{-1em}
	\subfigure[\small
		ResNet-BN-20 (training loss)
	]{
		\includegraphics[width=.475\textwidth,]{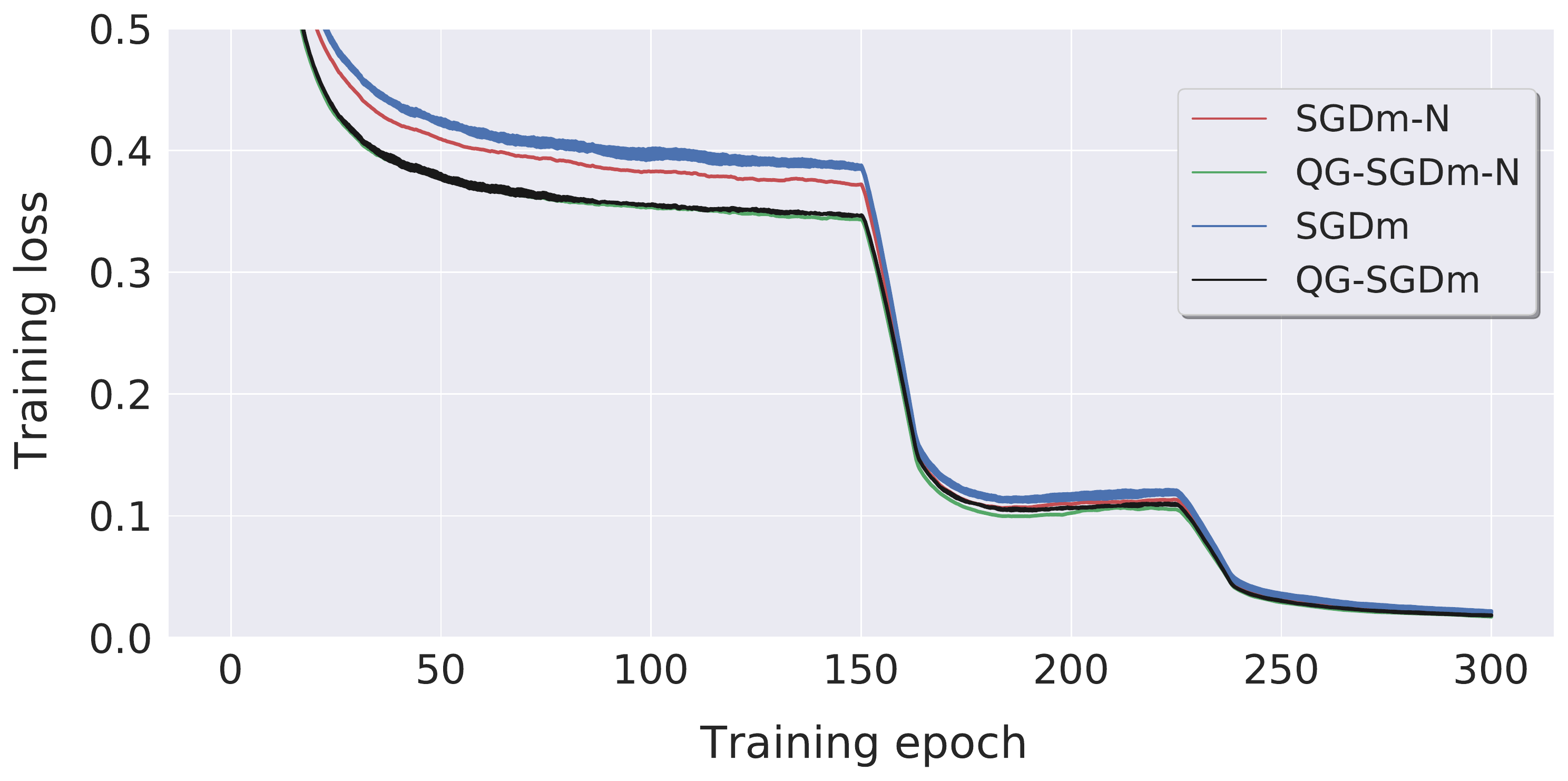}
		\label{fig:n1_resnet20bn_cifar10_understanding_method_itself_on_training_loss_complete}
	}
	\hfill
	\subfigure[\small
		ResNet-BN-20 (test top-1)
	]{
		\includegraphics[width=.475\textwidth,]{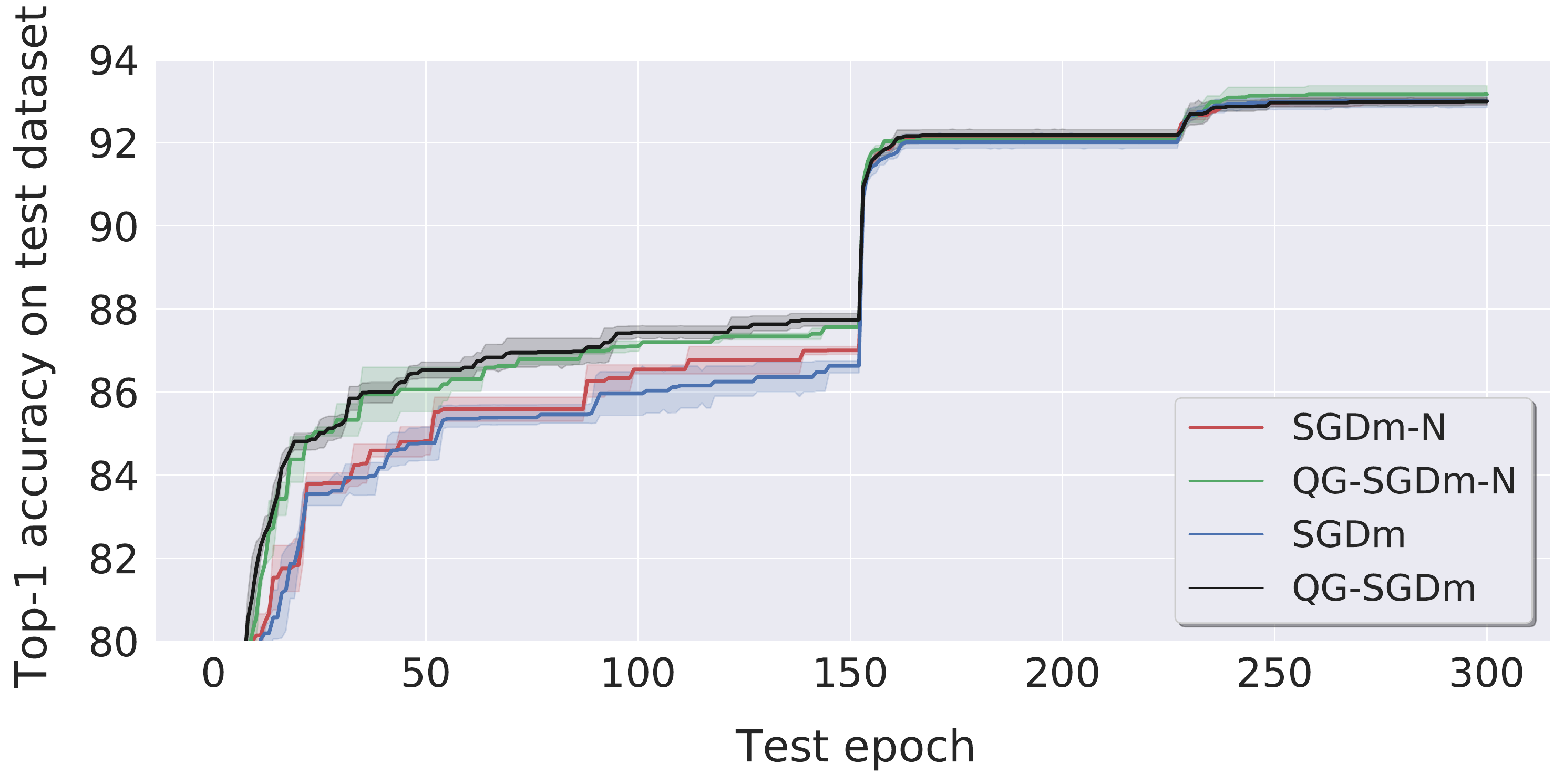}
		\label{fig:n1_resnet20bn_cifar10_understanding_method_itself_complete}
	}
	\hfill
	\subfigure[\small
		ResNet-GN-20 (training loss)
	]{
		\includegraphics[width=.475\textwidth,]{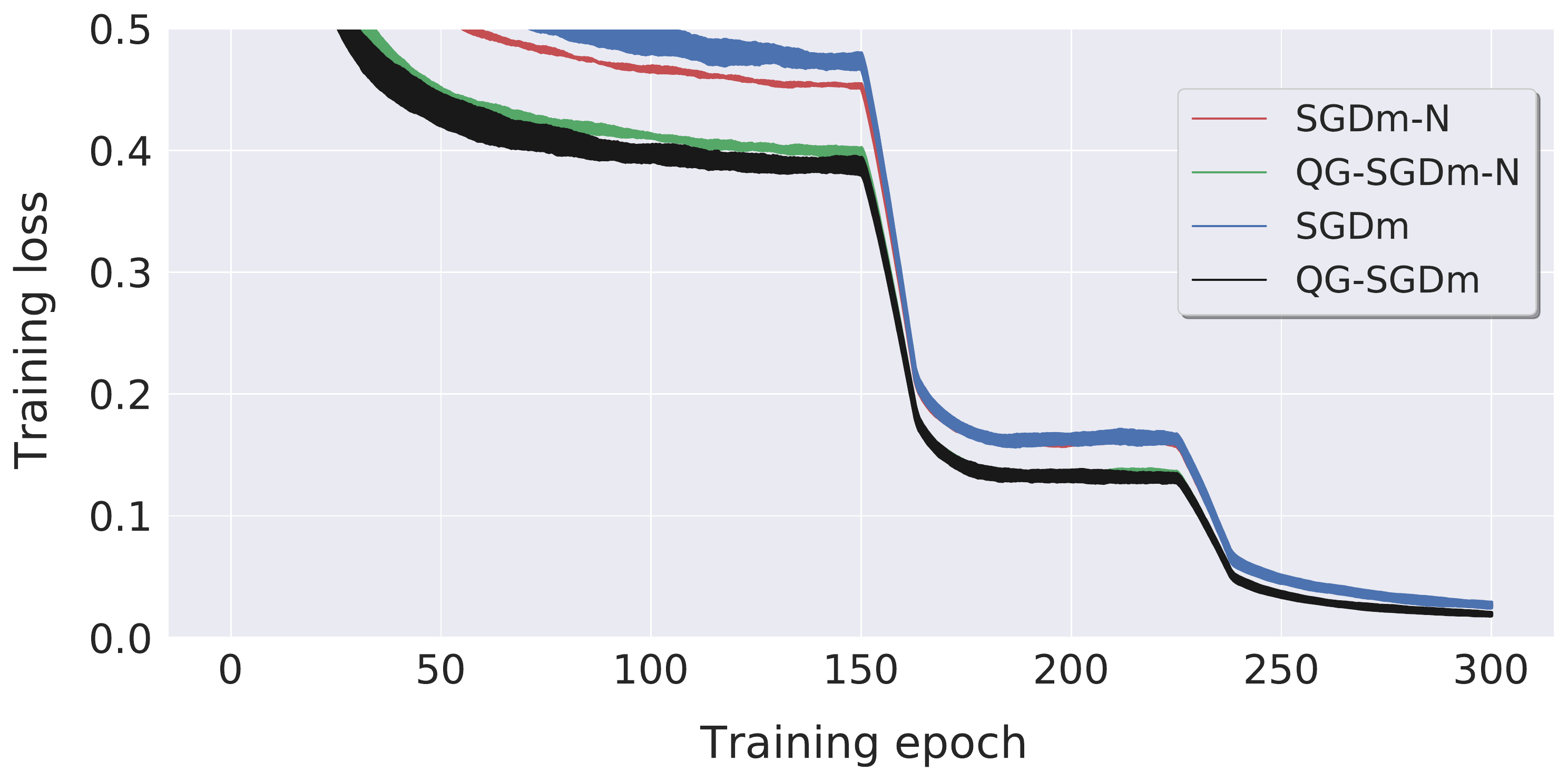}
		\label{fig:n1_resnet20gn_cifar10_understanding_method_itself_on_training_loss_complete}
	}
	\hfill
	\subfigure[\small
		ResNet-GN-20 (test top-1)
	]{
		\includegraphics[width=.475\textwidth,]{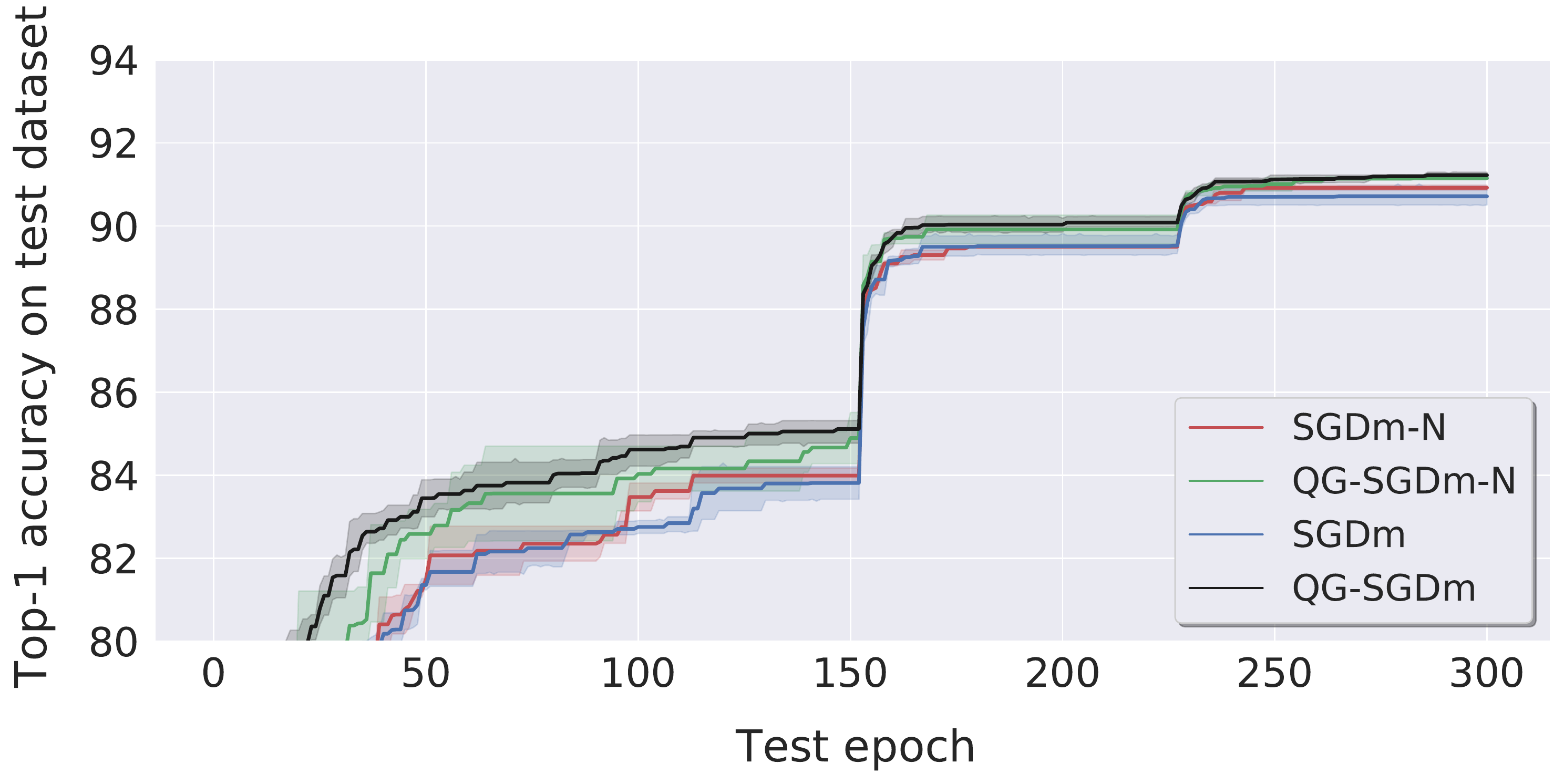}
		\label{fig:n1_resnet20gn_cifar10_understanding_method_itself_complete}
	}
	\resizebox{.6\textwidth}{!}{%
		\begin{tabular}{c|cc|cc}
			\toprule
			   & SGDm-N           & \salgoptsgdmn             & SGDm             & \salgoptsgdm              \\ \midrule
			BN & $93.01 \pm 0.12$ & $\textbf{93.17} \pm 0.30$ & $92.90 \pm 0.15$ & $\textbf{92.95} \pm 0.13$ \\
			GN & $90.92 \pm 0.08$ & $\textbf{91.15} \pm 0.06$ & $90.61 \pm 0.15$ & $\textbf{91.23} \pm 0.01$ \\ \bottomrule
		\end{tabular}
	}
	\caption{\small
		Understanding the learning behavior of \salgoptsgdm and \salgoptsgdmn,
		for training ResNet-20 on CIFAR-10 with mini-batch size of $32$.
	}
	\label{fig:understanding_optsgd_on_n1_complete}
\end{figure*}

\begin{figure*}[!h]
	\centering
	\vspace{-1em}
	\subfigure[\small
		ResNet-BN-20.
	]{
		\includegraphics[width=.45\textwidth,]{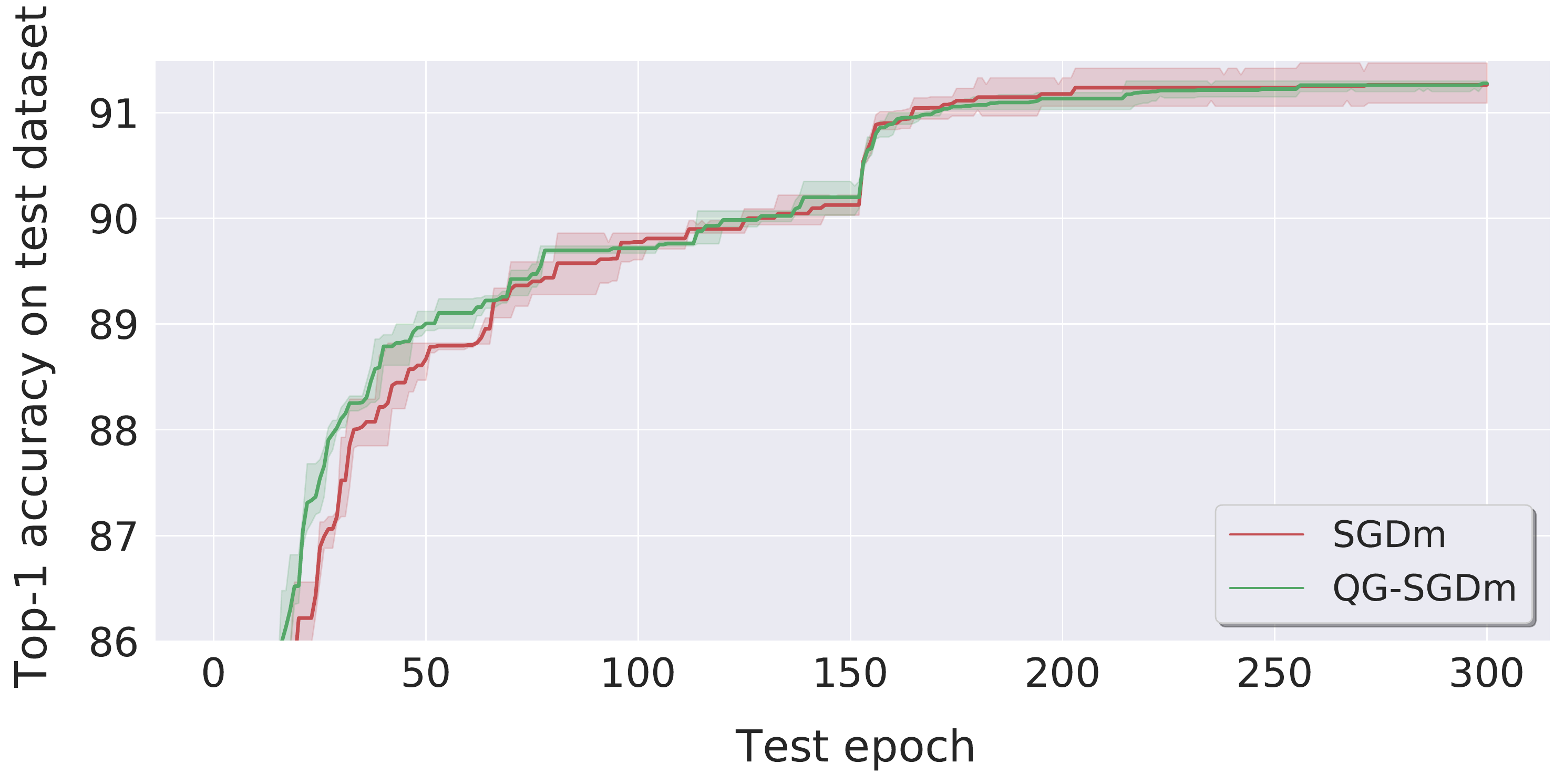}
		\label{fig:n1_resnet20bn_cifar10_understanding_method_itself_w_nesterov_momentum_wo_weight_decay}
	}
	\hfill
	\subfigure[\small
		ResNet-GN-20.
	]{
		\includegraphics[width=.45\textwidth,]{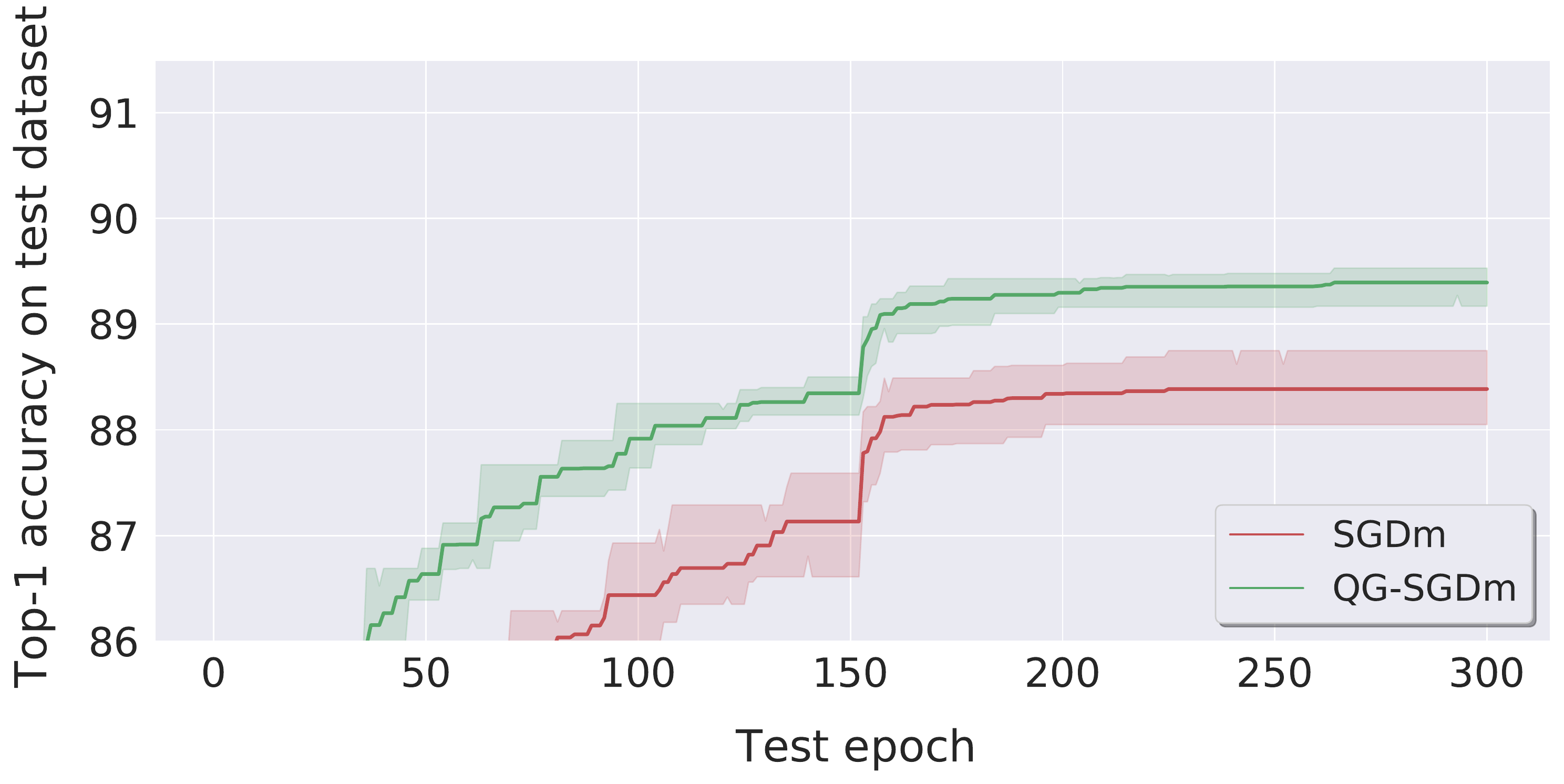}
		\label{fig:n1_resnet20gn_cifar10_understanding_method_itself_w_nesterov_momentum_wo_weight_decay}
	}
	\resizebox{.6\textwidth}{!}{%
		\begin{tabular}{c|cc|cc}
			\toprule
			& \multicolumn{2}{c|}{w/ weight decay} & \multicolumn{2}{c}{w/o weight decay} \\ \cmidrule(lr){2-3} \cmidrule(lr){4-5}
			   & SGDm-N           & \salgoptsgdmn             & SGDm-N           & \salgoptsgdmn             \\ \midrule
			BN & $93.01 \pm 0.12$ & $\textbf{93.17} \pm 0.30$ & $91.26 \pm 0.19$ & $\textbf{91.28} \pm 0.03$ \\
			GN & $90.92 \pm 0.08$ & $\textbf{91.15} \pm 0.06$ & $88.39 \pm 0.35$ & $\textbf{89.39} \pm 0.20$ \\ \bottomrule
		\end{tabular}
	}
	\vspace{-1em}
	\caption{\small
		Understanding the impact of weight decay on \salgoptsgdmn,
		for training ResNet-20 on CIFAR-10 with mini-batch size of $32$,
	}
	\label{fig:understanding_optsgd_on_n1_and_wo_weight_decay}
\end{figure*}

\begin{figure*}[!h]
	\centering
	\vspace{-1em}
	\subfigure[\small
		Weight norm $ \norm{ \xx_t }_2 $.
	]{
		\includegraphics[width=1.\textwidth,]{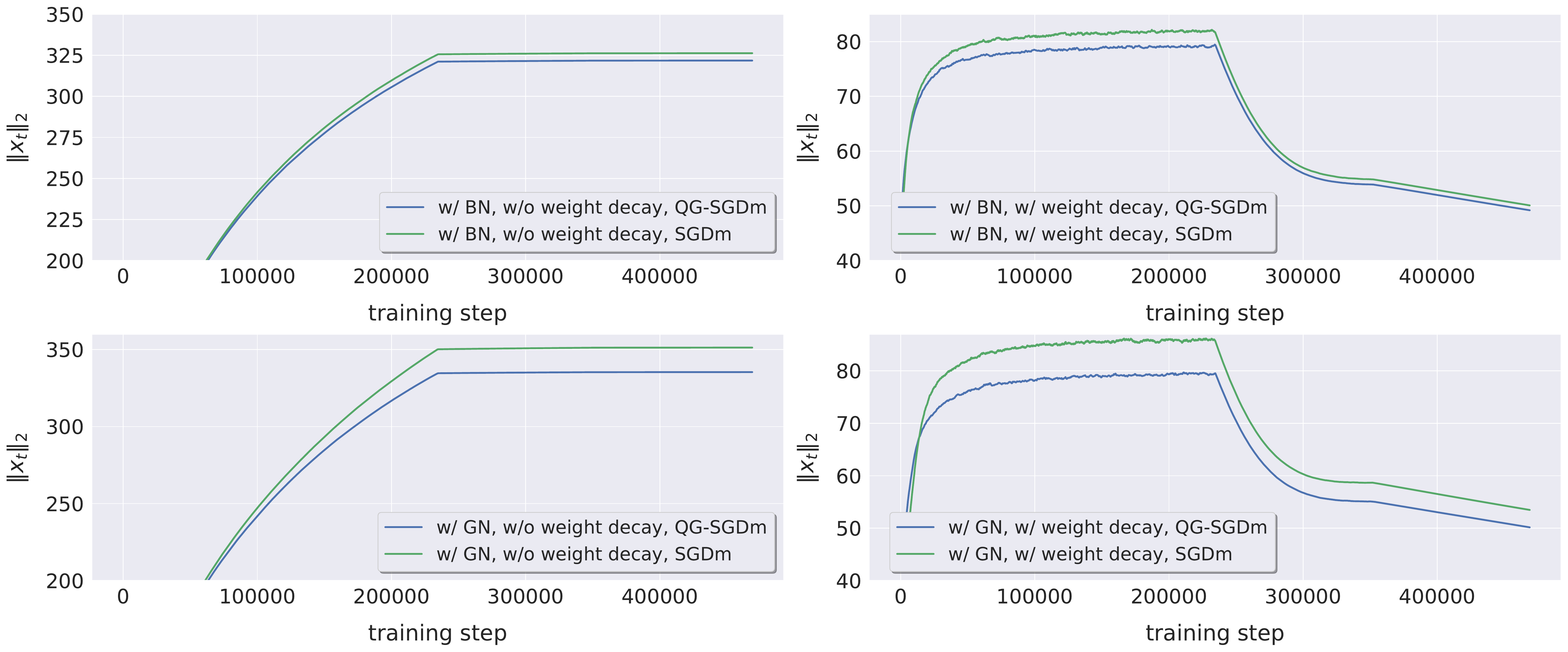}
		\label{fig:n1_resnet20_cifar10_understanding_weight_norm_heavyball_momentum_complete}
	}
	\hfill
	\subfigure[\small
		Effective stepsize $ \frac{\eta}{\norm{ \xx_t }_2^2} $.
	]{
		\includegraphics[width=1.\textwidth,]{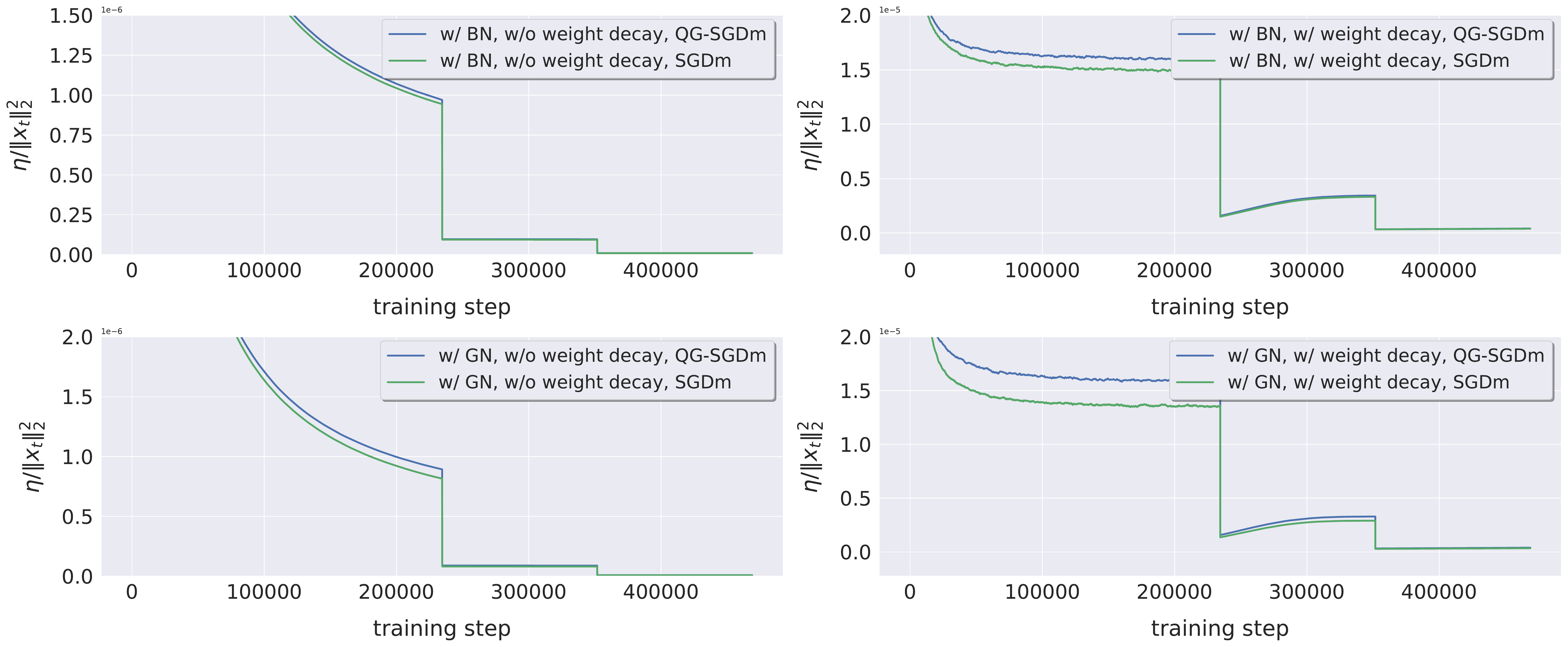}
		\label{fig:n1_resnet20_cifar10_understanding_effective_stepsize_heavyball_momentum_complete}
	}
	\caption{\small
		Understanding \salgoptsgdm through the lens of the effective step-size,
		for training ResNet-20 on CIFAR-10 with mini-batch size of $32$.
	}
	\label{fig:n1_resnet20_cifar10_understanding_from_deep_learning_aspect_complete}
\end{figure*}

\clearpage
\subsection{Understanding \algoptsgdm on the Single Worker Case via Toy Function} \label{appendix:toy_problem_trajectory}
Similar to~\citet{lucas2018aggregated}, we first optimize the Rosenbrock function,
defined as $f(x, y) = (y - x^2)^2 + 100 (x - 1)^2$.

Figure~\ref{fig:understanding_on_new_2d_toy_example}
illustrates the stabilized optimization trajectory in \salgoptsgdm (much less oscillation than SGDm).

\begin{figure*}[!h]
	\centering
	\vspace{-1em}
	\subfigure[\small $\beta = 0.9, \eta = 0.001$, initial point $(0, 0)$.]{
		\includegraphics[width=.475\textwidth,]{figures/understanding_via_new_toy_2d_example/trajectory_momentum09_lr0001_pos_bl.pdf}
		\label{fig:trajectory_momentum09_lr0001_pos_bl}
	}
	\hfill
	\subfigure[\small $\beta = 0.9, \eta = 0.001$, initial point $(2, 0)$.]{
		\includegraphics[width=.475\textwidth,]{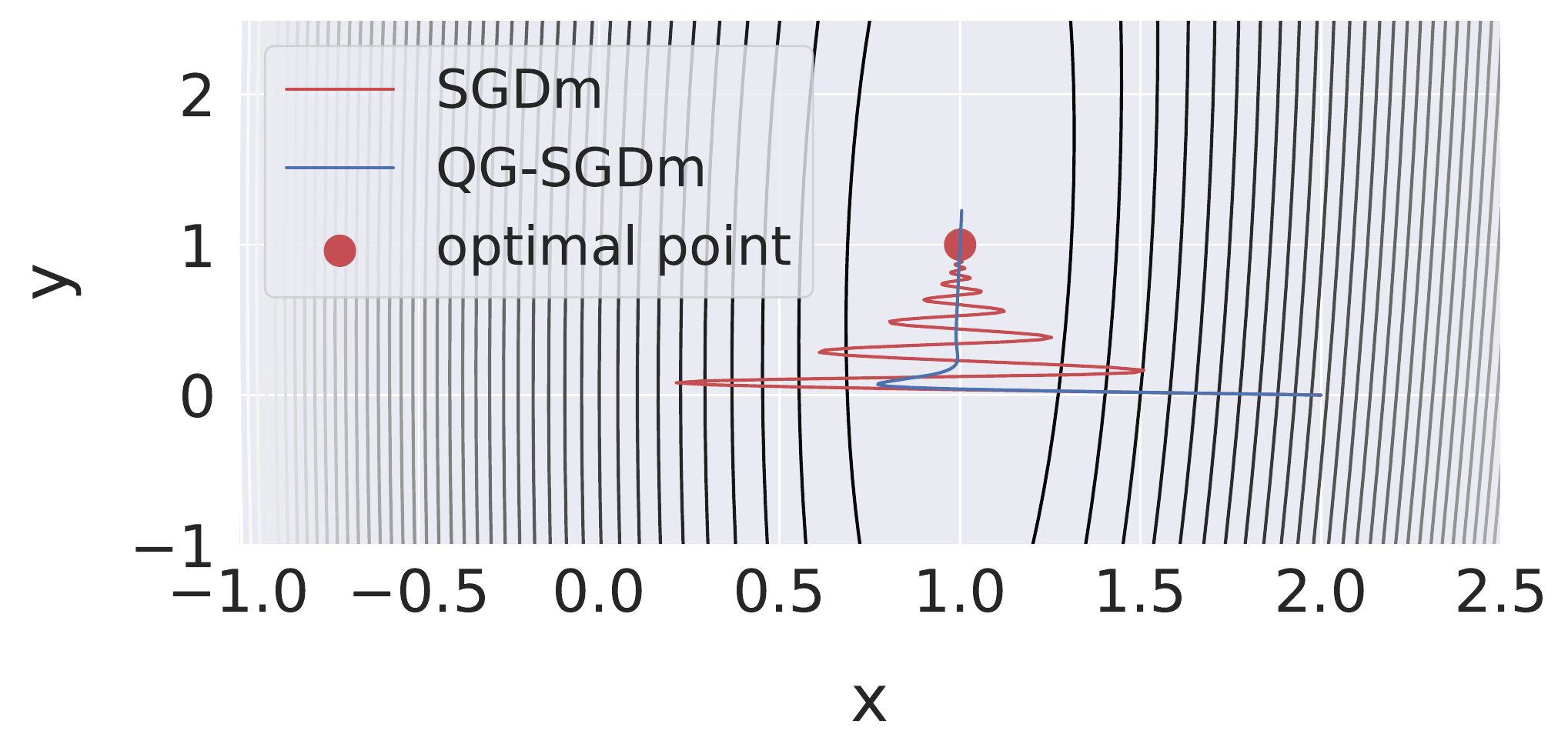}
		\label{fig:trajectory_momentum09_lr0001_pos_br}
	}
	\hfill
	\subfigure[\small $\beta = 0.9, \eta = 0.001$, initial point $(0, 2)$.]{
		\includegraphics[width=.475\textwidth,]{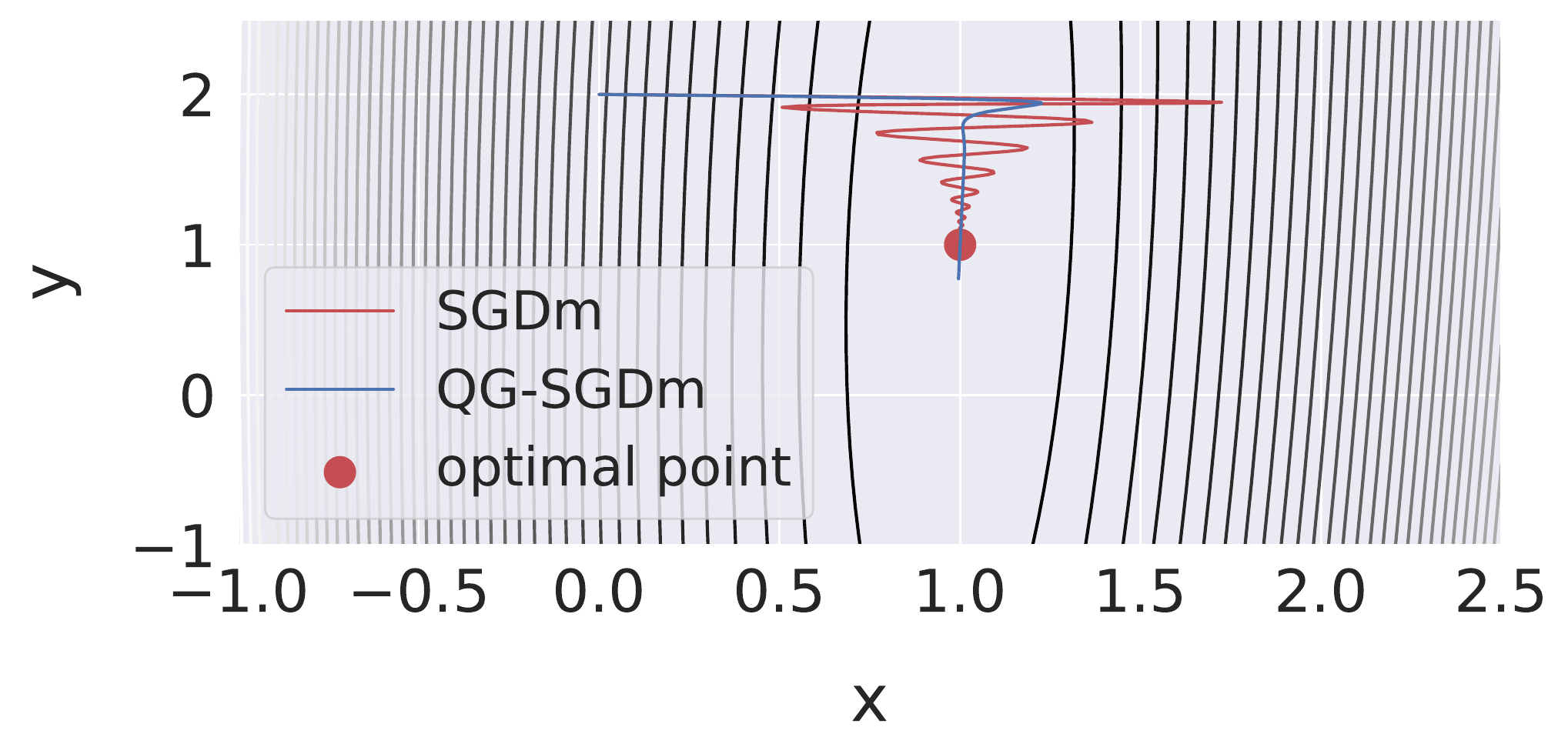}
		\label{fig:trajectory_momentum09_lr0001_pos_tl}
	}
	\hfill
	\subfigure[\small $\beta = 0.9, \eta = 0.001$, initial point $(2, 2)$.]{
		\includegraphics[width=.475\textwidth,]{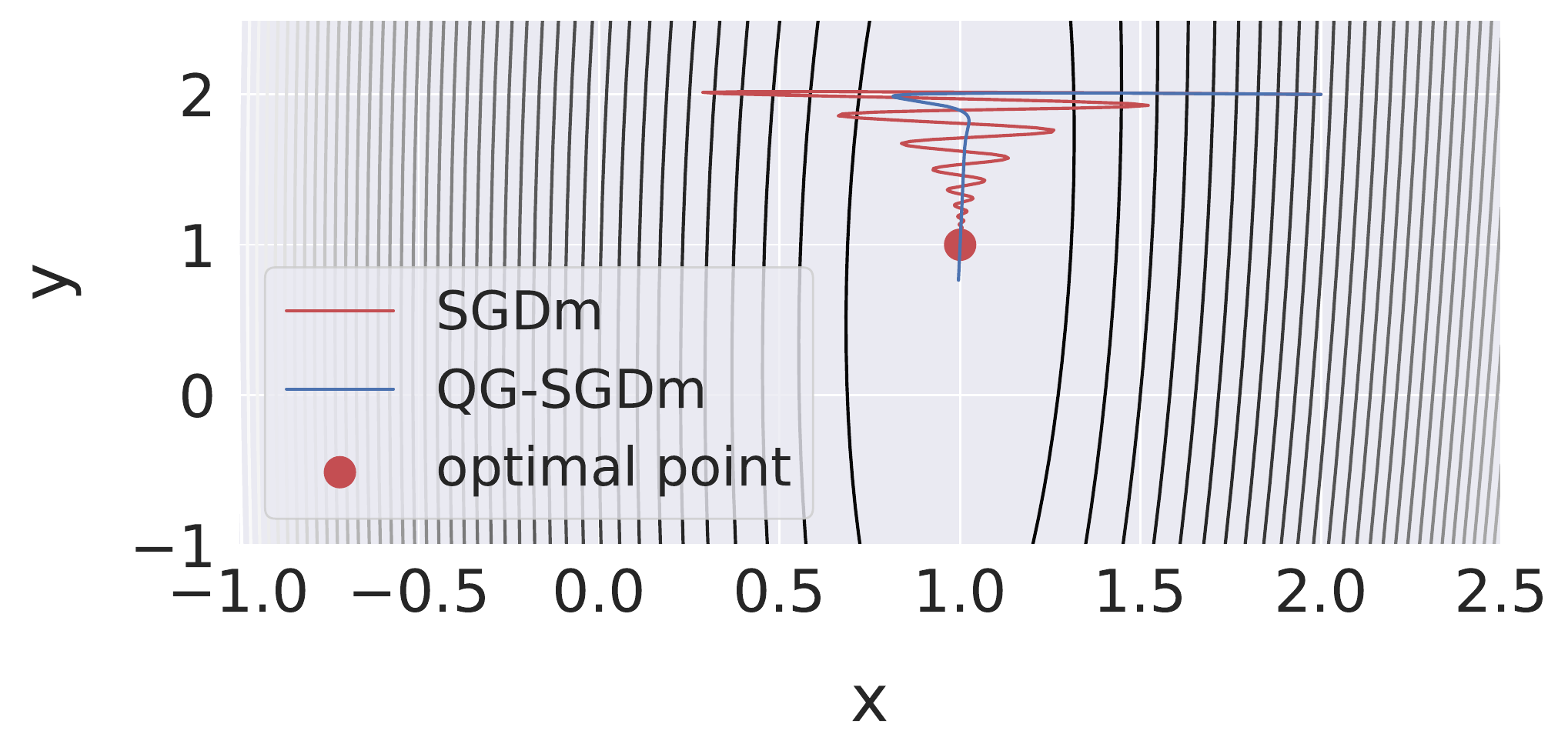}
		\label{fig:trajectory_momentum09_lr0001_pos_tr}
	}
	\hfill
	\subfigure[\small $\beta = 0.99, \eta = 0.001$, initial point $(0, 0)$.]{
		\includegraphics[width=.475\textwidth,]{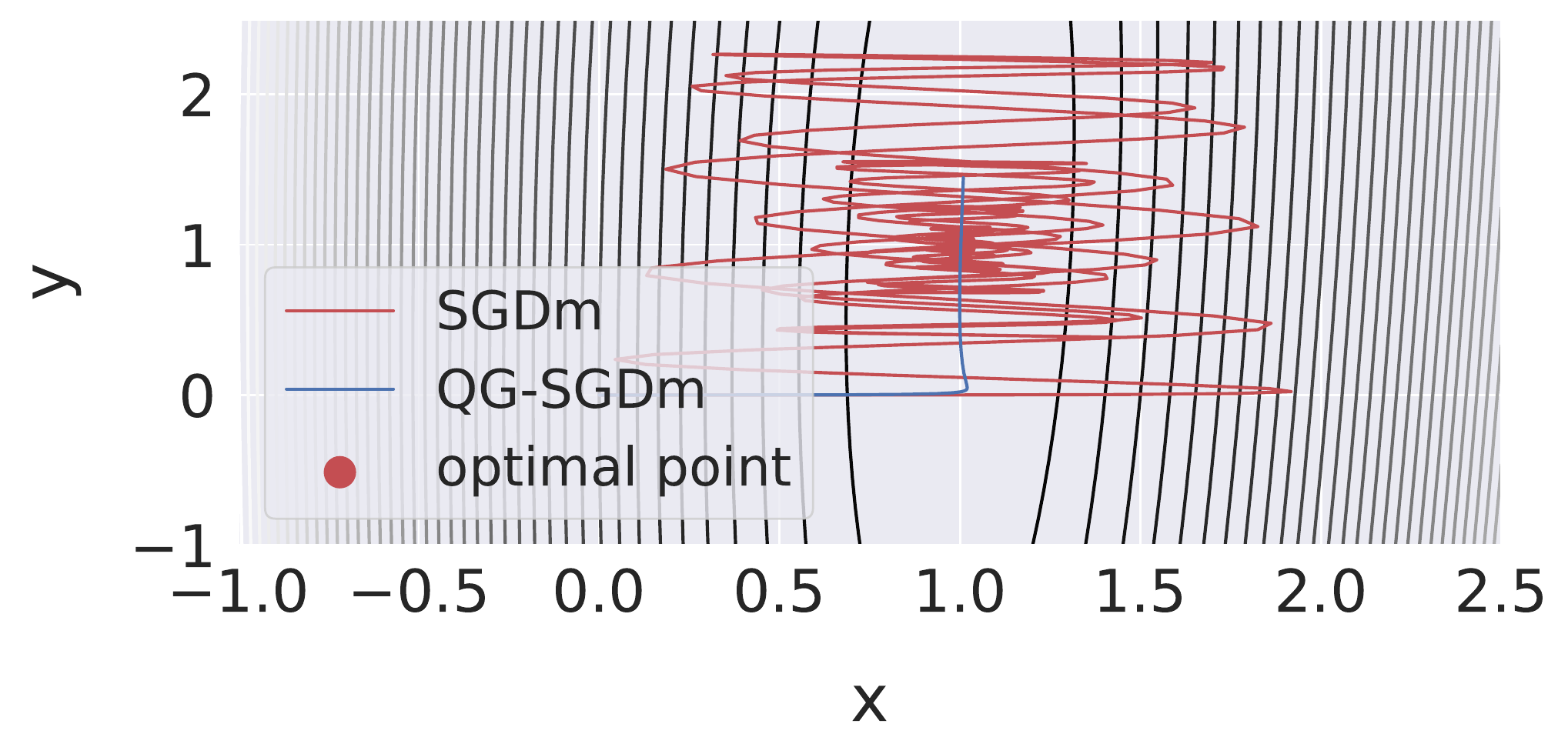}
		\label{fig:trajectory_momentum099_lr0001_pos_bl}
	}
	\hfill
	\subfigure[\small $\beta = 0.99, \eta = 0.001$, initial point $(2, 0)$.]{
		\includegraphics[width=.475\textwidth,]{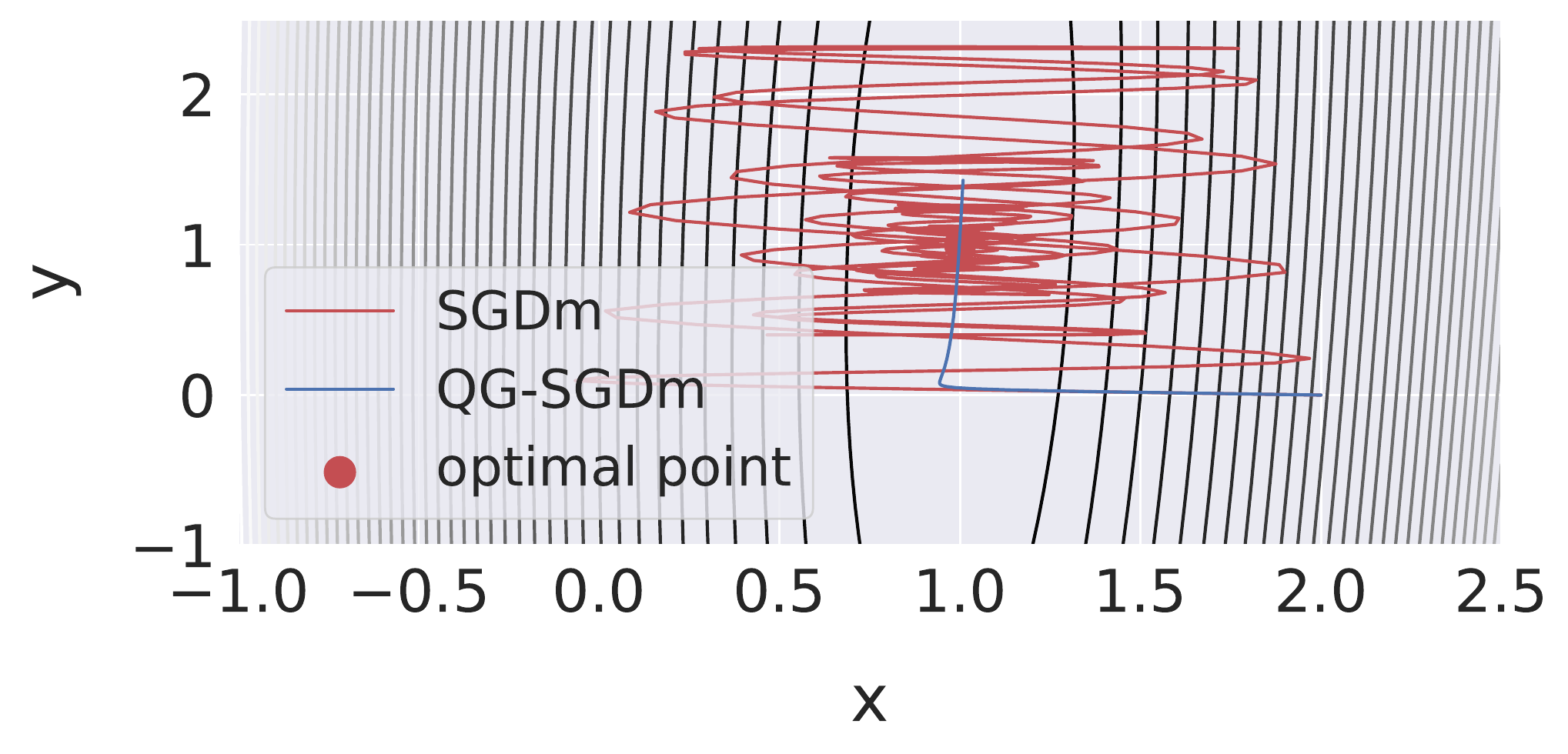}
		\label{fig:trajectory_momentum099_lr0001_pos_br}
	}
	\hfill
	\subfigure[\small $\beta = 0.99, \eta = 0.001$, initial point $(0, 2)$.]{
		\includegraphics[width=.475\textwidth,]{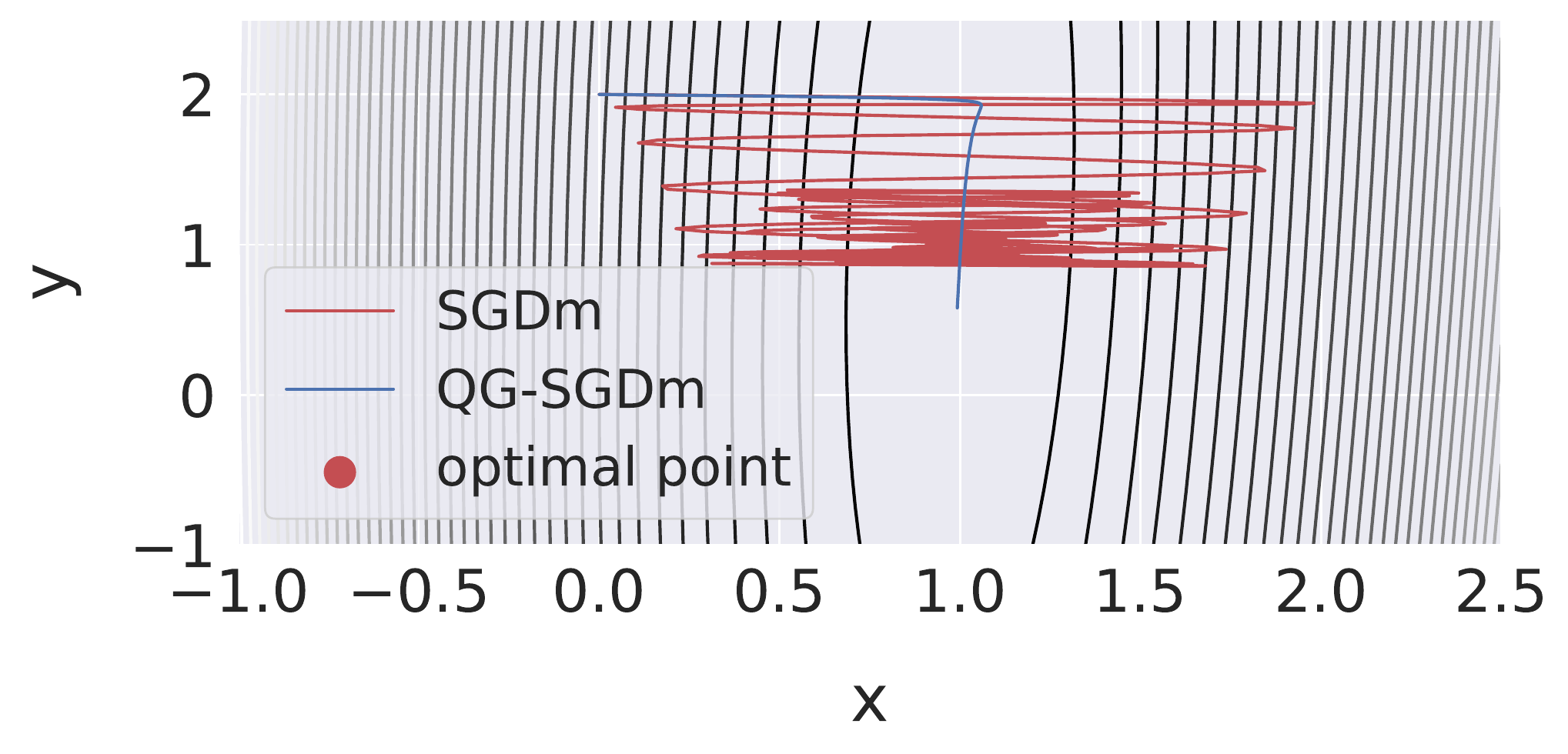}
		\label{fig:trajectory_momentum099_lr0001_pos_tl}
	}
	\hfill
	\subfigure[\small $\beta = 0.99, \eta = 0.001$, initial point $(2, 2)$.]{
		\includegraphics[width=.475\textwidth,]{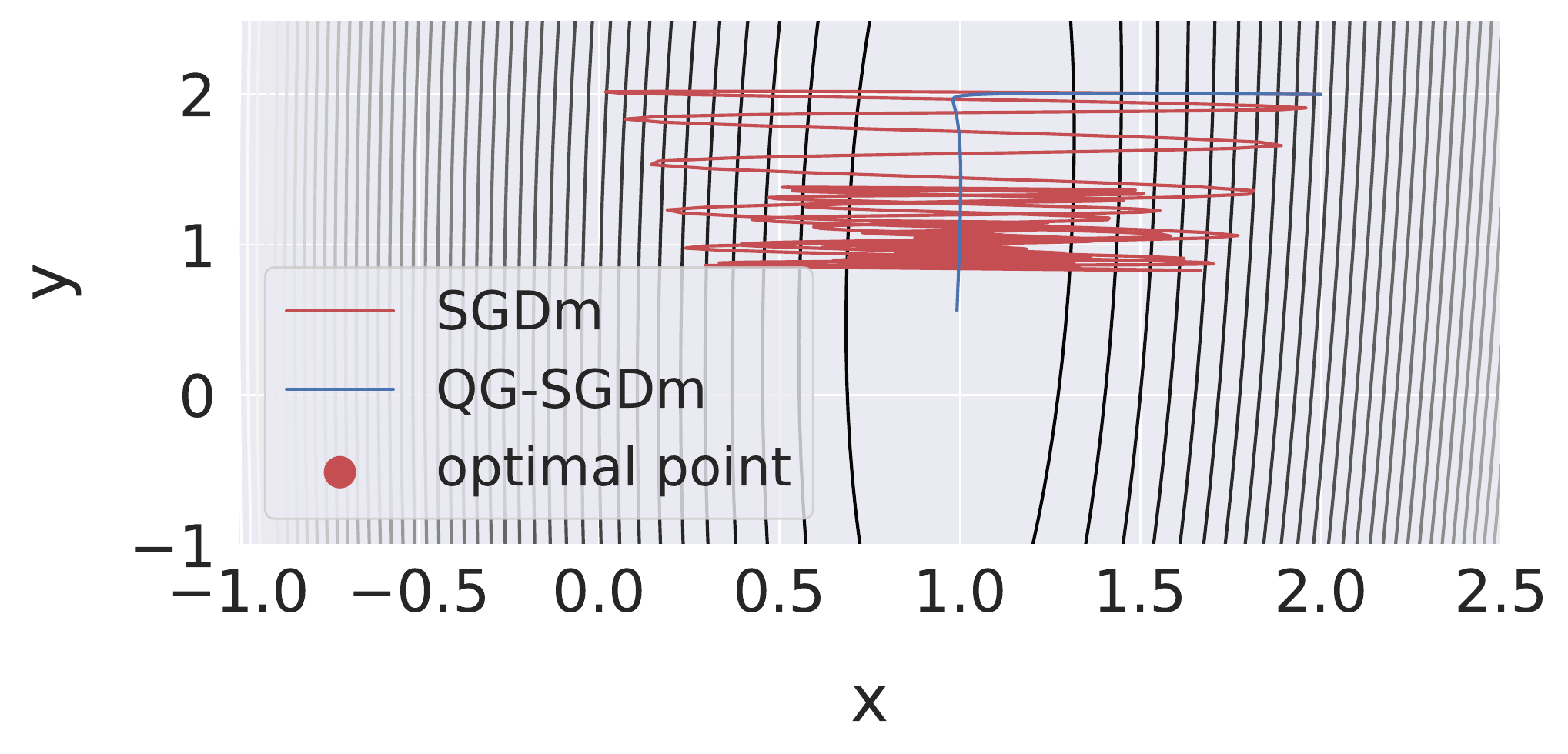}
		\label{fig:trajectory_momentum099_lr0001_pos_tr}
	}
	\hfill
	\caption{\small
		Understanding the optimization trajectory of \salgoptsgdm and Heavy-ball momentum SGD (SGDm),
		via a 2D toy function $f(x, y) = (y - x^2)^2 + 100 (x - 1)^2$.
		This function has a global minimum at $(x, y) = (1, 1)$.
		Red line corresponds to SGDm and blue line indicates \salgoptsgdm.
		Red line illustrates larger oscillation than \salgoptsgdm on the optimization trajectory.
	}
	\label{fig:understanding_on_new_2d_toy_example_complete}
\end{figure*}

We further study a simple non-convex toy problem~\citet{lucas2018aggregated}:
\begin{align*}
	f(x, y) = \log(e^x + e^{-x})
	+ 10 \log \left( e^{e^x \left( y - sin(a x) \right)} + e^{-e^x \left( y - sin(a x) \right)} \right) \,,
\end{align*}%
in Figure~\ref{fig:understanding_on_2d_toy_example}.
In our experiments, we choose $a = 8$ and $b = 10$,
and initialize the optimizer at $(x, y) = (-2, 0)$.

\begin{figure*}[!h]
	\centering
	\vspace{-1em}
	\subfigure[\small $\beta = 0.9, \eta = 0.05$.]{
		\includegraphics[width=.315\textwidth,]{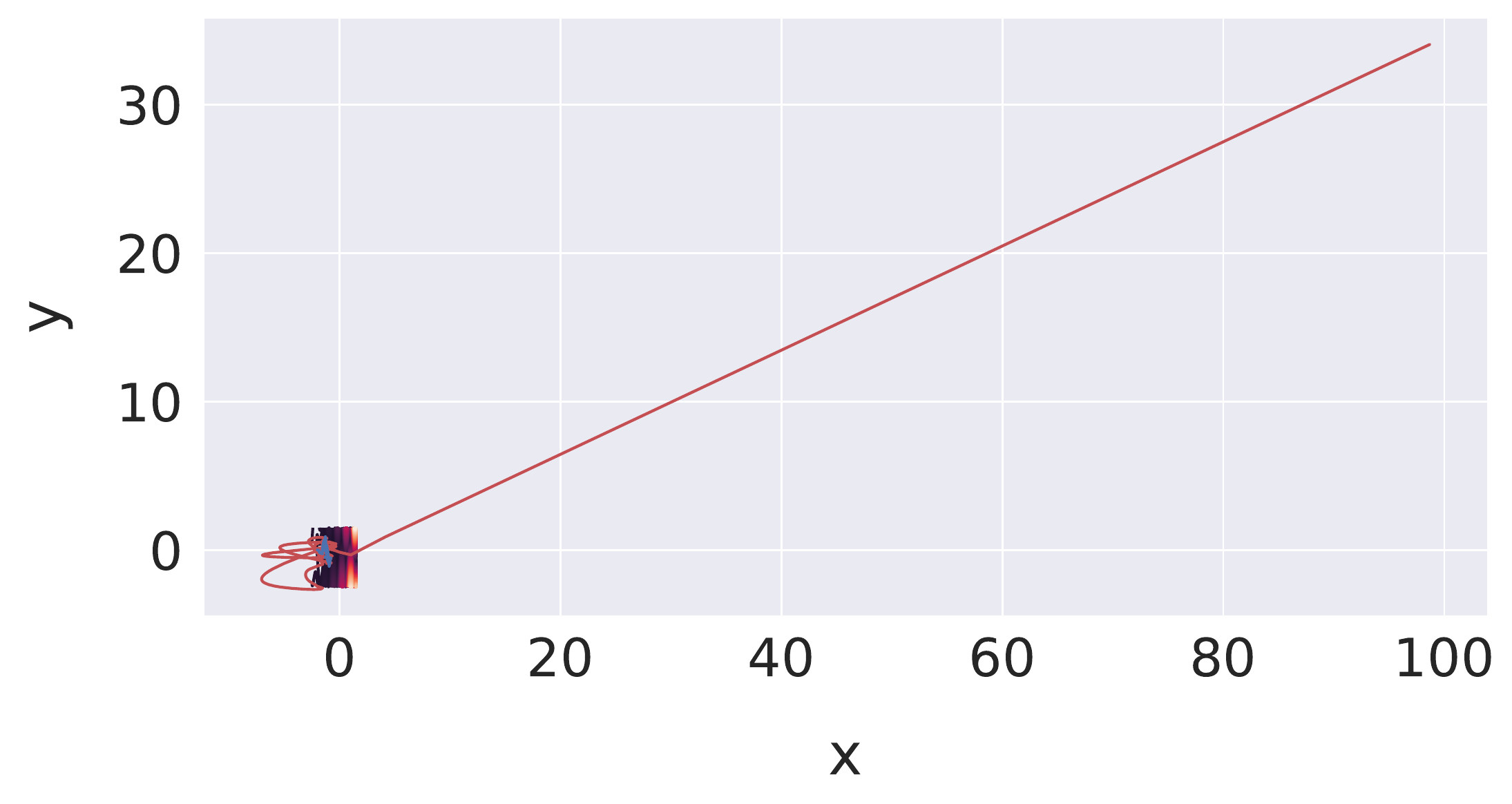}
		\label{fig:trajectory_momentum09_lr005}
	}
	\hfill
	\subfigure[\small $\beta = 0.9, \eta = 0.01$.]{
		\includegraphics[width=.315\textwidth,]{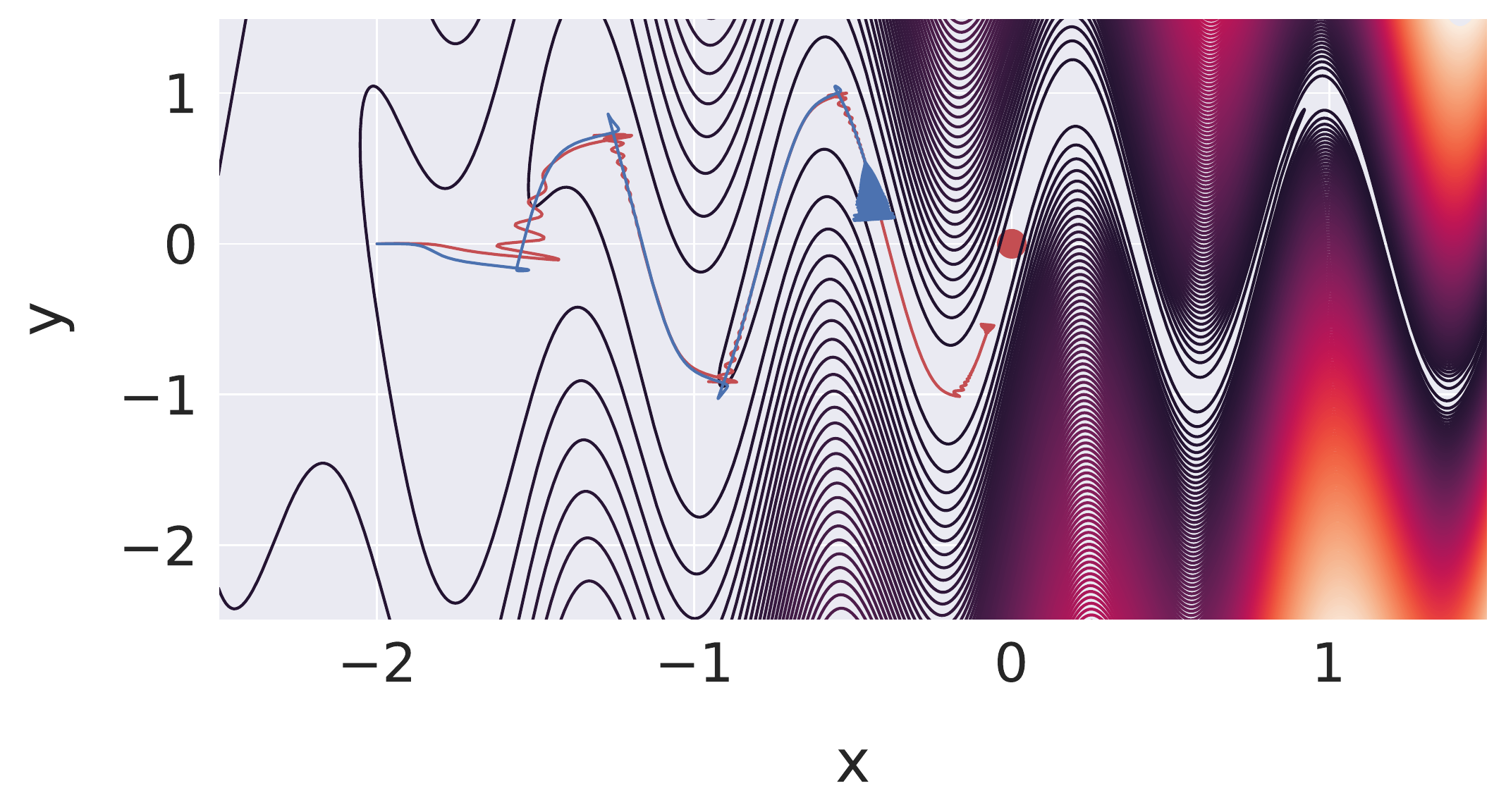}
		\label{fig:trajectory_momentum09_lr001}
	}
	\hfill
	\subfigure[\small $\beta = 0.9, \eta = 0.005$.]{
		\includegraphics[width=.315\textwidth,]{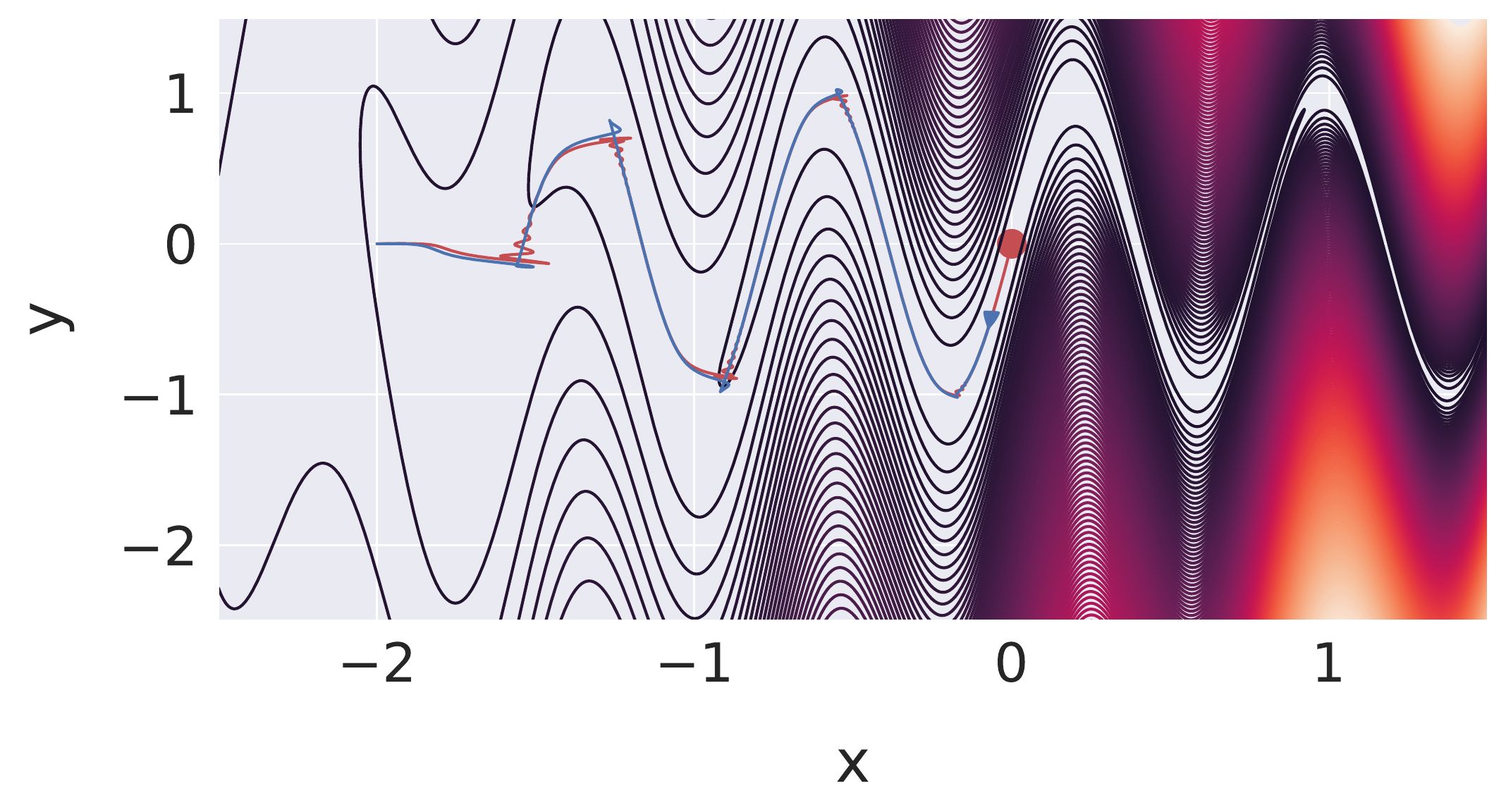}
		\label{fig:trajectory_momentum09_lr0005}
	}
	\hfill
	\subfigure[\small $\beta = 0.99, \eta = 0.01$.]{
		\includegraphics[width=.315\textwidth,]{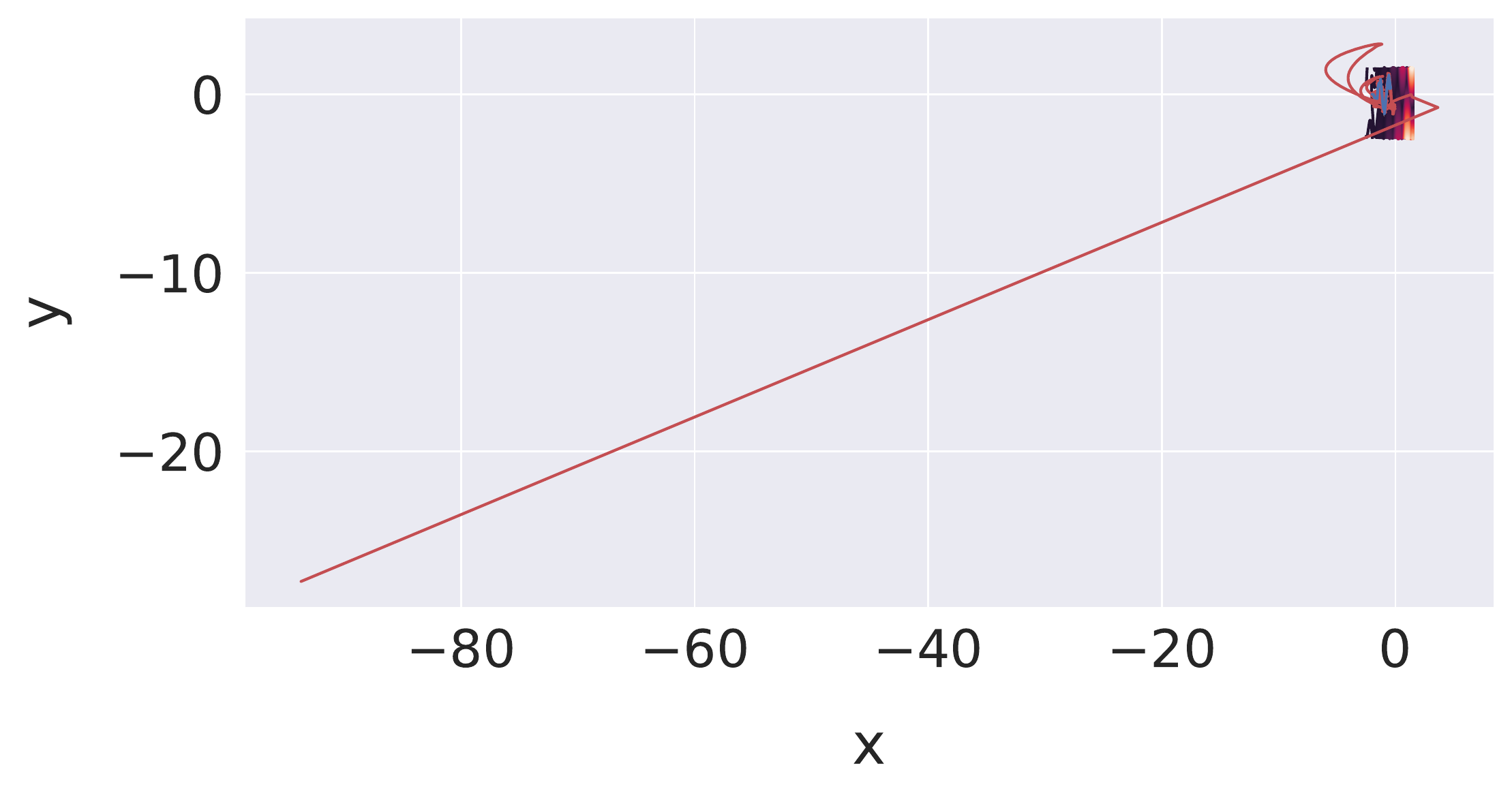}
		\label{fig:trajectory_momentum09_lr0005}
	}
	\hfill
	\subfigure[\small $\beta = 0.99, \eta = 0.005$.]{
		\includegraphics[width=.325\textwidth,]{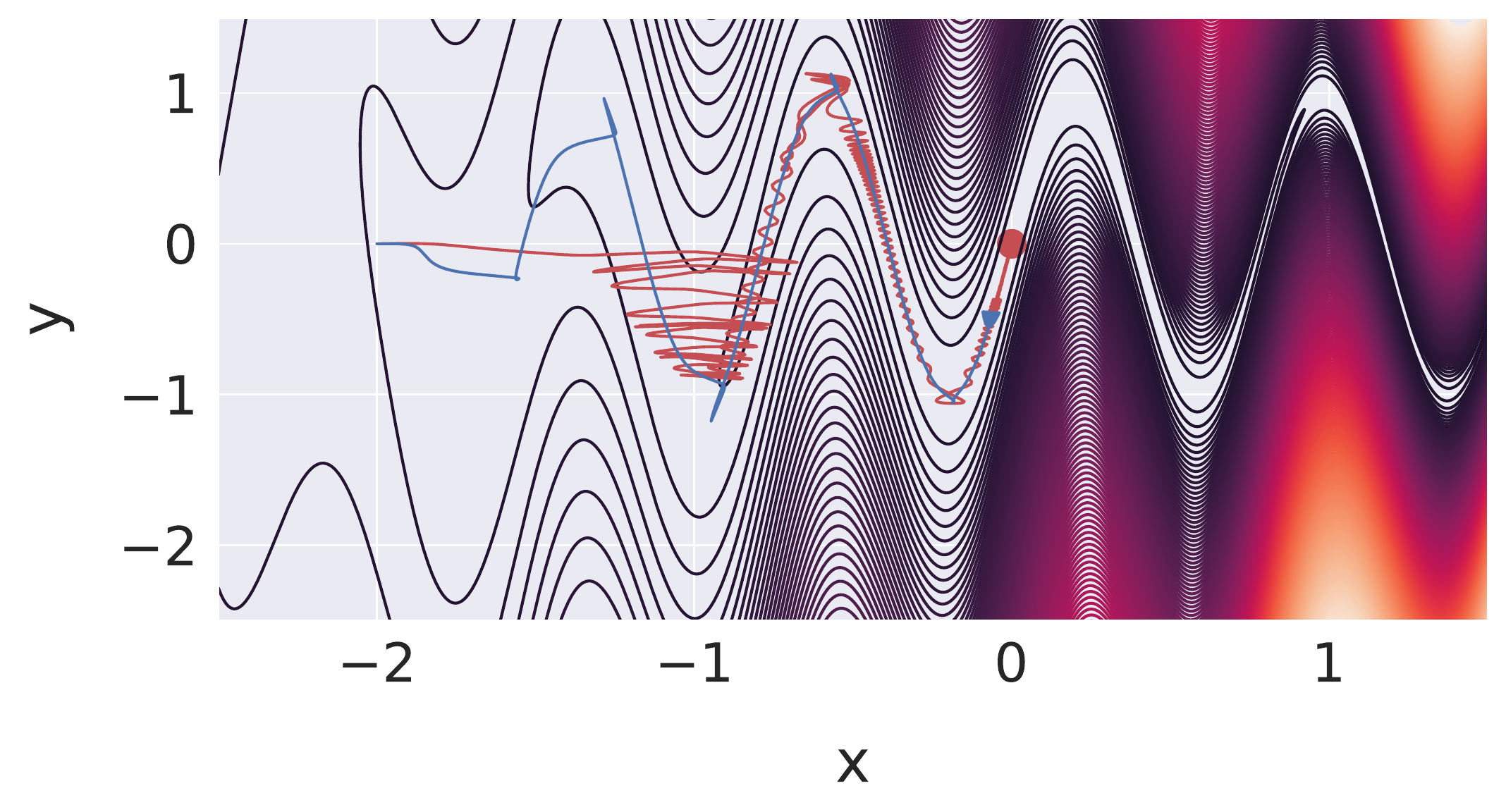}
		\label{fig:trajectory_momentum09_lr0005}
	}
	\caption{\small
		Understanding the optimization trajectory of \salgoptsgdm and Heavy-ball momentum SGD (SGDm),
		via a 2D toy function
		$
			f(x, y) = \log(e^x + e^{-x})
			+ 10 \log \left( e^{e^x \left( y - sin(8 x) \right)} + e^{-e^x \left( y - sin(8 x) \right)} \right)
		$.
		This function has an optimal value at $(x, y) = (0, 0)$.
		Red line corresponds to SGDm and blue line indicates \salgoptsgdm.
		Red line illustrates larger oscillation than \salgoptsgdm on the optimization trajectory.
	}
	\label{fig:understanding_on_2d_toy_example}
\end{figure*}

% \subsection{The Learning Curves on NLP tasks} \label{appendix:learning_curves_nlp_task}
% Figure~\ref{fig:learning_curves_nlp_tasks}
% visualizes the learning curves for training DistilBERT on AG News for $\alpha = \{1, 0.1\}$.

% \begin{figure*}[!h]
% 	\vspace{-1em}
% 	\centering
% 	\subfigure[\small
% 		Training DistillBERT on AG News with Ring topology for $n\!=\!16$ ($\alpha\!=\!1$).
% 	]{
% 		\includegraphics[width=.475\textwidth,]{figures/nlp_learning_curves/distilbert_agnews_k16_ring_non_iid_alpha1_test_top1.pdf}
% 		\label{fig:distilbert_agnews_k16_ring_non_iid_alpha1_test_top1}
% 	}
% 	\hfill
% 	\subfigure[\small
% 		Training DistillBERT on AG News with Ring topology for $n\!=\!16$ ($\alpha\!=\!0.1$).
% 	]{
% 		\includegraphics[width=.475\textwidth,]{figures/nlp_learning_curves/distilbert_agnews_k16_ring_non_iid_alpha01_test_top1.pdf}
% 		\label{fig:distilbert_agnews_k16_ring_non_iid_alpha01_test_top1}
% 	}
% 	\vspace{-1em}
% 	\caption{\small
% 		Learning curves for NLP tasks.
% 	}
% 	\label{fig:learning_curves_nlp_tasks}
% \end{figure*}

\subsection{The Learning Curves on CV tasks} \label{appendix:learning_curves_cv_task}
Figure~\ref{fig:learning_curves_cv_tasks} visualizes the learning curves
for training ResNet-EvoNorm-20 on CIFAR-10,
in terms of different degrees of non-i.i.d.-ness and network topologies (Ring and Social topology).

\begin{figure*}[!h]
	\vspace{-1em}
	\centering
	\subfigure[\small
		Training ResNet-EvoNorm-20 on CIFAR-10 with Social topology for $n\!=\!32$ ($\alpha\!=\!10$).
	]{
		\includegraphics[width=.315\textwidth,]{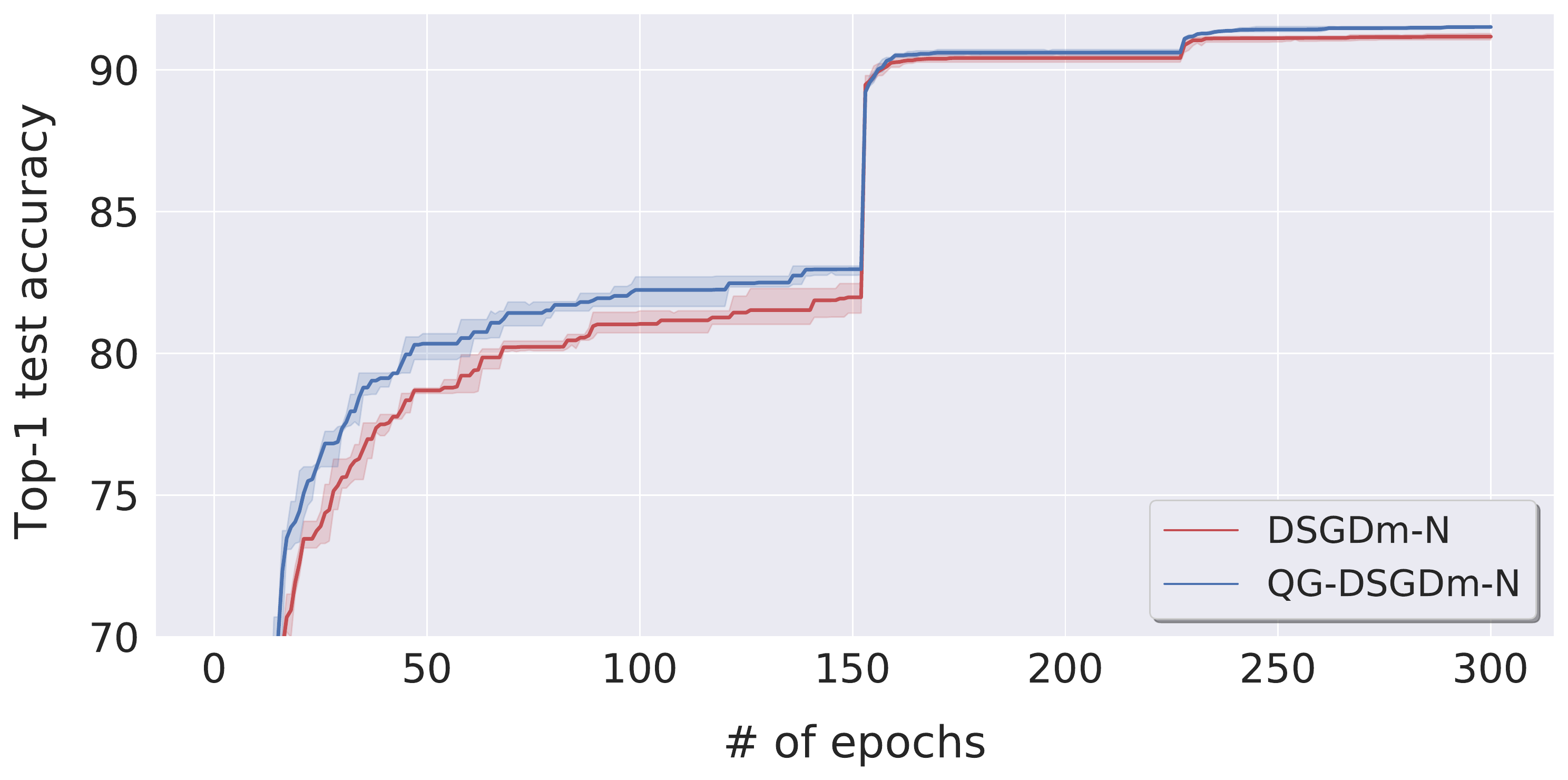}
		\label{fig:resnet_evonorm20_cifar10_k32_social_non_iid_alpha10_test_top1}
	}
	\hfill
	\subfigure[\small
		Training ResNet-EvoNorm-20 on CIFAR-10 with Social topology for $n\!=\!32$ ($\alpha\!=\!1$).
	]{
		\includegraphics[width=.315\textwidth,]{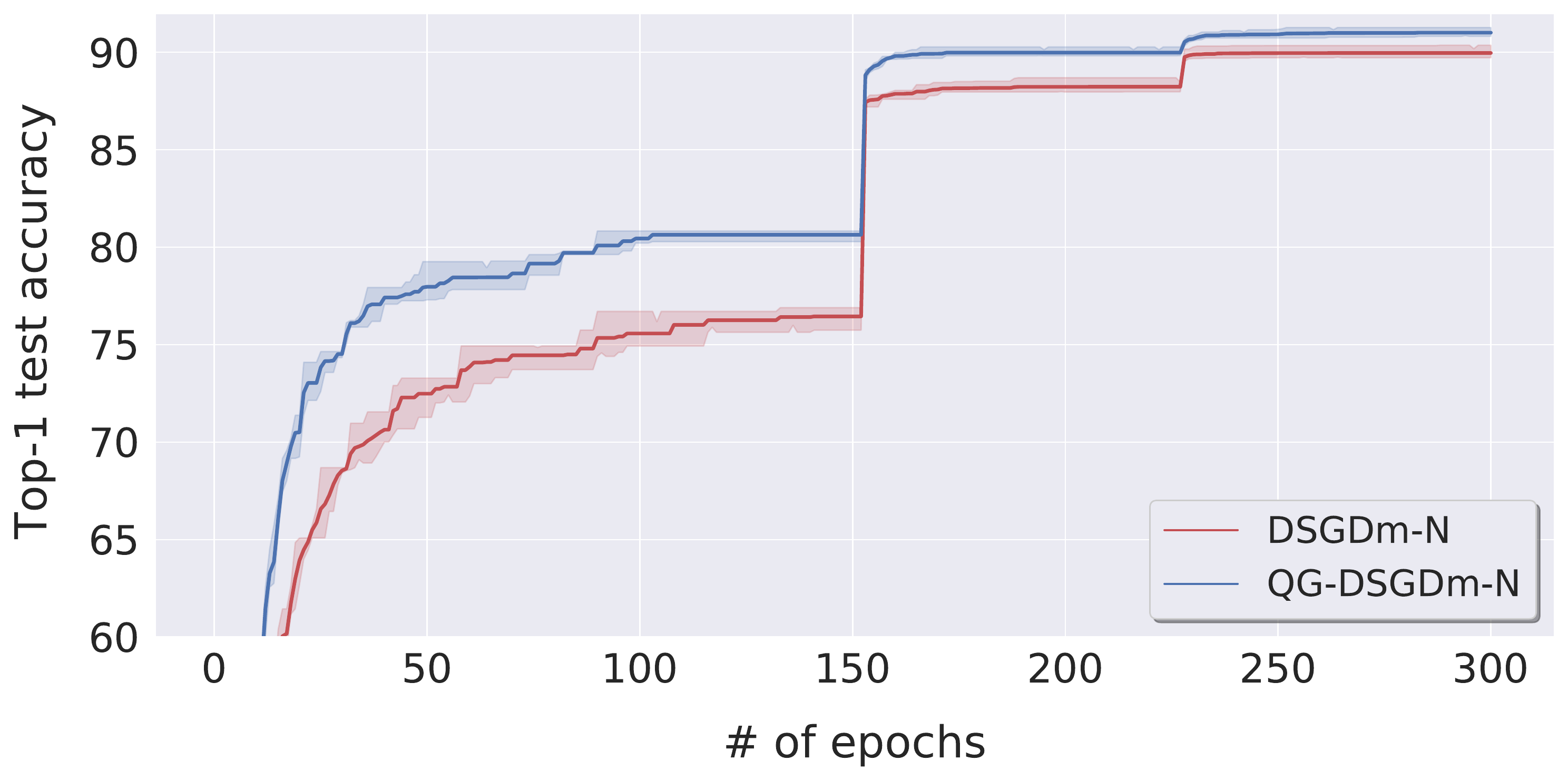}
		\label{fig:resnet_evonorm20_cifar10_k32_social_non_iid_alpha1_test_top1}
	}
	\hfill
	\subfigure[\small
		Training ResNet-EvoNorm-20 on CIFAR-10 with Social topology for $n\!=\!32$ ($\alpha\!=\!0.1$).
	]{
		\includegraphics[width=.315\textwidth,]{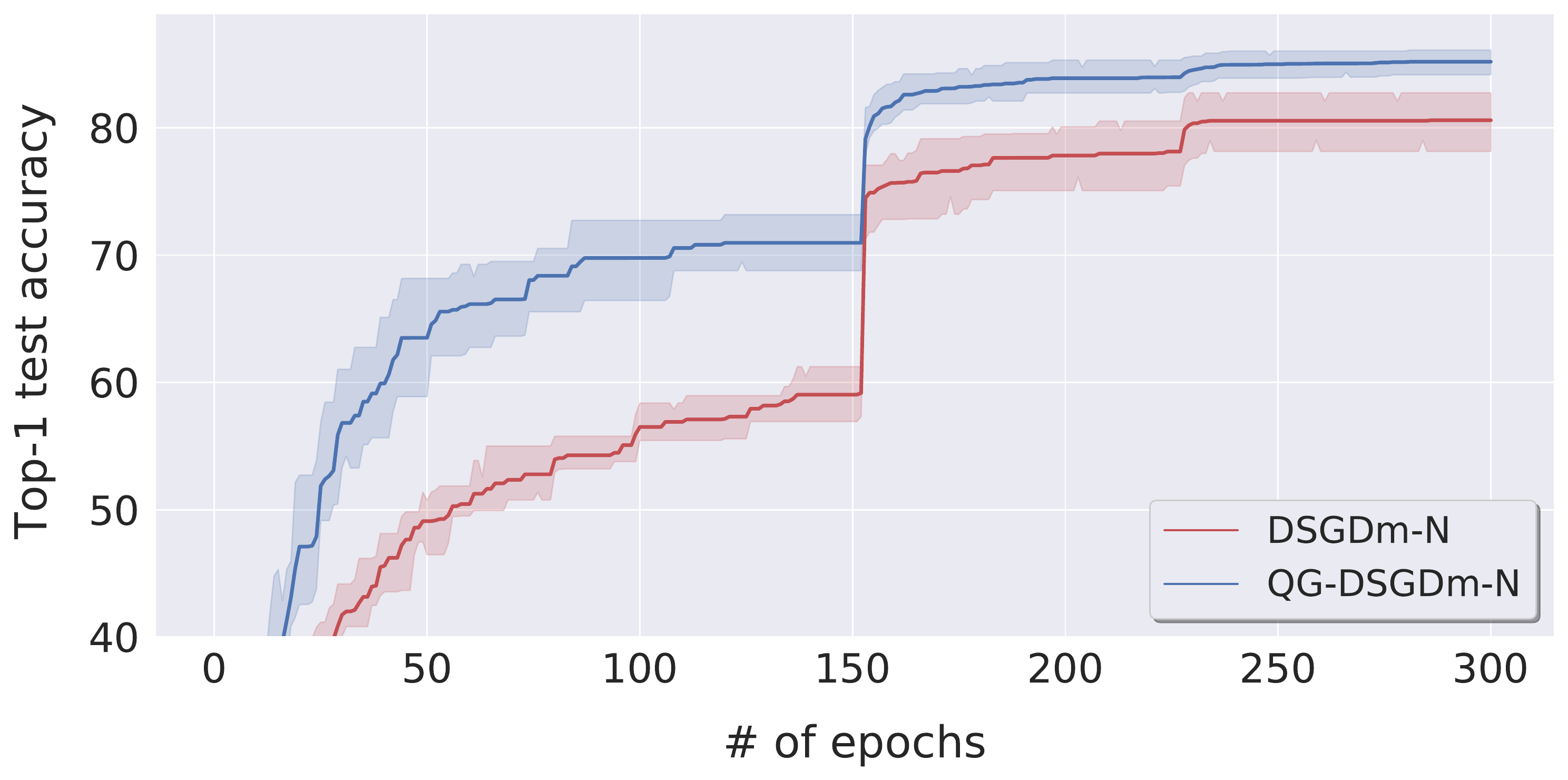}
		\label{fig:resnet_evonorm20_cifar10_k32_social_non_iid_alpha01_test_top1}
	}
	\subfigure[\small
		Training ResNet-EvoNorm-20 on CIFAR-10 with Ring topology for $n\!=\!16$ ($\alpha\!=\!10$).
	]{
		\includegraphics[width=.315\textwidth,]{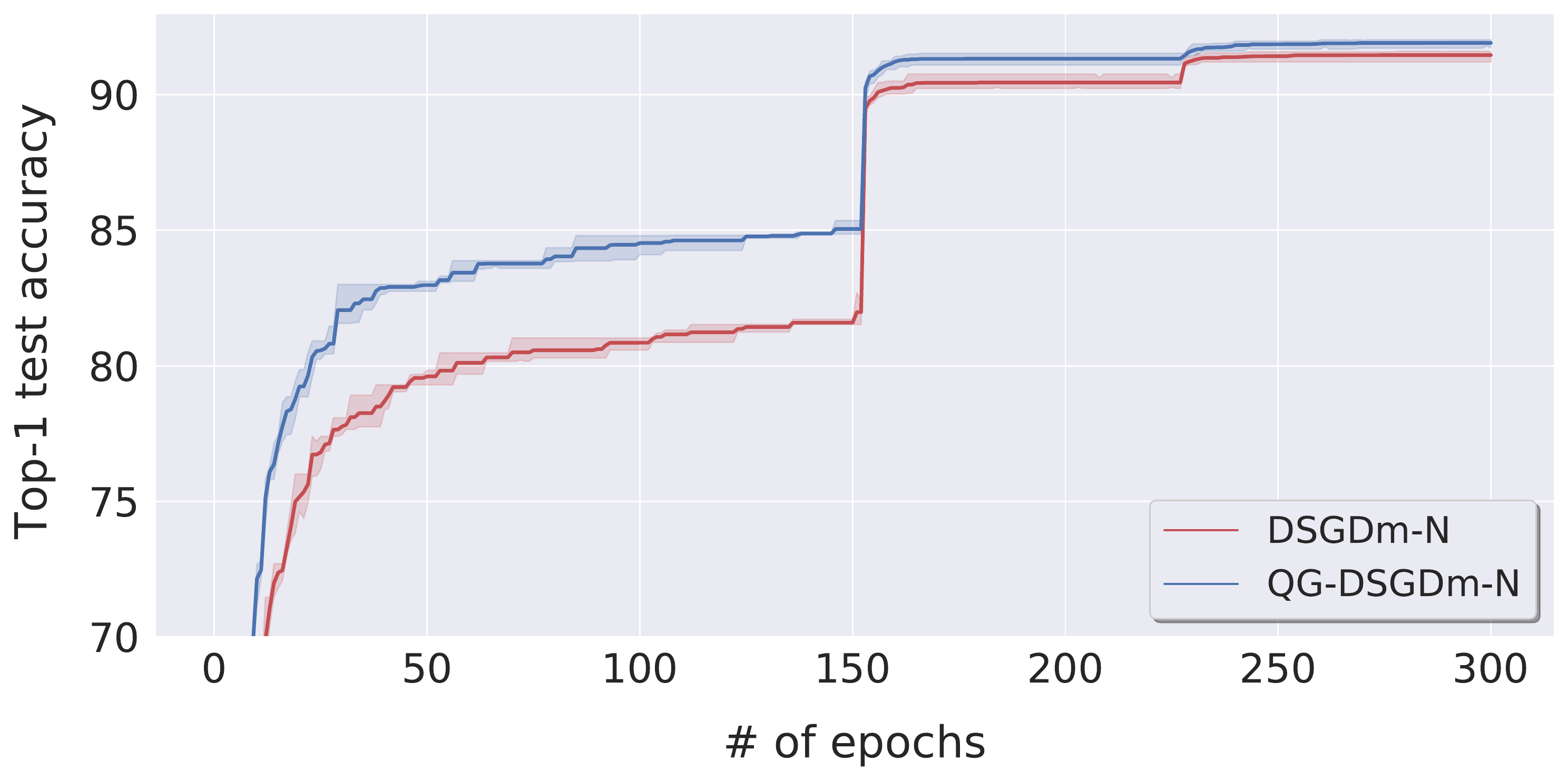}
		\label{fig:resnet_evonorm20_cifar10_k32_social_non_iid_alpha10_test_top1}
	}
	\hfill
	\subfigure[\small
		Training ResNet-EvoNorm-20 on CIFAR-10 with Ring topology for $n\!=\!16$ ($\alpha\!=\!1$).
	]{
		\includegraphics[width=.315\textwidth,]{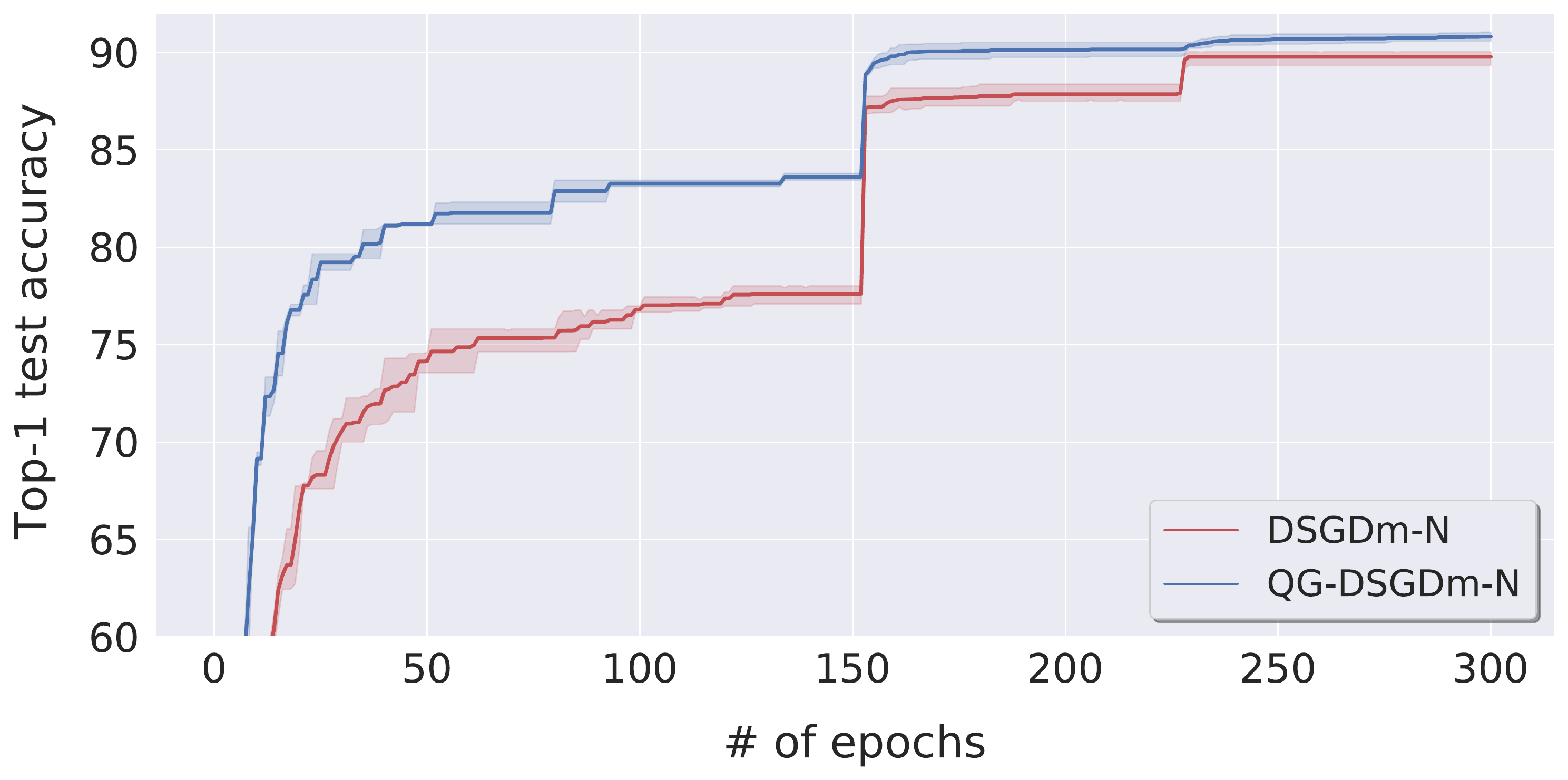}
		\label{fig:resnet_evonorm20_cifar10_k32_social_non_iid_alpha1_test_top1}
	}
	\hfill
	\subfigure[\small
		Training ResNet-EvoNorm-20 on CIFAR-10 with Ring topology for $n\!=\!16$ ($\alpha\!=\!0.1$).
	]{
		\includegraphics[width=.315\textwidth,]{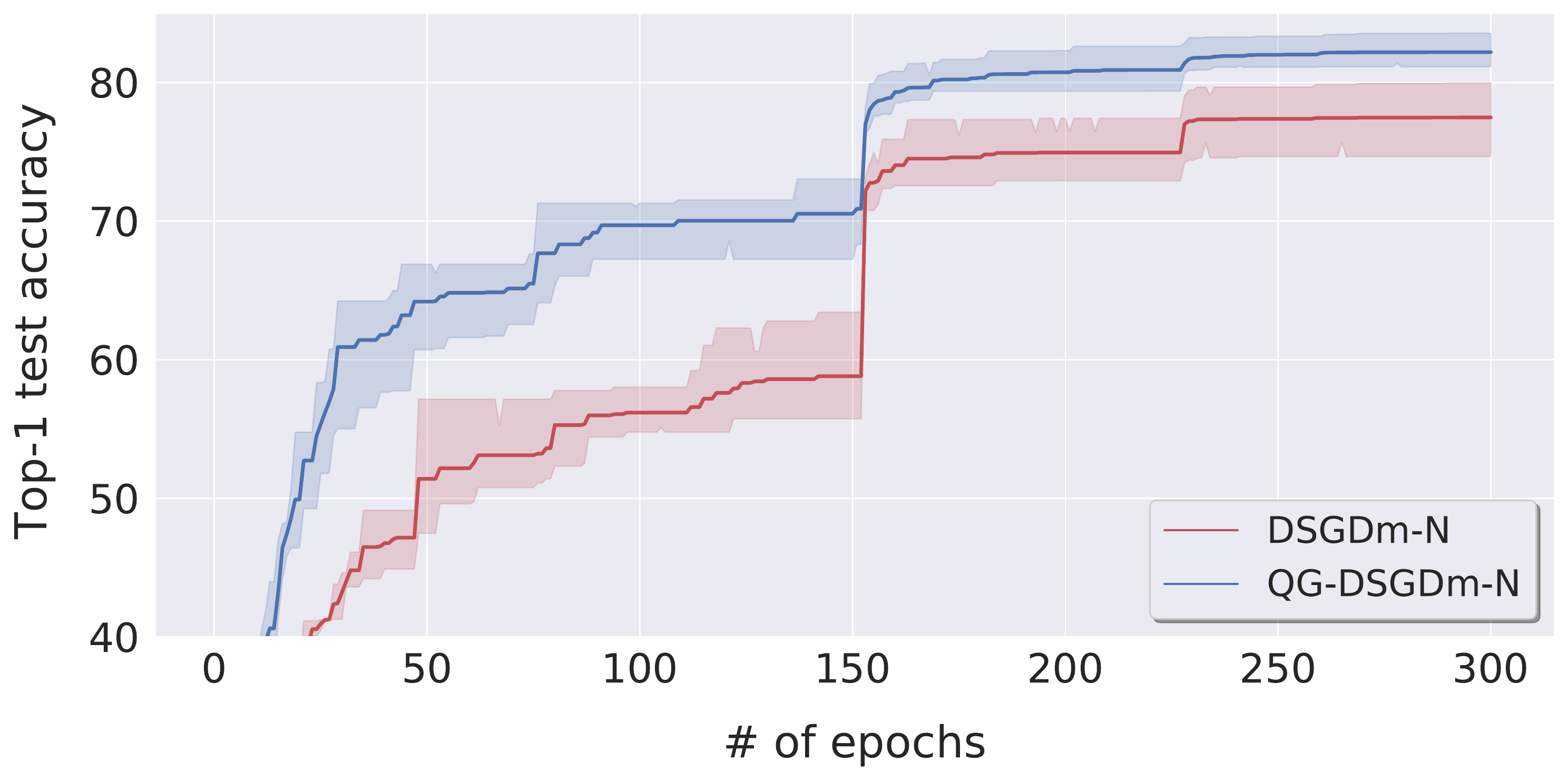}
		\label{fig:resnet_evonorm20_cifar10_k32_social_non_iid_alpha01_test_top1}
	}
	\hfill
	\subfigure[\small
		Training ResNet-EvoNorm-18 on ImageNet with Ring topology for $n\!=\!16$ ($\alpha\!=\!1$).
	]{
		\includegraphics[width=.475\textwidth,]{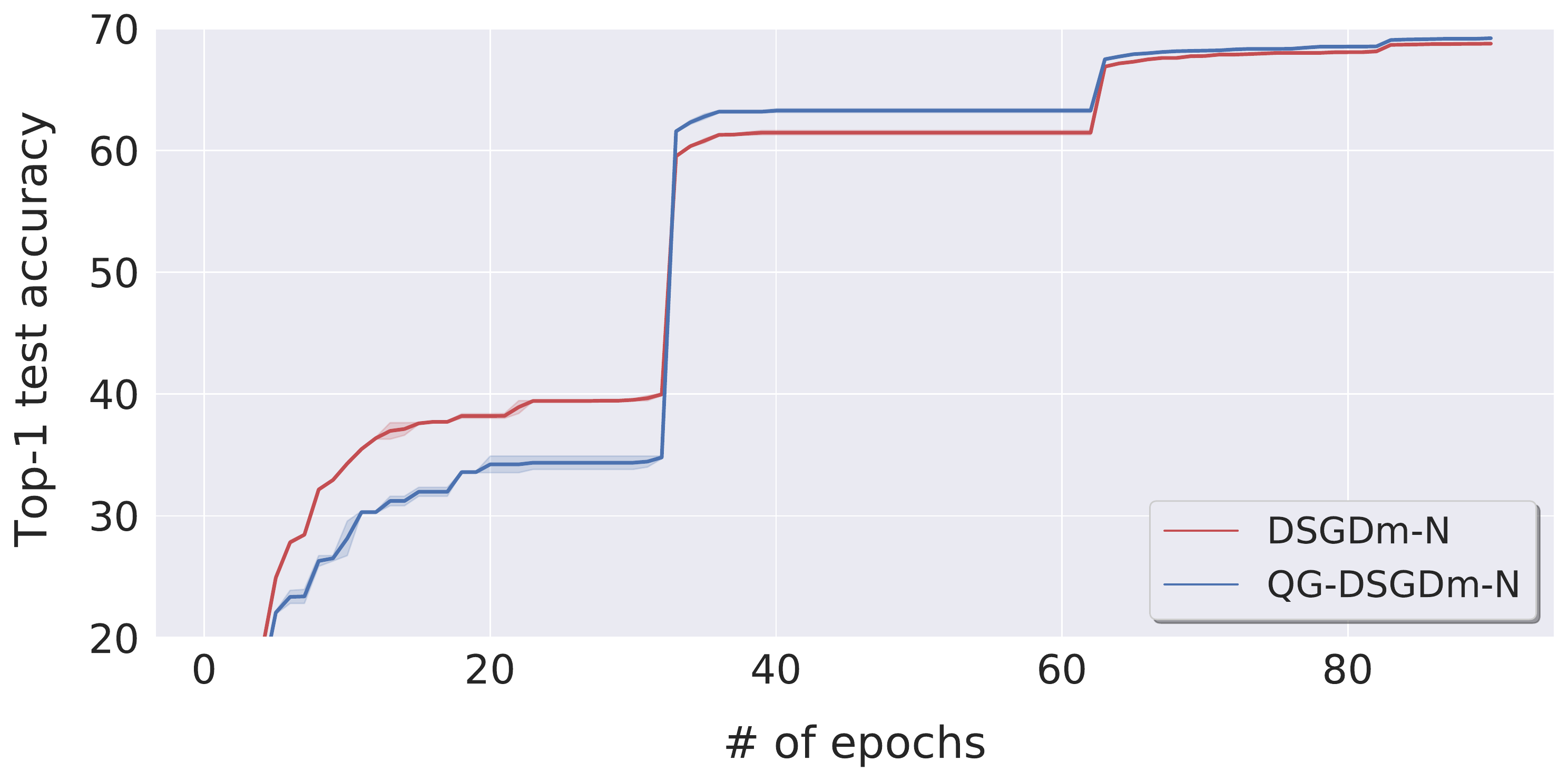}
		\label{fig:resnet_evonorm18_imagenet_k16_ring_non_iid_alpha1_test_top1}
	}
	\hfill
	\subfigure[\small
		Training ResNet-EvoNorm-18 on ImageNet with Ring topology for $n\!=\!16$ ($\alpha\!=\!0.1$).
	]{
		\includegraphics[width=.475\textwidth,]{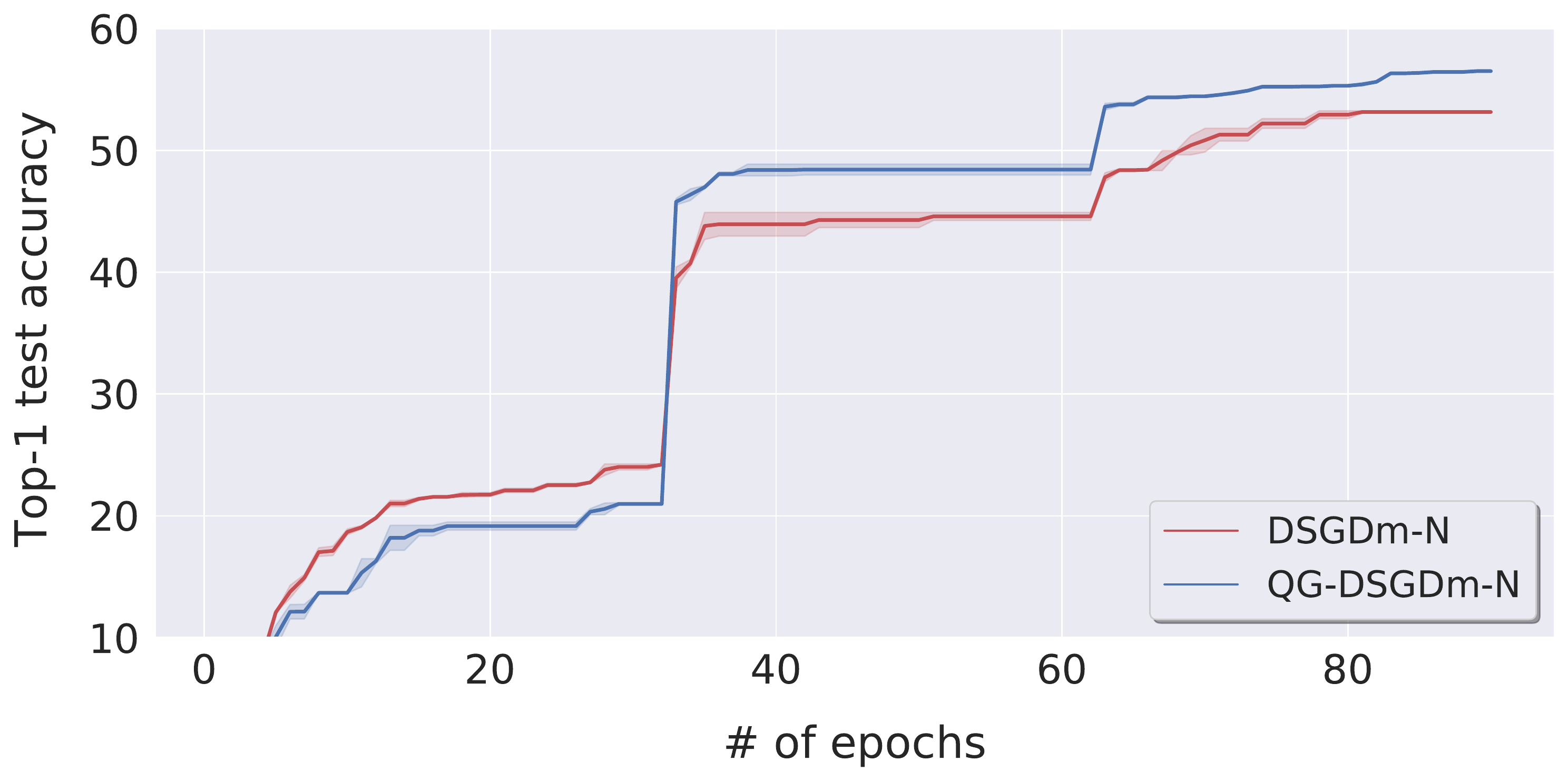}
		\label{fig:resnet_evonorm18_imagenet_k16_ring_non_iid_alpha01_test_top1}
	}
	\vspace{-1em}
	\caption{\small
		Learning curves for cv tasks.
	}
	\label{fig:learning_curves_cv_tasks}
\end{figure*}

\subsection{The Ineffectiveness of Tuning Momentum Factor for DSGDm-N}
\label{appendix:tuning_lr_and_momentum_for_sgdmN}

Table~\ref{fig:tuning_momentum_factors}
shows that tuning momentum factor for DSGDm-N cannot alleviate the training difficulty caused by heterogeneity.

\begin{figure*}[!h]
	\vspace{-1em}
	\centering
	\subfigure[\small
	]{
		\includegraphics[width=.475\textwidth,]{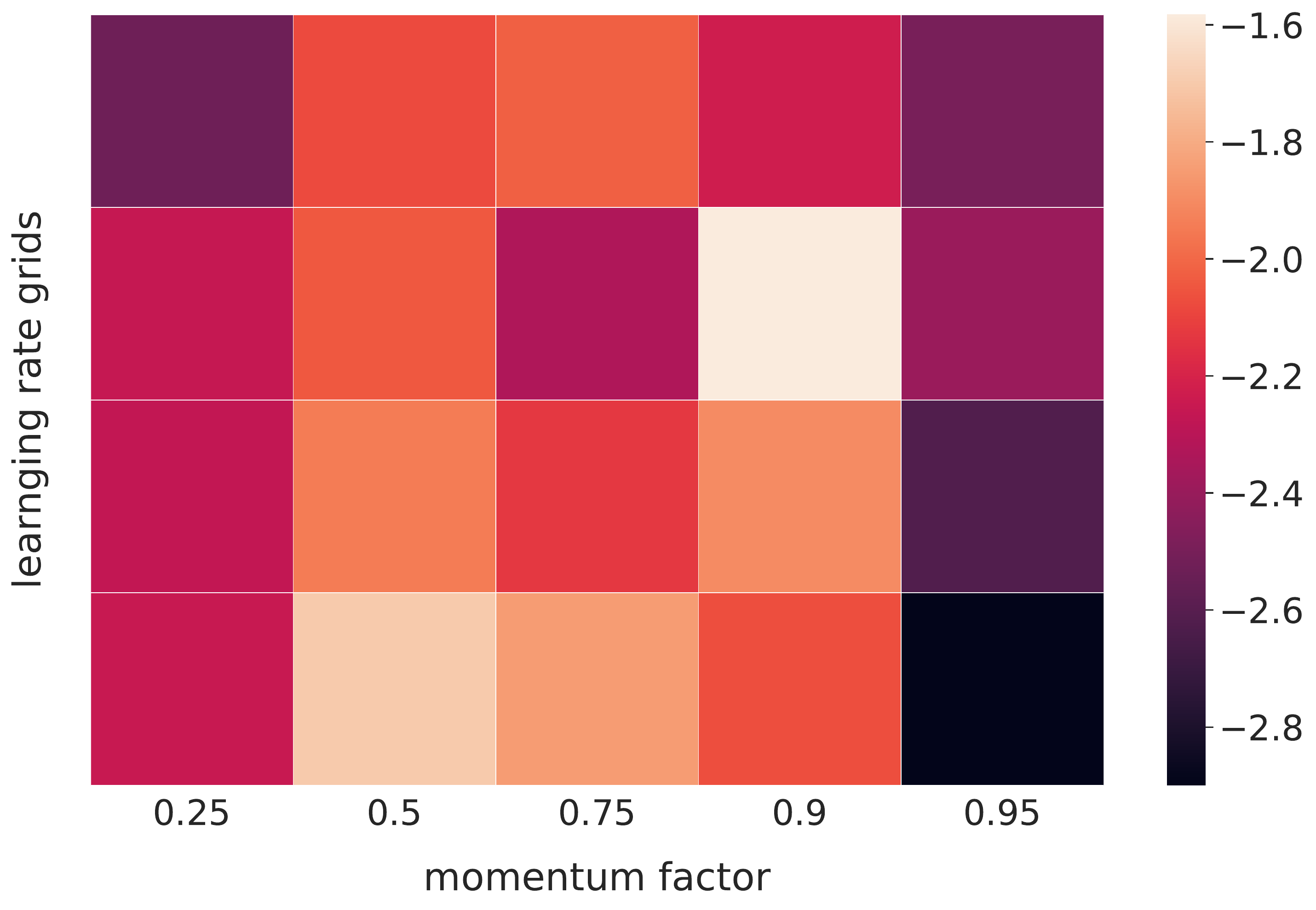}
		\label{fig:tuning_momentum_factors_alpha1}
	}
	\hfill
	\subfigure[\small
	]{
		\includegraphics[width=.475\textwidth,]{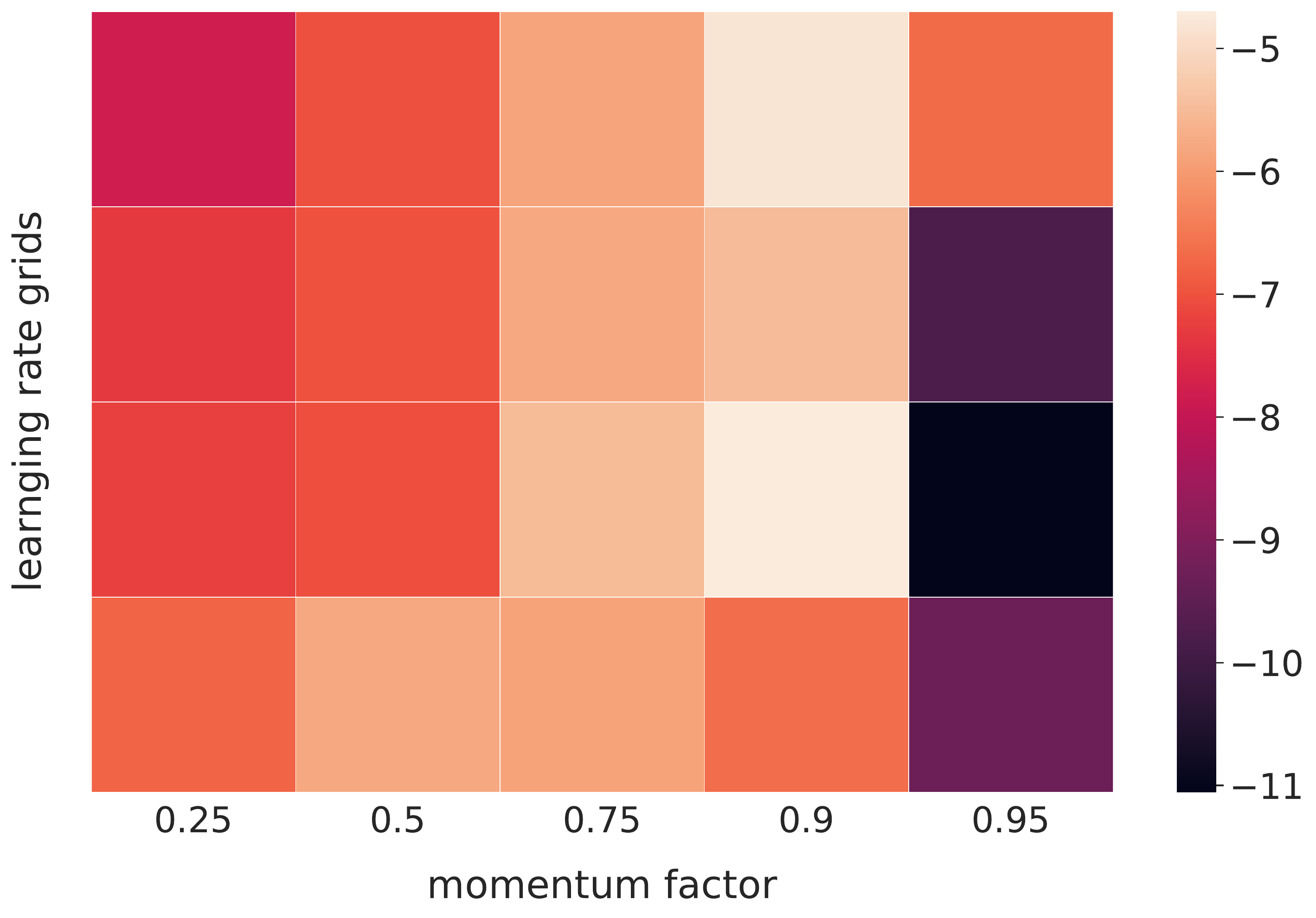}
		\label{fig:tuning_momentum_factors_alpha01}
	}
	\vspace{-1em}
	\caption{\small
		The ineffectiveness of tuning momentum factors for DSGDm-N,
		for training ResNet-EvoNorm-20 on CIFAR-10.
		We illustrate the performance gap between
		DSGDm-N (different combination of learning rate and momentum factor)
		and \algoptsgdmn (tuned learning rate from the grid with default momentum factor $0.9$).
	}
	\label{fig:tuning_momentum_factors}
\end{figure*}

\subsection{The Superior Performance of \algoptsgdmn Generalize Different Topology Scales}
\label{appendix:cv_different_ring_scale_results}

Table~\ref{tab:cv_different_ring_scale_results}
further showcases the generality of the predominant performance gain of \qg
on different topology scales ($n$).

\begin{table}[!h]
	\centering
	\caption{\small
		\textbf{The test top-1 accuracy of different decentralized algorithms
			evaluated on different topology scales and non-i.i.d.-ness},
		for training ResNet-EvoNorm-20 on CIFAR-10.
		The results are over three random seeds,
		with sufficient learning rate tuning.
		The table corresponds to Figure~\ref{fig:cv_different_ring_scale_results} in the main paper.
	}
	\vspace{-1em}
	\label{tab:cv_different_ring_scale_results}
	\resizebox{1.\textwidth}{!}{%
		\begin{tabular}{lcccccccc}
			\toprule
			\multirow{2}{*}{Methods}                 & \multicolumn{2}{c}{Ring ($n\!=\!16$)}  & \multicolumn{2}{c}{Ring ($n\!=\!32$)} & \multicolumn{2}{c}{Ring ($n\!=\!48$)} \\ \cmidrule(lr){2-3} \cmidrule(lr){4-5} \cmidrule(lr){6-7}
			             & $\alpha=1$                & $\alpha=0.1$              & $\alpha=1$                & $\alpha=0.1$              & $\alpha=1$                & $\alpha=0.1$              \\ \midrule
			SGDm-N (centralized) &            \multicolumn{2}{c}{$92.18 \pm 0.19$}                           & \multicolumn{2}{c}{$91.92 \pm 0.33$}              &            \multicolumn{2}{c}{$91.63 \pm 0.25$}              \\
			DSGDm-N      & $89.98 \pm 0.10$          & $77.48 \pm 2.67$          & $88.46 \pm 0.29$          & $78.17 \pm 1.63$          & $85.54 \pm 0.33$          & $73.67 \pm 0.90$          \\
			\algoptsgdmn & $\textbf{91.28} \pm 0.38$ & $\textbf{82.20} \pm 1.27$ & $\textbf{90.27} \pm 0.07$ & $\textbf{83.18} \pm 1.11$ & $\textbf{89.75} \pm 0.32$ & $\textbf{80.28} \pm 1.52$ \\
			\bottomrule
		\end{tabular}%
	}
\end{table}

\subsection{Multiple-step \algoptsgdmn variant} \label{appendix:multiple_step_algoptsgdmn}
Table~\ref{tab:ablation_study_for_multiple_step_algoptsgdmn}
illustrates the performance for the multiple-step variant of \algoptsgdmn.
We can witness that tuning the value of $\tau$ cannot lead to a significant performance gain.

\begin{table}[!h]
	\centering
	\caption{\small
		\textbf{Ablation study for the variant of multiple-step \algoptsgdmn }
		(illustrated in Algorithm~\ref{alg:multiple_step_algoptsgdm_varaint}),
		for training ResNet-EvoNorm-20 on CIFAR-10.
		The results are averaged over three seeds with tuned learning rate.
	}
	\vspace{-1em}
	\label{tab:ablation_study_for_multiple_step_algoptsgdmn}
	\resizebox{.5\textwidth}{!}{%
		\begin{tabular}{lcc}
			\toprule
			\multirow{2}{*}{Methods}                 & \multicolumn{2}{c}{Ring ($n\!=\!16$)} \\ \cmidrule(lr){2-3}
			                          & $\alpha=1$       & $\alpha=0.1$     \\ \midrule
			SGDm-N (centralized) &             \multicolumn{2}{c}{$92.18 \pm 0.19$}                                        \\
			\midrule
			DSGD                      & $88.88 \pm 0.26$ & $74.55 \pm 2.07$ \\
			DSGDm-N                   & $89.98 \pm 0.10$ & $77.48 \pm 2.67$ \\
			\midrule
			\algoptsgdmn ($\tau = 1$) & $91.28 \pm 0.38$ & $82.20 \pm 1.27$ \\
			\algoptsgdmn ($\tau = 2$) & $91.11 \pm 0.18$ & $82.25 \pm 1.68$ \\
			\algoptsgdmn ($\tau = 3$) & $91.04 \pm 0.01$ & $81.57 \pm 2.21$ \\
			\algoptsgdmn ($\tau = 4$) & $91.26 \pm 0.25$ & $82.55 \pm 1.55$ \\
			\bottomrule
		\end{tabular}%
	}
\end{table}

\subsection{Comparison with D$^2$ and Gradient Tracking (GT) methods} \label{appendix:comparison_with_gradient_tracking}
We comment on GT methods below (including D$^2$~\citep{tang2018d}), in order to 1) highlight the distinctions between different algorithms, and 2) justify the comparison with existing GT methods.
\begin{itemize}[nosep, leftmargin=12pt]
	\item Distinctions between algorithms:
	      1) Both D$^2$ and GT do not consider momentum in their algorithm design and theoretical analysis, while one of our main contributions is the design of quasi-global momentum---a simple yet effective approach for the SOTA decentralized deep learning training;
	      2) It is unclear how to integrate D$^2$ with momentum, given the original design intuition of D$^2$;
	      3) \algoptsgdm is different from D$^2$, where the updates of \algoptsgdm and D$^2$ follow $\scriptstyle \mW ( ( 1 + \beta \frac{\eta^{(t)}}{\eta^{(t-1)}} ) \mX^{(t)} - \beta \frac{\eta^{(t)}}{\eta^{(t-1)}} \mX^{(t-1)} - \eta^{(t)} \nabla f(\mX^{(t)}) )$ and $\scriptstyle \mW ( 2 \mX^{(t)} - \mX^{(t-1)} - \eta ( \nabla f(\mX^{(t)}) - \nabla f(\mX^{(t-1)}) ) )$ respectively (we simplify the comparison by letting $\mu \!=\! 0$ in \algoptsgdm);
	      4) Compared to \algoptsgdm, GT requires extra one communication step per update to approximate the global average of local gradients.

	\item D$^2$ cannot achieve comparable test performance on the standard deep learning benchmark.
	      \begin{itemize}[nosep, leftmargin=12pt]
		      \item D$^2$ requires a constant learning rate, which does not fit the SOTA learning rate schedule (e.g.\ stage-wise) in deep learning.
		            Note that D$^2$ can be rewritten as $\scriptstyle \mW ( \mX^{(t)} - \eta ( (\mX^{(t-1)} - \mX^{(t)} ) / \eta + \nabla f(\mX^{(t)}) - \nabla f(\mX^{(t-1)}) ) )$, and the update would break if the magnitude of $\scriptstyle \mX^{(t-1)} - \mX^{(t)}$ is a factor of $10 \eta$ (i.e. performing learning rate decay at step $t$).
		      \item It is non-trivial to improve D$^2$ for the SOTA deep learning training.
		            To support our argument, Table~\ref{tab:comparison_with_gt_methods} compares to an improved D$^2$ variant (noted as D$^2_+$) to address the issue of learning rate decay in D$^2$ (though it breaks the design intuition of D$^2$); the performance of D$^2_+$ is far behind our scheme.
		            The update of D$^2_+$ follows $\scriptstyle \mW ( \mX^{(t)} - \eta^{(t)} ( (\mX^{(t-1)} - \mX^{(t)} ) / \eta^{(t-1)} + \nabla f(\mX^{(t)}) - \nabla f(\mX^{(t-1)}) ) )$.
		      \item The numerical results of D$^2$ in~\citet{tang2018d,pan2020d,lu2019gnsd} cannot support the practicability of D$^2$:
		            1) The experiments of~\citet{tang2018d,pan2020d} only consider a very small scale setup (\# of nodes $n \!= \! 5$ or $8$) for CIFAR-10, while \cite{lu2019gnsd} only evaluates on the toy MNIST dataset---these setups are much less challenging than ours;
		            2) Only training loss curves are reported in~\citet{tang2018d,pan2020d,lu2019gnsd}, and the final training loss values of~\citet{tang2018d,pan2020d} are much higher than $0$ (i.e.\ not converge to a local minimum).
	      \end{itemize}
	\item
	      As suggested by anonymous reviewers, we compare with GT methods in Table~\ref{tab:comparison_with_gt_methods}: \algoptsgdmn outperforms GT by a large margin.
	      We would like to point out that
	      1) the observations of the marginal performance gain in GT are aligned with prior works, e.g.\ similar numerical results in Figure 3 \& 4 \& 5 of~\citet{xin2020variance};
	      2) GT may have more benefits in the extreme low training loss regime (e.g.\ less than $1e{-}4$) where gains might increase when combining with variance reduction techniques~\citep{xin2020variance}---however, we focus on the test performance for deep learning~\citep{defazio2018ineffectiveness};
	      3) recent work~\citep{yuan2020can} also proves that gradient tracking methods are in general much more sensitive than diffusion-based methods.
\end{itemize}